\documentclass[11pt,a4paper,twoside,openright]{extreport}

\usepackage{epsfig}
\usepackage{url}

\usepackage[utf8]{inputenc}

\usepackage{array}

\usepackage{scrextend}
\usepackage{graphicx}
\usepackage{subcaption}

\usepackage[nodisplayskipstretch]{setspace} 
\usepackage[table]{xcolor}

\usepackage{fancyhdr} 

\usepackage{varwidth}

\usepackage{enumitem}
\usepackage{makeidx} 
\usepackage{datetime} 
\usepackage{graphicx}
\usepackage{array}
\usepackage{subcaption}

\usepackage{a4wide}

\usepackage{csquotes}
\usepackage[english]{babel}

\usepackage{bm}
\usepackage{array,multirow}
\usepackage{epsfig}
\usepackage{epstopdf}
\usepackage{caption}

\usepackage{amsmath,amsfonts,amssymb,amsthm,mathtools}
\usepackage{bm}

\usepackage{accents}

\usepackage{relsize}
\usepackage[normalem]{ulem}
\usepackage{longtable}
\usepackage{enumitem}

\usepackage{algorithm}
\usepackage{algorithmic}

\usepackage{url}
\usepackage{hyperref}
\hypersetup{
    colorlinks,
    citecolor = [rgb]{0, 0, 0.6},
    filecolor =[rgb]{0, 0, 0.6},
    linkcolor = [rgb]{0, 0, 0.6},
    urlcolor = [rgb]{0.5, 0.5, 0.6}
}

\usepackage[citestyle=authoryear, bibstyle=numeric, natbib=true, hyperref, backend=bibtex]{biblatex}
\DeclareFieldFormat[inbook]{citetitle}{#1}
\DeclareFieldFormat[inbook]{title}{#1}

\newlist{steps}{enumerate}{1}
\setlist[steps, 1]{label = Step \arabic*:,leftmargin=5em,rightmargin=1em}

\definecolor{gray}{gray}{0.5}
\definecolor{dark_green}{rgb}{0, 0.5, 0}
\definecolor{dark_red}{rgb}{0.5, 0, 0}

\newtheoremstyle{highlight}
  {\topsep}
  {\topsep}
  {\itshape}
  {0pt}
  {\bfseries}
  {.}
  { }
  {\thmname{#1}\thmnumber{ #2}\thmnote{ (#3)}}

\newcommand{\Tau}{\mathcal{T}}

\theoremstyle{definition}

\newtheorem*{question*}{Question}
\newtheorem*{answer*}{Answer}

\newtheorem{theorem}{Theorem}
\newtheorem*{proof_sketch*}{Proof Sketch}
\newtheorem{lemma}{Lemma}[chapter]
\newtheorem*{claim*}{Claim}
\newtheorem{proposition}{Proposition}[chapter]
\newtheorem{definition}{Definition}[chapter]
\newtheorem{remark}{Remark}[chapter]
\newtheorem{example}{Example}[chapter]

\theoremstyle{highlight}

\renewenvironment{proof}{\noindent\textbf{Proof.}}{\qed}

\DeclareRobustCommand\iff{\;\Longleftrightarrow\;}

\newcommand{\dc}{\cellcolor{blue!25}}
\DeclareMathOperator*{\argmin}{\arg\!\min}

\newcommand{\newfontobj}[2]{
  \newcommand{#1}[1]{
    \expandafter\def\csname##1\endcsname{{#2 ##1}}}}

\usepackage{makecell}

\bibliography{references}

\newcommand{\comment}[1]{}

\let\temp\cleardoublepage
\renewcommand{\cleardoublepage}{\clearpage{\pagestyle{empty}\temp} \thispagestyle{empty}} 
\newcommand\thesistitlemedskip{0.2in}
\newcommand\thesistitlebigskip{0.5in}
\newcommand\thesistitleBigskip{1.2in}

\newdateformat{monthyear}{\monthname[\THEMONTH], \THEYEAR}
\def\title{Motion Planning on Visual Manifolds}
\def\author{M Seetha Ramaiah}

\def\degree{Doctor of Philosophy}
\def\department{DEPARTMENT OF COMPUTER SCIENCE \& ENGINEERING}

\def\institute{INDIAN INSTITUTE OF TECHNOLOGY KANPUR}

\def\date{June, 2018}

\makeindex

\begin{document}
\pagenumbering{roman}
\doublespacing
  \begin{titlepage}
  
    \begin{center}
      \vfill
      \mbox{}\\
      {\Huge\title} \\[\thesistitleBigskip]
      {\large \textit{A Thesis Submitted}} \\
      {\large in Partial Fulfillment of the Requirements} \\
      {\large for the Degree of}\\
      {\large \textit \degree} \\[\thesistitlebigskip]
      {\large \textit{by}} \\
      {\Large  \author} \\[\thesistitlebigskip]
      \vfill
      {\large \textit{to the}} \\[\thesistitlemedskip]
      \begin{figure}[htbp]
	\centering
	\includegraphics[width=4cm,bb=0 0 200 200]{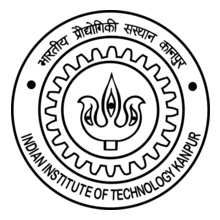}
\end{figure}
      {\large \bf \department} \\[\thesistitlemedskip]
      {\large \institute}\\[\thesistitlemedskip]
      {\large \bf \date}
      \vfill
    \end{center}
  \end{titlepage}
\cleardoublepage 
\setcounter{page}{3}

    
    


\includegraphics[bb=60 700 0 0]{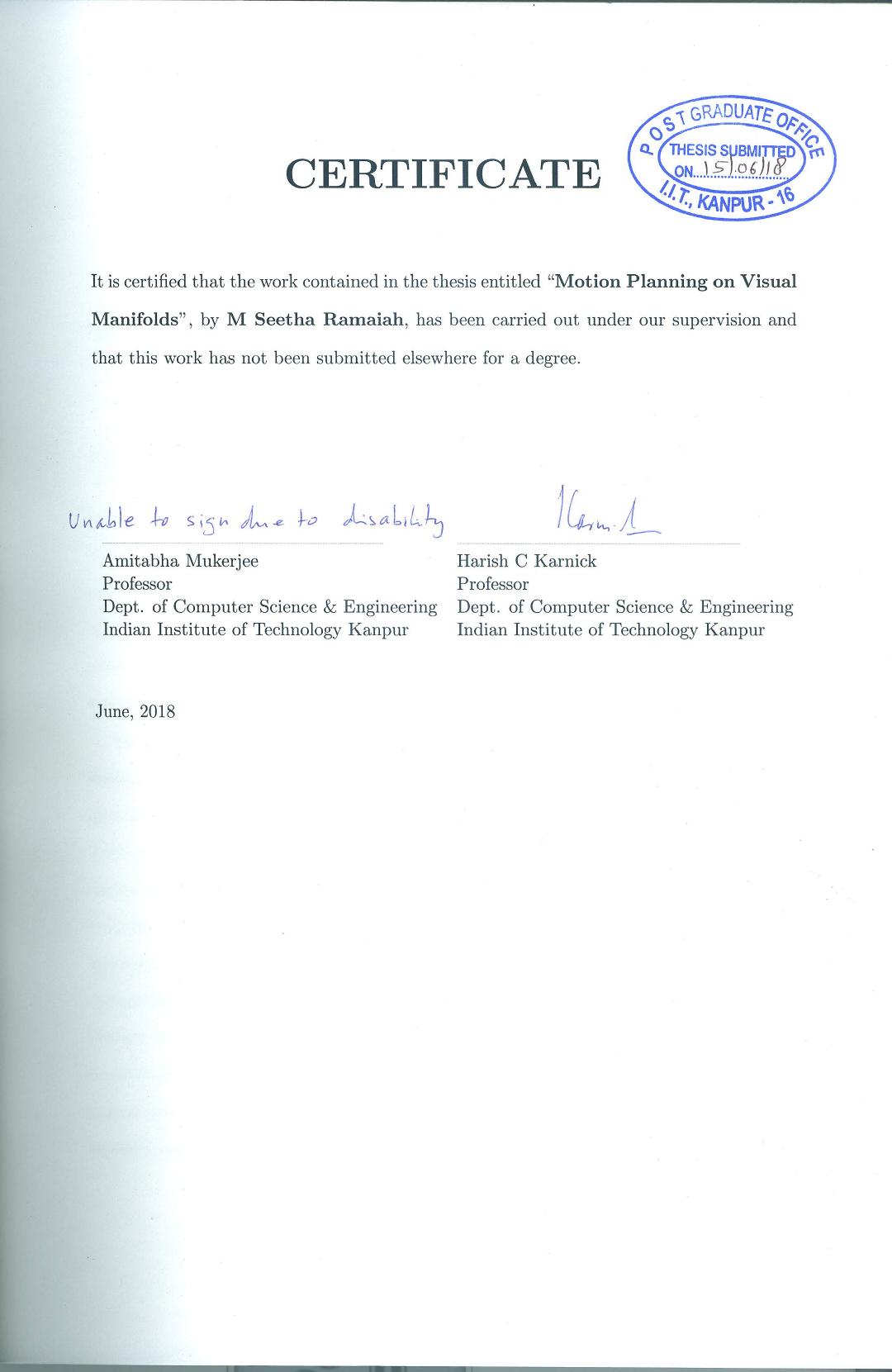}
\cleardoublepage 
\setcounter{page}{5}

\chapter*{SYNOPSIS}
\section*{Motion Planning}
\noindent Humans and most other animals are very good at avoiding obstacles and navigating in complex environments just by visually observing the world. They routinely use prior sensorimotor experience to build motor models, and use vision for gross motor tasks in novel environments. Achieving similar abilities in robots, without having to calibrate a robot’s own body structure, or estimate exact 3-D positions, is a touchstone problem for robotics. 

The notion of Configuration Space (also known as C-Space) is fundamental to conceptualizing multibody motion and to address the problem of \textit{Robot Motion Planning} -- planning obstacle-free paths from a source configuration to a destination configuration. Traditional solutions to the motion planning problem involve constructing an explicit representation of the robot in terms of the number of degrees of freedom, joint and link geometry, rules for forward kinematics and inverse kinematics to map from the configuration space to the work-space and vice versa. Once such a representation is available, obstacle-free paths can be planned in the configuration space. This approach is quite cumbersome because different representations need to be hand-engineered for different robots, obstacles and workspaces. If we consider the triple $\langle$robot, obstacles, workspace$\rangle$ as a system, the traditional approach is quite system specific. Any change in the system would require operating on a different hand-engineered model.

In this thesis, we propose an alternative characterization of configuration space, which we call the \textit{Visual Configuration Space} (VCS), to address the motion planning problem. Beginning with a dense enough set of images of the robot in random configurations, it is possible to discover the topology of its motion manifold, using manifold learning methods such as Isomap~\citep{tenenbaum2000global}. This is possible because, in spite of the robot images being very high-dimensional, the governing parameters of its motion are much fewer and as a result the robot images can only change in very restricted ways. This manifold, discovered using the robot images, is the visual configuration space.  A path between two configurations of the robot is a curve between the corresponding points on the VCS. Since we can only use a finite sample of images, we will only have a discrete representative of the VCS, which we call the \textit{Visual Roadmap} (VRM). The VRM is a graph in which each node corresponds to an image of the robot in some configuration with an edge between two nodes whose poses are near-by according to some metric in the image space. We use the VRM to compute shortest paths between configurations. 

We discuss the conditions under which the visual configuration space is homeomorphic to the canonical C-Space. We also discuss how obstacles in the workspace are mapped on to VCS and VRM and the issues involved in and advantages of using the VRM for motion planning. We present empirical results related to motion planning for various types of simulated planar robots and some real 3-D robots. 

\subsection*{Basic Framework}
Given a set of robot images in random configurations, we build a graph using these images and plan paths using this graph. In this graph, each node corresponds to a configuration of the robot and so the graph will have as many nodes as the number of images sampled. Two nodes will be connected if the corresponding images are such that one of them is a $k$-nearest neighbour of the other one, for some value of $k$. The corresponding edge will have a weight equal to the distance between those two images, for some distance metric such as the Euclidean distance. We call this $k$-nearest neighbours ($k$-NN) graph a \textit{Visual Roadmap} (VRM). The idea of VRM is an analog to Probabilistic Roadmap (PRM)~\citep{kavraki-latombe-overmars-96_PRM-high-dimensional}. While PRM samples the conventional C-space, VRM samples the VCS.

In addition to the raw images of the robot, we also consider random projections of raw images, a set of ideal points tracked on the robot body, Shi-Tomasi features~\citep{shi-tomasi_1994_cvpr_good-features-to-track} for each link of the robot, to save distance computation time. For these different representations, we use different distance metrics to construct the VRM. We use Euclidean metric for the raw images of the robot, for random projection vectors, and for ideal track points. For Shi-Tomasi feature based representation, we use Hausdorff distance. We study empirically, the effect of using these different representations and distance metrics to compute the visual roadmap.

When there are no obstacles in the workspace, finding a path for the robot from a source configuration to a destination configuration corresponds to finding a path in the $k$-NN graph.

\subsection*{Static Obstacle Avoidance}

When there are any static obstacles, we find the configurations in which the robot images have a non-empty intersection with the obstacle image and remove from the roadmap, the nodes corresponding to these configurations. Paths planned using the remaining graph will be \textit{almost} obstacle-free. The paths are not yet fully safe because the edge between two free configurations may not be free from obstacles. This is the problem of local planning and will be handled separately.

For spatial robots, avoiding an obstacle would require multiple cameras. We collect multiple images of each configuration from different fixed views. A given configuration is free, if there is at least one view in which the robot image does not intersect with the obstacle image. If the images from all the views have a non-empty intersection with the obstacle image in the respective views, then that configuration will be considered a collision configuration.

\subsection*{Local Planner}
Once we know that all the nodes in the $k$-NN graph correspond to free configurations, it is the responsibility of the local planner to make sure that all the edges connecting the free nodes are free from obstacles. For this purpose, we propose three approaches: (i) Local Tangent Space based planner, (ii) Track-points based local planner, and (iii) Local planner using Shi-Tomasi features. We present an empirical analysis of these three local planners for some simulated robots. 

\subsection*{Dynamic Obstacle Avoidance}
To handle dynamic obstacles, we make certain assumptions about the speed of the obstacle. These assumptions facilitate us to update only small portions of the graph that correspond to collision configurations, instead of having to recreate the whole graph all over again every time the obstacle moves. Firstly, we assume that the speed of the obstacle is less than the speed of the robot, so that there is enough time for the robot to plan alternate paths. Secondly, we assume that in each move, the obstacle moves at most by certain amount, so that the updates on the graph can be limited to a small number of layers of neighbours around the current obstacle region in the graph.


\section*{Modelling Body Schema}
We also discuss how similar manifold-based computational models can be used to model body schema and sensorimotor integration. Body schema of a cognitive agent is a representation of its body that allows it to infer the position and orientation of its limbs relative to its world, and to move and perform actions in that world.  Such a model could be learned by the agent in the early stages of its life by performing random motions and observing  their outcomes (motor babbling) and is grown and adapted to deal with more and more complex situations to perform complex tasks.

\section*{Head Motion Animation of Avatars in Virtual Environments}
We apply the idea of Visual Roadmap for animating avatars in a virtual environment designed to support remote collaboration between distributed work teams in which users are represented by avatars. There are usually long periods of time when a user is not actively controlling the avatar and working on his official task. We need to animate his avatar with a 'working-at-desk' animation that should be non-looping and sufficiently random for a single avatar as well as between multiple avatars to appear realistic. We present a technique for generating multiple head motions using the gaze space images to control the avatar motion. Our technique can automatically generate long sequences of motion without any user intervention. We present results from synthetic data.
\clearpage 


\setstretch{1.3} 

\pagestyle{empty} 
\newenvironment{dedication}
        {\vspace{6ex}\begin{quotation}\begin{center}\begin{em}}
        {\par\end{em}\end{center}\end{quotation}}
\begin{dedication}
To humanity \ldots
\end{dedication}

\chapter*{Acknowledgements}

I would like to express my deep gratitude towards my thesis supervisors Prof. Amitabha Mukerjee and Prof. Harish Karnick. Prof. Mukerjee has helped me find some fundamental problems in the area of Cognitive Robotics. He has offered a lot of discussions which helped me gain confidence in the area and a direction for my research. Besides research-related discussions, we have also had many refreshing discussions at his home and during cycling trips. The moral support that he gave me during not-so-good times is immeasurable. I would like to thank him for all the good things he has done to me.

Prof. Karnick has helped me get moving when Prof. Mukerjee became unavailable for guidance, due to his unfortunate medical condition. He offered me convincing explanations when I had doubts regarding the validity of my PhD work, for which I am quite grateful to him. He has also given me moral support when I needed it. Without his help, I would not have been able to complete this thesis. Also, his courses have helped me gain interest and knowledge in the foundations of Machine Learning, which is another area of my interest.

I would like to thank Prof. T V Prabhakar for motivating me to pursue PhD and getting me started with PhD. He was my guide during the initial times of my PhD. He encouraged me to pursue my interests without forcing anything on me. We also spent quite some time discussing things outside academics during some lunch/dinner sessions at his home. Without his support, I may not have continued or even started to pursue my PhD.

Besides my supervisors, I have also taken help from almost all the faculty members of the CSE department at IIT-Kanpur, in some form or the other. The help offered by Prof. Shashank Mehta and Prof. Sumit Ganguly deserves a special mention. Prof. Mehta helped me transition from one area of research to another when I was struggling to find research problems. Prof. Ganguly has given me valuable suggestions on many occasions. 

I would like to thank my collaborators for their help in the preparation of this thesis. I have collaborated with Dr. Geetika Sharma and Ankit Vijay at TCS innovation labs in Delhi. I have also collaborated with Arindam Chakraborty and Sadbodh Sharma at IIT Kanpur. I have received help from Shubham Goel, Ankit Maheshwari and Abhishek Kashyap who interned here at CSE, IIT Kanpur. 

I have received several kinds of support from the CSE Department at IIT Kanpur, without which my stay here would have been impossible. I have received financial support through various projects and through the Research-I Foundation fund. My stay at IIT-Kanpur has been very enjoyable because of the wonderful campus and the facilities provided by the Institute. Many thanks are due to the CSE Department and to the Institute for all the great things they have given me.

I would like to take this opportunity to thank all my friends at IIT-Kanpur. Special thanks are due to Kiran, Madhu, Ramesh, Hemanadhan, Ravi, Vishwesh, Narendra, Ashok, Guna, Venu, Srikanth and Prasanna. The time we spent together and the discussions we have had during lunch and dinner sessions were quite fun-filled and refreshing. I have also spent fruitful time with my lab mates Deepanjan, Saurabh, Sagarmoy, Ramprasad, Badri, Diptarka, Debarati, Keerti, Garima, Sumanta, Amit, Rajendra, and Mahesh. Without the company of these people, stay at IIT-Kanpur may not have been as enjoyable as it was.

Finally, I would like to thank my family for all their support during the rather long period of my PhD life. My wife Anusha deserves special thanks for all her support. She has made some commendable sacrifice to support the family, without which my pursuit of PhD may have been in great trouble.

\tableofcontents
\cleardoublepage 
\singlespacing
\chapter*{List of Publications}
\addcontentsline{toc}{chapter}{List of Publications}

\begin{enumerate}
    \item {
        \textbf{Body Schema as a Collection of Manifolds};
        M Seetha Ramaiah, Amitabha Mukerjee; Poster at \emph{International Conference on Cognition, Brain and Computation \textbf{CBC-2015}}. December 5–7, 2015. Ahmedabad, India.
    }
    \item {
        \textbf{Visual Generalized Coordinates};
        M Seetha Ramaiah, Amitabha Mukerjee, Arindam Chakraborty, Sadbodh Sharma;
        arXiv preprint arXiv:1509.05636, 2015.
    }
    \item {
        \textbf{The Baby at One Month: Visuo-motor discovery in the  infant robot};
    	Amitabha Mukerjee, M Seetha Ramaiah, Sadbodh Sharma, Arindam Chakraborty;
    	Poster at \emph{IEEE \textbf{ICRA-2013} Workshop on Bootstrapping Structural Knowledge
    	from Sensory-motor Experience}. May 6-10, 2013. Karlsruhe, Germany.
    }
    \item{
        \textbf{Head Motion Animation using Avatar Gaze Space};
    	M Seetha Ramaiah, Geetika Sharma, Ankit Vijay, Amitabha Mukerjee;
        Poster at \emph{IEEE Virtual Reality \textbf{VR-2013}}. March 18-20, 2013. Lake Buena Vista, FL, USA.
    }
\end{enumerate}

Contents of Chapter~\ref{chap:vcs} and Chapter~\ref{chap:rmp_in_vcs} are based on (2). Contents of Chapter~\ref{chap:manifold_body_schema} are based on (1) and (3). Contents of Chapter~\ref{chap:head_motion} are based on (4).

\cleardoublepage

\listoffigures
\addcontentsline{toc}{chapter}{List of Figures}
\cleardoublepage

\pagestyle{myheadings}

\pagenumbering{arabic}

\doublespacing 


\chapter{Introduction and Preliminaries} 

\label{chap:intro_n_prelim} 

\lhead{\textit{Introduction and Preliminaries}} 

\section{Motion Planning}
\label{sec:motion_planning}
\textit{Motion planning} is the problem of getting an agent, such as a robot, to move from one configuration to another configuration without hitting any obstacles present in its environment. This is a very important problem in robotics because planning obstacle-free paths is a necessary step to perform any task involving a robot's motion. For example, if a robotic arm has to reach a tool to fix something, it has to move from its current configuration to a configuration in which its end effector is ready to hold the tool. This is a very common task to be performed for any industrial robotic arm. Similarly for a mobile robot, moving from one place to another place in its environment is a common task. During this motion, we do not want any part of the robot to hit anything in the environment, in order that no damage be caused either to the robot or to any other object in the environment.

\subsection{Motion Planning: Examples}
\label{sec:motion_planning_examples}
Consider a circular mobile robot in a planar workspace as in Figure~\ref{fig:rmp:circle_example}. Given a source position and a goal position, how does the robot move from the source to the goal? Planning a path for a mobile robot such as this one corresponds to computing a continuous sequence of positions in the workspace, that begins with the source position and ends with the goal position, such that in no position of the sequence does the robot collide with any obstacle in the workspace.

\begin{figure}[ht!]
    \centering
    \includegraphics[width=0.7\textwidth,bb=0 0 1000 680]{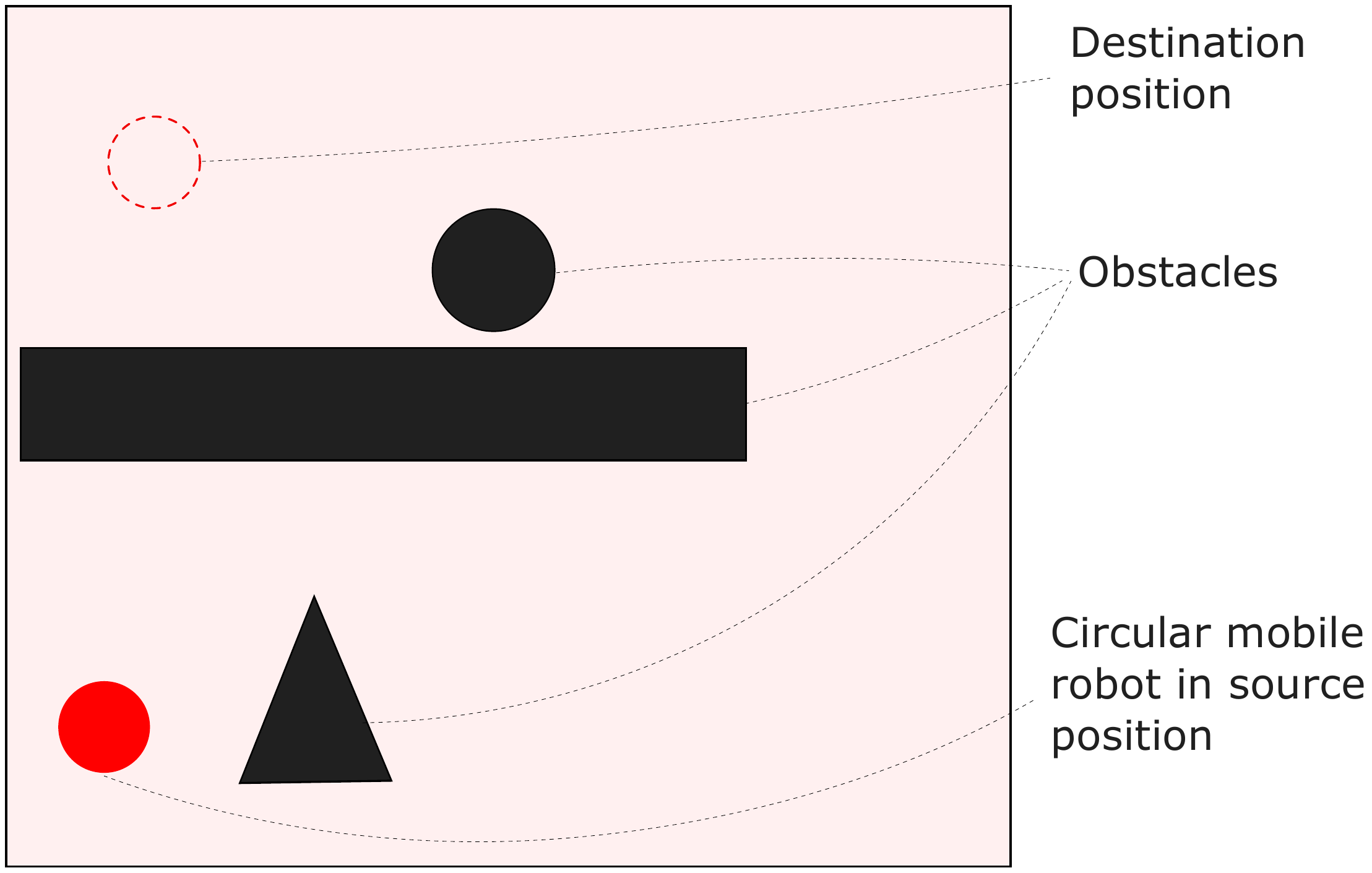}
    \caption{
        A circular mobile robot in a planar workspace with obstacles. The goal is to generate a sequence of motion instructions for the robot to move from the source position to the destination position without hitting obstacles.
    }
    \label{fig:rmp:circle_example}
\end{figure}

\begin{figure}[ht!]
    \centering
    \includegraphics[page=3, clip, trim = 5cm 3cm 5cm 7cm,  width=0.8\textwidth]{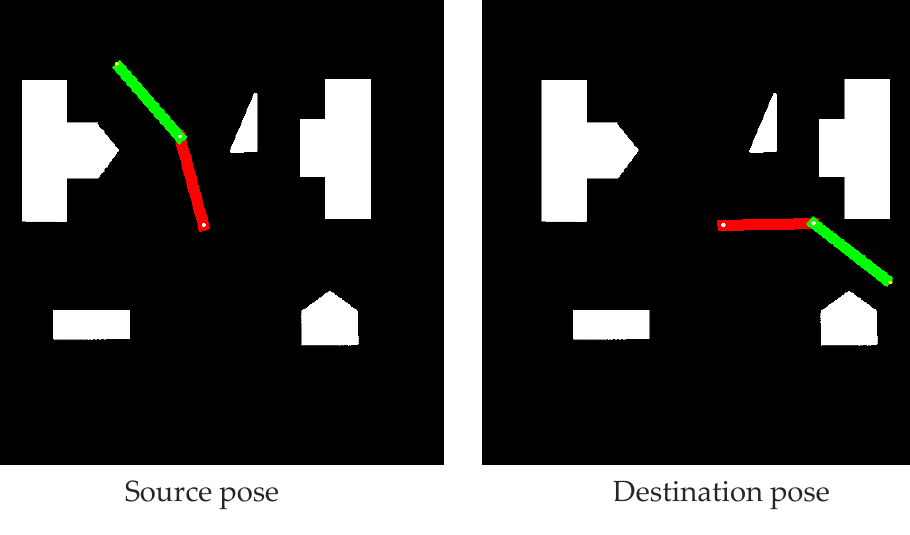}
    \caption{
        A 2-link planar arm in a planar workspace with obstacles (white objects). The goal is to guide the arm to change its configuration from the source pose to the destination pose, through a continuous motion without hitting obstacles.
    }
    \label{fig:rmp:2link_arm_example}
\end{figure}

Motion is not always necessarily translational. A lot of industrial robots are articulated chains of links connected by joints. Consider a 2-link planar arm as in Figure~\ref{fig:rmp:2link_arm_example}. Motion planning for such an arm corresponds to creating a continuous sequence of poses of the arm that begins with the source pose and ends with the destination pose such that no pose in the sequence leads to a collision with an obstacle in the workspace. 

Even though we used simple planar models for illustration, the motion planning problem is universally applicable to all kinds of autonomous robots such as industrial robots, humanoid robots, self-driving cars, unmanned air vehicles (UAVs), medical robots, firefighting robots, robots for defense and landmine clearance, etc. In addition to robotics, motion planning also has applications in many other fields like animating digital agents in virtual environments, architectural design walk-through, studying protein folding and RNA folding, etc\footnote{More details on the applications of motion planning can be found in~\citep{choset-05_robot-motion-theory}.}.


In this thesis, we propose a purely vision-based approach to motion planning, inspired by human cognition. We will apply the proposed visual approach to robot motion planning (chapter~\ref{chap:rmp_in_vcs}), modelling body schema and infant motor learning (chapter~\ref{chap:manifold_body_schema}) and animating avatars in virtual environments (chapter~\ref{chap:head_motion}). Before introducing the proposed approach, we will first look at the traditional approach to robot motion planning and the complexity associated with it. 

\section{Traditional Approach to Robot Motion Planning}
\label{sec:rmp_traditional_approach}

To make things slightly more specific, we define a \textit{multibody system} as a collection of rigid or flexible bodies, called links, connected through joints allowing for various kinds of translational and rotational motion of the links. A \textit{robot} is a programmable multibody system, equipped with some sensors and actuators. Under this definition, even human bodies qualify as robots. Programmability of robots allows us to make them autonomous. In this thesis, we will assume that the links of a robot are rigid. The \textit{workspace} ($\mathcal{W}$) of the robots considered in this thesis is a part of $\mathbb{R}^k$, where $k$ is either 2 or 3. If the workspace of a robot is a subset of $\mathbb{R}^2$, it is a \textit{planar robot} and if the workspace is a subset of $\mathbb{R}^3$, it is a \textit{spatial robot}. 

For the purposes of automatic motion planning, we need to be able to specify the robot completely. A fundamental question in this context is, what is a minimal representation for a complete specification of the robot. Answer to this question lies in the concept of \textit{configuration} and \textit{configuration space}~\citep{lozano1983spatial}. 

\subsection{Configuration Space}
The notion of \textit{configuration space} (also known as C-space) is fundamental to conceptualizing multibody motion. The \textit{configuration} of a multibody system is a complete description of the position of every point of the system, with respect to a fixed point in its world. The \textit{configuration space} of the system is the set of all possible configurations of the system. For a robotic system, typically, its configuration is given by a parameter vector which can be mapped to the positions of all points on the robot body. So, each configuration of the robot is a point in its configuration space. A configuration is denoted by $q$ and configuration space is denoted by $\mathcal{Q}$. Essentially, configuration space is a representational tool that allows us to view the robot as a point in a space. 

The configuration of a robot can be specified by a parameter vector, and hence the configuration space is a vector space. The minimum number of parameters needed to specify the configuration or the dimension of the configuration space, is called the \textit{degrees of freedom} (DOF) of the robot.

\subsubsection{Configuration Space: Examples}
For the circular mobile robot in Figure~\ref{fig:rmp:circle_example}, given its radius, which is fixed, coordinates of its center are sufficient to specify the position of every point on the robot. So, the set of all possible coordinate values that the center can take forms the configuration space of this robot. Hence, its C-space is $\mathbb{R}^2$. Thus, both the workspace and the configuration space, in this case, have the same topology\footnote{See Appendix~\ref{app:topology} for a quick refresher on topology.} (topology of $\mathbb{R}^2$). However, because of the size of the robot, there are some regions of the workspace where the center of the robot cannot be --- any point which is at a distance less than the radius of the robot, from a border of the workspace, cannot be the center of the robot. So, as shown in Figure~\ref{fig:cspace:circle_example}, its configuration space excludes some part of the workspace to accommodate for the size of the robot. This is done so that each point in the configuration space corresponds to the robot in some position in the workspace, and different points of the C-space correspond to the robot in different positions in the workspace. Since each configuration of this robot is a point in a 2-dimensional configuration space, it's a 2-DOF robot.

\begin{figure}[ht!]
    \centering
    \includegraphics[width=0.98\textwidth]{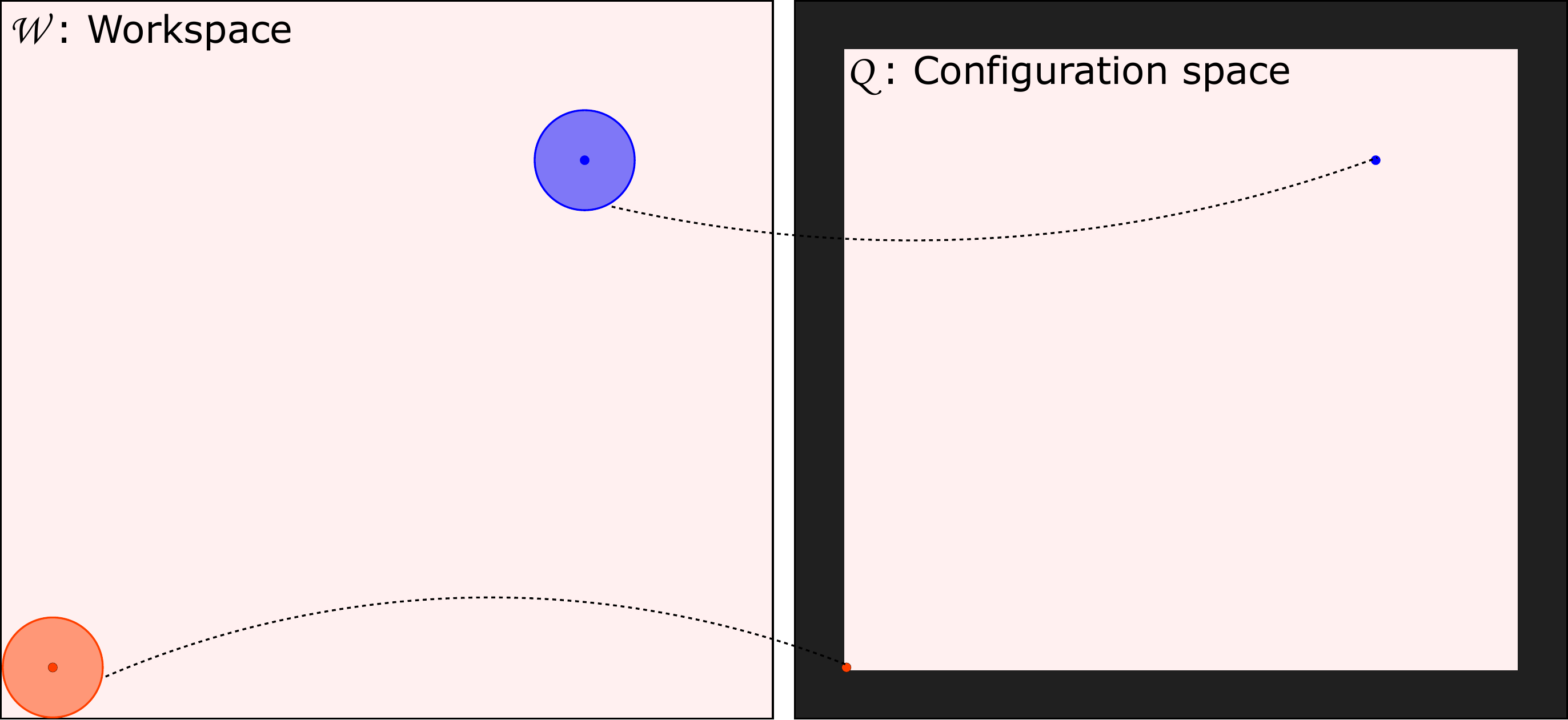}
    \caption{
        Configuration space of a circular mobile robot, with the position of the center being used as its configuration, excludes the points (the dark part in the right half of the figure) which cannot be the center of the robot in the workspace. Correspondence between the workspace and the configuration space representations of the robot is shown by the dotted lines. In this case, the workspace and the configuration space, both have an $\mathbb{R}^2$ topology.
    }
    \label{fig:cspace:circle_example}
\end{figure}

\begin{figure}[ht!]
    \centering
    \includegraphics[page=5, clip, trim = 1cm 4cm 1cm 6.8cm,  width=0.98\textwidth]{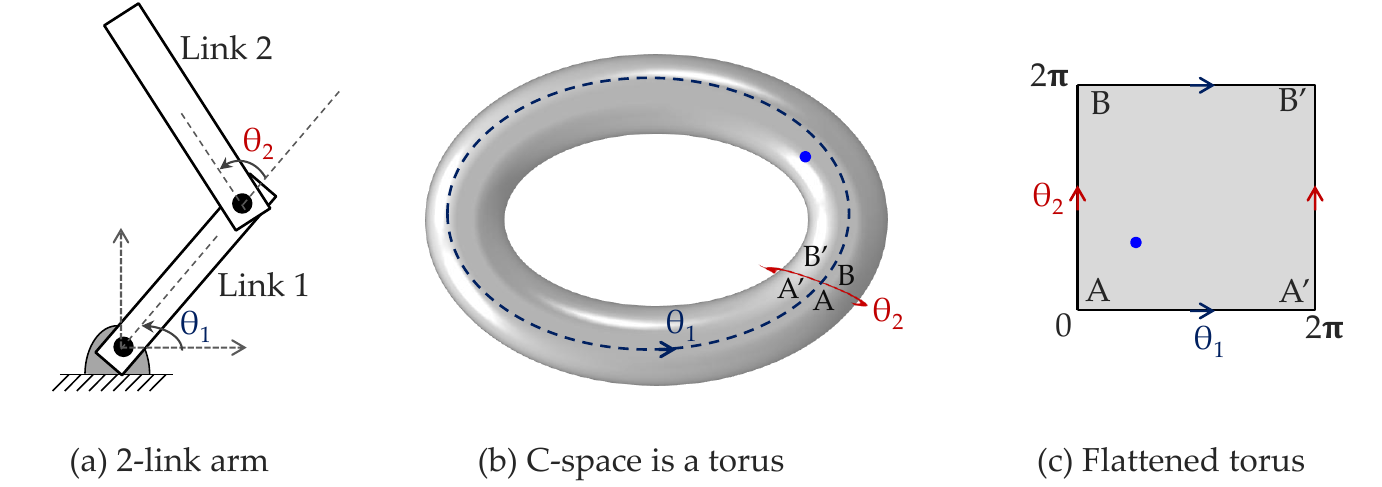}
    \caption{
        Configuration space of a 2-link robotic arm, with the joint angle vector $(\theta_1, \theta_2)$ being used as its configuration, is a torus which can be cut and flattened into a plane. The robot's configuration is shown as a point on the torus and in the plane. Because of the circular topology of each joint angle, each pair of opposite edges of the plane is actually the same line. This illustration is based on a figure from~\citep{choset-05_robot-motion-theory}.
    }
    \label{fig:cspace:2link_arm_example}
\end{figure}

Configuration space is not always a simple Euclidean space. For a 2-link robotic arm such as the one in Figure~\ref{fig:cspace:2link_arm_example}(a), we can use its joint angles $(\theta_1, \theta_2)$ to specify its configuration. If both the links can rotate fully around independently, then $\theta_1$ and $\theta_2$ take values from $[0, 2\pi]$. Since 0 and $2\pi$ are indistinguishable from each other, each joint angle has a circular topology $S^1$ and because of the independence of the two joint angles, the combined set of values has an $S^1 \times S^1 = T^2$ topology. That makes the configuration space a torus as shown in Figure~\ref{fig:cspace:2link_arm_example}(b). This torus can be cut and flattened into a plane as in Figure~\ref{fig:cspace:2link_arm_example}(c).

The basic motion planning problem is to compute a collision-free path between a pair of source and destination configurations. A path between two configurations $q_s, q_d$ is a curve $\gamma: [0, 1] \to \mathcal{Q}$ such that $\gamma(0) = q_s$ and $\gamma(1) = q_d$. When there are no obstacles in the workspace, a straight line (geodesic) between the two configurations would serve as a path. But when there are obstacles, things get complicated. We need to first map each workspace obstacle to a set of corresponding points in the C-space. A collision-free path from $q_s$ to $q_d$ would then be a curve from $q_s$ to $q_d$ that does not pass through any obstacle configurations. 

\subsection{Free Configuration Space}
An \textit{obstacle in the workspace} is a region of the workspace occupied by anything other than the robot, and hence not accessible to the robot. \textit{Free workspace}, denoted by $\mathcal{W}_{free}$, is the workspace excluding all the obstacles in the workspace:
\[
    \mathcal{W}_{free} = \mathcal{W} \setminus \bigcup\limits_i B_i,
\]
where $B_i$ is the region occupied by the $i^{th}$ obstacle in the workspace.

We define the \textit{shape} of the robot corresponding to a configuration $q$, denoted by $R(q)$, to be the set of all points of the workspace which are occupied by the robot, when its configuration is $q$. If $B_i$ is an obstacle in the workspace, the corresponding \textit{obstacle in C-space}, denoted by $CB_i$, is the set of configurations $q$ for which the corresponding shape of the robot has a non-empty intersection with $B_i$, i.e.,
\[
    CB_i = \{ q \in \mathcal{Q} : R(q) \cap B_i \ne \emptyset \}.
\]
The \textit{free configuration space}, denoted by $\mathcal{Q}_{free}$, is the configuration space minus the union of configuration space obstacles:
\[
    \mathcal{Q}_{free} = \mathcal{Q} \setminus \bigcup\limits_i CB_i.
\]
These definitions and notations are inspired by \citep{choset-05_robot-motion-theory}.

The shape of an obstacle in the C-space depends on the size, shape, and position of the obstacle and the robot in the workspace. Precise maps need to be defined between the workspace and the configuration space. A map from the configuration space to the workspace is called \textit{forward kinematics} and the inverse map is called \textit{inverse kinematics}. The definition of forward and inverse kinematics heavily depends on obstacle and robot geometry. So, the standard way of defining free configuration space requires the knowledge of the geometry of the obstacles as well as that of the robot in order to compute the free space. 

\subsubsection{Free Space Computation: Example}
Figure~\ref{fig:cspace_obstacle:circle_example} shows an obstacle in the workspace and in the configuration space, for the circular mobile robot considered in Figure~\ref{fig:cspace:circle_example}. As noted earlier, each point of the configuration space corresponds to the robot in some position of the workspace. The shape of the C-space obstacle depends on the geometry of the robot and of the obstacle. In particular, for this robot, we need to consider the shape and size of the robot and obstacles.

\begin{figure}[ht!]
    \centering
    \includegraphics[width=0.98\textwidth]{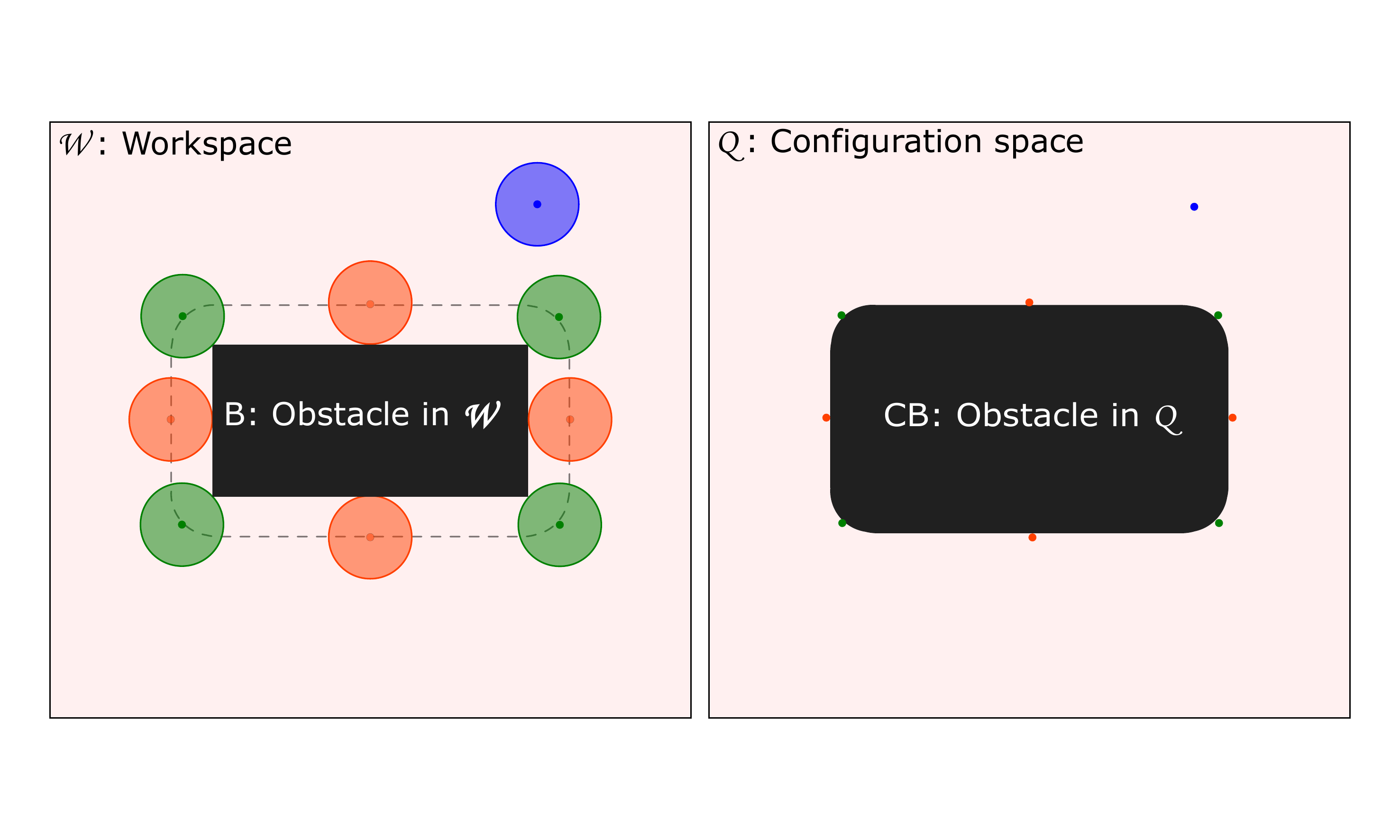}
    \caption{
        An obstacle in the workspace vs. in the configuration space for a circular mobile robot. The geometry of the robot and obstacle play a crucial role in computing the free C-space. Center of the robot cannot be inside the dotted region in the workspace, for any configuration of the robot, and hence the shape of a rectangular obstacle in the workspace takes the form of a bigger rounded rectangle in the C-space as shown on the right side. Some corresponding configurations are shown in the free workspace and in the free C-space. 
    }
    \label{fig:cspace_obstacle:circle_example}
\end{figure}

\subsection{Path Planning in Free C-Space}
A path from a source configuration to a destination configuration is a curve in the free C-space, that starts at the source configuration and ends at the destination configuration. 

There is a large variety of methods of motion planning using the traditional configuration space approach. The method of \emph{potential functions} treats the robot and obstacle configurations as positive charges and the goal configuration as a negative charge so that there is a natural attraction between the robot and the goal configuration and a natural repulsion between the robot and obstacle configurations. And the path planning problem is solved in terms of moving in a force field induced by these charges. 

Roadmap methods construct a data structure called a \emph{roadmap} in the form of a graph, which is a topological map of the environment based on the robot's sensory data. A roadmap is a graph embedded in the free configuration space, where each node corresponds to a free configuration of the robot and an edge between two nodes corresponds to a path between the corresponding configurations. Once a roadmap is computed, multiple path queries can be answered using it. Paths can be computed using standard graph algorithms such as the Dijkstra's algorithm for shortest path computation. 

A particular class of roadmap methods is cell-decomposition methods that divide the free C-space into a union of regions called cells. The most popular methods in this class are trapezoidal decomposition and Morse cell decomposition. Trapezoidal decomposition divides the free space into a collection of simple trapezoidal regions and treats each region as a node in a graph and connects the nodes that correspond to neighbouring regions. It is applicable to polygonal planar C-spaces which can be broken down into polygonal regions. Morse decomposition is a more general method that is particularly useful in scenarios where the robot has to cover the whole free space, such as a de-mining robot.

Another class of roadmap methods is probabilistic methods. The most popular ones in this category are \emph{probabilistic roadmap} (PRM)~\citep{kavraki-latombe-overmars-96_PRM-high-dimensional}, \emph{rapidly exploring random trees} (RRT)~\citep{lavalle1998rapidly} and their derivatives. They sample the free configuration space and build a roadmap using the sampled points. Even though the nodes of the network correspond to free configurations, it is still not guaranteed that the edge connecting two free configurations is completely in the free-space. Ensuring edge safety is the problem of \textit{local planning}. To check for the safety of an edge, a typical local planner in PRM interpolates several points between the configurations corresponding to the terminal nodes of that edge. If any of the interpolated points are not in the free C-space, then that edge is discarded; otherwise, it is retained in the roadmap. Paths planned on a roadmap after it is pruned using a local planner are safe. An extensive survey of sampling-based methods can be found in~\citep{karaman2011sampling, elbanhawi2014sampling}.

\subsection{Complexity of the Traditional Approaches}

Traditional solutions to the \textit{robot motion planning} (RMP) problem involve constructing an explicit representation of the robot in terms of the number of degrees of freedom, joint and link geometry, rules for forward kinematics and inverse kinematics to map from the configuration space to the workspace and vice versa. Once such a representation is available, obstacle-free paths can be planned in the configuration space. This approach is cumbersome because different representations need to be hand-engineered for different robots, obstacles, and workspaces. If we consider the triple $\langle \text{robot, obstacles, workspace} \rangle$ as a \textit{system}, the traditional approach is quite \textit{system specific}. Any change in the system would require operating on a different hand-engineered model.

In this thesis, we propose an approach which requires much less system-specific information to plan paths. In order to plan obstacle-free paths, we assume the availability of a set of images of the robot in various configurations in its workspace. Using these images alone, we construct models that can be used for planning safe paths. This approach is motivated by studies in human cognition which suggest that human babies around the age of one month vigorously move their limbs and observe them to learn latent representations for voluntary motion.

\section{An Inspiration from Human Cognition}
Motion planning is a problem that humans and other animals have to solve all the time to move from one place to another place and to perform the necessary motor tasks for their survival. They can easily navigate through complex environments, without hitting each other or any other objects in the environment. For example, a person can easily navigate through a busy marketplace or a road congested with heavy traffic. They do this, largely, by just visually observing the world. Vision also plays a major role in performing various tasks in the peripersonal space (i.e., part of the surrounding space which is immediately reachable to the agent's limbs, while the body is in a fixed position). From a very early stage of life, humans acquire sensorimotor models by observing the world visually and through touch. 

With the aim of bringing similar capabilities in robots, in this work, we propose to solve the motion planning problem purely based on vision. For this purpose, we develop a vision based characterization of configuration space, which we call \textit{visual configuration space} (VCS). We use the formalism of \textit{manifolds} to develop the notion of VCS. We use VCS in the following ways: 
\begin{enumerate}[label=(\roman*)]
    \item To address the problem of robot motion planning. (Chapter~\ref{chap:rmp_in_vcs})
    \item To model body schema and explain how a cognitive agent, such as a human baby, might learn the ability to perform motor tasks in its peripersonal space, as it grows and gains experience with its world. (Chapter~\ref{chap:manifold_body_schema})
    \item To generate animations for avatars in virtual environments. (Chapter~\ref{chap:head_motion})
\end{enumerate}

\subsection{Main Claims of This Thesis}
\begin{enumerate}
    \item {
        The proposed approach to the RMP problem is much less system specific than the traditional methods because it does not assume the knowledge of robot geometry,  kinematic and inverse kinematic maps of the robot and obstacle geometry.
    }
    \item {
        The proposed manifold based model of body schema for infant motor learning has more explanatory power than other models in the current literature on body schema.
    }
\end{enumerate}

In the subsequent sections, we discuss some preliminaries necessary to develop the concept of visual configuration spaces. 

\section{Configuration Manifold}
\label{sec:configuration_manifold}

Configuration space is a vector space defined by the coordinates used to specify the configuration of the robot. The dimension of this vector space is the degrees of freedom of the robot. However, in many cases, not all the vectors in this space are valid configurations of the robot. The set of actual values taken by the configuration parameters is usually a subspace of the vector space. This subspace is sometimes linear, especially for mobile robots, and most of the times non-linear, especially for articulated arms whose motion involves an $S^1$ topology. Often, this subspace satisfies the requirements of a \textit{manifold} and whenever it does, the configuration space is called a \textit{configuration manifold}.

\subsection{Manifolds}
Let $(X, \tau_X)$ and $(Y, \tau_Y)$ be two topological spaces. A function $f: X \to Y$ is said to be \textit{continuous} if the inverse image under $f$, of every open set in $\tau_Y$ is open in $\tau_X$, i.e., 
\[
\forall V \in \tau_Y, f^{-1}(V) = \{x \in X: f(x) \in V \} \in \tau_X.
\]
If $f$ is a continuous bijection and its inverse is also continuous, then $f$ is said be a \textit{homeomorphism} and if a homeomorphism exists between the spaces $X$ and $Y$, then they are said to be \textit{homeomorphic} to each other. A homeomorphism is some kind of equivalence of spaces. Two homeomorphic spaces can be transformed into each other through a series of continuous deformations like stretching, shrinking, twisting, etc. For example, a circle and an ellipse are homeomorphic, and a coffee cup and a torus are homeomorphic.

If $f$ and its inverse are smooth (i.e., infinitely differentiable), then $f$ is said to be a \textit{diffeomorphism}. If a diffeomorphism exists between the spaces $X$ and $Y$, then they are said to be \textit{diffeomorphic} to each other.

A space $(X, \tau_X)$ is said to be \textit{locally homeomorphic} to $(Y, \tau_Y)$, if for every point $x \in X$, there is an open set $U \in \tau_X$ containing $x$, which is homeomorphic to some open set $V$ in $\tau_Y$. 

A \textit{$p$-dimensional topological manifold} $\mathcal{M}$ is a Hausdorff topological space, with a countable basis for the topology, which is locally homeomorphic to $\mathbb{R}^p$. See Figure~\ref{fig:manifold_def} for an illustration. 

\begin{figure}[ht!]
    \centering
    \includegraphics[width = 0.5\textwidth]{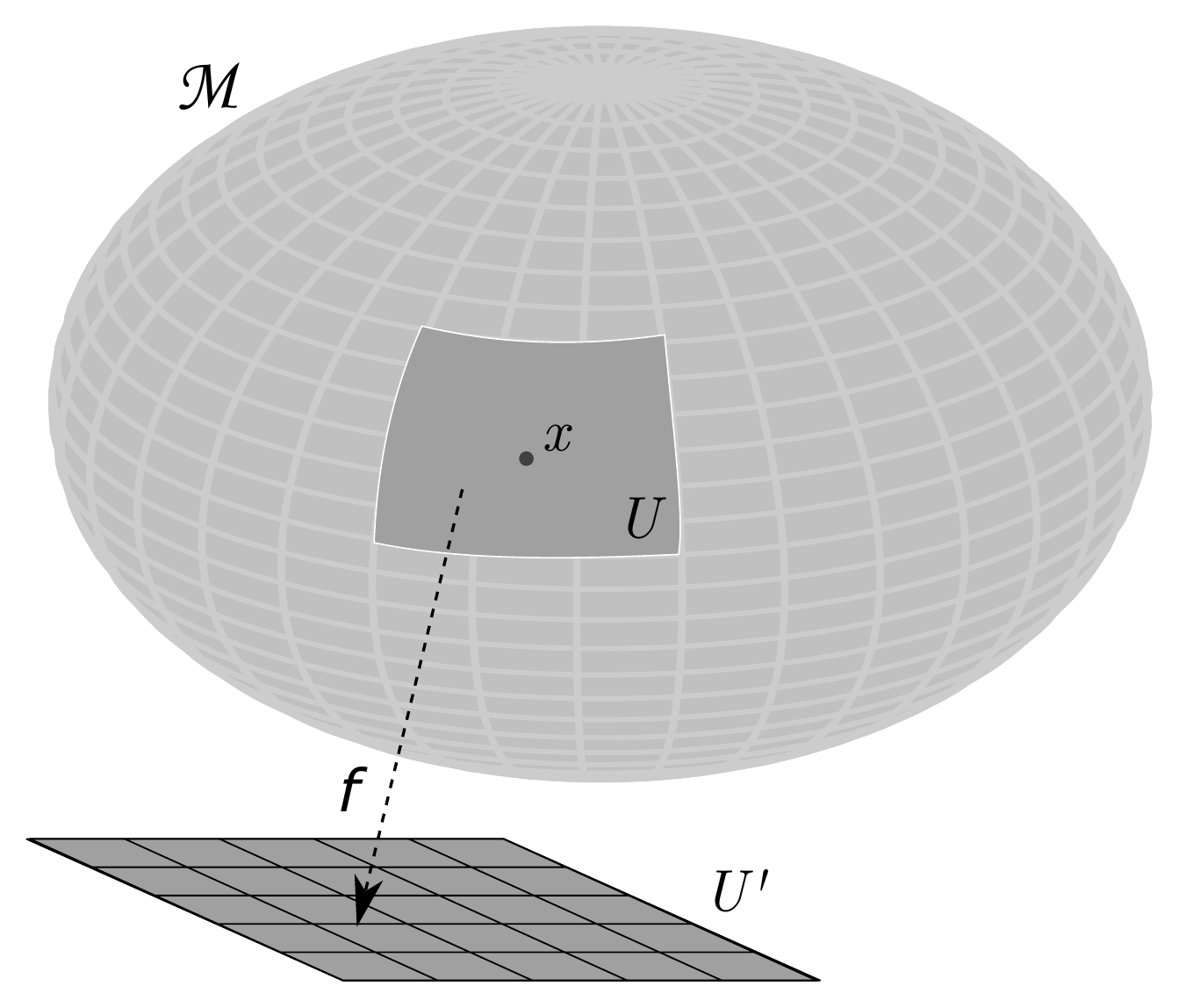}
    \caption{For every point $x \in \mathcal{M}$, there is an open neighbourhood $U$ containing $x$, an open set $U'\subset \mathbb{R}^p$ and a homeomorphism  $f: U \to U'$.}
    \label{fig:manifold_def}
\end{figure}

Intuitively speaking, $\mathcal{M}$ looks like a $p$-dimensional Euclidean space on a small enough scale. Another way of looking at it is, that the tangent space at any point on $\mathcal{M}$ is a $p$-dimensional vector space.

For each open set $U \subseteq \mathcal{M}$ that is homeomorphic to an open neighbourhood of $\mathbb{R}^p$ by a homeomorphism $\phi$, the pair $(U, \phi)$ is called a \textit{chart} of $\mathcal{M}$. A collection of charts that covers $\mathcal{M}$ is called an \textit{atlas} of $\mathcal{M}$. See \citep{choset-05_robot-motion-theory} for details of charts and atlases in the context of configuration spaces.

\subsubsection{Examples of Manifolds}
For example, a circle, an ellipse and a parabola can be formed by stitching together a collection of small lines; that is, all these objects are locally homeomorphic to $\mathbb{R}^1$ and hence they are 1-dimensional manifolds. See Figure~\ref{fig:manifold_examples} for examples of some 2-dimensional manifolds.

\begin{figure}[ht!]
    \centering
    \includegraphics[clip, trim = 2cm 3cm 1cm 4.5cm, width=\textwidth]{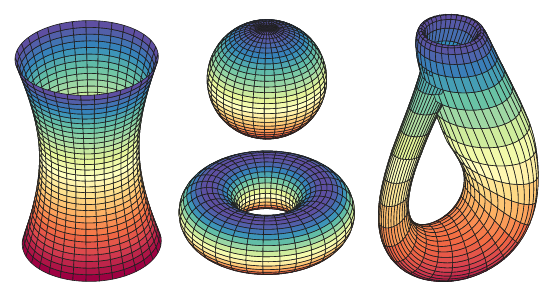}
    \caption{
        The surface of a hyperboloid, a sphere, a torus and a Klein bottle are some examples of 2-dimensional manifolds. These surfaces share the common property that they can be formed by stitching together a set of 2-dimensional patches; that is, all these surfaces are locally homeomorphic to $\mathbb{R}^2$ and hence they are all 2-dimensional manifolds.
    }
    \label{fig:manifold_examples}
\end{figure}

The configuration of an articulated robot with $p$ rotary links connected by $p$ joints, can be given by its $p$ joint angles $\theta_1,...,\theta_p$. If there are no constraints put on the joint angles, then its configuration space is a $p$-dimensional torus, which is a $p$-dimensional manifold. In particular, for a 2-link arm, whose links can rotate fully around is a 2-dimensional torus as was seen in Figure~\ref{fig:cspace:2link_arm_example}. 

In the next section, we discuss some methods of discovering the intrinsic structure of a manifold, which we employ in later chapters. These methods are called manifold learning methods, or more generally dimensionality reduction methods because they use a set of high-dimensional points to discover the intrinsic low-dimensional structure of the subspace on which the given points lie.

\section{Manifold Learning}
A dimensionality reduction method is a procedure that maps points living in a high-dimensional space onto a low-dimensional space without losing much information. This is possible when the given data points lie on or near a low-dimensional subspace of the high-dimensional ambient space. More precisely, given $X \in \mathbb{R}^{m \times n}$, a set of $n$ points living in an $m$ dimensional space, dimensionality reduction methods seek a $p$-dimensional embedding $Y \subseteq \mathbb{R}^{p \times n}$ of $X$, with $p \le \min(m, n-1)$, when the data points lie on or near a $p$-dimensional subspace of the $m$-dimensional ambient space. Usually, for a lot of real-life data sets, $p \ll m$. See ~\citep{ghodsi2006dimensionality}.

\begin{figure}[ht!]
    \centering
    \includegraphics[page=9, clip, trim = 3.5cm 5cm 4cm 5cm, width=0.5\textwidth]{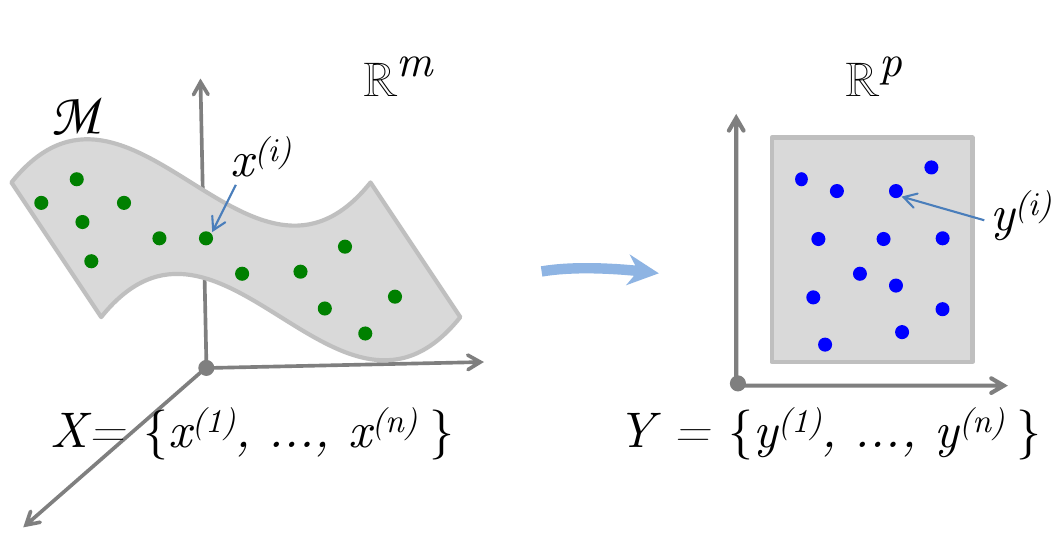}
    \caption{
        Manifold learning seeks to obtain a low-dimensional representation of high-dimensional points lying on or near a low-dimensional manifold.
    }
    \label{fig:manifold_learning_def}
\end{figure}

There are several methods of dimensionality reduction which can be classified into two broad types: (i) linear methods, and (ii) non-linear methods. Linear methods perform well when the given data points lie on or near a linear subspace of the ambient space. If the data lies on a non-linear subspace, then we will need to employ non-linear methods. Non-linear dimensionality reduction is also called \textit{manifold learning}, since the non-linear subspace on which the given data points lie, is assumed to be a manifold. 

Given a set of $n$ points $X = \{x^{(1)}, x^{(2)}, \ldots, x^{(n)}\} \subset \mathbb{R}^m$, lying on or near a $p$-dimensional manifold $\mathcal{M}$, with $p < m$, \textit{manifold learning} algorithms seek to learn a $p$-dimensional representation $Y = \{y^{(1)}, y^{(2)}, \ldots, y^{(n)}\} \subset \mathbb{R}^p$ of $X$. See Figure~\ref{fig:manifold_learning_def}.

\subsection{Principal Components Analysis}
\label{sec:pca}

Principal components analysis (PCA)\citep{jolliffe1986principal,hotelling1933analysis} is a linear dimensionality reduction method that reduces the data dimension while preserving the variance in data. In other words, PCA is a statistical procedure that converts a set of data points measured on a set of possibly correlated variables into a set of values of uncorrelated variables, called \textit{principal components} (or PCs). One way of doing this is to find a sequence of directions in the descending order of data variance, and project the data points onto some of these directions, so as to retain most of the variance after the projection. Then the resulting points are a low-dimensional representation of the original data.

Another way of looking at PCA is as a rotation of the coordinate axes, after mean-centering the given data, to align them with the directions of maximum variance, and then discarding some of the insignificant directions, i.e., the directions in which the data does not vary much. See Figure~\ref{fig:pca_example} for an illustration. 
\begin{figure}[ht!]
    \centering
    \includegraphics[page=10, clip, trim=0 4.2cm 0cm 6cm, width=0.98\textwidth]{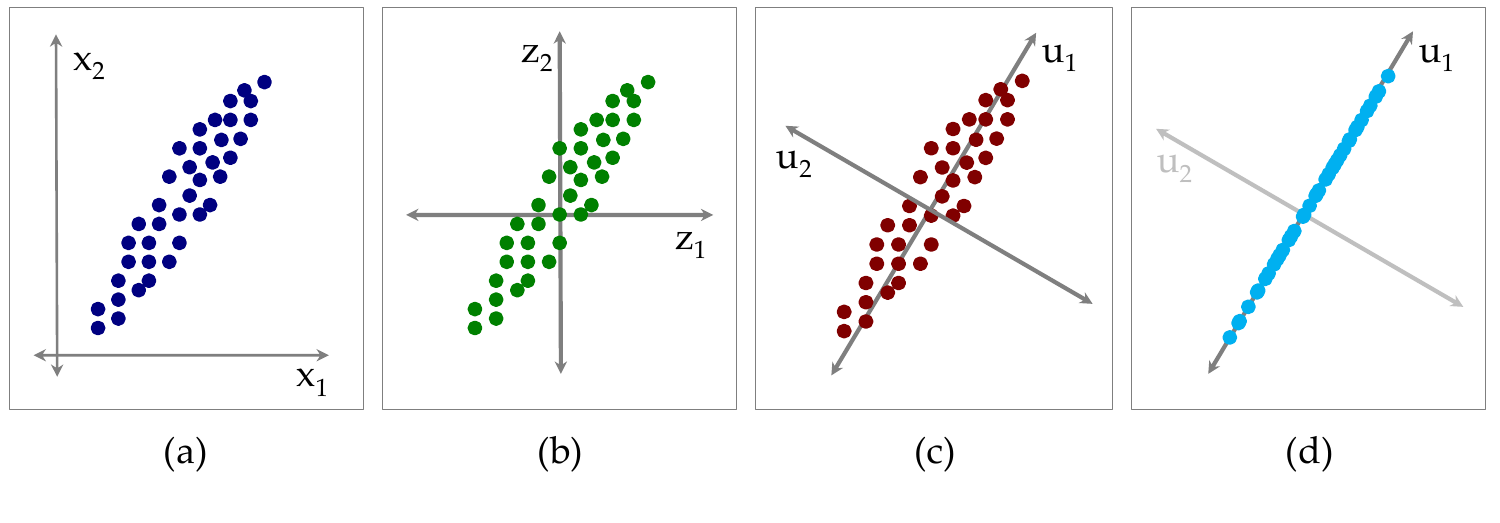}
    \caption{
        \textit{Illustration of PCA on a 2-dimensional data set:}
        (a) Original 2-D data points (b) Mean-centring the data (c) Rotation of the coordinate axes to align them with the directions of maximum variance in data (d) Projection of points onto $u_1$ (the first principal component) and discarding $u_2$ (the second principal component).
    }
    \label{fig:pca_example}
\end{figure}

Given a set $X \in \mathbb{R}^{m \times n}$ of $n$ iid random points in $\mathbb{R}^m$; the goal of PCA is to find an orthonormal basis $U = \{\hat{u}_1, \hat{u}_2, \ldots, \hat{u}_m\}$ such that the columns of $Y = U^TX$ are in the descending order of variance. Without loss of generality, assume that $X$ is mean-centered along each of the $m$ directions so that the empirical mean $\bar{X} = 0$ and the covariance matrix $C = XX^T$.

A $p$-dimensional embedding of $X$ can be obtained by projecting the data to the $p$ eigenvectors of $C$, corresponding to the top $p$ eigenvalues.

\subsubsection{Reconstruction and Interpolation}
\label{sec:pca:reconstruction}

Along with the projections on to the top $p$ principal components, PCA also gives a projection matrix $W_{m \times p}$ s.t. $W^Tx \in \mathbb{R}^p, \forall x \in \mathbb{R}^m$. So, if $y^{(i)} \in Y$ is the PCA projection of $x^{(i)} \in X$, we can have $x^{(i)}$ reconstructed from $y^{(i)}$, as $\hat{x}^{(i)} = W y^{(i)}$. This is applicable to any point in the projection space, even to those points which are not in $Y$. For example, we can interpolate between $y^{(i)}$ and $y^{(j)} \in Y$, as $y_\alpha = \alpha y^{(i)} + (1-\alpha)y^{(j)}$, for some $\alpha \in (0, 1)$, and construct the corresponding point in the original space as $\hat{x}_\alpha = W y_\alpha$.

\subsection{Multidimensional Scaling}
Multidimensional Scaling (MDS) is another method of dimensionality reduction, which tries to preserve the pairwise distances between points in the high-dimensional space. Given $X_{m \times n}$, let $D^X$ be the matrix of pairwise distances such that $D^X_{ij} = dist(x^{(i)}, x^{(j)})$. To find a lower dimensional embedding $Y$, conceptually, MDS solves the following optimization problem:
\[
    Y = \argmin_Y {\sum\limits_{i=1}^n\sum\limits_{j=1}^n (D^X_{ij} - D^Y_{ij})^2}, 
\]
where $D^Y_{ij} = dist(y^{(i)}, y^{(j)})$. Solution of this optimization problem is obtained through an eigenvalue decomposition (EVD) of the similarity matrix $S = X^TX$, after $X$ is mean-centered. If $V \Lambda V^T$ is the EVD of $S$, where $\Lambda$ is the diagonal matrix of eigenvalues and the columns of $V$ are the eigenvectors of $S$, then
\[
    Y_{p \times n} = \sqrt{\Lambda_p} V_p^T,
\]
is a $p$-dimensional embedding of $X$, where $\Lambda_p$ and $V_p$ are the matrices of top-$p$ eigenvalues and eigenvectors.

If we have access to the mean-centered data matrix $X$, we can apply EVD on $X^TX$ and get $Y$. Alternatively, we can also apply SVD on $X$ as done in PCA and get the $p$-dimensional embedding as 
\[
    Y_{p \times n} = \Sigma_p V_p^T,
\]
where $\Sigma_p$ is the diagonal matrix of top-$p$ singular values and the columns of $V_p$ are the top-$p$ right-singular vectors of $X$, which are same as the top-$p$ eigenvectors of $X^TX$.

If we are not given the data matrix $X$, but only the pairwise distance matrix $D^X$, then we can compute the similarity matrix $S$ by the following formula:
\[
    S = -\frac{1}{2}HD^XH, \text{ where } H = I_n - \frac{1}{n}\textbf{1}\textbf{1}^T, 
\]
with $I_n$ being the $n \times n$ identity matrix and $\textbf{1}$ being an $n \times 1$ column vector of all ones. Here, $H$ is called the centering matrix, which is needed to make sure that the similarities are calculated for the mean-centered data. This is needed because distances are translation-invariant while similarities are not.

The classical version of metric MDS preserves Euclidean distances, but any other distance metric can be used. In particular, Isomap preserves (approximate) geodesic distances between points on the manifold on which the given data points lie in the original space. The geodesic distance between two points on a manifold is the length of the shortest path between them, along the manifold.

\subsection{Isomap}
\label{sec:isomap}

The Isomap algorithm is a non-linear extension of the metric MDS algorithm. It generates a low-dimensional embedding of high-dimensional points, that approximately preserves the geodesic distance between the points lying on a manifold. 

As before, we are given a finite sample $X$ of $n$ high-dimensional points in $\mathbb{R}^{m}$ --- i.e., $X \in \mathbb{R}^{m \times n}$ --- lying on or near a smooth and uniform low-dimensional manifold of dimension $p$. We want to learn a low dimensional representation $Y \in \mathbb{R}^{p \times n}$ of $X$, such that $\forall x \in X, \exists y \in Y : x = f(y) + \epsilon$, where $f$ is the function that generated $X$ from a latent parameter space (the underlying manifold).

Isomap first computes a nearest-neighbours graph $G(V, E)$ such that $|V| = n$ and each vertex corresponds to a data point in $X$ and an edge is added between two vertices $v^{(i)}, v^{(j)}$ if the corresponding points $x^{(i)}, x^{(j)} \in X$ are near neighbours. Neighbourhoods can be decided in two ways: \textit{$\epsilon$-neighbourhood} and $k$-NN. In the first method, $(v^{(i)}, v^{(j)}) \in E$ if $dist(x^{(i)}, x^{(j)}) < \epsilon$, for a predetermined real value $\epsilon > 0$. In the second method, $(v^{(i)}, v^{(j)}) \in E$ if $x^{(i)}$ is one of the $k$ nearest neighbours of $x^{(j)}$ or vice versa, for some predetermined integer $k > 0$. In each case, an edge $(v^{(i)}, v^{(j)})$ is weighted by $dist(x^{(i)}, x^{(j)})$. This distance function is usually the Euclidean distance.

Then a distance matrix $D_{n \times n}$ is computed such that $D_{ij}$ is the shortest path length on $G$. The shortest path length between two vertices of $G$ is an approximation of the geodesic distance between the corresponding points on the manifold. See Figure~\ref{fig:isomap_illustration}.

\begin{figure}
    \centering
    \includegraphics[width=0.98\textwidth]{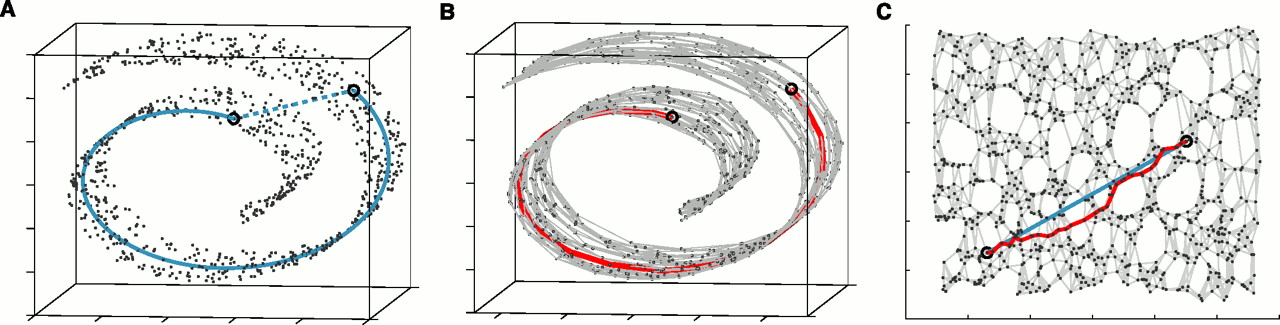}
    \caption{
        Illustration of geodesic approximation using shortest path distance on the neighbourhood graph. (A) A set of points lying on a ``Swiss roll" in $\mathbb{R}^3$, with the Euclidean distance (dotted straight line in blue) and the geodesic distance (solid curved line in blue) between two points (marked by small circles). (B) The neighbourhood graph $G$ computed on a set of 1000 data points with K = 7, along with the shortest path on $G$ (in red), approximating the geodesic between the chosen points. (C) The 2-dimensional Isomap embedding of points, along with the actual geodesic (in blue) and the approximate geodesic (in red). The embedding preserves the shortest path distances on $G$, with the hope of preserving the actual geodesic distances. This illustration has been taken from \citep{tenenbaum2000global}. 
    }
    \label{fig:isomap_illustration}
\end{figure}

Finally, MDS is applied on $D$ to get a low-dimensional embedding $Y_{p \times n}$.

We next discuss \textit{run length encoding}, which we will use in chapter~\ref{chap:rmp_in_vcs}, for fast collision detection.

\section{Run Length Encoding (RLE)}
\label{sec:rle_intro} 
Run Length Encoding (RLE) is a data compression mechanism that works by representing each run of a symbol (i.e., a sequence in which the same symbol occurs consecutively) by a pair $(c, s)$, where $s$ is the symbol in the run and $c$ is the length of that run. It is space-efficient when the data contains long runs.  For example, the string
\begin{center}
    \texttt{00000000001111111111111110000000000001111111111} 
\end{center}
has the encoding
\begin{center}
    $\langle (10, 0), (15, 1), (12, 0), (10, 1) \rangle$,
\end{center}
which tells that the given string has 10 $0s$ followed by 15 $1s$ followed by 12 $0s$ followed by 10 $1s$. Here, in the encoding, symbols other than numbers are used only for readability and explanation purposes.

If the data contains only binary digits, then we can store just the run lengths of each bit with the assumption that the first run is always a run of $0s$ and the runs alternate between $0s$ and $1s$. Under this encoding, the string in the previous example can be represented as
\begin{center}
    $\langle 10, 15, 12, 10 \rangle$.  
\end{center}
If the data begins with a 1, then the first run can be treated as a run of $0s$
of length zero.

\subsection{RLE for Binary Images}
\label{subsec:rle_bin_img}
%

A binary image is a pixel grid which can be represented as a matrix of $0s$ and
$1s$. An example of this is shown in Figure~\ref{fig:binary_image_example}, with
a white pixel represented by a 0 and a dark pixel represented by a 1.

\begin{figure}[t]
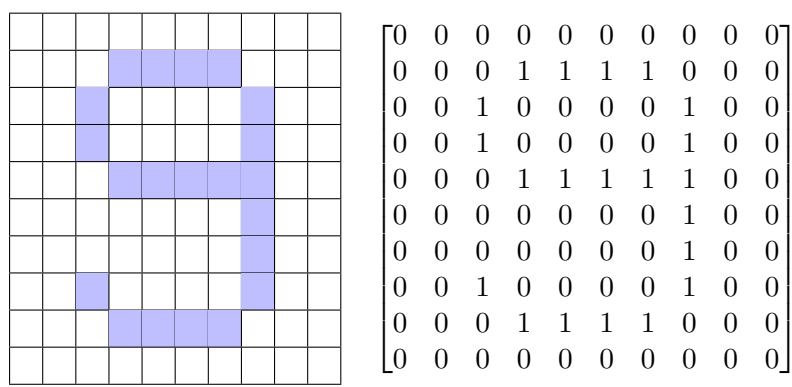

\centering
\begin{tabular}{ll}
\begin{tabular}{|l|l|l|l|l|l|l|l|l|l|}
    
    \hline
    & & & & & & & & & \\
    \hline
    & & & \dc & \dc & \dc &    \dc & & & \\
    \hline
    & & \dc & & & & & \dc &  &  \\
    \hline
    & & \dc & & & & &  \dc & & \\
    \hline
    & & & \dc & \dc & \dc &    \dc & \dc & &    \\
    \hline
    & &  & & & & & \dc &  &  \\
    \hline
    & & & & & & &  \dc & & \\
    \hline
    & & \dc  & & & & & \dc  & & \\
    \hline
    & & & \dc & \dc & \dc & \dc & & & \\
    \hline
    & & & & & & & & & \\
    \hline
    
\end{tabular} &
$
\begin{bmatrix}
    0 & 0 & 0 & 0 & 0 & 0 & 0 & 0 & 0 & 0 \\

    0 & 0 & 0 & 1 & 1 & 1 & 1 & 0 & 0 & 0 \\

    0 & 0 & 1 & 0 & 0 & 0 & 0 & 1 & 0 & 0 \\

    0 & 0 & 1 & 0 & 0 & 0 & 0 & 1 & 0 & 0 \\

    0 & 0 & 0 & 1 & 1 & 1 & 1 & 1 & 0 & 0 \\

    0 & 0 & 0 & 0 & 0 & 0 & 0 & 1 & 0 & 0 \\

    0 & 0 & 0 & 0 & 0 & 0 & 0 & 1 & 0 & 0 \\

    0 & 0 & 1 & 0 & 0 & 0 & 0 & 1 & 0 & 0 \\

    0 & 0 & 0 & 1 & 1 & 1 & 1 & 0 & 0 & 0 \\

    0 & 0 & 0 & 0 & 0 & 0 & 0 & 0 & 0 & 0 \\

\end{bmatrix}
$
\end{tabular}
\caption{A binary image shown as a pixel grid and as a binary matrix.}
\label{fig:binary_image_example}
\end{figure}

For binary images RLE can be applied row-wise, in which case the image in Figure~\ref{fig:binary_image_example} will have the following encoding:
\begin{center}
    $10; \langle 10, 3, 4, 3, 2, 1, 4, 1, 2, 2, 1, 4, 1, 2, 3, 5, 2, 7, 1,
    2, 7, 1, 2, 2, 1, 4, 1, 2, 3, 4, 3, 10 \rangle$,
\end{center}
where the 10 at the beginning indicates the width of the image in pixels, without which it will not be possible to know when to start a new row while trying to reconstruct the original image from the encoding.

To decode such an encoded sequence, one just needs to go through the sequence and create data bits as per the run lengths specified in the sequence. In the beginning, a variable, indicating the sum of run lengths processed so far, is initialized to 0. A new row is created in the image when the run lengths sum to the image width.

Alternatively, we can first flatten the image matrix into a one-dimensional array by concatenating all the rows into a single row and then apply RLE on the resulting array. In this case, the image in Figure~\ref{fig:binary_image_example} will be encoded as:
\begin{center}
    $10; \langle 13, 4, 5, 1, 4, 1, 4, 1, 4, 1, 5, 5,
    9, 1, 9, 1, 4, 1, 4, 1, 5, 4, 13 \rangle$.
\end{center}

\subsection{Interval Based RLE for Binary Images}
\label{subsec:interval_rle}
Equivalently, we can store the positions of intervals of $1s$ in the image array. An integer interval $[lb, ub)$ can be taken to mean that the array contains a 1 from position $lb$ (short for lower bound) till and excluding position $ub$ (short for upper bound), with the indices beginning from 0. All other positions can be assumed to have a 0. If we do that for the image in Figure~\ref{fig:binary_image_example}, the resulting encoding will be as follows:
\begin{center}
    $10 \times 10; \langle [13, 17), [22, 23), [27, 28), [32, 33), [37,    38),
    \;\;\;\;\;\;\;\;\;\;\;\;$
    \\
    $\;\;\;\;\;\;\;\;\;\;\;\;\;\;\;\;\;
    [43, 48), [57, 58), [67, 68), [72, 73), [77, 78), [83, 87) \rangle$.
\end{center}
Here, $10 \times 10$ at the beginning indicates the original image size without which reconstruction of the image from the encoding would not be possible. An interval such as $[13, 17)$ indicates that the image array contains a 1 from position 13 to position 16 (indexing is assumed to be zero-based and the intervals do not include upper bounds). 

We will use interval based RLE for the purposes of collision detection in robot motion planning in Chapter~\ref{chap:rmp_in_vcs}.

\chapter{Literature Survey} 

\label{chap:literature_survey} 

\lhead{\emph{Literature Survey}} 

Humans and animals routinely use prior sensorimotor experience to build motor models, and use vision for gross motor tasks in novel environments. Achieving similar abilities in a robotic system, without having to calibrate the robot's own body structure, or estimate the exact 3-D positions, is a touchstone problem for robotics (e.g. see~\citep{engelberger-Joseph-1980_robotics-in-practice} ch.9). Such an approach would enable a robot to work in less controlled environments, as is being increasingly demanded in social and interactive applications for robots. 

There have been two methods for approaching this problem --- either based on learning a \emph{body schema}~\citep{poincare-1895-space-and-geometry, hoffmann2010body, pierce-kuipers-97_map-learning-uninterpreted-sensors, philipona-oregan-2003_perception-of-structure-unknown-sensors, arleo-smeraldi-04_cognitive-navigation_nonuniform-gabor_reinforcement, stober-fishgold-kuipers-09_sensor-map-discovery}, or by fitting a canonical robot model~\citep{sigaud-sala-11_online-regression-learning-robot-models_survey}. Body schema approaches have not scaled up to full scale robotic models or used for  global motion planning, and robot model regression requires intrusive structures on the robot~\citep{sturm2013approaches} and even then it cannot sense the environment.

Another approach, \emph{visual servoing} attempts to estimate the motion needed for small changes in image features. However, visual servoing models cannot construct models spanning large changes in robot pose, since the pseudo-inverse of the image Jacobian can be computed only over small motions.  Recently, global motion planning algorithms have been proposed by stitching together local visual servos~\citep{kazemi-gupta-10_path-planning-visual-servoing}, but these require that the goal be constantly visible. 


A number of schemes have been proposed for learning the body schema from visual inputs (Review:~\citep{hoffmann2010body}). Most of them involve varying degrees of knowledge about the kinematics~\citep{hikita2008visual,martinez2010body}. Some require special decals pasted on the robot to remain visible~\citep{sturm2013approaches}.  Alternately, one may discover the dimensionality of the input-output relation by analyzing locally linear tangent spaces~\citep{philipona-oregan-nadal:2003}. Other work, often termed \emph{developmental robotics}, attempts to establish visuomotor correlations by analyzing random motions (motor babbling)~\citep{caligiore-baldassarre-08_motor-babbling-developmental-reaching-w-obstacles}, by observing smooth patches in optical flow data~\citep{olsson-nehaniv-polani-06_from-sensorimotor-to-actions-info-theory}, or via clusters of sensory data that permit the recognition of object categories and shapes~\citep{modayil-kuipers-07aaai_robot-learning-grounded-object-ontology, modayil-10_discovering-sensor-space}. Other approaches have focused on discovering the topology~\citep{ranganathan-dellaert-11_online-probabilistic-topological-mapping}, or on constructing dynamical system models~\citep{shatkay-kaelbling-02_learning-hmms-sensorimotor-topological-geometric}. 

\citep{sturm2013approaches} proposes methods of learning the kinematic structure and properties of a mobile manipulation robot and methods of learning novel manipulation tasks from human demonstrations. This work uses visual self-observation to learn the robot’s own body schema from scratch and use it to learn the properties of articulated objects. Here, the visual observations of the links of the arm are modelled as a Gaussian process and the body schema of the robot is defined in terms of Bayesian networks that describe the kinematics of the system. This work also describes approaches to automatically update the body schema after the robot uses a tool.

\citep{hoffmann2014minimally} addresses cognitive developmental robotics through a case study on a quadruped robot. This work interprets the results of the case study from an \textit{enactive} perspective. Enactive robots construct their identity by interacting with the environment continuously. \citep{lanillos2017enactive} is another work that acquires the sensorimotor self through enaction. It describes algorithms to learn the body schema and to enable tool-extension, and shows their usefulness in generalizing the computational model of enactive self.

\textit{Goal babbling} is an extension of the findings in developmental psychology that suggest that human infants try to make goal-directed movements, in addition to random motor babbling, in the early stages of development. \citep{schmerling2015goal} investigates how \textit{goal babbling} relates to a visuomotor coordination task in  Aldebaran Nao. This work also suggests that goal babbling is effective in coordinating the motion of head and arm.

\citep{schillaci2016exploration} surveys the studies in exploration behaviors such as random motor babbling, internal body representations such as body schema, and sensorimotor simulation processes for cognitive development in artificial agents such as robots. 

\citep{lanillos2018adaptive} is a work that formulates body learning as a problem of predictive coding. In this work, the authors propose a method of obtaining a forward model which encodes the sensor values as a function of the body variables, and solve it using Gaussian process regression. They model the problem of body estimation as one of minimizing the discrepancy between the robot's belief about its body configuration and the observed posterior. They test the proposed formulation on a real multisensory robotic arm and show how different sensory modalities help in refining the body estimation. 

\citep{zenha2018incremental} considers touch events to incrementally adapt a body schema in robots. The authors of this work enable a humanoid robot to incrementally estimate model inaccuracies by allowing it to touch some known planar surfaces like walls, through motor babbling, thereby making it adapt its own body schema using the contact information alone. They formulate this problem as an adaptive parameter estimation using Extended Kalman Filter, that uses planar constraints obtained at each contact detection. They perform a set of experiments to compare different incremental update methods, using a simulated version of the iCub humanoid robot.
\chapter{Visual Configuration Space}
\label{chap:vcs}

\section{Configuration Space and Generalized Coordinates}
The notion of \textit{Configuration Space} (also known as C-Space) is fundamental to conceptualizing multibody motion. The \textit{configuration} of a multibody system is \textit{any} complete description of the position of every point of the system, with respect to a fixed point in its world. The \textit{configuration space} of the system is the set of all possible configurations of the system. For a robotic system, typically, its configuration is given by a parameter vector which can be mapped to the set of positions of all points on the robot body. Each configuration is a point in the configuration space. The \textit{degrees of freedom} of the robot is the dimension of the configuration space, which is given by the minimum number of independent parameters required to specify the configuration. These independent parameters are called the \textit{generalized coordinates}. The configuration of a system with $d$ degrees of freedom can be specified in terms of $d$ generalized coordinates. In this and the subsequent chapters, we will sometimes use the word \textit{pose} synonymously with \textit{configuration}.


For a planar robotic arm with two links, the canonical choice for generalized coordinates is to use the joint angles $(\theta_1, \theta_2)$. Given the knowledge of a robot's geometry, such as the shape of the links, position of joints on the links, link lengths, and other such parameters, all of which are fixed for a given multibody system, generalized coordinates like joint angles are sufficient to specify the position of every point on the robot body. More generally, the generalized coordinates of an open chain articulated robot with $d$ links connected by $d$ revolute joints, can be given by its $d$ joint angles  $(\theta_1,...,\theta_d)$. 

However, this is only one of the many (potentially infinite) choices of coordinates, each resulting in a different C-space. Generalized coordinates need not specify joint angles or any motion parameter --- they just need to uniquely specify the position of every point on the body. One of our main aims is to show that an alternate set of generalized coordinates can be learned from the robot's appearance alone, i.e. from a set of images of the robot in various configurations.

\section{Visual  Configuration Space (VCS)}
\label{sec:vcs}

In order to understand the idea of a \textit{Visual Configuration Space} (VCS), let us consider the images of a robot, all of which are captured by a camera in a fixed position in the robot's workspace, such that each image corresponds to a configuration of the robot, for example for the 2-DOF robot of Figure~\ref{fig:2dof_arm_image_space}. Let the set of all such images, where each image corresponds to one configuration of the robot, be denoted by $\mathcal{I}$. If each image is $600 \times 600$ RGB pixels, it is represented by $3*600*600 \approx 10^6$ integer values each in the range $[0, 255]$. Here the factor $3$ corresponds to the three color channels: red, green and blue.  Then, $\mathcal{I} \subset \mathbb{R}^{10^6}$, so that each image of the robot is a point in a million dimensional space. So, at first sight, the elements of $\mathcal{I}$ seem to be very high dimensional points.

However, not every point in $\mathbb{R}^{10^6}$ corresponds to an image of the robot in a valid configuration. To see why this is the case, consider an image formed by randomly turning on some of the pixels in a $600 \times 600$ RGB image. The probability that such an image looks like an image of the robot in a valid configuration is vanishingly small. Thus, only a very small fraction of the points of $\mathbb{R}^{10^6}$ make up robot images. 

Moreover, given an image $x \in \mathcal{I}$, it can be altered in only as many ways as the degrees of freedom of the robot, without the resulting image leaving $\mathcal{I}$, because all the pixels corresponding to a moving link of the robot vary together. So the robot images lie on a very small subspace $\mathcal{I}$ of the ambient space $\mathbb{R}^{10^6}$. As we will see later in this chapter, for a $d$-DOF robot, the intrinsic dimensionality of $\mathcal{I}$ would be $d$.

\begin{figure}[h]
    \centering
    \begin{subfigure}{0.9\textwidth}
        \includegraphics[width=\textwidth]{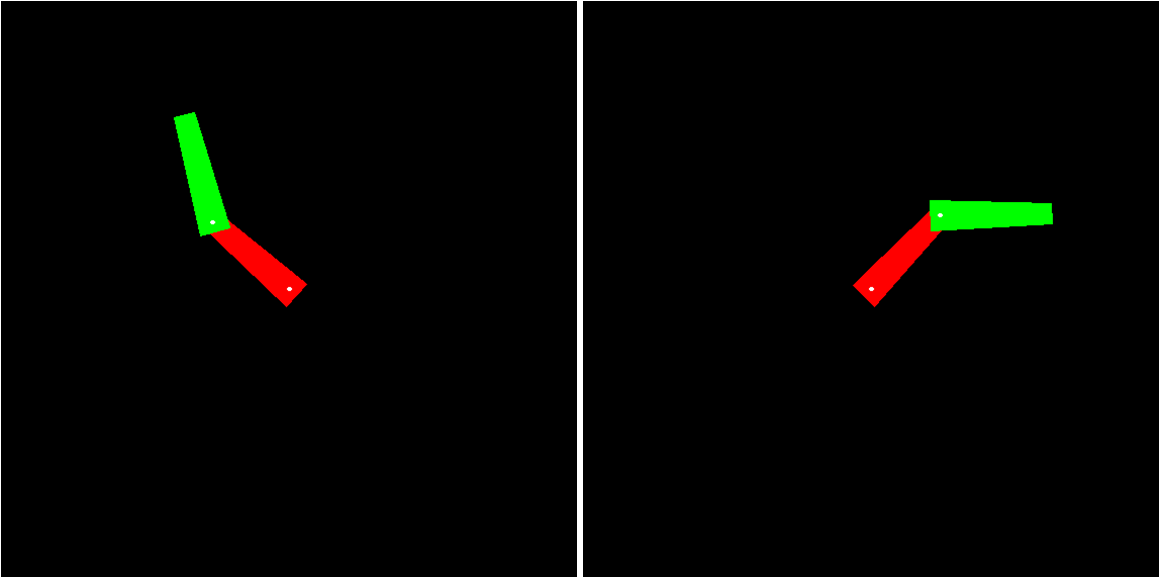}
        \caption{Two images of a 2-DOF planar arm, each in a different
        configuration of the arm.}
        \label{fig:2dof_arm_image}
    \end{subfigure}
    
    \begin{subfigure}{0.9\textwidth}
        \includegraphics[width=\textwidth]{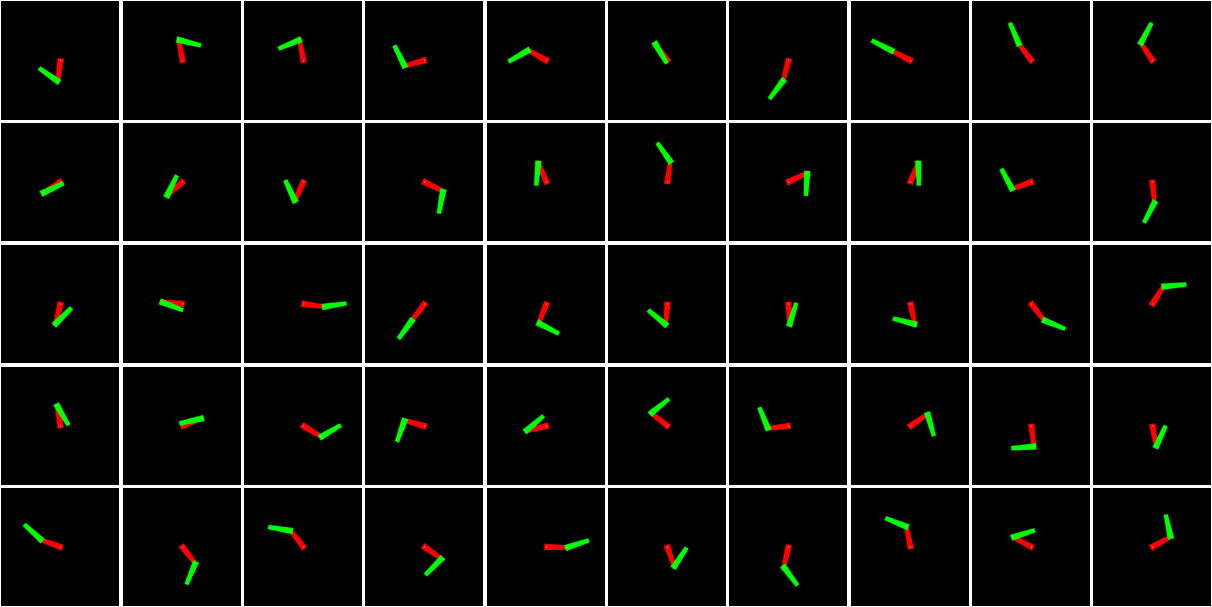}
        \caption{A sample of 50 images from the image space of the arm in
        Figure~\ref{fig:2dof_arm_image}.}
        \label{fig:2dof_arm_image_sample}
    \end{subfigure}
    \caption{
        \textit{Image space of a 2-DOF planar arm}: each image corresponds to one configuration of the robot. If the images are of $600 \times 600$ RGB pixels, then each image is a point in a space of $3*600*600 \approx 10^6$ dimensions.
    }
    \label{fig:2dof_arm_image_space}
\end{figure}

For the 2-DOF robot of Figure~\ref{fig:2dof_arm_image_space}, the intrinsic dimensionality of $\mathcal{I}$ is 2. Further, if we allow both joints to rotate fully around from $0$ to $2\pi$ radians, as each link keeps rotating starting from a home configuration, at some point the image sequence of that link returns to the original image, because that link reaches the original position after a full rotation of $2\pi$ radians. Thus the topology of $\mathcal{I}$ is not Euclidean ($\mathbb{R}^2$), but $S^1 \times S^1$, i.e., a 2-dimensional torus. And in general, the topology of $\mathcal{I}$ for a $d$-DOF articulated arm, without any restrictions on the link rotations, is a $d$-torus.

This low-dimensional subspace of the ambient Euclidean space is a manifold which is homeomorphic to a canonical C-space, under some conditions discussed in section~\ref{sec:visual_distinguishability}. We call this manifold a \textit{Visual Configuration Space} (VCS), as it is discovered from a set of images. The coordinates assigned to the points on the VCS will be called \textit{Visual Generalized Coordinates}, akin to the generalized coordinates of the conventional C-space.

We now establish the conditions under which the space of all images of the robot would form a manifold. 

\subsection{Visual Distinguishability Assumption}
\label{sec:visual_distinguishability}
Let $R_q$ be the set of all points of the workspace occupied by the robot (its volume) in configuration $q$ (i.e., $R_q$ is the shape of the robot in configuration $q$) and let $R(\mathcal{Q}) = \{R_q: q \in \mathcal{Q}\}$ be the set of all robot shapes. Let $\phi: \mathcal{Q} \rightarrow R(\mathcal{Q})$ and $\psi: R(\mathcal{Q}) \rightarrow \mathcal{I}$ be the functions that map a configuration to a shape and
a shape to an image respectively. Then the visual distinguishability assumption requires that the function $\psi \circ \phi: \mathcal{Q} \rightarrow \mathcal{I}$ be a bijection as illustrated in Figure~\ref{fig:visual_manifold_theorem}.

\begin{figure}[h] 
    \centering
    \includegraphics[page=3, clip, trim = 1cm 4cm 1cm 4cm, width=\textwidth]{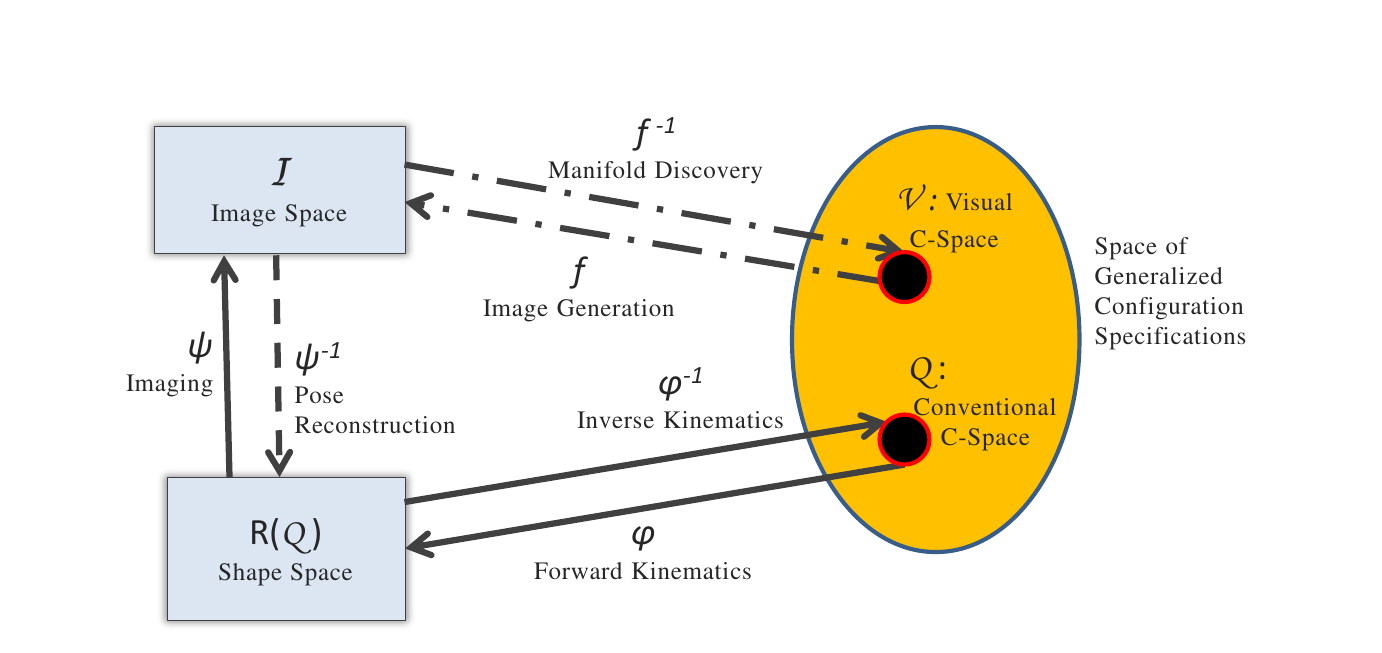}
    \caption{ 
        \textit{Correspondence between robot configuration, robot shape, image and visual configuration spaces}: 
        In order for visual generalized coordinates to exist, two robot configurations cannot generate the same image. If this condition holds, we show that any coordinate for the image manifold constitutes a generalized coordinate system. Under such conditions, the map $\psi\circ\phi$ between $\mathcal{Q}$, the shape space $R(\mathcal{Q})$ and the image space $\mathcal{I}$ is bijective, and the image manifold does not self-intersect.   The latent space $\mathcal{V}$ is a specification of generalized coordinates on the image manifold and is a member of the collection of C-spaces. The bijective map $f: \mathcal{V} \leftrightarrow \mathcal{I}$ relates robot images to unique points in $\mathcal{V}$. 
    }
    \label{fig:visual_manifold_theorem}
\end{figure}

\begin{figure}[h]
\centering
    \begin{subfigure}{0.35\textwidth}
        \includegraphics[width=\columnwidth]{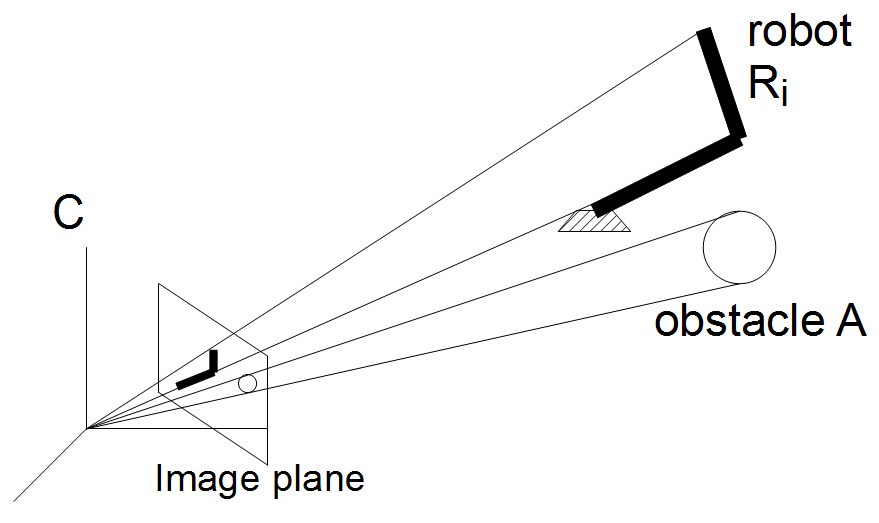}
        \caption{}
    \end{subfigure}
    \begin{subfigure}{0.25\textwidth}
        \includegraphics[width=\columnwidth]{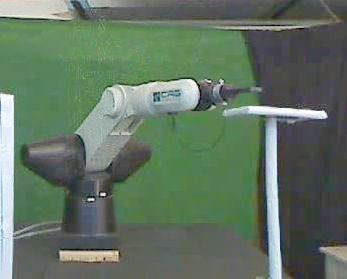}
        \caption{}
    \end{subfigure}
    \begin{subfigure}{0.25\textwidth}
        \includegraphics[width=\columnwidth]{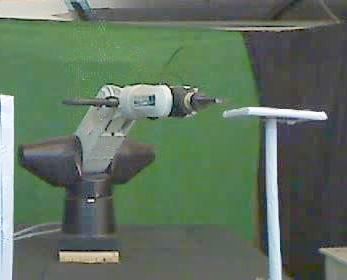}
        \caption{}
    \end{subfigure}
    \caption{
        \textit{Imaging the workspace}.  (a) The robot and obstacle lie along the projection bundle from the optical center via their image regions in the virtual image plane.  If these bundles do not intersect, $R \cap A = \emptyset$.  However, the converse is not true.  (b, c) Images of CRS A465 6-axis robot appear to be neighbouring poses, but close observation reveals that the base joint $\theta_1$ has rotated by nearly 180 degrees, while $\theta_2$ and $\theta_3$ have changed sign. Such situations are avoided in the analysis by additional cameras, or by adding decals.
    }
    \label{fig:visual-symmetry}
\end{figure}

In general, the imaging transformation $\psi$ is not invertible --- i.e., the 3D positions are not recoverable from the image. In such cases, the visual distinguishability assumption does not hold. This can happen when there is any visual symmetry in the robot's motion, as in Figure~\ref{fig:visual-symmetry}. For example, for a cylinder rotating about its own axis, unless the boundary is coloured to break the visual symmetry, all configurations result in the same image. This can also happen when the changes in the robot pose are hidden from the camera view, and hence two different configurations of the robot result in the same image. 

To ensure visual distinguishability in all such cases, we assume that at least one of the following requirements is satisfied:
\begin{itemize}
    \item {
        The robot is coloured differently at each local patch on its body, for example by use of some colour texture. 
   }
   \item {
        There is a restricted range of motion, that permits distinguishability of all robot poses.
   }
   \item {
        There are multiple cameras, as discussed in section~\ref{sec:collision_detection_vcs}, and for every pair of configurations of the robot, the resulting images look different in at least one view.
   }
\end{itemize}
This is the \textit{visual distinguishability} assumption. Formally, this assumption holds when $\psi \circ \phi: \mathcal{Q} \rightarrow \mathcal{I}$ is a bijection.

The imaging transformation $\psi\circ\phi$ maps each configuration $q$ to an image $I_q$ projected by the boundary $\delta R_q$ of shape $R_q$.  If the visual distinguishability assumption holds, then both $\phi$ and $\phi^{-1}$ exist. Moreover, these maps are continuous because small changes in the robot configuration lead to small changes in its shape and the corresponding images and vice versa. So, whenever $\mathcal{Q}$ is a manifold, $\mathcal{I}$ is also a manifold of the same dimension (i.e., for a $d$ DOF robot whose configuration space is a manifold, the image space is a $d$-dimensional manifold). We refer to this image manifold as the \textit{Visual Configuration Space}. We will later discuss how to build a Visual Roadmap (VRM) on this manifold.

Even though for practical purposes, we use discrete pixel values on a discrete pixel grid to represent digital images, the intensity of the optical signal that makes up the image and the image plane are continuous in theory. So, the image space changes smoothly as long as the robot motions are smooth. 


\subsection{Visual Manifold Theorem}
\begin{definition}
A \textit{Smoothly Moving Piece-wise Rigid body} (SMPR) is any multibody system with a smooth map from its configuration space $\mathcal{Q}$ to its shape space $R(\mathcal{Q})$.
\end{definition} 


\begin{definition}
A \textit{visually distinguishable system}  is one for which the visual distinguishability assumption holds.  
\end{definition}

Hence, for a visually distinguishable SMPR, $\psi \circ \phi: \mathcal{Q} \rightarrow \mathcal{I}$ is a homeomorphism.

\begin{theorem}
For a visually distinguishable SMPR, whenever $\mathcal{Q}$ is a manifold, $\mathcal{I}$ is a manifold of the same dimension.
\end{theorem}
\begin{proof}
The imaging transformation $\psi \circ \phi$ maps each configuration $q$ to an image $I_q$ projected by the boundary $\delta R_q$ of shape $R_q$. If the \textit{visual distinguishability} assumption holds, then
both $\phi^{-1}$ and $\psi^{-1}$ exist. Since $\psi$ is a perspective projection with no singularities, $\psi$ and $\psi^{-1}$ are continuous. The functions $\phi$ and $\phi^{-1}$ map infinitesimal changes in a configuration to infinitesimal changes in the corresponding shape and vice versa; hence $\phi$ and $\phi^{-1}$ are continuous. Therefore, the image space $\mathcal{I}$ is locally homeomorphic to the configuration space $\mathcal{Q}$.  Hence  $\mathcal{I}$ constitutes a manifold of the same dimension as that of $\mathcal{Q}$, whenever $\mathcal{Q}$ is a manifold.
\end{proof}

While these properties hold for the continuous image space, in practice we work with a representative sample $X = \{x^{(1)}, ..., x^{(n)}\} \subset \mathcal{I}$. We now discuss how the VCS can be discovered using manifold learning algorithms.

\subsection{VCS Discovery through Manifold Learning}
Since the robot images lie on a very small subspace of the ambient space, it should be possible to assign a low-dimensional point to each image. This can be achieved by any of a number of non-linear dimensionality reduction (NLDR) algorithms~\citep{lee-verlysen-07_nonlinear-dimensionality-reduction}. These algorithms discover the subspace on which the given high-dimensional data points lie and give a low-dimensional representation to each of those points. One such algorithm is the well-known Isomap algorithm~\citep{tenenbaum2000global}, which generates a low-dimensional representation for the high-dimensional points, in such a way as to preserve the distances between points. 

\begin{figure}[h!]
    \begin{subfigure}{\textwidth}
        \centering
        \includegraphics[clip, trim = 3cm 9cm 4cm 9cm, width=0.7\columnwidth]{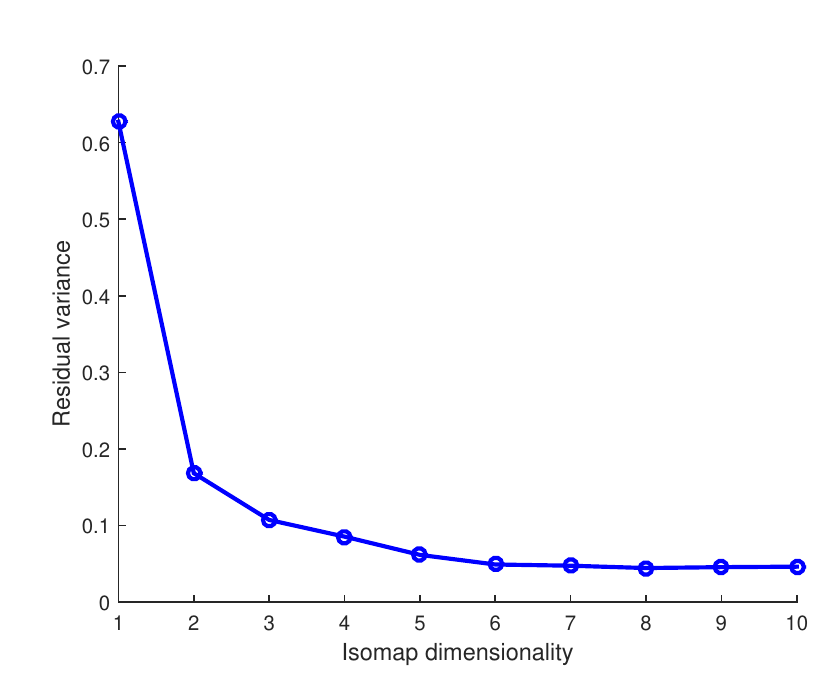}
        \caption{Scree plot of Isomap dimensions}
        \label{fig:isomap_2dof_arm:scree_plot}
    \end{subfigure}
    
    \begin{subfigure}{\textwidth}
        \centering
        \includegraphics[width=0.7\columnwidth]{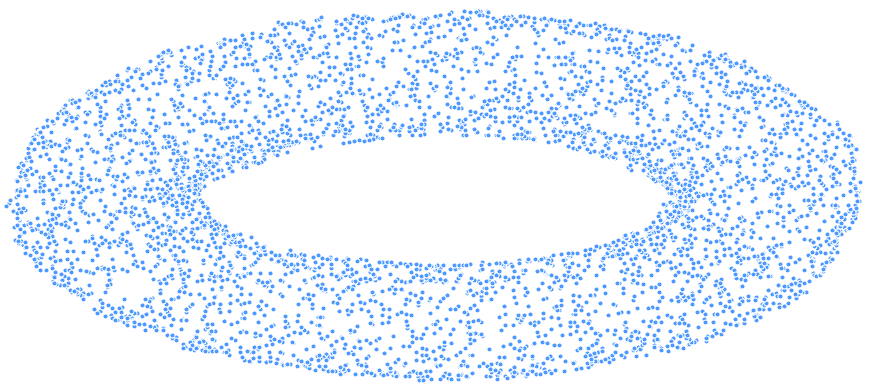}
        \caption{A 3-dimensional Isomap embedding}
        \label{fig:isomap_2dof_arm:3d_embedding}
    \end{subfigure}
    
    \caption{
        Isomap on a set of 5000 images of the 2-DOF articulated robotic arm of Figure~\ref{fig:2dof_arm_image_space}: (a) The scree plot has a `knee' at dimension 2, which suggests that the intrinsic dimension of the image manifold is 2, which is same as the DOF of the arm. (b) The embedding very closely resembles the toroidal structure of the topology of the canonical C-space of a 2-link arm in Figure~\ref{fig:cspace:2link_arm_example}. Even though the intrinsic dimension of the image manifold is 2, this Isomap embedding is not globally homeomorphic to $\mathbb{R}^2$ and hence cannot be embedded in 2-D, without losing some information along $\theta_2$.
    }
    \label{fig:isomap_2dof_arm}
\end{figure}

In many such algorithms, the first step is to construct a nearest neighbourhood graph (based on a measure of distance in $\mathcal{I}$). If the graph is connected, then distances between any two distal images can now be approximated by a shortest path through the edges connecting near neighbours. The Isomap algorithm constructs a low-dimensional embedding by attempting to preserve this geodesic distance. 

For the demonstrations presented here, we restrict ourselves to Isomap for this reason: though far from exact, it gives a closer resemblance to the global metric distances compared to other algorithms.  In order to estimate the robot DOFs, we simply try out a range of target dimensions and choose the lowest dimension that is able to adequately explain the variance in the data (based on residual variance). See Figure~\ref{fig:isomap_2dof_arm}.

If we are able to discover the DOF of the robot, the lower-dimensional space can be described in terms of $d$ latent parameters $v_1, \ldots, v_d$, which act as (state) parameters of the robot. These are visual analogs of nonlinear combinations of the configuration parameters of the robot and the space of these visual latent parameters is the \textit{Visual Configuration Space} (VCS).

However, such methods have difficulty in introducing new data points and in interpolating local data. So, for the purposes of motion planning, we avoid computing the manifold altogether, and restrict ourselves to a piecewise algorithm, as in~\citep{kambhatla-leen-97_dimension-red-by-local-PCA,yang-wang-05_better-scaled-local-tangent-space}.

\subsection{Difficulties with Manifold Discovery Algorithms}
\label{sec:non-euclidean}
For robots which involve a motion with an $S^1$ topology, the C-space and hence the VCS is not globally Euclidean. For example, the C-space of a freely-rotating 2-DOF articulated robot is $S^{1} \times S^{1} = \bf{T}^2$, which is a torus~\citep{choset-05_robot-motion-theory}. Traditional nonlinear dimensionality reduction (NLDR) algorithms assume that the target space for dimensionality reduction is a  Euclidean space (a subspace of $\mathbb{R}^D$). This means that a $d$-torus manifold, which is $d$-dimensional, cannot be globally mapped to an $\mathbb{R}^d$ space, with which it is locally homeomorphic. Another practical difficulty with NLDR algorithms is that it is very challenging to add new points to the manifold without recomputing the entire structure~\citep{bengio2004out}.  

\begin{figure}[h!]
    \centering
    \includegraphics[width=0.6\columnwidth]{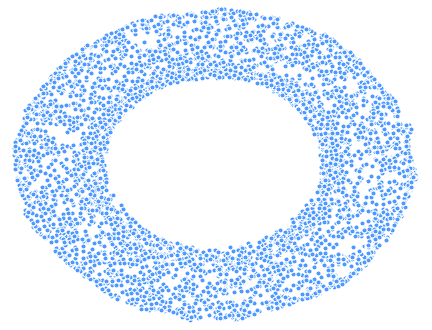}
    \caption{A 2-dimensional Isomap embedding of the 2-DOF arm images of Figure~\ref{fig:2dof_arm_image}.}
    \label{fig:isomap_2dof_arm:2d_embedding}
\end{figure}

Hence, a 2-dimensional Isomap embedding of the image manifold, as shown in Figure~\ref{fig:isomap_2dof_arm:2d_embedding}, resembles the torus in capturing the variability along the $\theta_1$ dimension of this space, but not $\theta_2$. 

At the same time, the global non-linear coordinate is little more than a convenience and does not materially affect the modelling, which can be done in a piecewise linear manner. Thus, we avoid computing global coordinates altogether and use the local neighbourhood graphs for planning global
paths and local tangent spaces, discovered using Principal Component Analysis (PCA), for checking the safety of edges (local planner). These local tangent spaces, in theory, correspond to charts which when stitched together form an atlas for the image manifold.

We note that much of this topological complexity would be reduced for most of the real world robots because they have a restricted range of rotation and their motion manifold does not have a torus topology. For example, the Scara arm demonstrated later in section~\ref{sec:scara} has $-135^\circ$ to $135^\circ$ for both $\theta_1$ and $\theta_2$. This implies that the mapping, though it is part of the surface of a torus, can be stretched and would fit in a subspace of $\mathbb{R}^2$. 

We next describe how obstacles are mapped on the VCS for collision detection.


\subsection {Collision Detection in VCS}
\label{sec:collision_detection_vcs}
In the imaging process, robot and obstacle are mapped to a bundle of rays converging on the camera optical center (Figure~\ref{fig:visual-symmetry}). Robot configurations that do not intersect with this bundle are guaranteed to be in free space.  
 
\begin{figure}[th]
\begin{center}
\includegraphics[width=0.6\columnwidth]{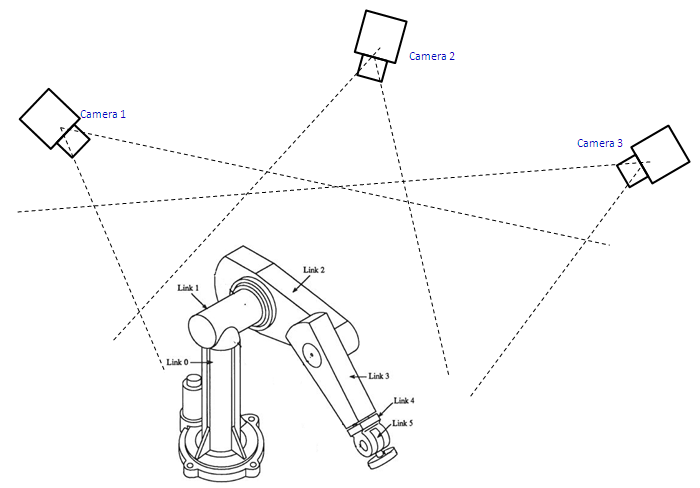}
\end{center}
\caption{
\textit{Conservative modelling of 3-D obstacles}. 
    For 3-D obstacles, the robot image must not overlap with the obstacle in at least one camera view. This is a depiction of this situation, shown via a 2D simplification. If some part of the robot occludes the obstacle in the cones for \textit{all} the cameras, then our system will consider it to be a collision situation.  
} 
\label{fig:multi_view_cameras}
\end{figure}

Let $\prescript{C}{}{R_i}$ be the bundle subtended at camera optical center $\prescript{C}{}{O}$ by the robot in configuration $q^{(i)}$, $\prescript{C}{}{B}$ be the bundle subtended at $\prescript{C}{}{O}$ by the obstacle $B$ and $\prescript{I}{}{R_i}$, $\prescript{I}{}{B}$ be the image regions corresponding to the robot and the obstacle.

\begin{lemma}
\label{lem:vc1}
If $\prescript{C}{}{R_i} \; \cap \; \prescript{C}{}{B} = \emptyset$, then $B \cap R(q^{(i)}) = \emptyset$.
\end{lemma}
\begin{proof}
If the obstacle $B$ had a non-empty intersection with the robot region in the workspace for configuration $q^{(i)}$, so that $B \cap R(q^{(i)}) \ne \emptyset$, then the optical bundles subtended by the robot and the obstacle would also overlap, so that $\prescript{C}{}{R_i} \; \cap \; \prescript{C}{}{B} \ne \emptyset$.
\end{proof}

Thus, robot configurations for which the bundles do not intersect with the obstacle bundle are guaranteed to be in the free space $\mathcal{Q}_{free}$. Note that the converse is not true. Even when a configuration of the robot is free, it is possible, in some view of the camera, that the obstacle is in between the camera and the robot, or the robot is in between the camera and the obstacle and hence that configuration may appear as a collision configuration from that view.

\begin{lemma}
\label{lem:vc2}
$\prescript{C}{}{B} \; \cap \; \prescript{C}{}{R} = \emptyset$ iff $\prescript{I}{}{B} \; \cap \; \prescript{I}{}{R} = \emptyset$. 
\end{lemma}
\begin{proof}
$\prescript{I}{}{R}$ and $\prescript{I}{}{B}$ are sections of cones $\prescript{C}{}{R}$ and $\prescript{C}{}{B}$; so $\prescript{C}{}{B}$ and $\prescript{C}{}{R}$ are pyramids with their apexes at $\prescript{C}{}{O}$, and $\prescript{I}{}{B}$ and $\prescript{I}{}{R}$ are their projections on the image plane. So, a non-empty intersection of images requires and is required by a non-empty intersection of the corresponding optical bundles.
\end{proof}

\begin{theorem}
\textit{(Visual Collision Theorem)}
For a robot in a given pose $q^{(i)}$, if  $\prescript{I}{}{R_i} \; \cap \; \prescript{I}{}{B} = \emptyset$, then $q^{(i)} \in \mathcal{Q}_{free}$.
\end{theorem}
\begin{proof}
Follows from Lemma~\ref{lem:vc1} and Lemma~\ref{lem:vc2}.
\end{proof}


We note that the above is a necessary condition, but it is often rather conservative.  Indeed, the inverse condition defines \textit{occlusion} situations: where $B \cap R = \emptyset$ but $\prescript{C}{}{R} \; \cap \; \prescript{C}{}{B}$ is non-null. This limitation is a result of the information loss in the imaging process. This can cause particular difficulties for articulated arms.  In such cases, one may use multiple cameras; since the \textit{visual collision theorem} holds for all cameras, we may define any space as free if $\prescript{C}{}{R} \; \cap \; \prescript{C}{}{B} = \emptyset$ in at least one view.  In this situation, both robot and obstacle are less conservatively modelled as the intersection of multiple cones.

For non-planar motions, such as for 3D robot motion, one can consider multiple cameras; free space is guaranteed to be the disjunction of the free spaces guaranteed by each camera.  Hence, the obstacle is guaranteed to be a subset of the conjunction. 


Visual Configuration Space (VCS) is the manifold on which the sampled robot images lie. However, discovering this manifold is an expensive operation and we do not need the concrete visual generalized coordinates of the VCS, for motion planning purposes. We can work with a graph embedded on the VCS, which is constructed using the robot images. We call this graph a \textit{Visual Roadmap} (VRM). 

In the next chapter, we describe how VRM is constructed and discuss how it could be used to address various issues in robot motion planning and illustrate its use to plans motions for some simulated planar robots and real 3-D robots. 
\chapter[Robot Motion Planning in Visual Configuration Space]{Robot Motion Planning in \\ Visual Configuration Space} 

\label{chap:rmp_in_vcs} 
\lhead{\emph{Motion Planning in VCS}} 

\textit{Robot Motion Planning} (RMP) is the problem of getting a robot to move from one configuration to another configuration without hitting any obstacles present in the workspace. Traditional solutions to the RMP problem involve constructing an explicit representation of the robot in terms of the number of degrees of freedom, joint and link geometry, rules for forward kinematics and inverse kinematics to map from the configuration space to the workspace and vice versa. Once such a representation is available, obstacle-free paths can be planned in the configuration space. This approach is cumbersome because different representations need to be hand-engineered for different robots, obstacles and
workspaces. If we consider the triple $\langle \text{robot, obstacles, workspace} \rangle$ as a \textit{system}, the traditional approach is quite \textit{system specific}. Any change in the system would require operating on a different hand-engineered model.

We propose an approach to motion planning that requires much less system-specific information to plan paths. We use the ideas developed in the previous chapter and plan obstacle-free motions on a graph that we call \textit{Visual Roadmap} (VRM), which is embedded on the Visual Configuration Space of the robot. This approach is motivated by studies in human cognition which suggest that human babies around the age of one month vigorously move their limbs and observe them to learn latent representations for voluntary motion.

\section{Visual Roadmap (VRM)}
The VCS is the manifold on which the robot images lie. However, discovering this manifold is an expensive operation and we do not need the concrete visual generalized coordinates of the VCS, for motion planning purposes. We can work with a graph embedded on the VCS, which is constructed using the robot images. 

We assume the availability of a set of images of the robot in various configurations in its workspace. Given a set of robot images, we build a neighbourhood graph $G$ of these images and plan paths using this graph. In this graph, each node corresponds to an image of the robot in some configuration and so the graph will have as many nodes as the number of images sampled. Two nodes will be connected by an edge, if the corresponding images are such that one of them is a $k$-nearest neighbour of the other one, for some value of $k$. The corresponding edge will have a weight equal to the distance between those two images, under some distance metric on robot pose images. We call this graph a \emph{Visual Roadmap} (VRM). 

We begin with a random image sample $X = \{x^{(1)}, x^{(2)}, \ldots, x^{(n)}\}$ of robot configurations $\{q^{(1)}, q^{(2)}, \ldots, q^{(n)}\}$. We construct a graph $G(V, E)$, which has $n$ vertices, each vertex $v^{(i)}$ corresponding to a robot image $x^{(i)}$ and hence to a robot configuration $q^{(i)}$. An edge is added between $v^{(i)}$ and $v^{(j)}$, if and only if $x^{(i)}$ and $x^{(j)}$ are near-by. For example, we can find the $k$-nearest neighbours of each image $x^{(i)}$ in the image space and an edge between $v^{(i)}$ and the nodes corresponding to the $k$-nearest neighbours of $x^{(i)}$. This $k$-nearest neighbours ($k$-NN) graph $G$ is a VRM. We note that the construction of VRM does not require any knowledge of the robot geometry.

\begin{figure}[t]
    \centering
    \includegraphics[page=2, clip, trim = 0cm 0cm 0cm 4cm,  width=\textwidth]{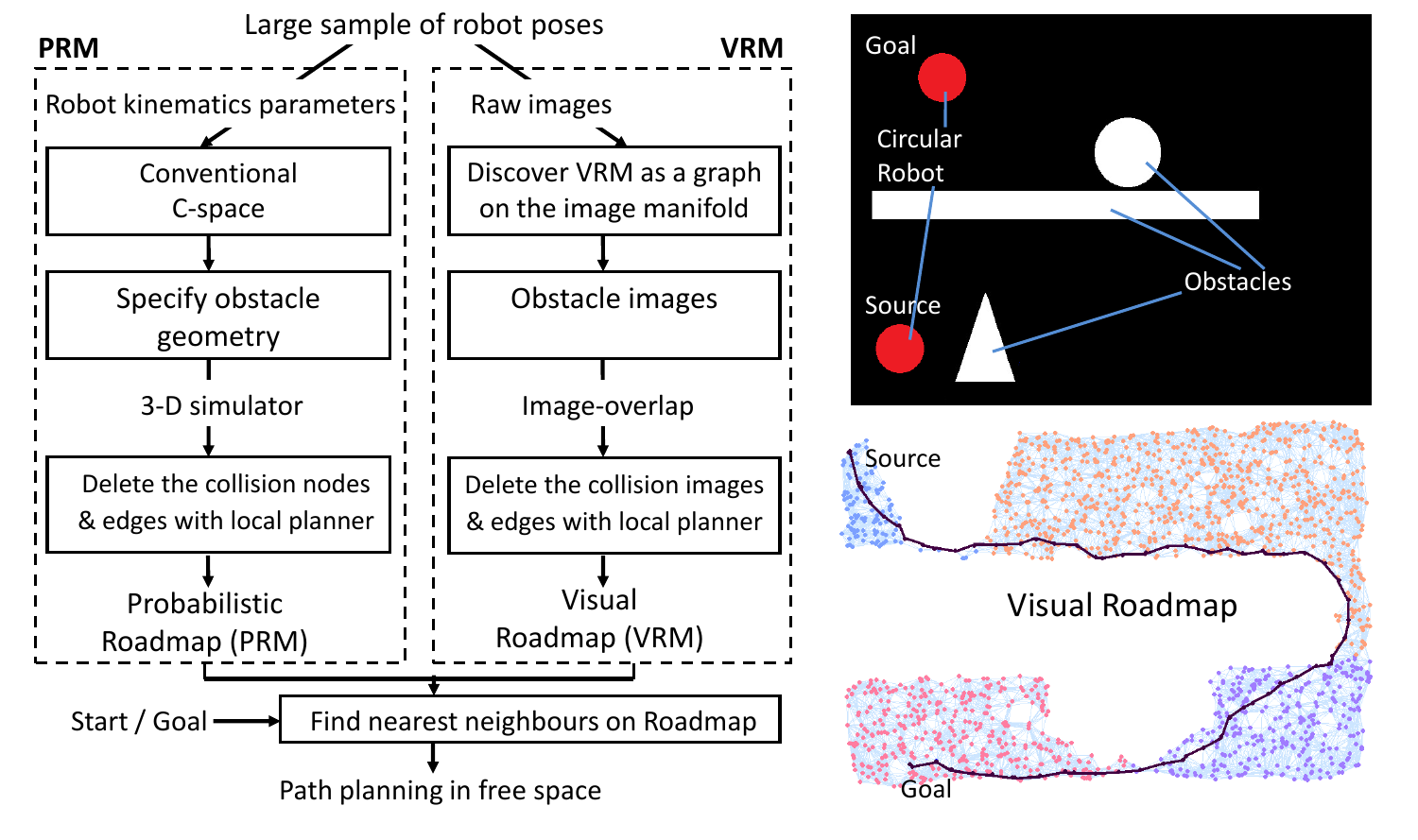}

    \caption{
        \textit{Visual Roadmap as an analogue of Probabilistic Roadmap (PRM)}.
        In the Visual Roadmap (VRM) approach, a graph is constructed from the neighbourhoods in the image space. This requires no knowledge of robot kinematics or geometry.  Just as with PRM, one now deletes nodes overlapping the obstacles, and constructs a path on the remaining edges of the graph. The process is illustrated with a simulated mobile robot : the manifold (bottom right) is constructed solely from a sample of 2000  images. The VRM graph is shown with all obstacle nodes removed and a path identified for a given source and goal. The manifold discovery process preserves topology but may flip or deform the map, as it has done in this case.
    } 
    \label{fig:flowchart}
\end{figure}

When there are no obstacles in the workspace, finding a path for the robot from a source configuration to a destination configuration corresponds to finding a path in the VRM graph.

Figure~\ref{fig:flowchart} shows an overview of the VRM algorithm, demonstrated on a simulated mobile robot.  The idea of VRM is analogous to Probabilistic Roadmap (PRM)~\citep{kavraki-latombe-overmars-96_PRM-high-dimensional}. While PRM samples the conventional C-space, VRM samples the VCS.

\subsection{Static Obstacle Avoidance}
When there are any static obstacles, we find the configurations in which the robot images have a non-empty intersection with the obstacle image and remove from the roadmap the nodes corresponding to these configurations. Paths planned using the remaining graph will be \textit{almost} obstacle-free. The paths are not yet fully safe because the edge between two free configurations may not be free from obstacles. This is the problem of local planning and will be handled separately.

For spatial robots, avoiding an obstacle would require multiple cameras. We collect multiple images of each configuration from different fixed views. A given configuration is free if there is at least one view in which the robot image does not intersect with the obstacle image. If the images from all the views have a non-empty intersection with the obstacle image in the respective views, then that configuration will be considered a collision configuration.

\subsection{Local Planner}
Once we know that all the nodes in the VRM graph correspond to free configurations, it is the responsibility of the local planner to make sure that all the edges connecting the free nodes are free from obstacles. For this purpose, we propose three approaches: (i) Local Tangent Space based planner, (ii) Track-points based local planner, and (iii) Local planner using Shi-Tomasi features. We discuss these approaches in Section~\ref{sec:local_planner}.

\subsection{Dynamic Obstacle Avoidance}
To handle dynamic obstacles, we make certain assumptions about the speed of the obstacle. These assumptions allow us to update only small portions of the graph that correspond to collision configurations, instead of having to recreate the whole graph all over again every time the obstacle moves. Firstly, we assume that the speed of the obstacle is less than the speed of the robot so that there is enough time for the robot to plan alternate paths. Secondly, we assume that in each move, the obstacle moves at most by a certain amount, so that the updates on the graph can be limited to a small number of layers of neighbours around the current obstacle region in the graph. We discuss the details of dynamic obstacle avoidance in Section~\ref{sec:dynamic_obstacle}.

\section{Motion Planning without Obstacles}
When there are no obstacles in the environment, we directly plan a path on the VRM. Given a source configuration, say $s$, and a destination configuration, say $t$, as images, first we identify the nearest configurations to $s$ and $t$ in the initial image-set and add them to the VRM temporarily. Next, we find a shortest path, say $p$, on the VRM from $s$ to $t$ using an algorithm such as Dijkstra's shortest path finding algorithm.

After finding a feasible path $p$, it is the job of a controller to decide how the robot would move from one node to another node on the path $p$. Each node on $p$ is a waypoint for the controller. If the configuration parameters are known for all the configurations in the image-set, the controller could use them to execute the path; otherwise, a visual servo controller may be employed.

Figure~\ref{fig:circular_mobile_robot_ex_wo_obs} illustrates the use of VRM to plan paths for a circular mobile robot in a planar workspace.  For this robot, the C-space, as well as the workspace, are both $\mathbb{R}^2$. This example uses a sample size of 5000 images and a neighbourhood size of 10, to build the VRM.

\begin{figure}
    \centering
    \begin{subfigure}{0.45\textwidth}
        \includegraphics[width=\columnwidth]{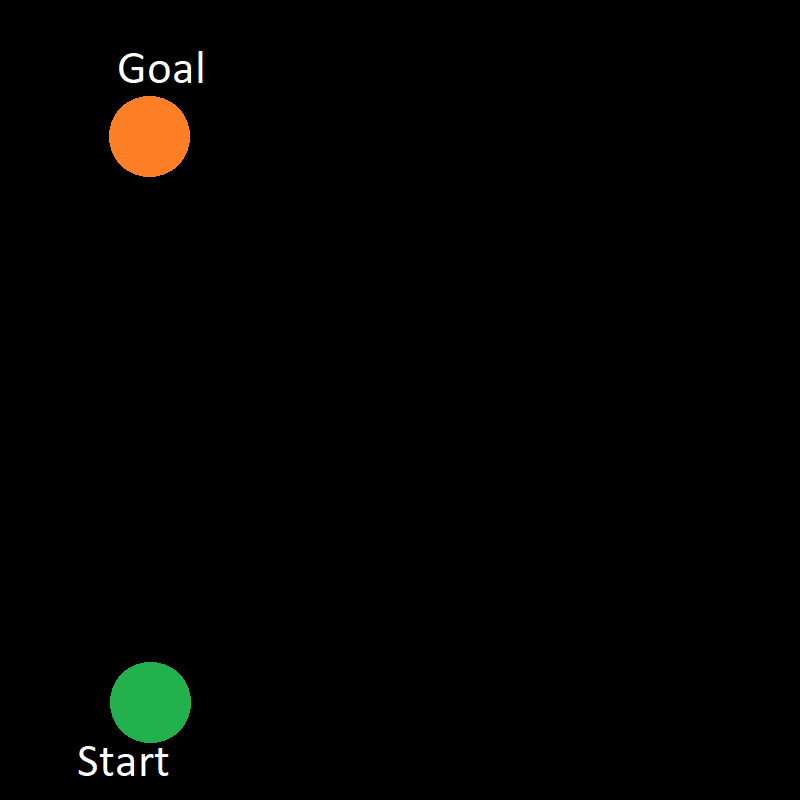}
        \caption{}
    \end{subfigure}
    \begin{subfigure}{0.45\textwidth}
        \includegraphics[width=\columnwidth]{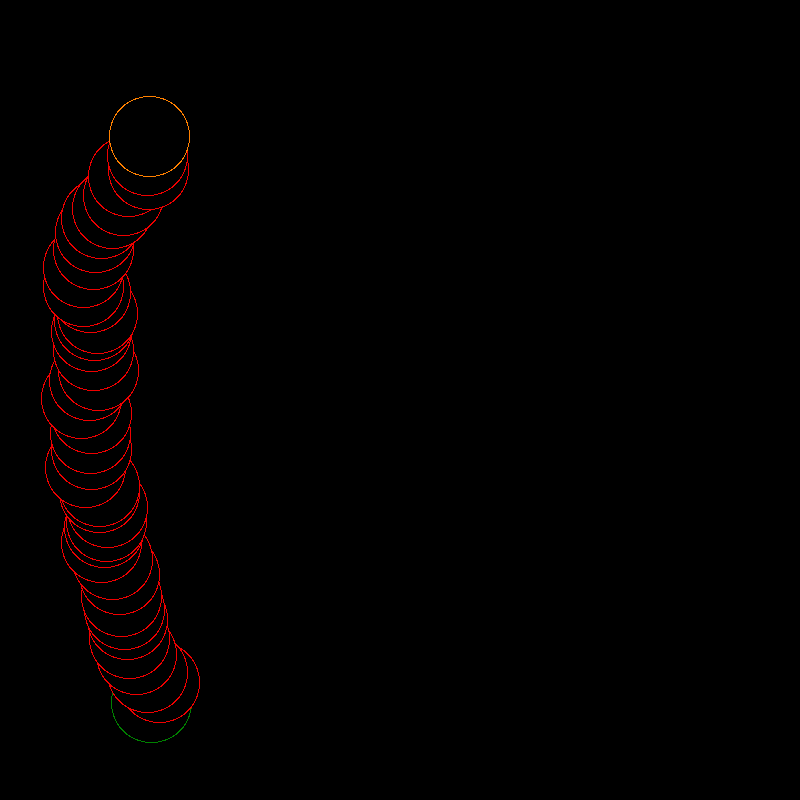}
        \caption{}
    \end{subfigure}
    
    \begin{subfigure}{0.45\textwidth}
        \includegraphics[width=\columnwidth]{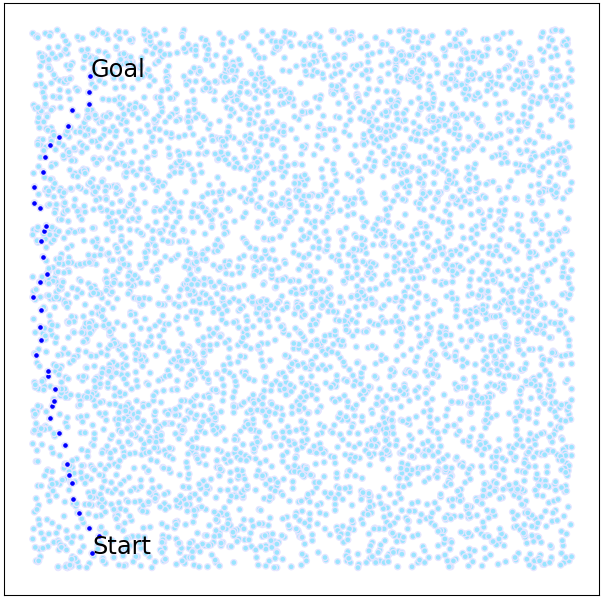}
        \caption{}
    \end{subfigure}
    \begin{subfigure}{0.45\textwidth}
        \includegraphics[width=\columnwidth]{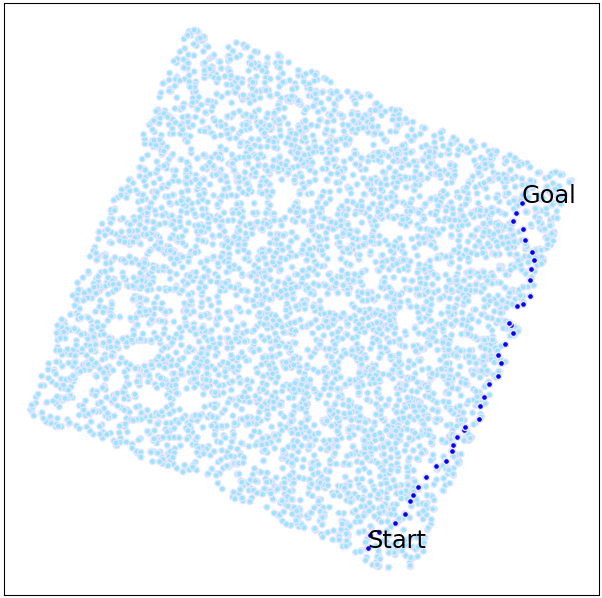}
        \caption{}
    \end{subfigure}
    \caption{
        \textit{Path planning on VRM, for a circular mobile robot}. (a) Start and goal configurations of a circular mobile robot marked in a workspace with no obstacles. A path is computed on the VRM, from the start to the goal configuration. This example uses $N=5000$ images and neighbourhood size $k=10$. (b) The computed path is shown in the workspace, as a superimposition of an outline of the robot images corresponding to the configurations on the path. (c) The same path is shown, as a sequence of blue dots, on the conventional C-space formed by the actual positions of the robot's centre. (d) The path is shown on the VCS computed by Isomap. Notice that Isomap preserves the topology of the C-space and the metric relations among neighbours, but may flip and/or rotate the map, as it has done in this case.
    }
    \label{fig:circular_mobile_robot_ex_wo_obs}
\end{figure}

\begin{figure}
    \centering
    \begin{subfigure}{0.45\textwidth}
        \includegraphics[width=\columnwidth]{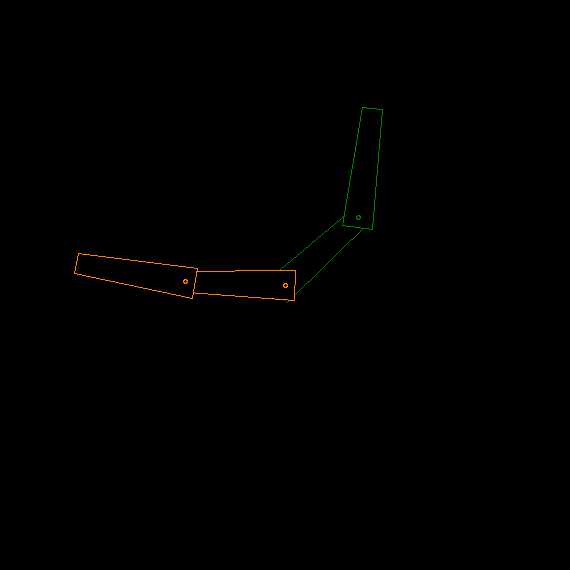}
        \caption{}
    \end{subfigure}
    \begin{subfigure}{0.45\textwidth}
        \includegraphics[width=\columnwidth]{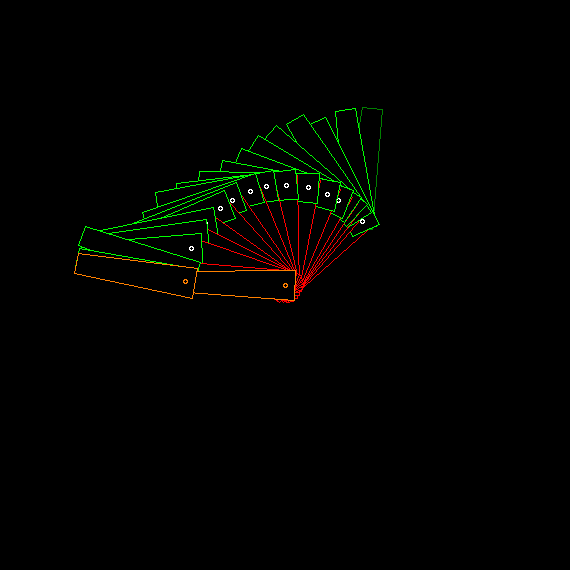}
        \caption{}
    \end{subfigure}
    
    \begin{subfigure}{0.45\textwidth}
        \includegraphics[width=\columnwidth]{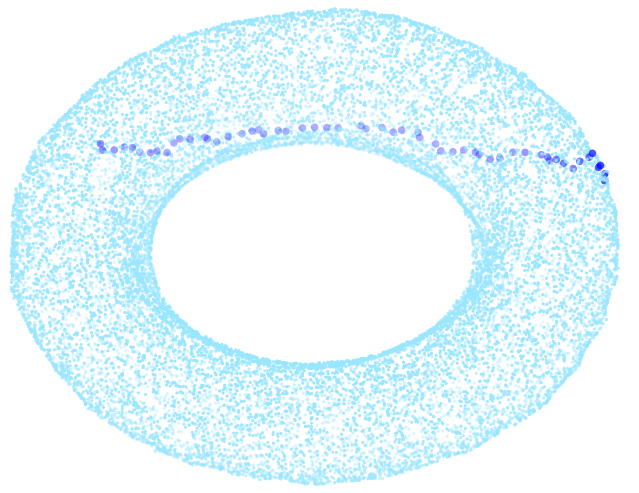}
        \caption{}
    \end{subfigure}
    \begin{subfigure}{0.45\textwidth}
        \includegraphics[width=\columnwidth]{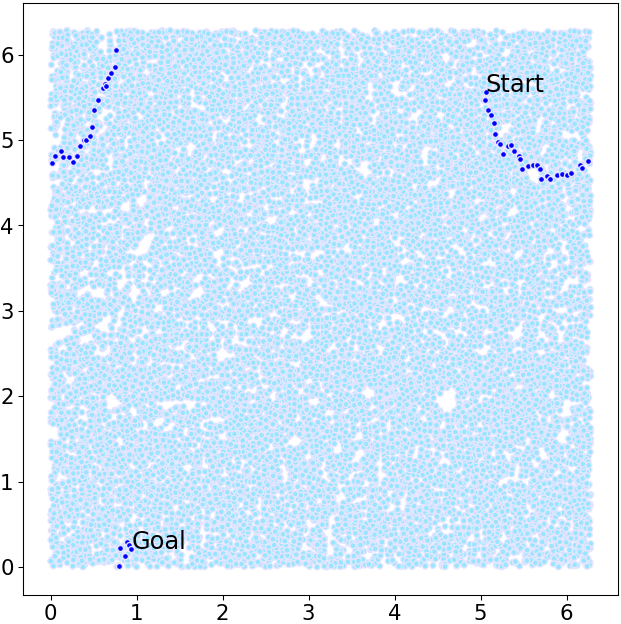}
        \caption{}
    \end{subfigure}
    \caption{
        \textit{Path planning on VRM, for an articulated arm}. (a) Start (green) and goal (orange) configurations of a 2-link arm in its workspace. A path is computed on the VRM, from the start to the goal configuration. This example uses $N=20000$ images and neighbourhood size $k=10$. (b) The computed path is shown in the workspace, as a superimposition of an outline of the robot images corresponding to the configurations on the path. 
        (c) The VCS, discovered by Isomap, has a toroidal shape since the C-space of this arm has an $S^1 \times S^1$ topology.
        (d) We cut the torus in such a way that it will stretch into a rectangle and show the path between the start and goal configurations. Notice that because of the toroidal topology, there is a wrap-around effect in the plot.
    }
    \label{fig:2dof_arm_ex_wo_obs}
\end{figure}

Figure~\ref{fig:2dof_arm_ex_wo_obs} illustrates the use of VRM to plan paths for a 2-DOF articulated arm in a planar workspace. Both of its links can rotate fully around and hence its C-space has an $S^1 \times S^1$ topology. Correspondingly, the VCS discovered by Isomap has a toroidal structure. This example uses a sample size of 20000 images and a neighbourhood size of 10, to build the VRM. 

Computation of VCS using Isomap is done only for the purpose of demonstrations here. The actual path planning is done on the VRM, which is just a neighbourhood graph and hence computing VRM is much simpler than computing VCS.

\section{Static Obstacle Avoidance}
\label{sec:static_obstacles}

For handling obstacles, we take the background subtracted images, and test for intersection with the obstacle image. We find the configurations in which the robot images have a non-empty intersection with the obstacle image and remove from the roadmap the nodes corresponding to these configurations.  A non-empty intersection implies that the configuration is not free and we remove the corresponding node and its incident edges from $G$. If $b \in \mathbb{R}^p$ is the obstacle image vector, then the set of nodes to be removed from $G$ is $V_{collision} = \{v^{(i)}: x^{(i)} * b \ne \bf{0}\}$, where $*$ denotes entry-wise product (Hadamard product) and $\bf{0}$ is the zero-vector. Thus, we obtain a modified graph in which every node represents a free configuration.

However, the edges may still touch some part of the obstacle in an intermediate pose. Guaranteeing edge-safety is the responsibility of the local planner, discussed in Section~\ref{sec:local_planner}. Note that this process applies to any number of static obstacles. 

Figure~\ref{fig:circular_mobile_robot_ex_obs_multiple} illustrates the use of VRM to plan paths for a circular mobile robot in a planar workspace, with multiple obstacles. While the topological properties of the C-space (with obstacles) are preserved in the VCS, there can be a rotation and reflection of the points as can be seen in the figure. Such transformations are not problematic, as the usability of the computed paths is not affected by them.

\begin{figure}
    \centering
    \begin{subfigure}{0.45\textwidth}
        \includegraphics[width=\columnwidth]{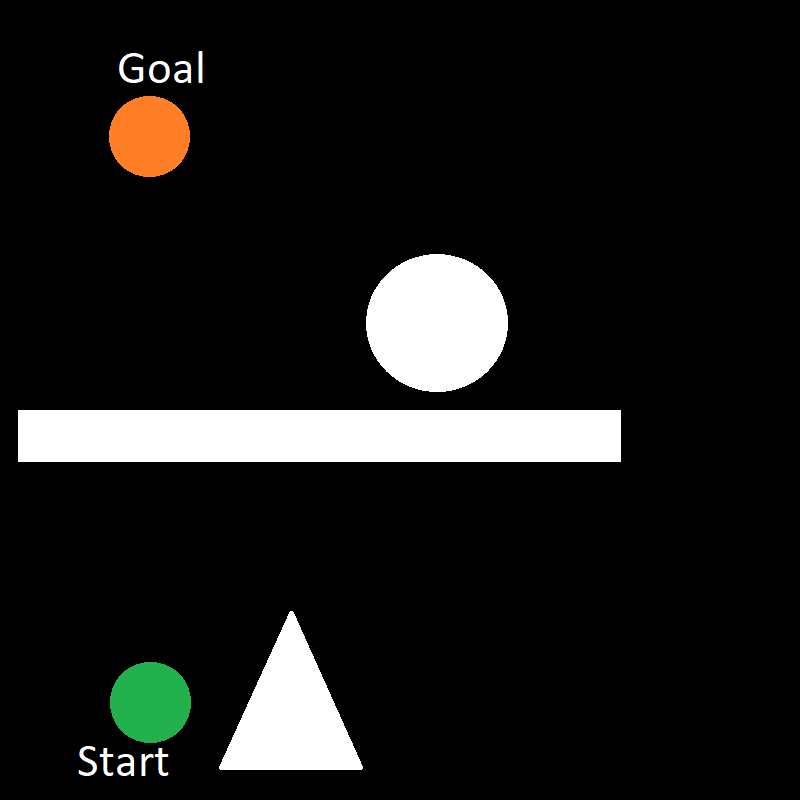}
        \caption{}
    \end{subfigure}
    \begin{subfigure}{0.45\textwidth}
        \includegraphics[width=\columnwidth]{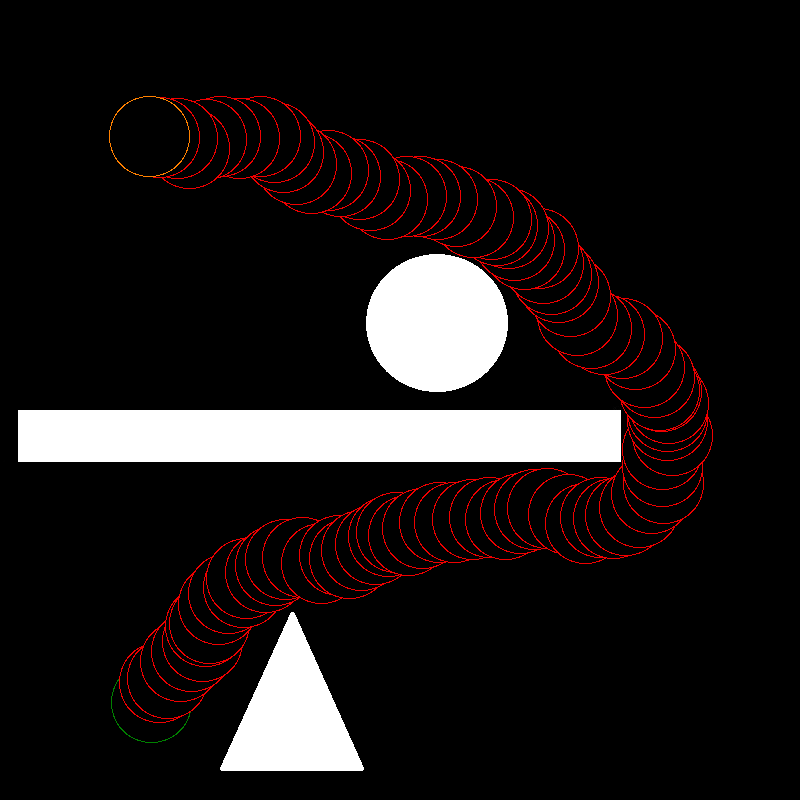}
        \caption{}
    \end{subfigure}
    
    \begin{subfigure}{0.45\textwidth}
        \includegraphics[width=\columnwidth]{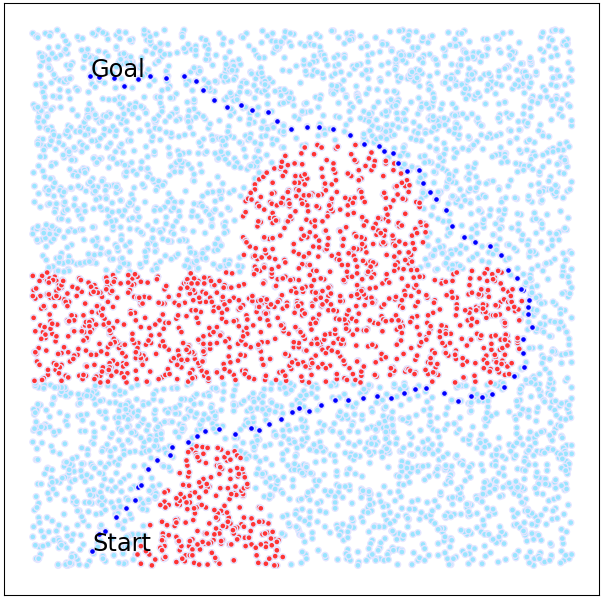}
        \caption{}
    \end{subfigure}
    \begin{subfigure}{0.45\textwidth}
        \includegraphics[width=\columnwidth]{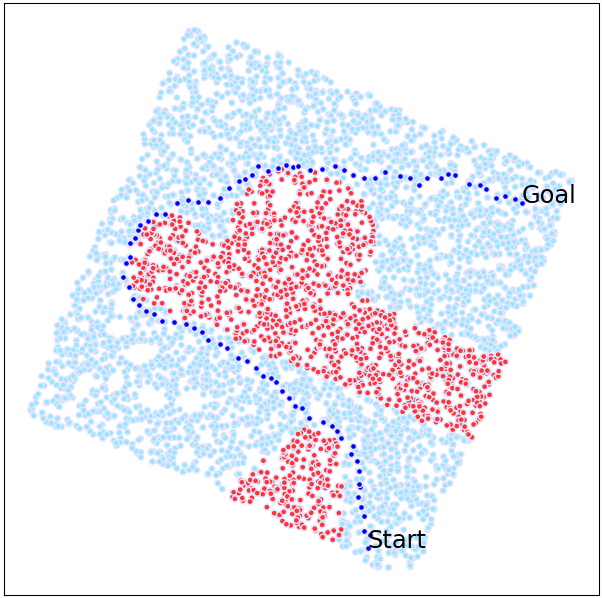}
        \caption{}
    \end{subfigure}
    \caption{
        \textit{Path planning on VRM, for a circular mobile robot, in presence of obstacles}. (a) Start and goal configurations of a circular mobile robot marked in a workspace with multiple obstacles (white objects). An obstacle-free path is computed on the VRM from the start to the goal configuration, after discarding the collision configurations. This example uses $N=5000$ images and neighbourhood size $k=10$. (b) The computed path is shown in the workspace, as a superimposition of an outline of the robot images corresponding to the configurations on the path. (c) The same path is shown, as a sequence of blue dots, on the conventional C-space formed by the actual positions of the robot's centre. Here, the obstacle configurations are shown as red points. (d) The path is shown on the VCS computed by Isomap. Notice that Isomap preserves the topology of the C-space and the metric relations among neighbours, but may flip and/or rotate the map, as it has done in this case.
    }
    \label{fig:circular_mobile_robot_ex_obs_multiple}
\end{figure}

\begin{figure}
    \centering
    \begin{subfigure}{0.45\textwidth}
        \includegraphics[width=\columnwidth]{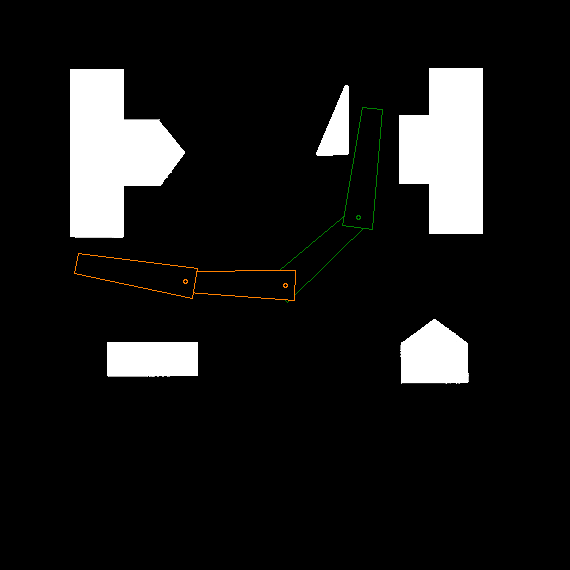}
        \caption{}
    \end{subfigure}
    \begin{subfigure}{0.45\textwidth}
        \includegraphics[width=\columnwidth]{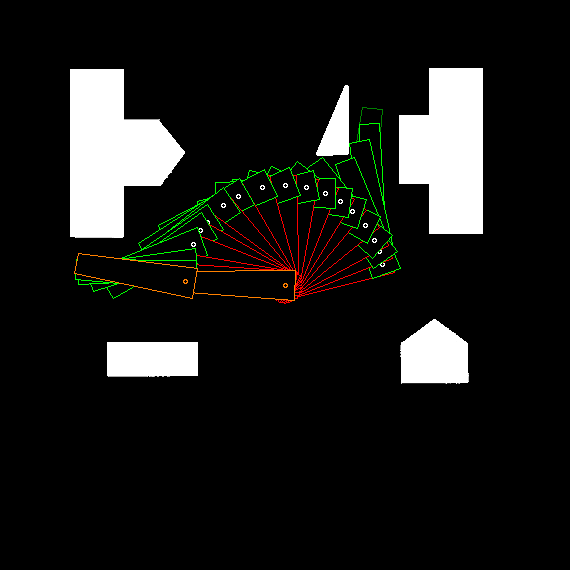}
        \caption{}
    \end{subfigure}
    
    \begin{subfigure}{0.45\textwidth}
        \includegraphics[width=\columnwidth]{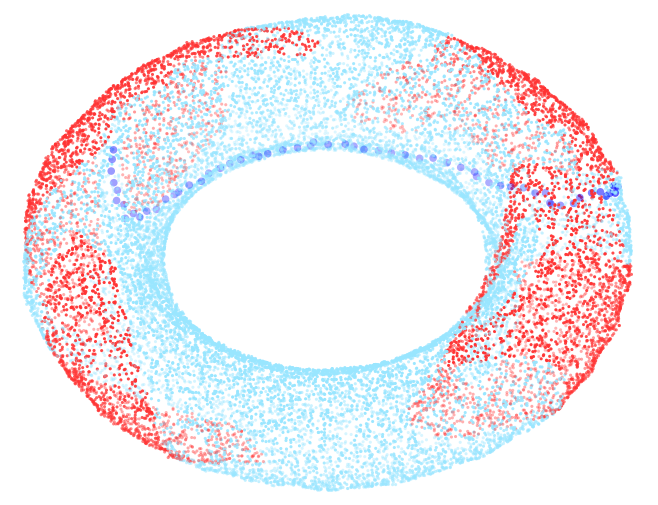}
        \caption{}
    \end{subfigure}
    \begin{subfigure}{0.45\textwidth}
        \includegraphics[width=\columnwidth]{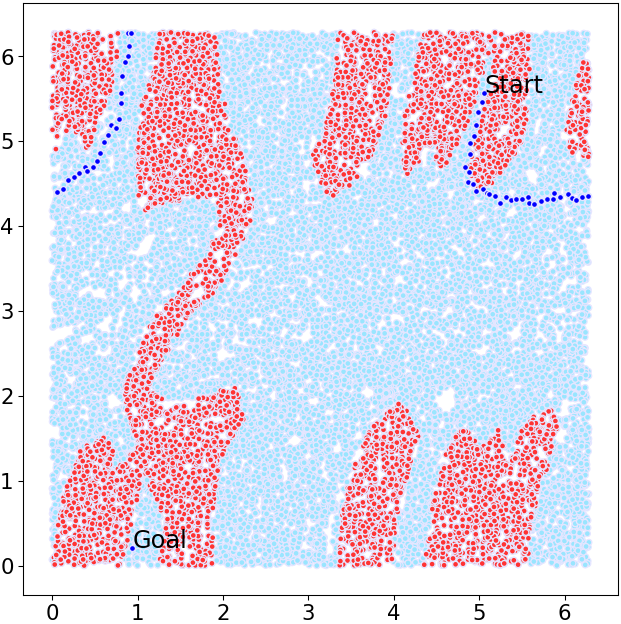}
        \caption{}
    \end{subfigure}
    \caption{
        \textit{Path planning on VRM, for an articulated arm, in presence of obstacles}. (a)  Start (green) and goal (orange) configurations of a 2-link arm in its workspace with multiple obstacles (white objects). An obstacle-free path is computed on the VRM from the start to the goal configuration, after discarding the collision configurations. This example uses $N=20000$ images and neighbourhood size $k=10$. (b) The computed path is shown in the workspace, as a superimposition of an outline of the robot images corresponding to the configurations on the path. (c) The same path is shown, as a sequence of blue dots, on the VCS discovered by Isomap. Here, the obstacle configurations are shown as red points. In this 3-D embedding of the VCS, it is not clear that the path shown is obstacle-free. (d) We cut the torus in such a way that it will stretch into a rectangle and plot the points of the free space and obstacle space, and show the obstacle-free path between the start and goal configurations. Notice that because of the toroidal topology, there is a wrap-around effect in the plot.
    }
    \label{fig:2dof_arm_ex_obs_narrow}
\end{figure}

Figure~\ref{fig:2dof_arm_ex_obs_narrow} illustrates the use of VRM to plan paths for a 2-DOF articulated arm in a planar workspace with many obstacles. Both of its links can rotate fully around and hence its C-space has an $S^1 \times S^1$ topology. Correspondingly, the VCS discovered by Isomap has a toroidal structure. This example uses a sample size of 20000 images and a neighbourhood size of 10, to build the VRM. 

For the purposes of collision detection, we use binary images without affecting the correctness of the process and make it faster. See Section~\ref{sec:fast_collision_detection} for details.

\section{Local Planner in VRM}
\label{sec:local_planner}
We say that an edge $(u, v) \in E(G)$ is \emph{safe}, if every point on the geodesic from $u$ to $v$ in the configuration space is in the free space. Assuming that every node of $G$ is in the free space, we need to guarantee that every edge is also safe. We describe three local planners that work with robot images and can be used on visual roadmaps. To make sure that an edge is safe, these methods construct a new image that estimates the swept volume of the robot in the workspace and check this image for collision. 

To illustrate these local planners, we will use a 3-DOF arm and the obstacles shown in Figure~\ref{fig:3dof_arm_obs_lp}.

\begin{figure}[h]
    \centering
    \begin{subfigure}{0.32\textwidth}
        \includegraphics[width=\columnwidth]{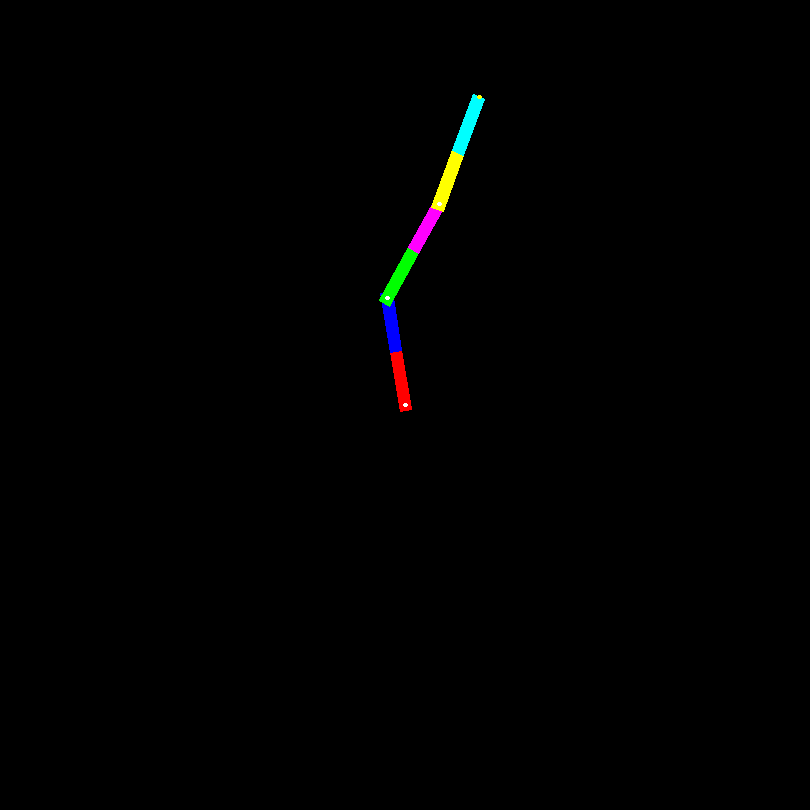}
        \caption{}
    \end{subfigure}
    \begin{subfigure}{0.32\textwidth}
        \includegraphics[width=\columnwidth]{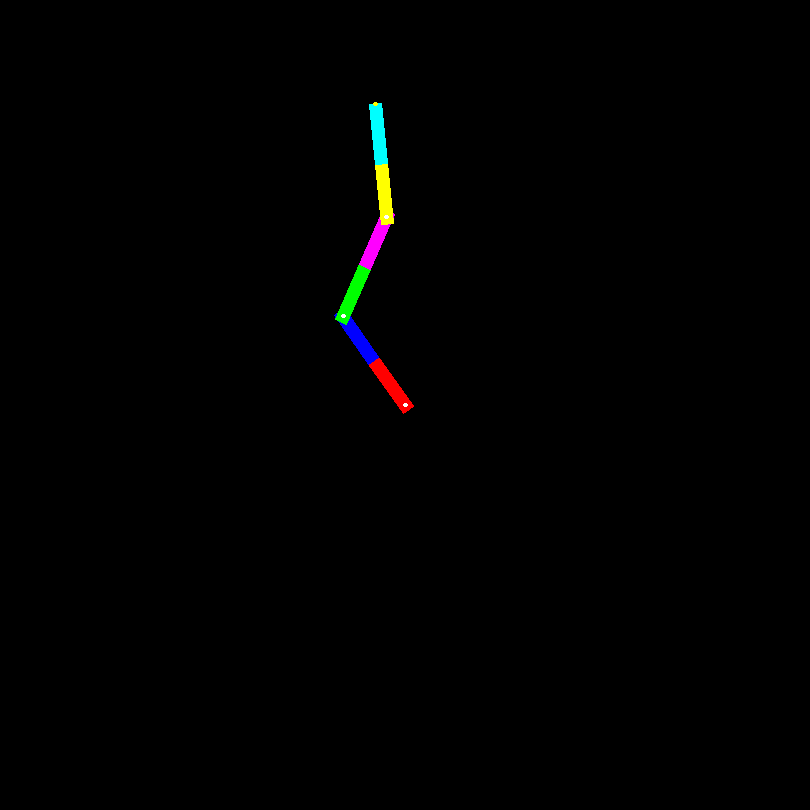}
        \caption{}
    \end{subfigure}
    \begin{subfigure}{0.32\textwidth}
        \includegraphics[width=\columnwidth]{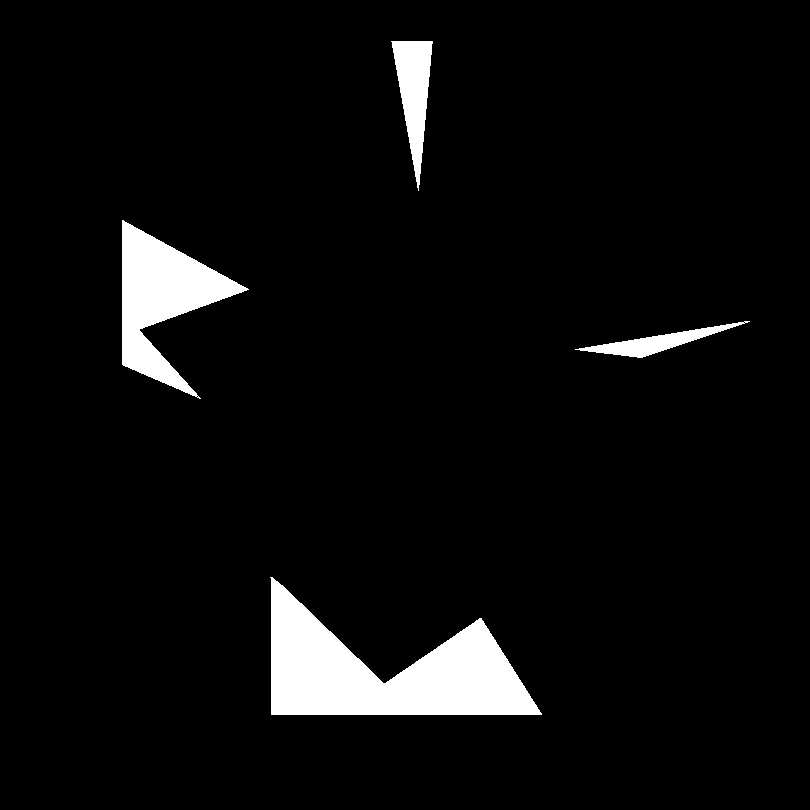}
        \caption{}
    \end{subfigure}
    \caption{
        Two poses of a 3-DOF arm and the obstacles used to illustrate the proposed local planners on VRM.
    }
    \label{fig:3dof_arm_obs_lp}
\end{figure}

\subsection{Interpolation on the Local Tangent Space (LTS)}
\label{pca_lp}
 
For each edge $(u, v) \in E$, let $X^{(u, v)} = \{x_q: q \in \mathcal{N}(u) \cap \mathcal{N}(v)\}$ be the ${p \times m}$ matrix of images corresponding to the intersection of neighbours of $u$ and neighbours of $v$ (including $u$ and $v$), where $m$ is the cardinality of $X^{(u, v)}$. 
To see if $(u, v)$ is \emph{safe}, we interpolate the intermediate images on the tangent space spanned by $X^{(u, v)}$, obtained using PCA. The target dimension is the number of degrees of freedom $d$, and PCA
maps $X^{(u, v)}$ to a $Y^{(u, v)}$ ($d \times m$). In addition to $Y^{(u, v)}$, PCA also gives a ${p \times d}$ orthonormal matrix $W^{(u, v)}$ such that $X^{(u, v)} = W^{(u, v)}Y^{(u, v)}$ or $Y^{(u, v)} = W^{(u, v)^T}X^{(u,  v)}$. We then interpolate between $y_u$ and $y_v$ to construct $y^{(\alpha)} =  \alpha*y_u + (1-\alpha)*y_v$ for various values of $\alpha \in (0, 1)$. For each $\alpha$, the image $x^{(\alpha)} = Wy^{(\alpha)}$ must be in free space.

If this image has a non-empty overlap with the obstacle image, then the edge under consideration is marked \textit{unsafe}. In practice, the resulting image is a poor interpolation, and rejects many valid edges; however, the probability of an edge being unsafe after being passed by the local planner is low (i.e. it is conservative).

\begin{figure}[h]
    \centering
    \begin{subfigure}{0.35\textwidth}
        \includegraphics[width=\columnwidth]{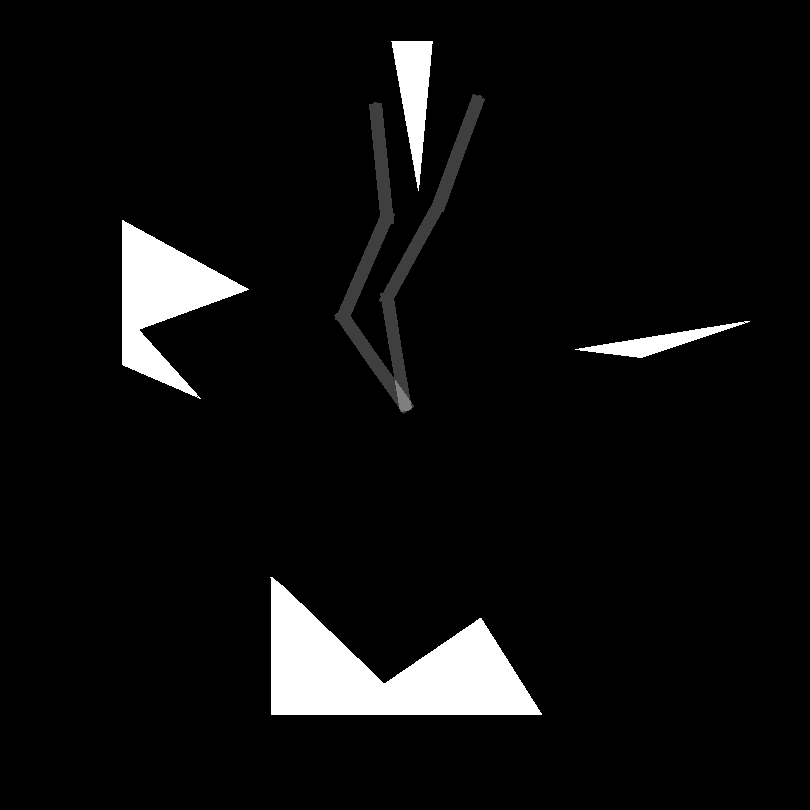}
        \caption{}
        \label{fig:pca_lp_interpolation_intersection}
    \end{subfigure}
    \begin{subfigure}{0.35\textwidth}
        \includegraphics[width=\textwidth]{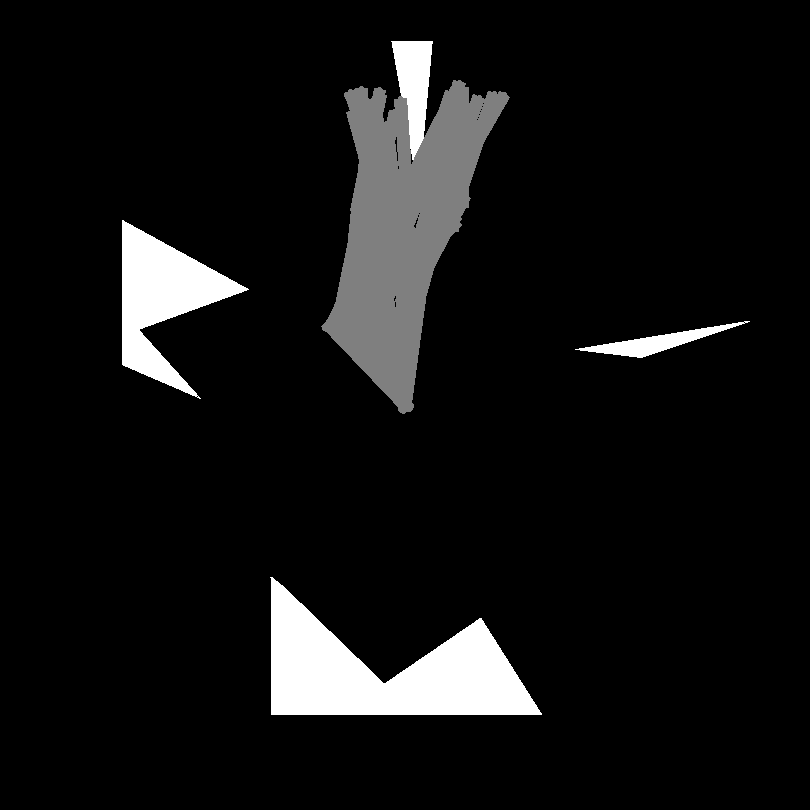}
        \caption{}
        \label{fig:pca_lp_interpolation_union}
    \end{subfigure}
    \caption{
        \textit{PCA based interpolation}: (a) Interpolation on the LTS using the intersection of neighbourhoods of the two terminal nodes of the edge under consideration. (b) Interpolation using union of neighbourhoods.
    }
    \label{fig:pca_lp_interpolation}
\end{figure}

The image obtained by a linear interpolation on the local tangent space (LTS) is a weighted sum of the images in $X^{(u, v)}$. Thus, for collision detection purposes, it is sufficient to look at the superimposition of images in $X^{(u, v)}$. This achieves the same effect as the PCA based method described above and avoids the PCA computation.

Figure~\ref{fig:pca_lp_interpolation} illustrates this interpolation. In the example shown in Figure~\ref{fig:pca_lp_interpolation_intersection}, we see just a superimposition of two configuration images, because the intersection of the neighbourhoods of the nodes corresponding to these configurations is empty. In such cases, we may consider a union of neighbourhoods instead of an intersection. The union-based interpolation would result in an image shown in Figure~\ref{fig:pca_lp_interpolation_union}. However, the union based interpolation would discard many more valid edges from the VRM, and would be more conservative. 

\subsection{Ideal Tracked Points (ITP)}
\label{sec:itp}
Here we assume that a set of points on the robot body can be tracked in all poses (including occlusions). Then to see if an edge $(u, v)$ is safe, we join each pair of corresponding tracked-points to create
a new image, as in Figure~\ref{fig:itp_lp_interpolation}. This image is used for collision detection. If this image has a non-empty overlap with the obstacle image, then the edge under consideration is marked \textit{unsafe}.

\begin{figure}
    \centering
    \includegraphics[page=3, clip, trim=3cm 1.2cm 3cm 5cm, width=\columnwidth]{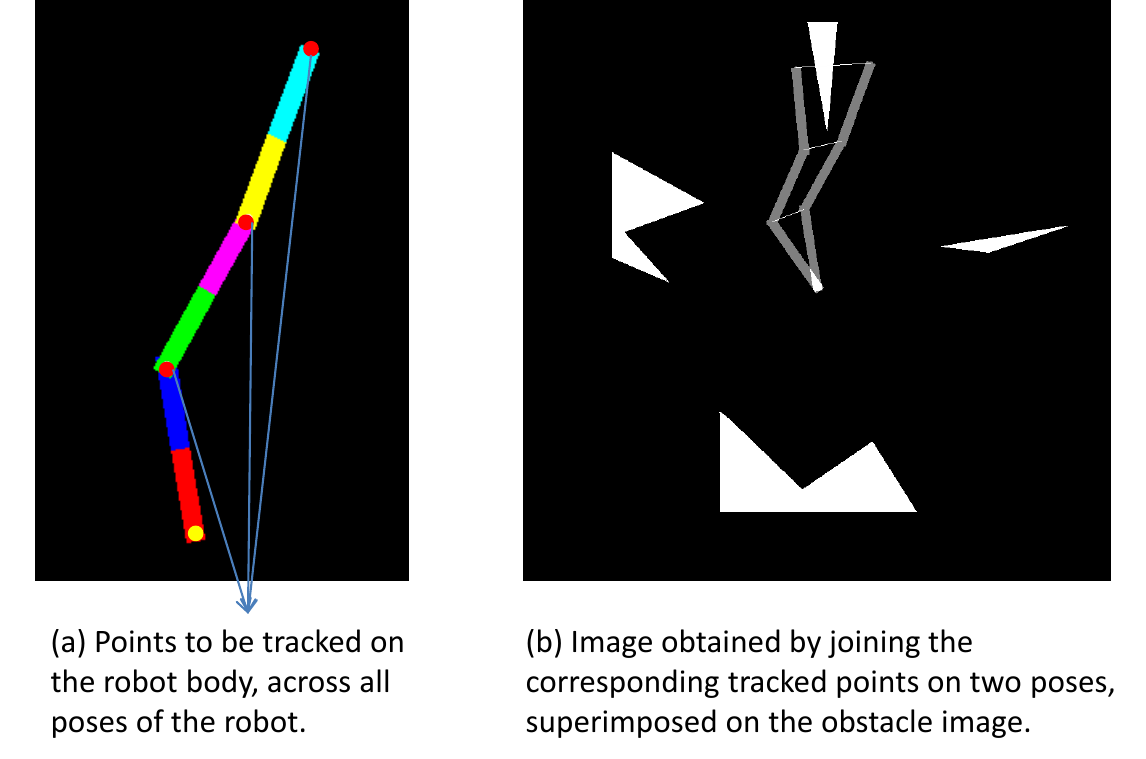}
    \caption{
        \textit{Local planner using Ideal Tracked Points}. We assume that some fixed points on the robot body are tracked across all the poses. We use the image obtained by joining the corresponding tracked points, for collision detection. 
    }
    \label{fig:itp_lp_interpolation}
\end{figure}

\begin{figure}
    \centering
    \includegraphics[width=0.5\columnwidth]{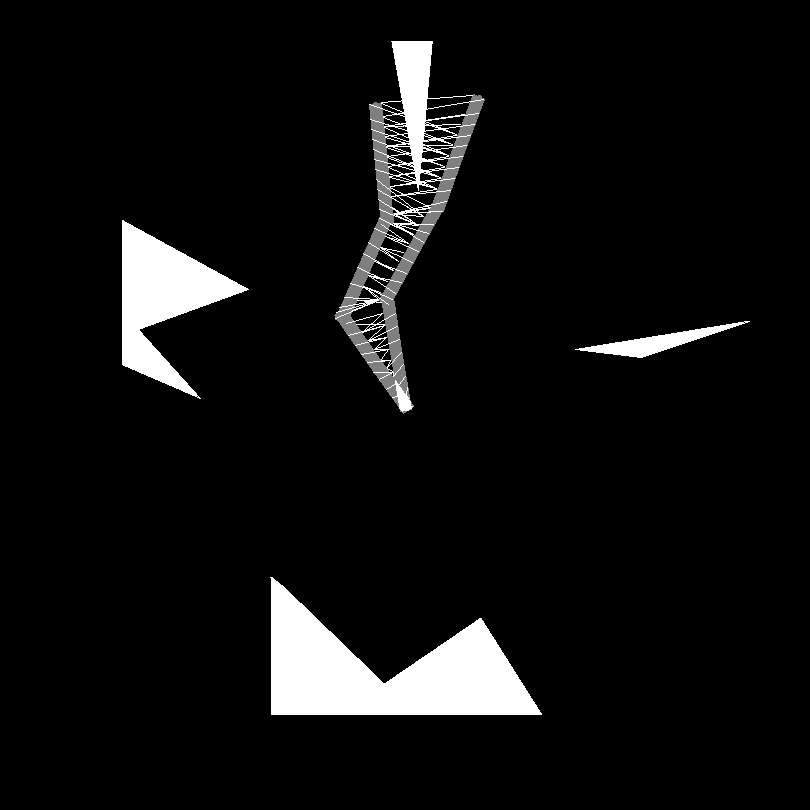}
    \caption{Local planner using Joins of Nearest Shi-Tomasi features link-wise (JNST).}
    \label{fig:stfl_lp_interpolation}
\end{figure}

\subsection{Joins of Nearest Shi-Tomasi features (JNST)}
In practice, occlusion precludes the tracking of any set of points on the robot body.  So, ITP method is not always practical. Here, we propose an approximation to the ITP method, based on high-contrast points known as the Shi-Tomasi features~\citep{shi-tomasi_1994_cvpr_good-features-to-track}. These are high contrast salient regions such as corners, ends of lines, intersection points etc.,  in an image, and are good for tracking. We assume that each link of the robot can be separated and that the Shi-Tomasi features are computed on each link. 

As before, to ensure safety for an edge $(u, v)$, we create the superimposed image of $u$ and $v$. Here, we do not know the correspondences between points in the two images. The Joins of Nearest Shi-Tomasi features approach (JNST) involves associating each feature point on each link in $u$ with the nearest feature point in the corresponding link in $v$.  We do the same in both directions and add a line between each pair of such nearest points as before and use the resulting image for collision detection (Figure~\ref{fig:stfl_lp_interpolation}). If the image thus obtained has a non-empty overlap with the obstacle image, then the edge under consideration is marked \textit{unsafe}. We will refer to this local planner as JNST in the subsequent sections. 

\subsection{Start and Goal States}

For motion planning on the VRM, we need to map  the source ($s$) and target ($t$) images onto the
VRM $G$. We first ensure that the poses $s,t$ themselves are in free space.  We then add these to $G$ and connect them with their $k$-nearest neighbours in $X$. We then run a local planner on the new edges and find the shortest path between $s$ and $t$ as before. Adding a new node (image) to the graph is a computation that requires $O(nk)$ distance computation steps for finding. Time for distance calculation depends on the metric used. This approach again is almost identical to traditional roadmap methods~\citep{choset-05_robot-motion-theory}, except that the tests are all visual. 

\section{Empirical Analysis: Metrics and local planners}
\label{sec:empirical_results}
Factors affecting the quality of paths in VRM include sampling density, the metric used, and the local planner. We now present an empirical study of these aspects on a planar 3-link simulated arm and a set of obstacles shown in Figure~\ref{fig:3dof_arm_and_obstacle_lp}. 
\begin{figure}[t]
\centering
    \begin{subfigure}[t]{0.45\columnwidth}
        \includegraphics[width=\columnwidth]{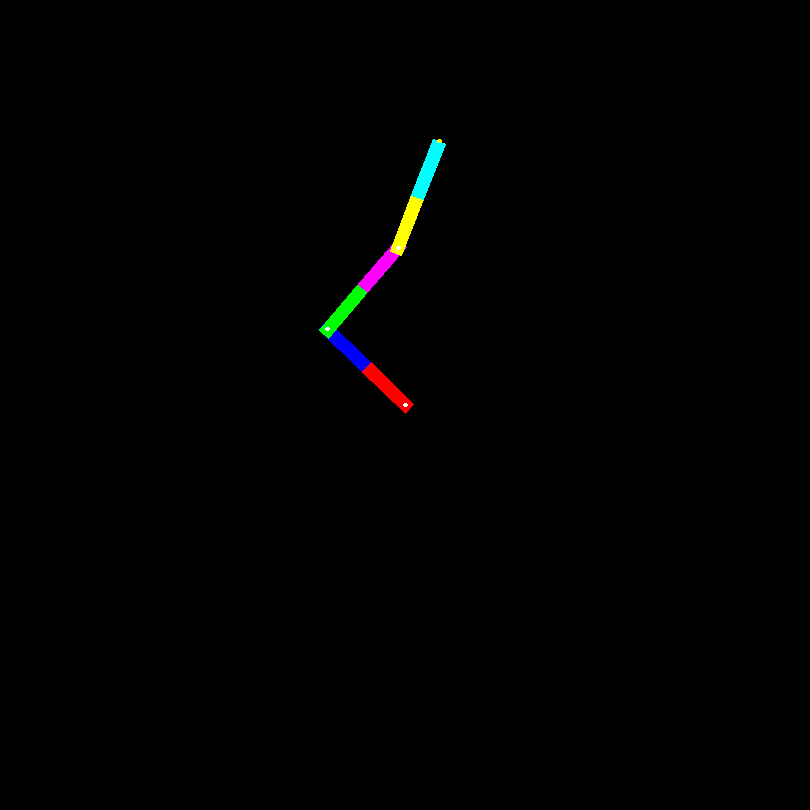}%
        \caption{}
    \end{subfigure}
    \begin{subfigure}[t]{0.45\columnwidth}%
        \includegraphics[width=\columnwidth]{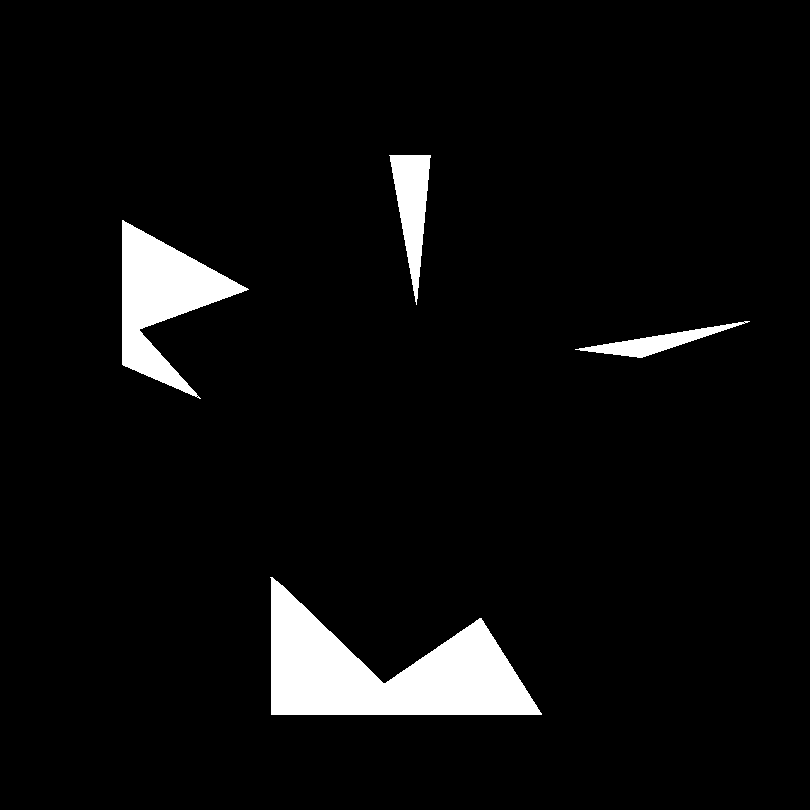}%
        \caption{}
    \end{subfigure}
    \caption{A 3-DOF simulated arm and a set of obstacles used for empirical analysis.}
\label{fig:3dof_arm_and_obstacle_lp}
\end{figure}

\subsection{Gold Standard Local Planner}
In the traditional configuration space, two configurations are assumed to be joined by a linear join between them. To see if an edge $(u, v)$ is actually safe, we generate intermediate pose images by interpolating joint angle vectors at an $\epsilon$ resolution. We observe that a linear interpolation in joint angle space need not be the same as an interpolation on visual C-space, but we assume the difference would be fairly small for a reasonable sampling density.  If all these images are collision-free, we treat $(u, v)$ to be safe. The performance of the local planners is evaluated relative to this gold standard local planner. Results reported here use $\epsilon = 1^\circ$.



\begin{figure}[t]
\centering
    \begin{subfigure}[t]{0.48\columnwidth}%
        \includegraphics[width=\columnwidth]{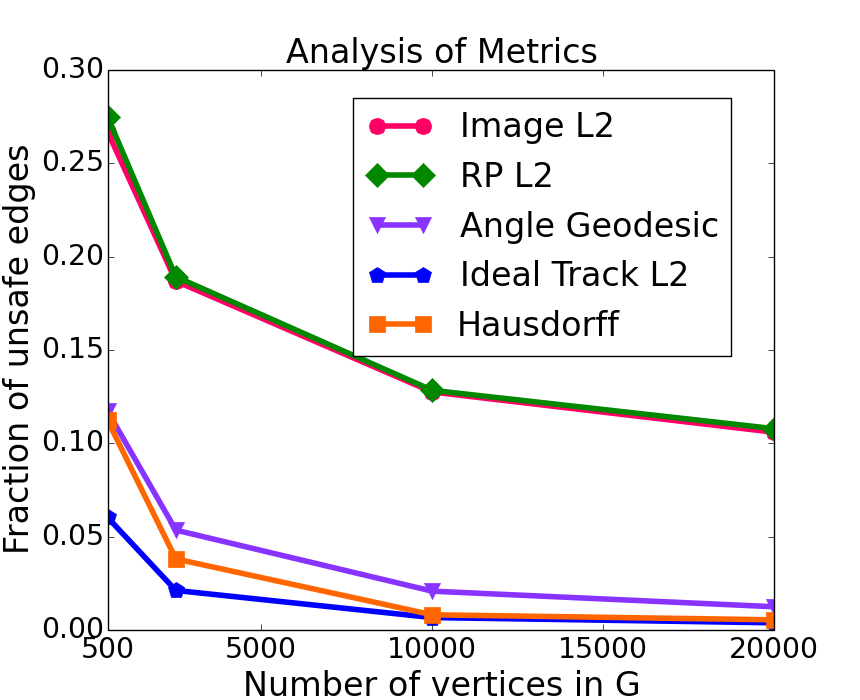}%
        \caption{Comparison of Metrics}
        \label{fig:comparison_metrics}
    \end{subfigure}
    \begin{subfigure}[t]{0.48\columnwidth}%
        \includegraphics[width=\columnwidth]{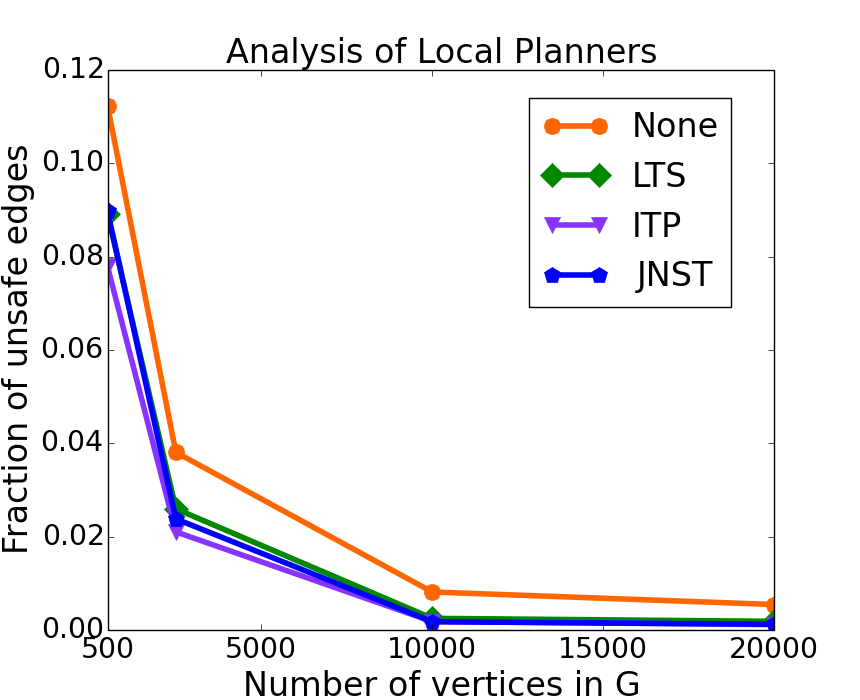}%
        \caption{Comparison of local planners}
        \label{fig:comparison_lp}
    \end{subfigure} 
    \caption{
        \textit{Empirical analysis of metrics and local planners}.
        (a) Edge failures under various metrics, without a local planner: for each case, we consider a representation of the robot and an appropriate metric to measure the distance between two poses during the computation of the neighbourhoods in the VRM.  
        (b) Local planner performance plots based on Hausdorff metric on the link-wise Shi-Tomasi feature-based representation of robot images. The JNST local planner performs almost as well as the ITP local planner, which is not implementable in practice. This suggests that JNST local planner can be used as a good approximation for ITP local planner.
    }
    \label{fig:lp_plots}
\end{figure}


\subsection{Effect of Sampling Density and Distance Metric}
\label{subsec:metrics}

Plots in Figure~\ref{fig:comparison_metrics} clearly suggest that the sampling density (i.e., the number of images used to construct the visual roadmap) heavily affects the fraction of unsafe edges and hence the quality of paths. The denser the sample is, the better the paths.

We present the effect of several representations of the configuration space along with appropriate distance metric for each case. Table~\ref{tab:metrics} lists the different representations and the corresponding distance metric used to compute neighbourhoods.

\begin{table}[h]
    \centering
    \caption{
        Different representations of the configuration space and the distance metric used with each    
        representation. Short forms mentioned here have been used in Table~\ref{tab:metrics_n_lp}.
    }
    \begin{tabular}{| l | l | l | l |}
    \hline
    Representation & Distance Metric & Short Form \\
    \hline
    1. Raw RGB images of the robot & $L_2$  &  Img $L_2$ \\
    2. Random projections of images & $L_2$ & RP $L_2$ \\
    3. Joint angle vector of the robot & Geodesic & $\theta$-G \\
    4. Ideal tracked points & $L_2$ & ITP $L_2$ \\
    5. Shi-Tomasi features link-wise & Hausdorff & ST-H\\
    \hline
    \end{tabular}
    \label{tab:metrics}
\end{table}

\begin{table}[h]
    \centering
    \caption{
        Percentage of bad edges remaining after pruning the VRM using each local planner on a graph with 20000 nodes with different metrics. See Table~\ref{tab:metrics} for an explanation of these metric spaces.
    }
    \begin{tabular}{| c | l | l | l | l | l |}
    \hline
    \multirow{2}{*}{Local Planner} & \multicolumn{5}{c|}{Metric Space} \\
    \cline{2-6}
    & Img $L_2$ & RP $L_2$ & $\theta$-G & ITP $L_2$ & ST-H \\
    \hline
    None & 10.59 & 10.79 &  1.25 &  0.39 & 0.55 \\
    LTS & 9.18 &  9.34 &  0.43 &  0.09 & 0.19 \\
    ITP & 7.97 &  8.11 &  0.17 &  0.11 & 0.12 \\
    JNST & 9.58 &  9.74 &  0.16 &  0.12 & 0.12 \\
    \hline
    \end{tabular}
    \label{tab:metrics_n_lp}
\end{table}

To find the distance between two images we just flatten all the channels of each image into a single vector and use the standard Euclidean  ($L_2$) distance on the resulting vectors. In our experiments we used 30,000 (100x100x3) dimensional vectors for image distance.

Random Projection (RP)~\citep{bingham2001random,dasgupta2000experiments} is a dimensionality reduction method that preserves $L_2$ distances. In our experiments, we projected the 30,000-dimensional image vectors onto 2000 Gaussian random unit vectors to obtain a 2000-dimensional representation of each image. The experiments show that the $L_2$ distance on RP vectors does almost as well as that on the image vectors. Since the distance computation is done on much smaller vectors, the graph construction gets much faster while preserving the neighbourhoods, when using RP.

The distance between two joint angle vectors is computed as the sum of the shortest circular-distances (i.e., treating 0 and $2\pi$ to be the same angle) between individual components. This is in some sense the geodesic distance between the two vectors.

The ideal tracked point (ITP) $L_2$ distance between two configurations is computed as the $L_2$ distance between the vectors obtained by concatenating all the tracked point coordinates of each configuration.

Finally, the Hausdorff distance between two configurations is computed as the sum of Hausdorff distances between the sets of Shi-Tomasi feature points on the corresponding links for the two configurations. Given two sets of points $A$ and $B$, Hausdorff distance $d_H(A, B)$ is defined as 
\[
d_H(A, B) = \max \left\{ \adjustlimits \sup_{a \in A} \inf_{b \in B} d_E(a, b), \adjustlimits \sup_{b \in B} \inf_{a \in A} d_E(a, b) \right\},
\]
where $d_E(a, b)$ is the Euclidean distance between $a$ and $b$.

As can be seen from Figure~\ref{fig:comparison_lp} and Table~\ref{tab:metrics_n_lp}, JNST local planner performs almost as well as ITP local planner.

\section{Dynamic Obstacle Avoidance}
\label{sec:dynamic_obstacle}
In this section, we consider the problem of planning paths when there are moving obstacles. To plan a path from a given source to a destination configuration, the planner initially treats the current state of the obstacle as a static obstacle and finds the shortest path to the destination, after removing the collision configurations from the graph. Then, before making the next step on the already-computed path, it checks for a change in the obstacle's position. If there is a change in the obstacle's position, then it updates the graph as required and finds a new path from the current configuration to the destination. When there are multiple moving obstacles, the planner updates the graph every time a change is detected in any of these moving obstacles. 

We assume that the speed of a moving obstacle is within some range so that during the update step, we do not have to update the entire graph, but only a small subset of it. In particular, we maintain a set of boundary nodes among the collision nodes, which have the property that the corresponding robot configurations just touch the boundary of the obstacle in consideration. We say that the robot \textit{touches} the boundary of an obstacle, if the amount of overlap between the robot image and the obstacle image is below a certain threshold. When an obstacle moves, we only look at the nodes which are within a certain number of hops away from any of the boundary nodes. The number of hops to consider depends on the speed of the obstacle and needs to be determined experimentally. Some of these near-by nodes will become collision nodes for the new position of the obstacle and some of the old collision nodes now become free nodes. We mark the nodes accordingly in the graph and proceed to plan a path from the current state to the goal state. We repeat this process until the robot reaches the goal state.

It is possible that a moving obstacle reaches a position such that the goal configuration becomes a collision configuration. During this time, we can have the robot wait in its current configuration until that obstacle moves away and the goal configuration becomes free.

\begin{figure}[t]
    \centering
    \includegraphics[page=4, clip, trim = 2cm 1cm 2cm 4cm,  width=0.7\textwidth]{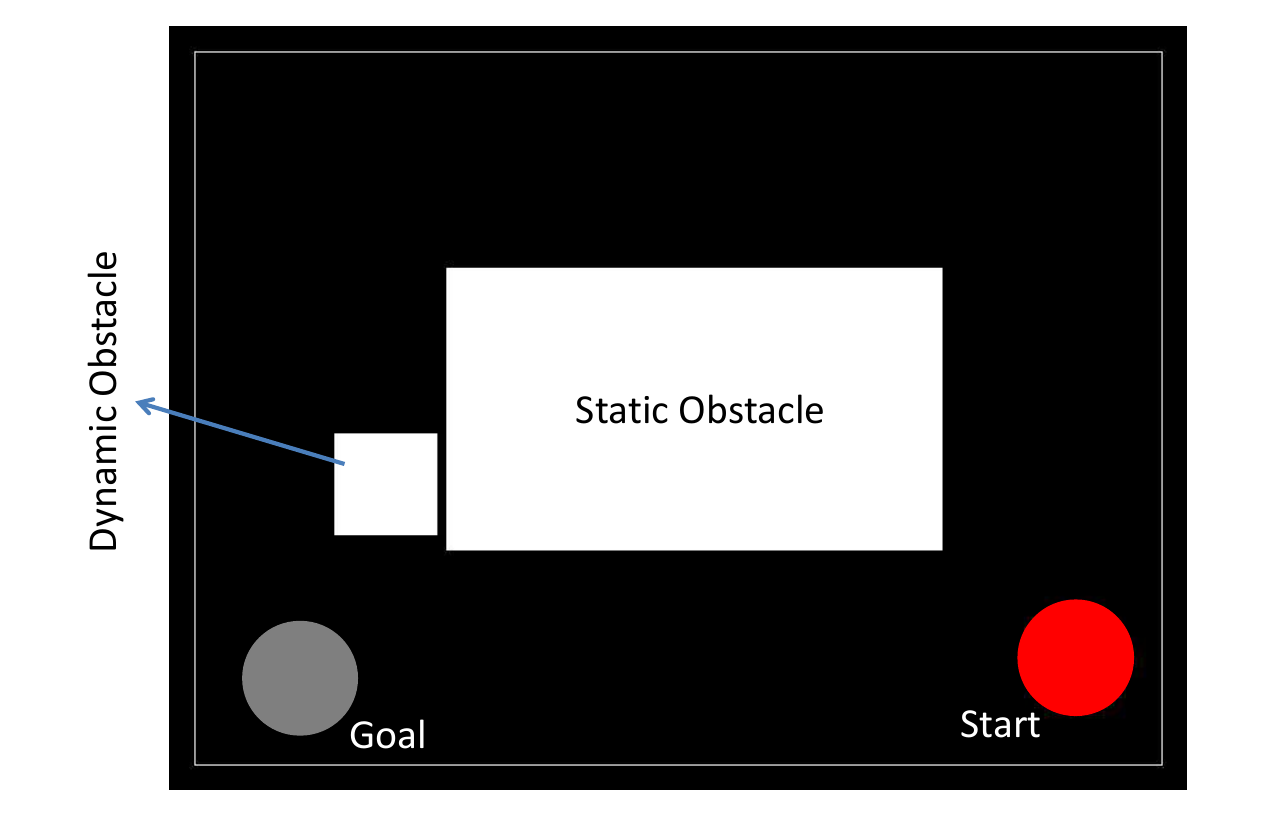}
    \caption{
        Circular mobile robot in a planar workspace with a static obstacle and a dynamic obstacle, with the start and goal configurations marked. The objective is to plan a path for the robot from the start state to the goal state, a path that avoids both static and dynamic obstacles.
    }
    \label{fig:dyn_obs_ex_start_goal}
\end{figure}

Figure~\ref{fig:dyn_obs_ex_start_goal} shows a circular mobile robot in a planar workspace. The workspace has both static and dynamic obstacles. We wish to plan paths for this robot between any start and goal configurations, such that the robot avoids hitting any of the obstacles. Figure~\ref{fig:dyn_obs_ex1} shows the frames of the motion executed by this robot, for the start and goal configurations marked in Figure~\ref{fig:dyn_obs_ex_start_goal}.

\clearpage

\begin{figure}[H]
    \centering
    \begin{subfigure}[b]{0.24\textwidth}
        \includegraphics[width=\textwidth]{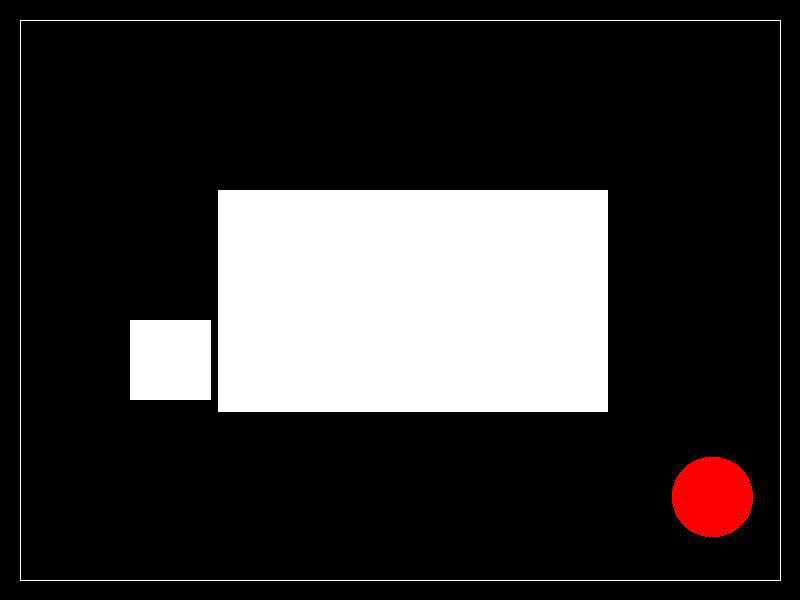}
        \caption{Frame 1}
    \end{subfigure}
    \begin{subfigure}[b]{0.24\textwidth}
        \includegraphics[width=\textwidth]{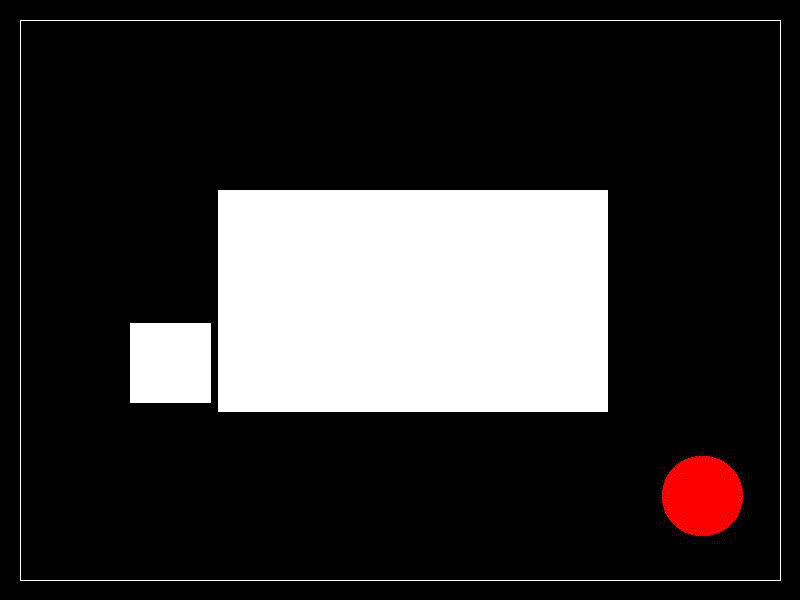}
        \caption{Frame 2}
    \end{subfigure}
    \begin{subfigure}[b]{0.24\textwidth}
        \includegraphics[width=\textwidth]{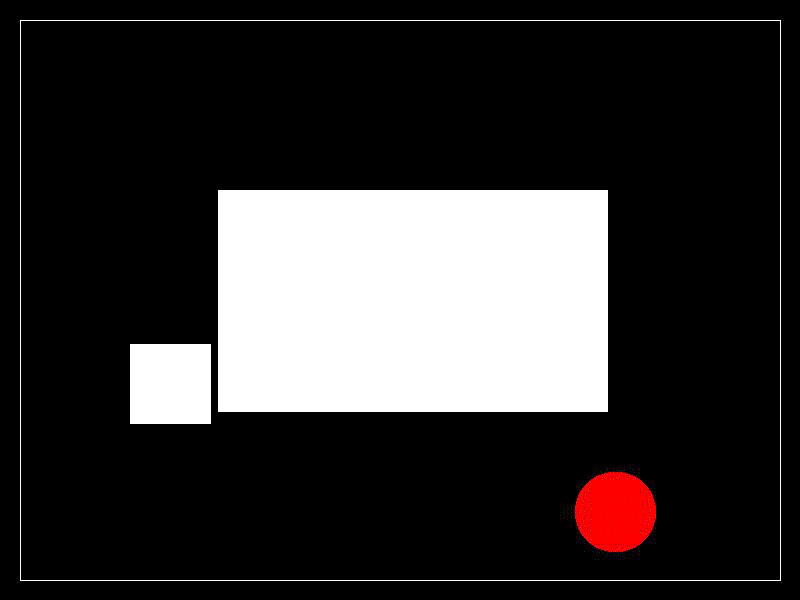}
        \caption{Frame 3}
    \end{subfigure}
    \begin{subfigure}[b]{0.24\textwidth}
        \includegraphics[width=\textwidth]{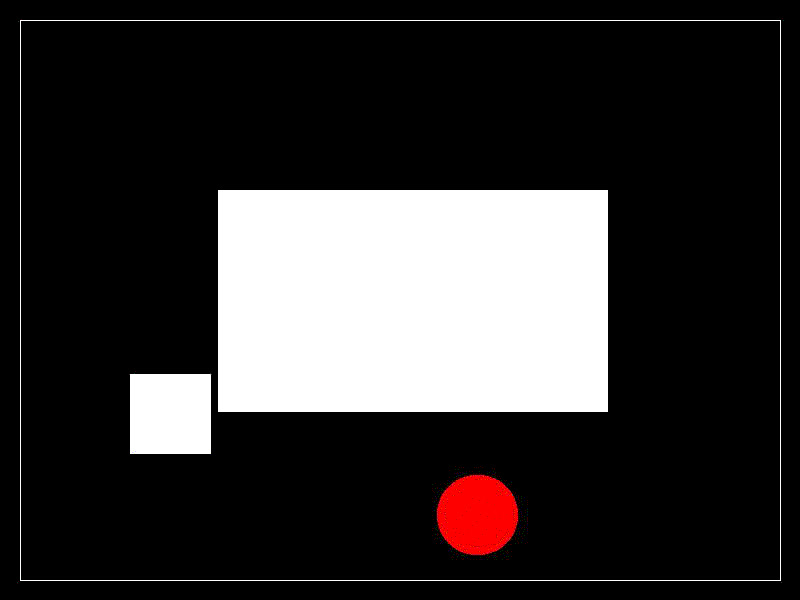}
        \caption{Frame 4}
    \end{subfigure}
    
    \begin{subfigure}[b]{0.24\textwidth}
        \includegraphics[width=\textwidth]{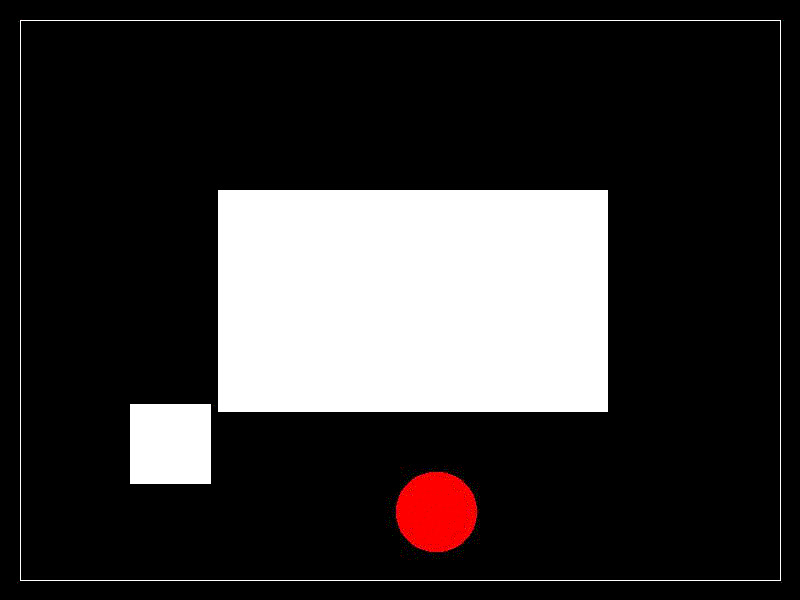}
        \caption{Frame 5}
    \end{subfigure}
    \begin{subfigure}[b]{0.24\textwidth}
        \includegraphics[width=\textwidth]{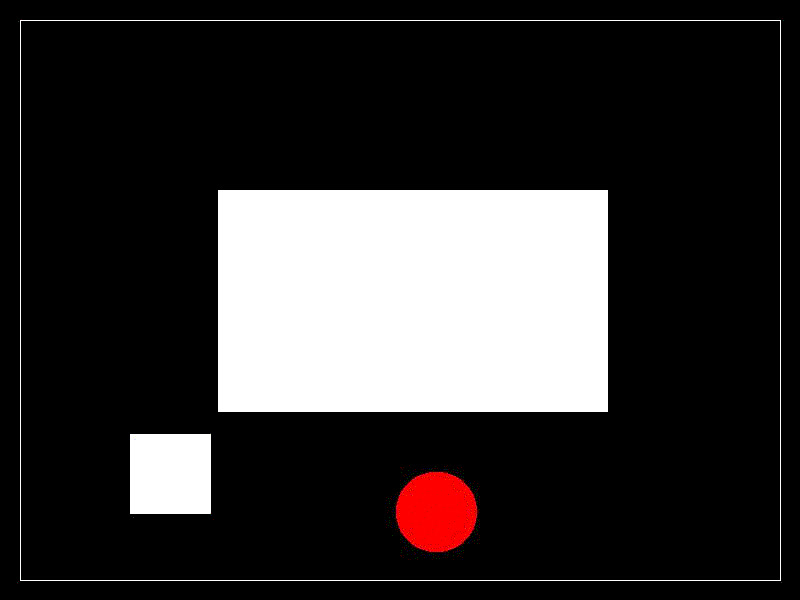}
        \caption{Frame 6}
    \end{subfigure}
    \begin{subfigure}[b]{0.24\textwidth}
        \includegraphics[width=\textwidth]{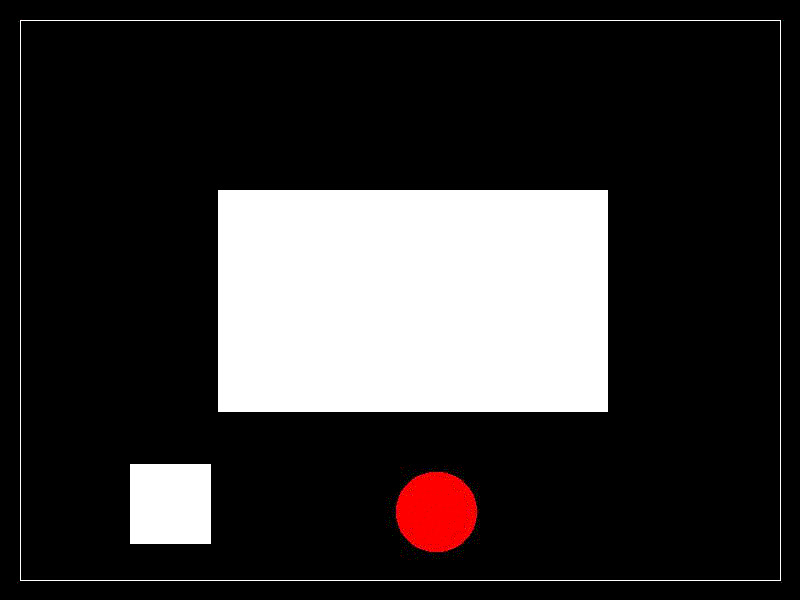}
        \caption{Frame 7}
    \end{subfigure}
    \begin{subfigure}[b]{0.24\textwidth}
        \includegraphics[width=\textwidth]{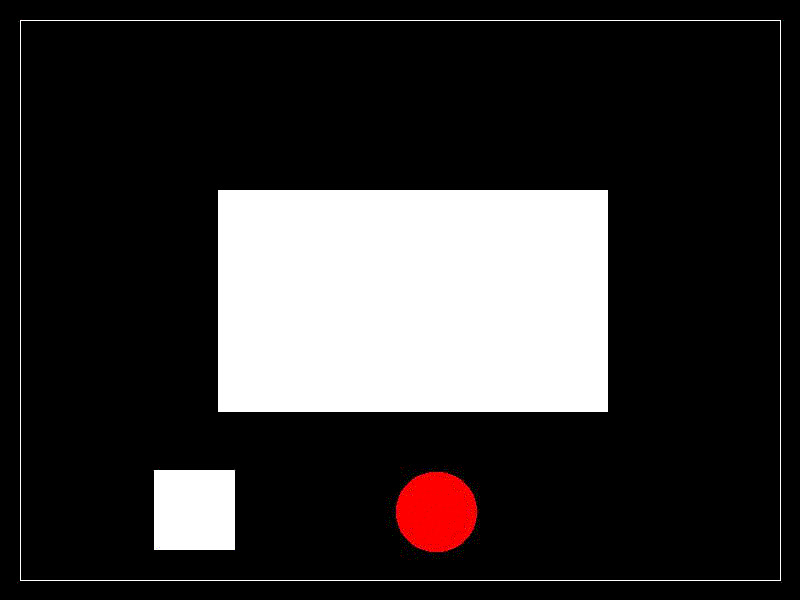}
        \caption{Frame 8}
    \end{subfigure}
    
    \begin{subfigure}[b]{0.24\textwidth}
        \includegraphics[width=\textwidth]{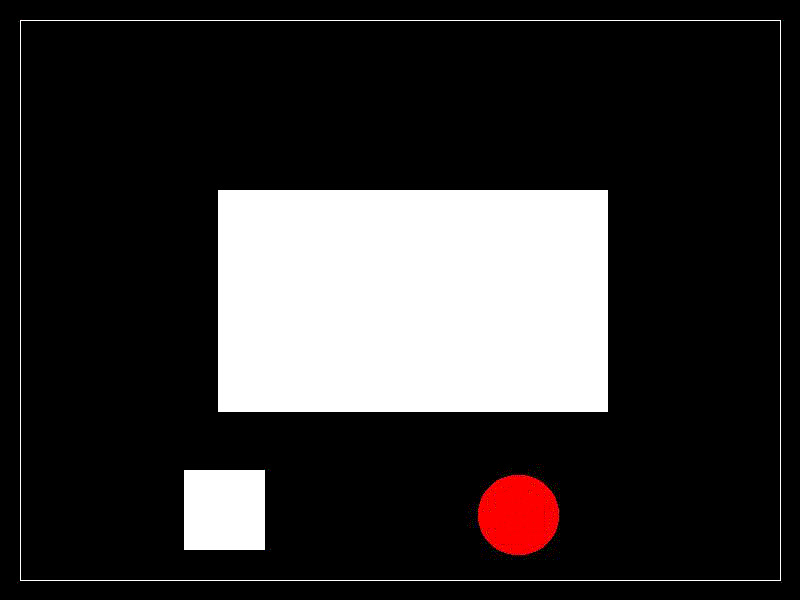}
        \caption{Frame 9}
    \end{subfigure}
    \begin{subfigure}[b]{0.24\textwidth}
        \includegraphics[width=\textwidth]{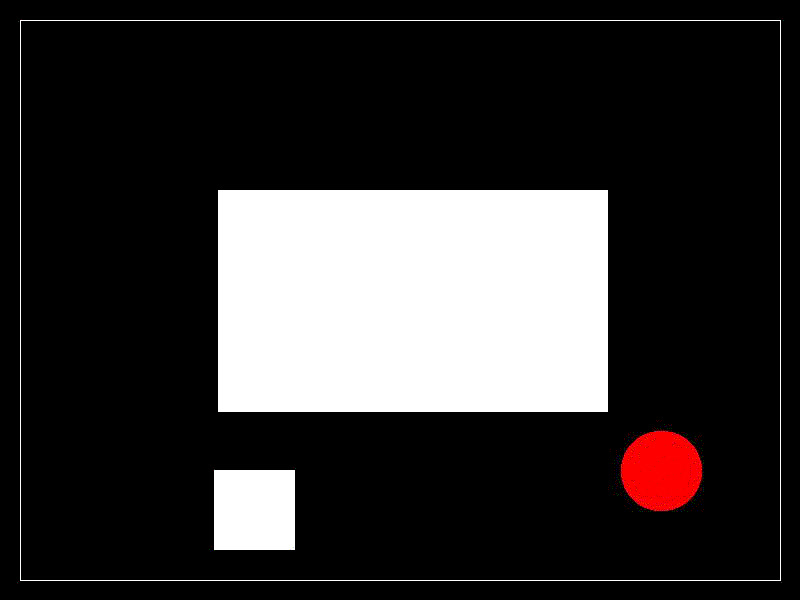}
        \caption{Frame 10}
    \end{subfigure}
    \begin{subfigure}[b]{0.24\textwidth}
        \includegraphics[width=\textwidth]{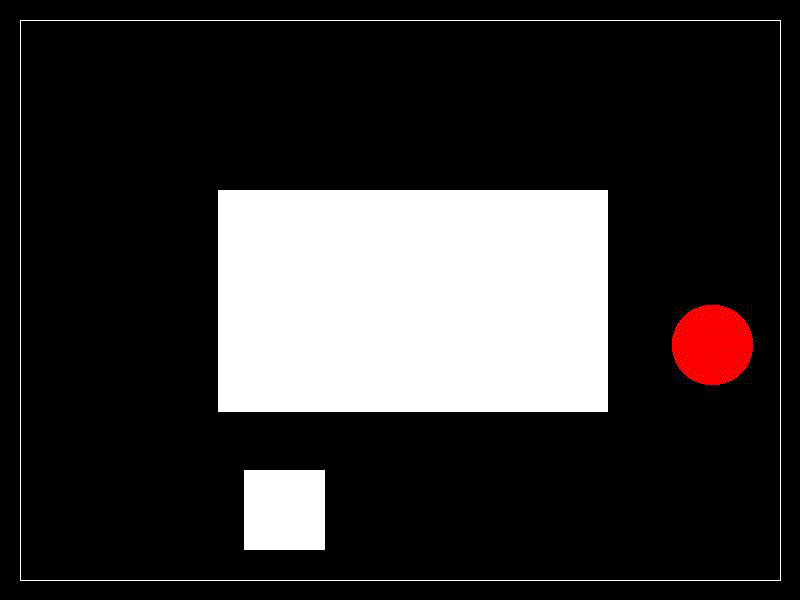}
        \caption{Frame 11}
    \end{subfigure}
    \begin{subfigure}[b]{0.24\textwidth}
        \includegraphics[width=\textwidth]{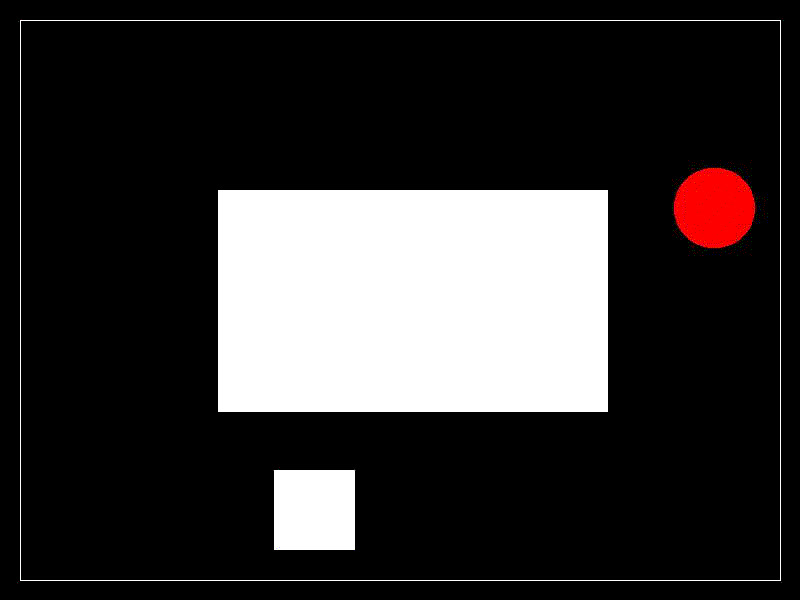}
        \caption{Frame 12}
    \end{subfigure}
    
    \begin{subfigure}[b]{0.24\textwidth}
        \includegraphics[width=\textwidth]{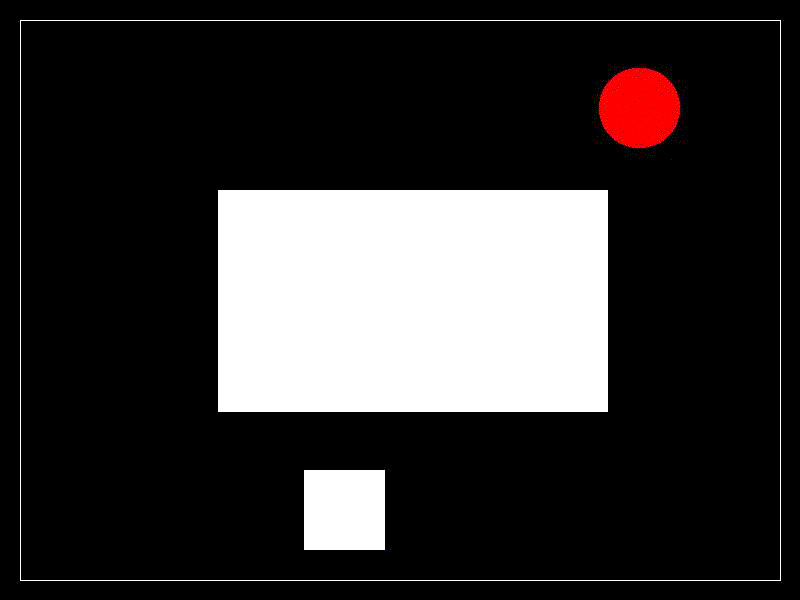}
        \caption{Frame 13}
    \end{subfigure}
    \begin{subfigure}[b]{0.24\textwidth}
        \includegraphics[width=\textwidth]{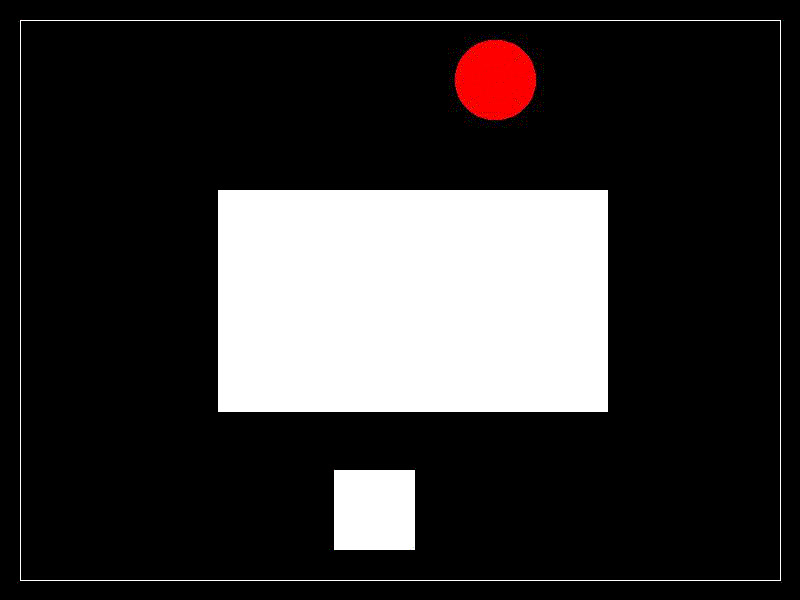}
        \caption{Frame 14}
    \end{subfigure}
    \begin{subfigure}[b]{0.24\textwidth}
        \includegraphics[width=\textwidth]{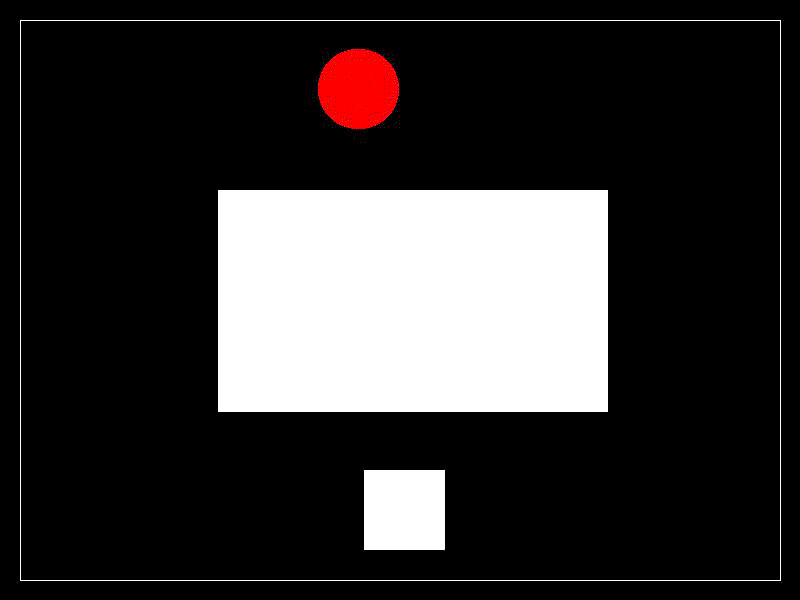}
        \caption{Frame 15}
    \end{subfigure}
        \begin{subfigure}[b]{0.24\textwidth}
        \includegraphics[width=\textwidth]{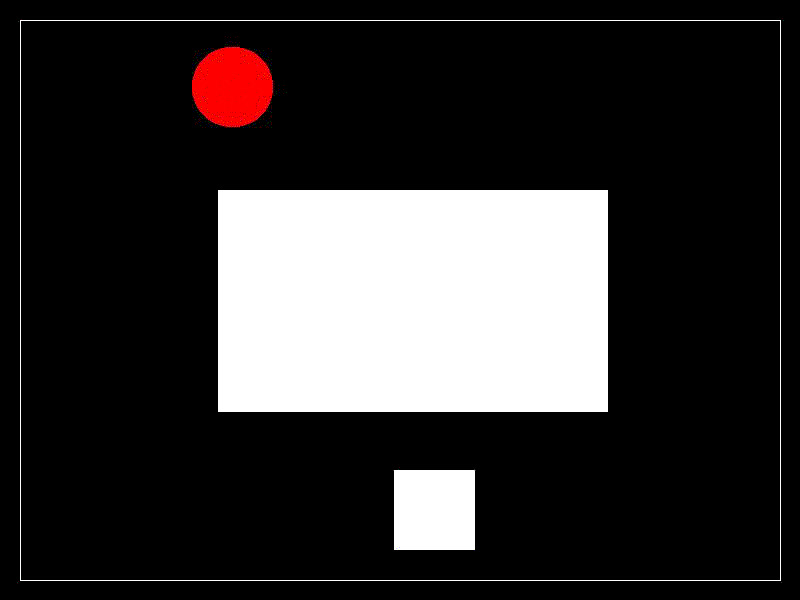}
        \caption{Frame 16}
    \end{subfigure}
    
    \begin{subfigure}[b]{0.24\textwidth}
        \includegraphics[width=\textwidth]{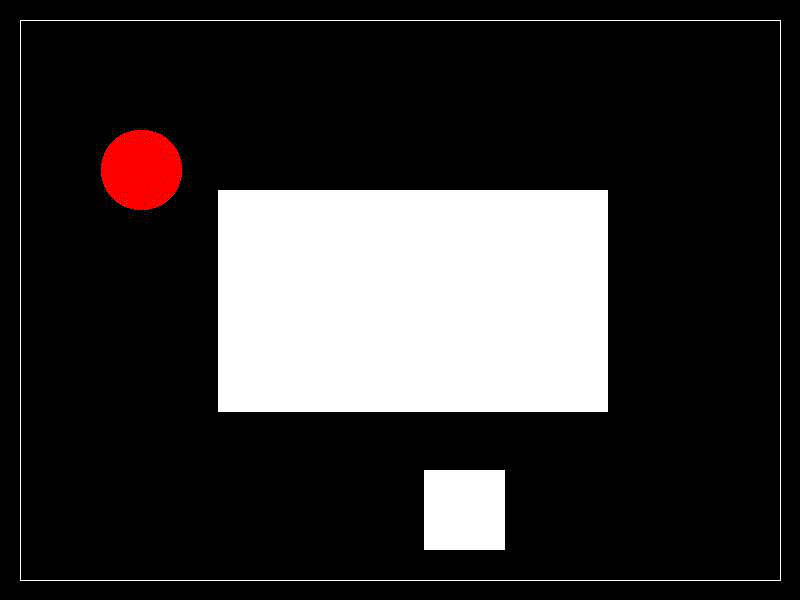}
        \caption{Frame 17}
    \end{subfigure}
    \begin{subfigure}[b]{0.24\textwidth}
        \includegraphics[width=\textwidth]{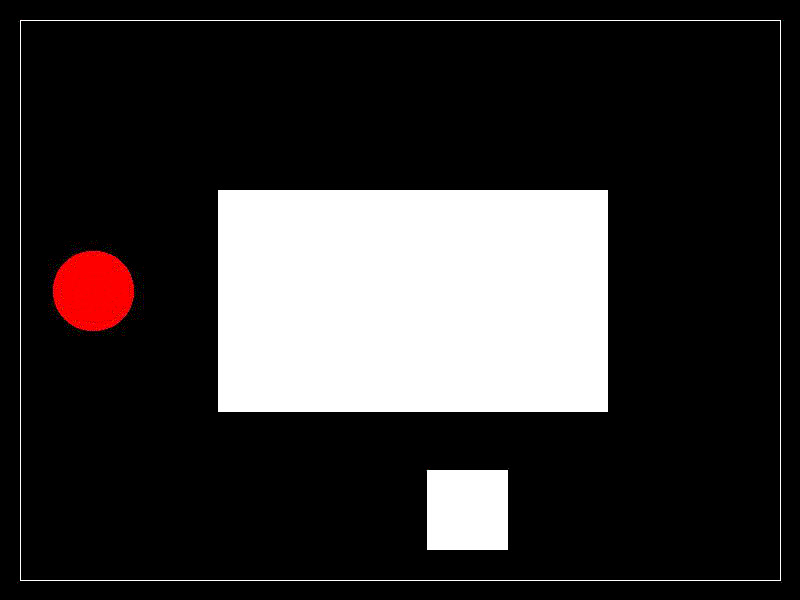}
        \caption{Frame 18}
    \end{subfigure}
    \begin{subfigure}[b]{0.24\textwidth}
        \includegraphics[width=\textwidth]{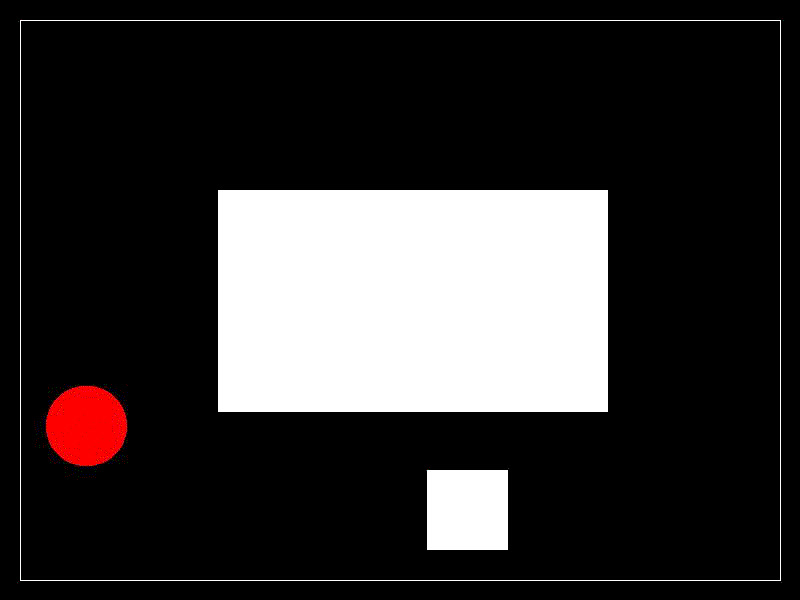}
        \caption{Frame 19}
    \end{subfigure}
    \begin{subfigure}[b]{0.24\textwidth}
        \includegraphics[width=\textwidth]{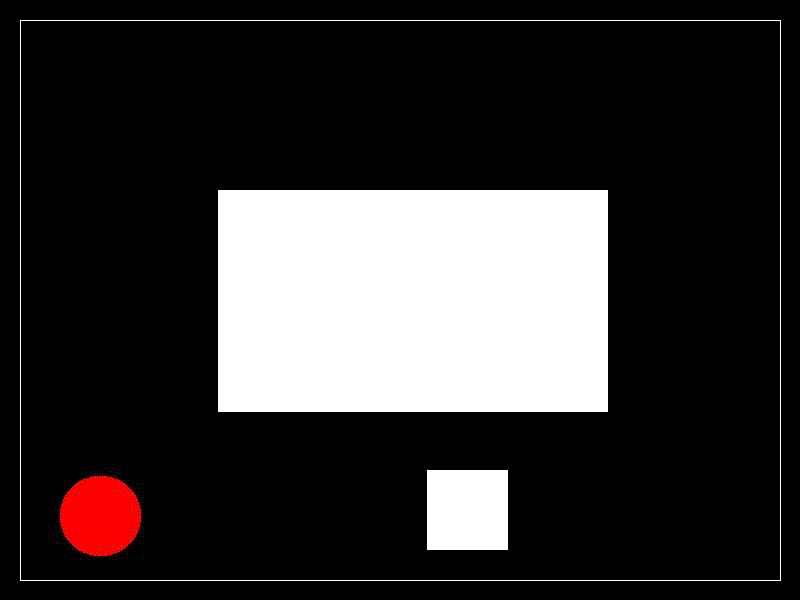}
        \caption{Frame 20}
    \end{subfigure}
    \caption{
        \textit{Motion planning with dynamic obstacle avoidance.} 
        Frames in a path followed by the robot to move from the start position to the goal position indicated in Figure~\ref{fig:dyn_obs_ex_start_goal}, while avoiding the static and dynamic obstacles. We can see that from Frame 5 to Frame 8 only the obstacle is moving and the robot is still since the moving obstacle is colliding with the goal state in these frames. Once the obstacle moves out of the goal position, the robot again starts moving from Frame 9. However, since the obstacle is now coming in the way of the robot, the robot cannot take the path it started with. So, it backtracks and takes a round-about route from the current position to the goal position.
    }
    \label{fig:dyn_obs_ex1}
\end{figure}

\clearpage

\section{Fast Collision Detection using RLE}
\label{sec:fast_collision_detection}


In robot motion planning, collision detection (i.e., detecting whether the robot hits an obstacle in a given configuration) is extremely important because otherwise, the robot can damage itself and/or the world. In this section, we discuss fast collision detection using run-length encoding (RLE) of binary images. See Section~\ref{sec:rle_intro} for a description of RLE.

\subsubsection{Background Subtraction}
\label{subsec:bg_sub}
Given a set of images of a robot in several different configurations in its workspace, an image of just the background can be obtained by taking the mean of all these images. Subtracting this background image from the robot image in each configuration will give a background subtracted image for that configuration. 

The quality of the background image obtained by this method heavily depends on the number of images that were averaged, the uniformity of lighting conditions and absence of any external noise leading to significantly visible changes in the environment, while taking images of the robot. In the presence of sources of any changes visible to the camera, more advanced methods will be needed to
extract the background more reliably. The collision detection method discussed here assumes a high-quality background subtraction.

\subsubsection{Collision Detection by Image Intersection}
\label{subsec:collision_detection_img}
To find if the robot collides with an obstacle in a given configuration of the robot, we can overlap the background subtracted images of the robot and the obstacle and see if there is an intersection between the two images. If the intersection is non-empty, then there is a collision; otherwise, that configuration is free. This, in the worst case, requires $O(p)$ comparisons, where $p$ is the number of pixels in the robot/obstacle image.

\subsubsection{Collision Detection using RLE}
\label{subsec:collision_detection_rle}
Collision detection can be made much faster by using the interval based RLE (see Section~\ref{subsec:interval_rle}) representations of the background subtracted images of the robot and obstacle. A background-subtracted image can be treated as a binary image with a point on the robot body or the obstacle represented by a 1 and a point of the background by a 0.

Let $r = \langle [rl_i, ru_i): i=0 \ldots m-1 \rangle \text{ and } o = \langle [ol_j, ou_j): j=0 \ldots n-1 \rangle$ be the interval based RLE encodings of the robot and obstacle images respectively, ignoring the image size. Here $rl_i$ and $ru_i$ are the lower and upper bounds of the $i^{th}$ interval of $r$, the robot image RLE; $ol_j$ and $ou_j$ are the lower and upper bounds of the $j^{th}$ interval of $o$, the obstacle image RLE; and $m$ and $n$ are the number of intervals in $r$ and $o$ respectively. Then the robot collides with the obstacle if and only if 
\[ 
    \exists i < m, j < n:  [rl_i, ru_i) \cap [ol_j, ou_j) \ne \emptyset. 
\]
A non-empty intersection between two integer intervals $[rl_i, ru_i)$ and $[ol_j, ou_j)$ can occur in four different ways as listed below: 
\begin{enumerate}
  \item $ol_j \le rl_i < ou_j \le ru_i$
  \item $rl_i \le ol_j < ru_i \le ou_j$
  \item $ol_j \le rl_i < ru_i \le ou_j$
  \item $rl_i \le ol_j < ou_j \le ru_i$.
\end{enumerate}

Equivalently, the robot collides with the obstacle if and only if 
\[ \exists i < m, j < n: min(ru_i, ou_j) > max(rl_i, ol_j). \]

A pseudocode implementing this method is shown in Algorithm~\ref{alg:collision_detection_rle}. It takes $O(min(m, n))$ comparisons, which is usually much smaller than the number of pixels in the robot/obstacle image.

\begin{algorithm}[ht]
\caption{ Collision Detection using Image RLE }
\label{alg:collision_detection_rle}
\begin{algorithmic}[1]
    \REQUIRE $rl, ru, ol, ou, m, n$ // see the text in 
    Section~\ref{subsec:collision_detection_rle} for details.
    \ENSURE $true$ if there is a collision, $false$ otherwise.
    
    \STATE $i \leftarrow 0, j \leftarrow 0$

    \WHILE { $i < m \textbf{ and } j < n$ }
        \IF {  $min(ru_i, ou_j) > max(rl_i, ol_j)$ }
            \RETURN $true$
        \ENDIF

        \IF { $ru_i \le ou_j$ }
            \STATE $i \leftarrow i + 1$
        \ELSE
            \STATE $j \leftarrow j + 1$
        \ENDIF
    \ENDWHILE

    \RETURN $false$
\end{algorithmic}
\end{algorithm}

\subsection{Degree of Penetration of Robot into the Obstacle}
\label{sec:degree_of_penetration}
Sometimes we may want to allow the robot to touch the obstacle without any impact. In such cases, it is useful to know how deep the robot would penetrate into the obstacle, in a given configuration. For this, we consider the number of overlapping pixels in the robot and obstacle images as a measure of the degree of penetration. For two intervals $[rl_i, ru_i) \text{ and } [ol_j, ou_j)$, the measure of overlap $\mu_{ij}$ is given by
    \[\mu_{ij} = min(ru_i, ou_j) - max(rl_i, ol_j)\]
and the overall measure of penetration $\mu$ is given by
    \[\mu = \sum_{i, j} \mu_{ij}.\]

A pseudocode implementing this method is shown in Algorithm~\ref{alg:measure_of_penetration_rle}. Since each interval of the robot RLE can overlap with at most one interval of the obstacle RLE, the time required to compute the measure of penetration is $O(m + n)$, as can be seen from Algorithm~\ref{alg:measure_of_penetration_rle}.

\begin{algorithm}[ht]
\caption{ Measure of Penetration using Image RLE }
\label{alg:measure_of_penetration_rle}
\begin{algorithmic}[1]
    \REQUIRE $rl, ru, ol, ou, m, n$ // see the text in 
    Section~\ref{subsec:collision_detection_rle} for details.
    \ENSURE measure of penetration of the robot into the obstacle
    
    \STATE $i \leftarrow 0, j \leftarrow 0, \mu \leftarrow 0$
    \WHILE { $i < m \textbf{ and } j < n$ }
        \STATE $\mu = \mu + min(ru_i, ou_j) - max(rl_i, ol_j)$
        \IF { $ru_i \le ou_j$ }
            \STATE $i \leftarrow i + 1$
        \ELSE
            \STATE $j \leftarrow j + 1$
        \ENDIF
    \ENDWHILE

    \RETURN $\mu$
\end{algorithmic}
\end{algorithm}


\section {Demonstrations on Real Robots}
\label{sec:real-robots}

\subsection{Planar SCARA Robot}
\label{sec:scara}

We now demonstrate the algorithm for a real robot, an MTAB SCARA 4 DOF arm, in which two revolute joints move the first two links in a plane, so the motion has two degrees of freedom. Two more degrees of freedom (at the wrist) are not used in this demonstration, so the robot is effectively planar. See Figure~\ref{fig:scara-intro}. The length of link-1 is 200 mm and that of link-2 200mm. The range of angles that each link can traverse is  -135 to +135 degrees. For precautionary purposes and robot safety, we limit the range to -125 to + 125 degrees in our work. 

\begin{figure}[ht!]
    \centering
    \begin{subfigure}[b]{0.55\columnwidth}
        \includegraphics[width=\columnwidth]{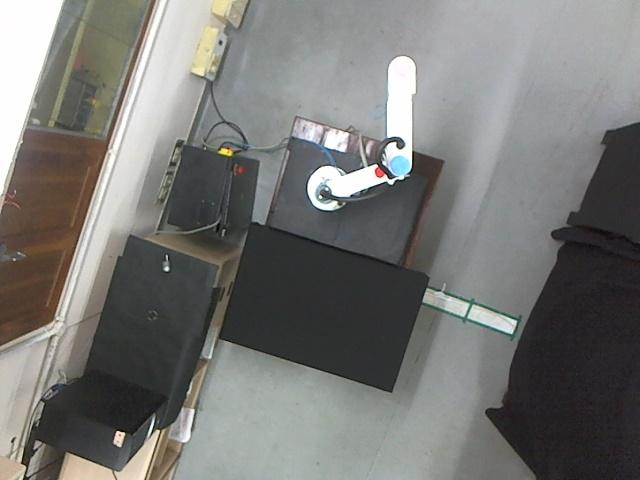}
        \caption{}
    \end{subfigure}
    
    \begin{subfigure}[b]{0.55\columnwidth}
        \includegraphics[width=\columnwidth]{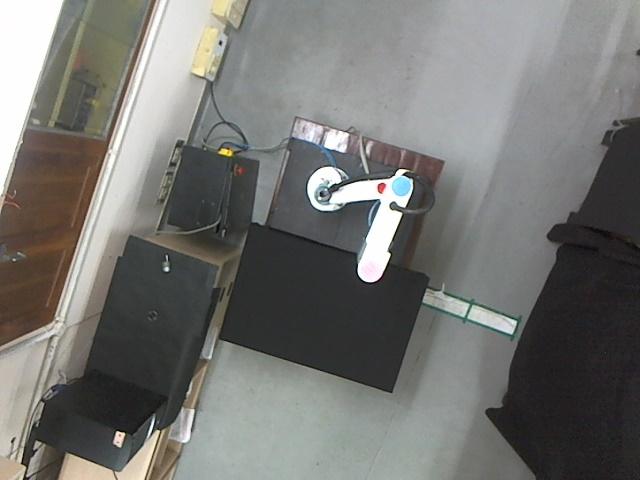}
        \caption{}
    \end{subfigure}
    
    \begin{subfigure}[b]{0.55\columnwidth}
        \includegraphics[clip, trim=1cm 7.5cm 1cm 8cm, width=\columnwidth]{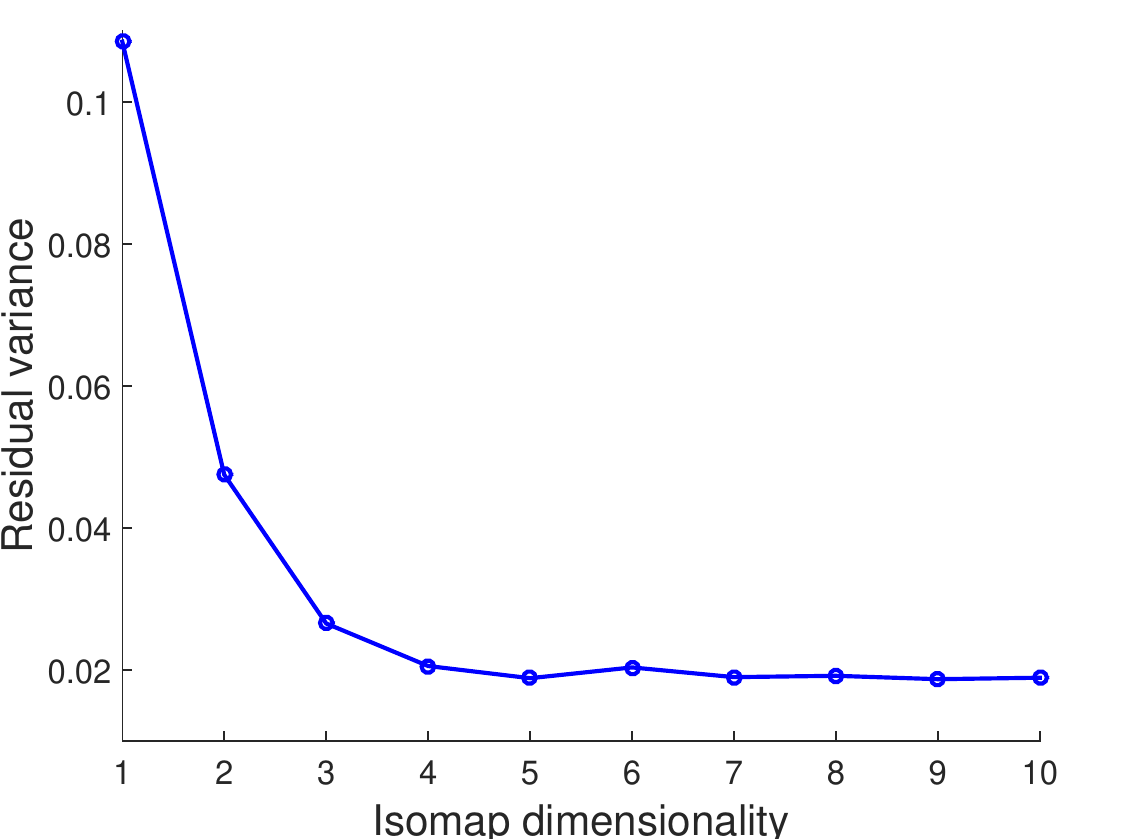}
        \caption{}
    \end{subfigure}
    \caption{
        \textit{MTAB SCARA robotic arm}:
        (a, b) Two of the 4000 images of the arm.  (c) Scree plot of Isomap suggesting that around 95\% of the variance in the data is explained by 2 dimensions.
    }
    \label{fig:scara-intro}
\end{figure}

We observe this robot with an overhead camera at a frame rate of 25 fps. 4000 images are sampled from a video while the robot is moving between random poses throughout its workspace. These images form the basic image dataset for our analysis. Background subtraction is performed on each image to generate the foreground robot. A visual roadmap (VRM) is computed using the background subtracted images. Here, we used the Euclidean metric to find neighbourhoods. 

\clearpage

\begin{figure}
    \centering
    \begin{subfigure}[b]{0.55\columnwidth}
        \includegraphics[width=\columnwidth]{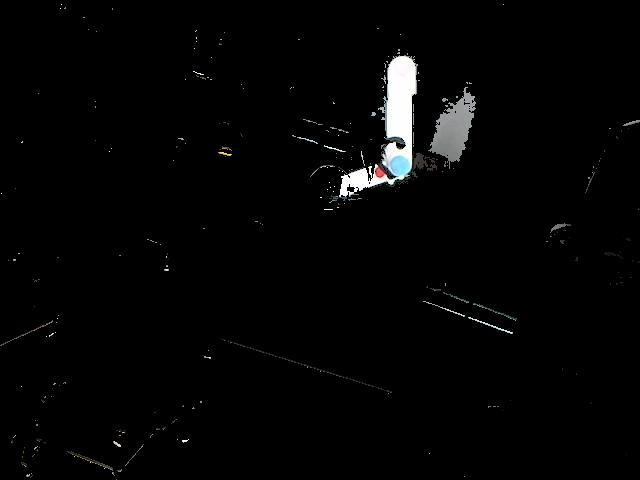}
        \caption{}
    \end{subfigure}
    
    \begin{subfigure}[b]{0.55\columnwidth}
        \includegraphics[width=\columnwidth]{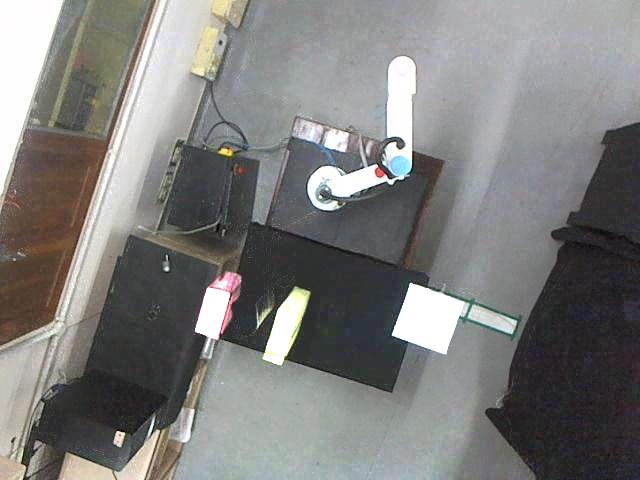}
        \caption{}
    \end{subfigure}
    
    \begin{subfigure}[b]{0.55\columnwidth}
        \includegraphics[width=\columnwidth]{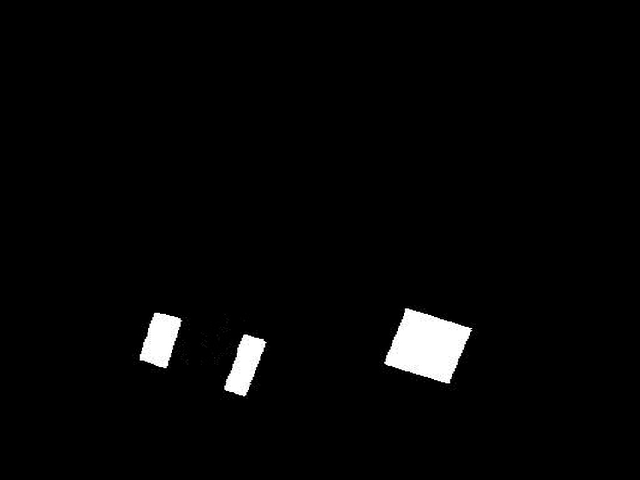}
        \caption{}
    \end{subfigure}
    \caption{
        \textit{Incorporating obstacles.} (a) Background subtracted image of the arm. (b) Image with obstacles. (c) Obstacles after image subtraction.  
    }
    \label{fig:scara-bgsub-obs}
\end{figure}

\clearpage

Thereafter, several obstacles are introduced in the workspace and the obstacles are discovered via background subtraction (\ref{fig:scara-bgsub-obs}). Note that owing to the motion being planar, a single camera view is quite adequate. Nodes corresponding to the obstacle configurations on the VCS are removed from the VRM.  A planned path is shown in Figure~\ref{fig:scara-vcs-vrm}. Note that the Isomap computations shown in the figures are only for the purpose of illustration. The actual path planning does not require Isomap. We only need to compute the VRM, which is a neighbourhood graph computed using the sampled random images of the robot. Figure~\ref{fig:scara_path_ws} shows some images of SCARA executing a path.

\begin{figure}
    \centering
    \includegraphics[width=\columnwidth]{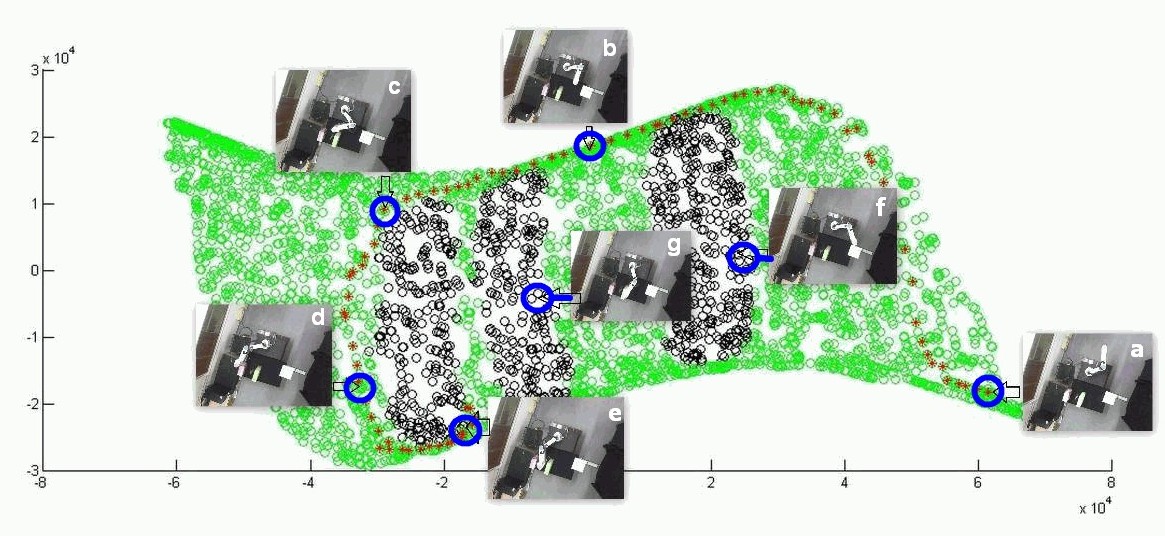}
    
    \caption{
        \textit{Path planning for SCARA using Visual Roadmap (VRM).} 
        \textit{Visual Configuration Space} (VCS) discovered by Isomap. The points corresponding to free configurations are shown in green and the points corresponding to obstacle configurations are shown in black. An obstacle-free path from a start configuration to a goal configuration, computed on the VRM embedded on the VCS, is shown in red. Images of SCARA in some of the configurations pointing to the corresponding points on the VCS are also shown here. 
    }
    \label{fig:scara-vcs-vrm}
\end{figure}

\begin{figure}
    \centering
       
    \begin{subfigure}[b]{0.55\columnwidth}
        \includegraphics[width=\columnwidth]{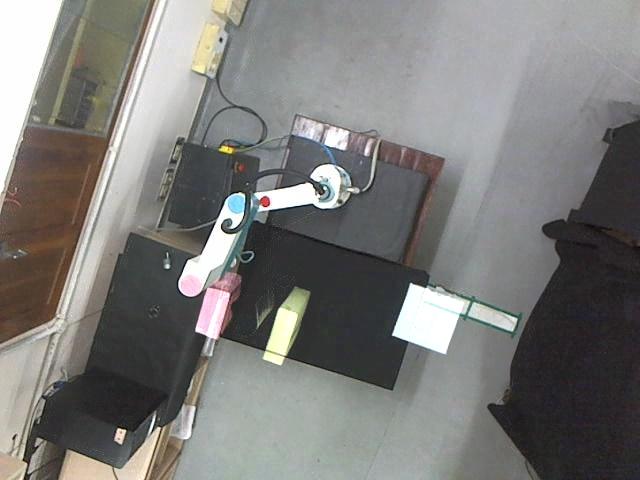}
        \caption{}
    \end{subfigure}
    
    \begin{subfigure}[b]{0.55\columnwidth}
        \includegraphics[width=\columnwidth]{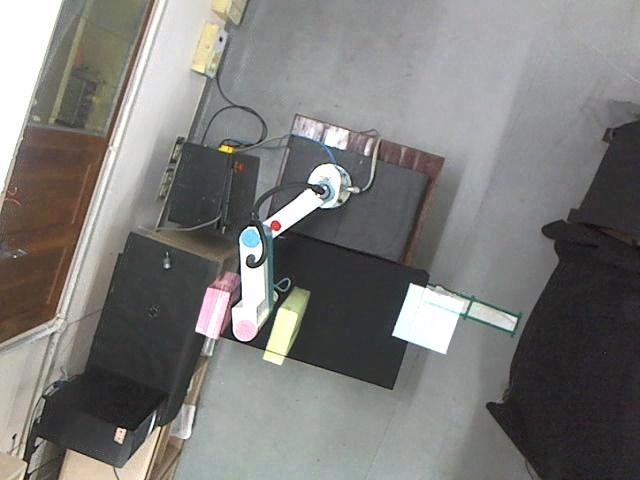}
        \caption{}
    \end{subfigure}
    
    \begin{subfigure}[b]{0.55\columnwidth}
        \includegraphics[width=\columnwidth]{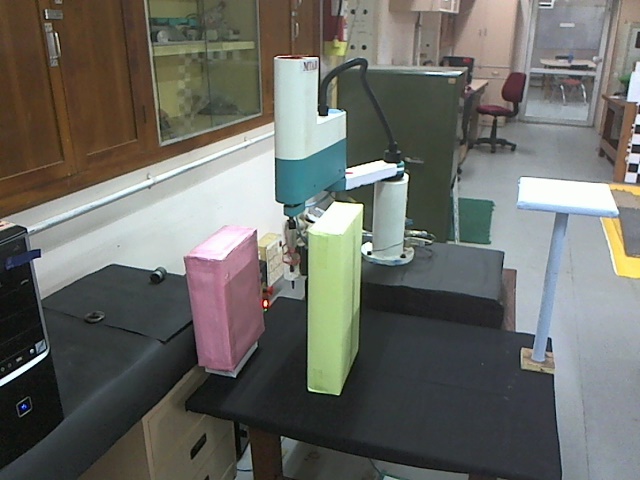}
        \caption{}
    \end{subfigure}
 
    \caption{
        Path being executed by SCARA. The last image in (c) is of the same configuration as in (b), from a different view.
    }
    \label{fig:scara_path_ws}
\end{figure}

\clearpage

\subsection{CRS A465 robot arm}
\label{sec:crs}
Here, we have a robot in a 3-D workspace, a 6-DOF CRS A465 robot. Out of its 6 DOFs, we used only 3. Since the motion is spatial (not restricted to a plane), a single camera view will not suffice for identifying collision situations. Hence, we construct the VRM using images from multiple views by stitching together the images coming from all the cameras for each sampled configuration. 

\begin{figure}[h]
    \centering
    \includegraphics[width=\columnwidth]{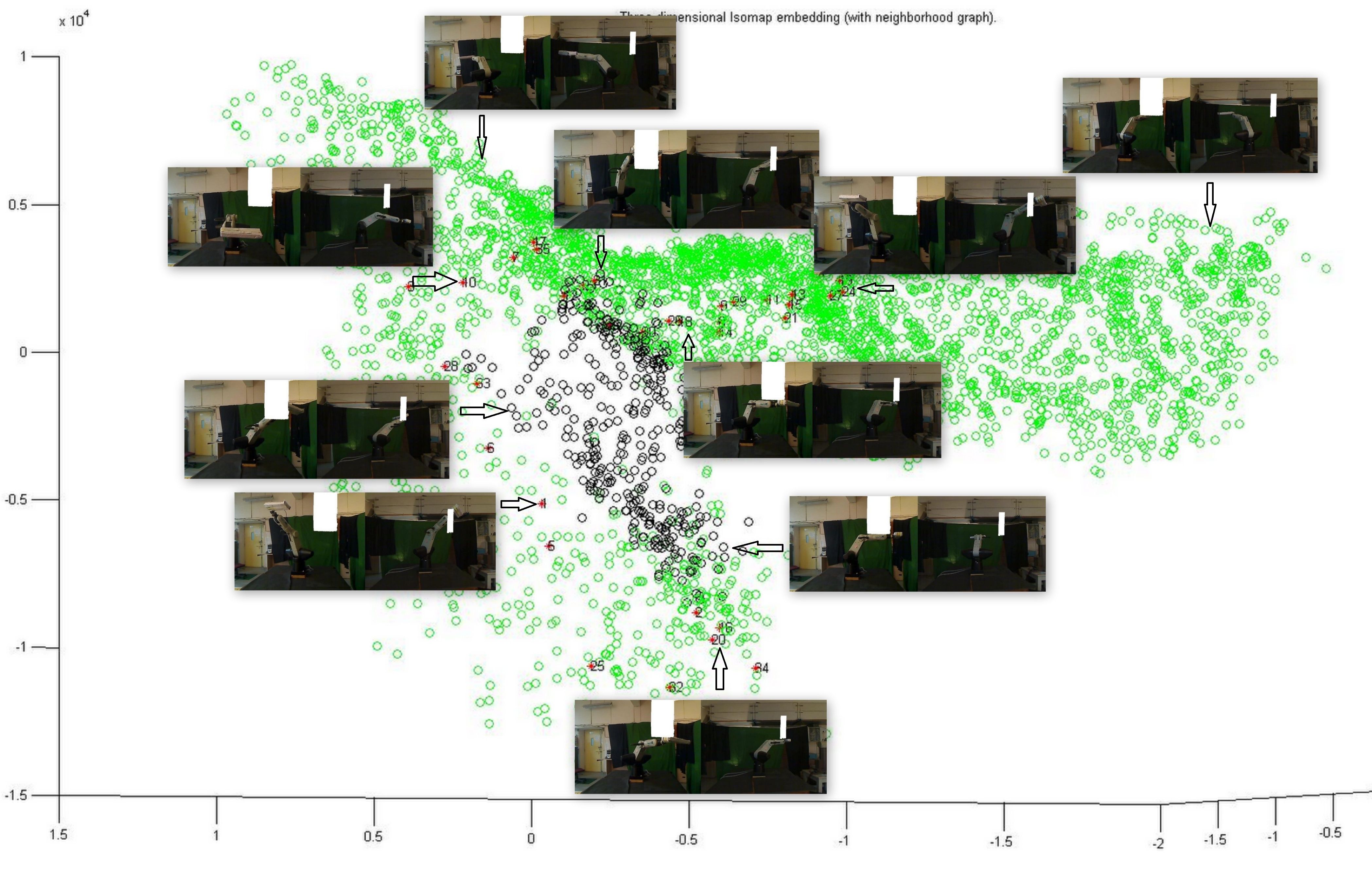}
    \caption{
        \textit{VRM for a 6-DOF CRS A465 robot, using images from two cameras.} 
        Since this is a 3-D workspace, obstacles cannot be distinguished from a single view. Here we use multiple cameras, and the intersection of the cones provides a (conservative) model for both obstacle and robot. To identify potential collision states, background-subtracted images of the obstacle (shown in white here) are overlaid on each foreground robot image. Only if the robot overlaps the obstacle in all the views, it is a potential collision node.  Collision nodes are shown in black.
    } 
    \label{fig:CRS-joint-manifold}
\end{figure}

We compute the neighbourhoods using a Euclidean metric on the stitched images. Using the foreground obstacle and robot images, we designate a robot state as an obstacle only if it overlaps the obstacle in all of the camera views. The 3-DOF workspace and a path are shown in Figure~\ref{fig:CRS-joint-manifold}, with the obstacle nodes marked in black. We used two cameras for this demonstration. 

\section{Effect of Noise}
\label{sec:effect_of_noise}

Each source of noise would introduce some additional dimensions in the manifold learning procedure, requiring many more images to be sampled. For example, additional noise due to illumination changes in case of a 2 DOF robotic arm, would make its configuration manifold 3 dimensional. In addition to requiring a dense enough image sample of the robot in various poses, we would then require images under various lighting conditions as well. While the DOFs of the robot have a direct relation with the robot’s motion, the additional dimensions due to noise would only make the manifold discovery process harder, without being useful. 

Assuming all the sources of noise can be controlled, the proposed approach can be used to plan motions for industrial robotic arms which are fixed in a position, where the environment does not change frequently and the obstacles are introduced rarely. With the help of external modules like tool-grasping, it can be used in assembly line automation. The method can also be applied to plan motions for mobile robots to perform various tasks (with the help of additional modules again) in restricted environments where the robot is always observable with cameras and obstacles are not frequent. 

As and when new technologies like 3D sensing become available, they can be incorporated into the proposed framework seamlessly to have better noise reduction and better background removal, thus leading to a reduction in the sample complexity and a qualitative improvement in the computed paths. In this thesis, we focused mainly on demonstrating the possibility of using visual manifolds as a fundamentally new construct in cognitive robotics to address various related problems, rather than worrying about extraneous factors such as background noise, which are more in the realm of computer vision. 

In the next chapter, we use similar manifold-based computational models to explain how a cognitive agent such as a human infant might learn the necessary representations to perform actions in its peripersonal space.
\chapter{Manifold Body Schema} 
\label{chap:manifold_body_schema} 

\lhead{\emph{Manifold Body Schema}}

In this chapter, we demonstrate the use of the manifold based framework developed in the previous chapters in modelling things beyond motion planning, in particular how an infant cognitive agent (e.g., a human infant) might \textit{learn to reach} objects in its surroundings. A body schema of a cognitive agent is a representation of its body that allows it to infer the position and orientation of its limbs relative to its world, and to move and perform actions in that world. We propose a \textit{computational model of body schema} based on manifolds. We suggest how it can be \textit{acquired} just by observing one's own body, and how it can be \textit{updated} as the body grows. We demonstrate how it could be used for tasks such as moving to a desired pose, swatting and reaching of objects within peripersonal space, avoiding obstacles and planning motions in the peripersonal space.

\section{Introduction}
\label{sec:intro}

Humans and animals are very good at producing smooth reaches and other movements from visually presented inputs. The inverse and forward aspects of this problem --- i.e. going from the task space to the actuator space and vice versa --- have long been a focus of interest~\citep{kawato-90_feedback-error-learning-NN-for-motor-neuro,wolpert-ghahramani-00_computational-movement-neuroscience}. However, in much of this work, the task is presented in terms of 3D coordinates, and not in terms of visual inputs (e.g. an image of the desired goal pose).  The few attempts to incorporate a visual map also work via a set of intermediate workspace coordinates~\citep{ofjall-felsberg-15_learning-vision-robot-control_autonomous}. In this work, we present an approach based on visuomotor babbling; i.e. the input is a set of images sampled from the motion space, and corresponding actuation parameters (joint angles). This differs from earlier work in motor babbling in that 
no workspace coordinates are needed~\citep{caligiore-baldassarre-08_motor-babbling-developmental-reaching-w-obstacles,lee-11_intrinsic-activitity-from-motor-babbling-to-play,rolf-steil-10_goal-babbling-to-learn-inverse-kinematics,rolf-asada-13_learning-inverse-models-w--goal-babbling}.

\subsection{Maps as Visuo-Motor manifolds}
Another significant difference of the present approach is that it is based on the discovery of low-dimensional manifolds in the image space. In executing any motor task, the sensory feedback co-varies with the actuator
signals, so that the input, output and combined configurations lie on a low-dimensional subspace of the sensory, motor, or joint spaces.  The number of signals (or dimension $D_i$) in each of these modalities vary, e.g. visual data may involve $O(10^8)$ photoreceptors, whereas motor output may involve thousands of muscle spindle
signals.  However, if a limb is being moved whose pose can be specified by fixing $d$ joint angles (its degrees of freedom), then the subspace occupied by either the images, actuator, or proprioception signals, all lie
on a $d$-manifold embedded in the high-dimensional sensorimotor space. Thus, the efferent muscle spindle feedback, afferent corticospinal commands, visual images of the hand, as well as other contingent sensory information like tactile perception, images of the full arm and body, or hand-only images, all can be singly or jointly represented on suitable $d$-dimensional non-linear manifolds.  The parameters on any of these manifolds --- even the visual manifold --- constitute a set of generalized coordinates, which can be used to uniquely specify a pose, and hence to describe a motion. Thus, these manifolds constitute multiple ``representations'' for sensorimotor space~\citep{grush-00-brain-mind_self-world-space}, and may be considered part of the elusive ``body schema''.

For implementing such a system in a neural substrate, we consider a set of neurons connected with its nearby nodes in the image space, motor space or some combination of these spaces.  Now, we know that a neuron can be used to compute the local principal components (PCA) on this neighbourhood~\citep{oja-82_simplified-neuron-model-as-principal-component-analyzer,philipona-oregan-nadal:2003}, which is an approximation of the tangent space at this point on the manifold. This local basis space then represents a ``chart''; the set of all such charts constitute the ``atlas'' for the manifold.  Both forward and inverse mappings --- i.e. from proprioception to what the arm should look like and target hand images to motor commands --- can now be read off from the same manifold.

We show how such a system may arise in a developing infant, and how it may explain aspects such as an early motor awareness of her own body, and the ability to reach for objects in peripersonal space. Newborns less
than a month old move the visible arm more than the one not attended to, and in the dark, they exert themselves to keep it where it is visible~\citep{van1997keeping}.  Such actions have been taken to be indicative of the possibility that the neonate may be learning a map between vision and proprioception, and discover new
possibilities for its motions~\citep{von-Hofsten-04trics_action-perspective-on-motor-development}. When in a darkened room with a beam of light, the infant attempts to keep the arm in the light and slows down the motion of the limb when it is about to reach the beam. The phenomenon of motor babbling is well-documented,
and~\citep{thelen-79_rhythmical-stereotypies-in-normal-human-infants} has noted how new motor patterns arise
just as the infant was gaining postural control over a previous part of the workspace, and reflect ``a degree of functional maturity but as yet incomplete voluntary control.''  In this work, we take this to mean that the developing system has explored most parts of its reachable space.  While some have even argued that such movements may be intentional and prospective~\citep{adolph-berger-06_motor-development}, we wish to seek the mapping to motor function without entering into these larger debates.

\subsection{An Empiricist View of Body Schema}

Foals, calves and other hoofed animal infants can start walking very soon after they are born (within hours). A newborn monkey baby can hang on to its mother while the mother jumps from one tree to another tree. These animals have much better innate motor skills than human babies. In humans, most of the motor skills are learned. We take an empiricist view to model the body schema of cognitive agents that do not have innate motor skills. See Figure~\ref{fig:body_schema_empiricist_view}.

\begin{figure}[ht!]
    \centering
    \includegraphics[page=2, clip, trim = 1cm 2cm 1cm 6cm, width=0.8\textwidth]{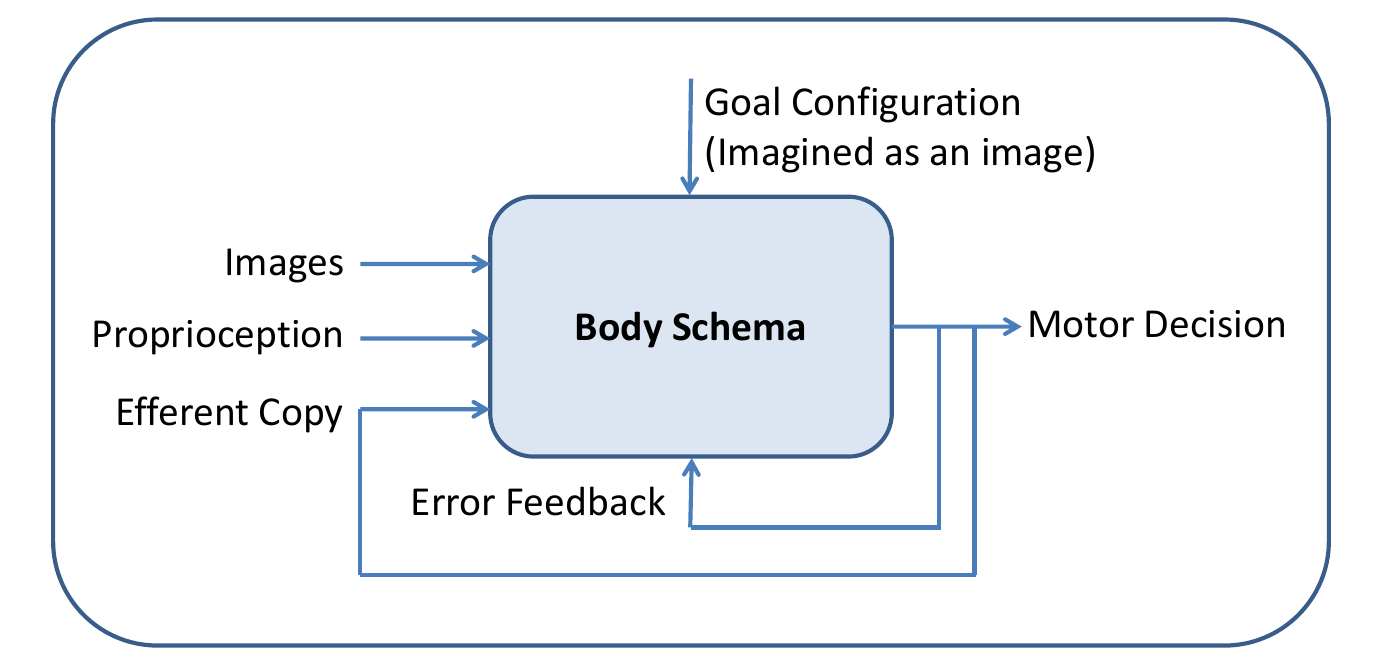}
    \caption{
        \textit{An empiricist view of body schema.} In the early stages, the agent does not have good enough visuomotor models to execute a motor task such as reaching an object in its peripersonal space. Its past visual experience is available in the form of images. Also available are the proprioceptive signals. In order to reach an object, a goal state is imagined as an image. When the agent tries to achieve the goal state, a motor decision is made, of which an efferent copy is retained. The error in the proprioceptive signals of the actually reached state from those of the desired goal state is fed back for a correction of the body schema. As the agent gains more and more experience with its own body and its world, its body schema gets better and better for making motor decisions for desired actions accurately.
    }
    \label{fig:body_schema_empiricist_view} 
\end{figure}

\section{Image Manifolds}

Figure~\ref{fig:overview-example} shows an agent with a red elliptical body and a single 3 jointed arm, having a rigid U-shaped hand at the end.  The range of angles for each joint is as follows: shoulder ($\theta_1)$ : -30 to 130$^\circ$; elbow ($\theta_2$) : 0 to 170$^\circ$; and wrist ($\theta_3$): -30 to 60 $^\circ$, so that the motion of the arm can hypothetically range over a cube in the joint angle space. 

Over the past two decades, abundant evidence has accumulated about the tight integration of motor and perceptual spaces.~\citep{prinz-97_perception-perception-and-action-planning,jordan-stork-02_action-planning-affects-spatial-localization}. However, the exact mechanics of this integration has not been explicated. There is also a question as to whether the mapping is a close association that does not participate in the function directly, or whether the same cells that command the interneurons are also involved in generating a visual expectation~\citep{gallese-lakoff-05_the-brains-concepts-sensory-motor}.

Following an idea gathering strength in neuroscience~\citep{ganguli-sompolinsky-12_compressed-sensing-sparsity-dimensionality}, we would like to suggest that the infant in the first few weeks of life is actually learning the integration of these perceptual and motor cues by projecting these onto a low-dimensional nonlinear manifold.  See Figure~\ref{fig:joint-manifold}. Such an approach is also neurally plausible in that a neural architecture could be computing non-linear manifolds as combinations of local tangent spaces computed via PCA using single neurons.   

\begin{figure}[t]
  \begin{center}
    
    \begin{subfigure}[b]{0.4\columnwidth}
        \centering
        \includegraphics[width=\columnwidth]{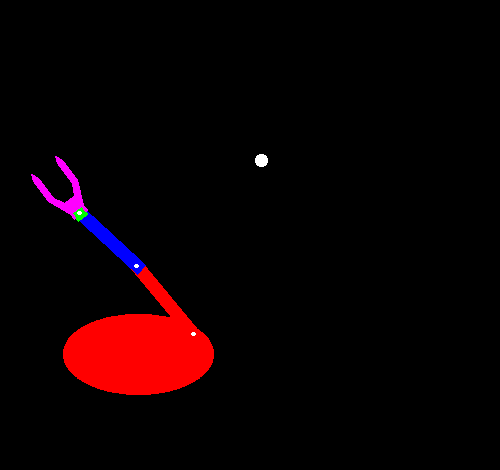}
        \caption{}
    \end{subfigure}
    \hspace{1mm}
    \begin{subfigure}[b]{0.47\columnwidth}
      \centering
        \includegraphics[width=\columnwidth]{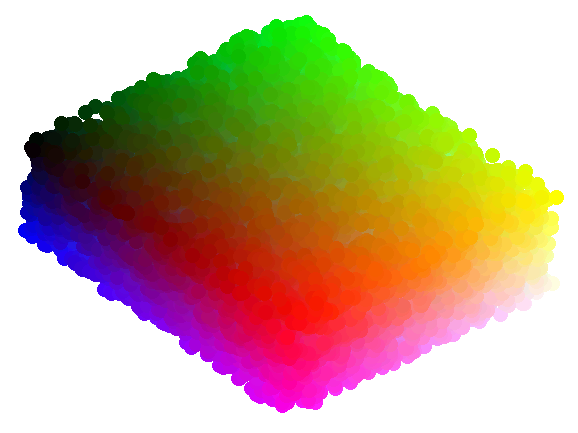}
        \caption{}
    \end{subfigure}
    
    \begin{subfigure}[b]{0.47\columnwidth}
      \centering
        \includegraphics[width=\columnwidth]{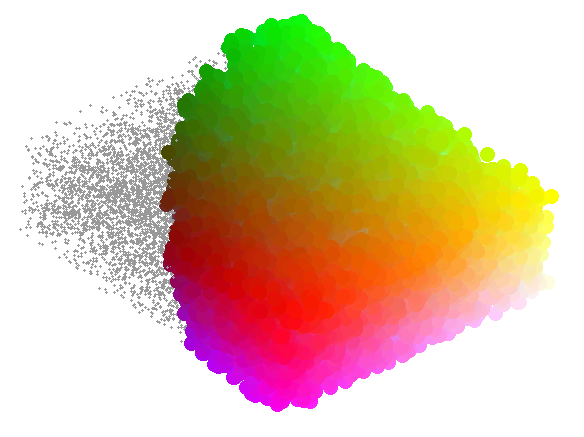}
        \caption{}
    \end{subfigure}
    \begin{subfigure}[b]{0.4\columnwidth}
        \centering
        \includegraphics[width=\columnwidth]{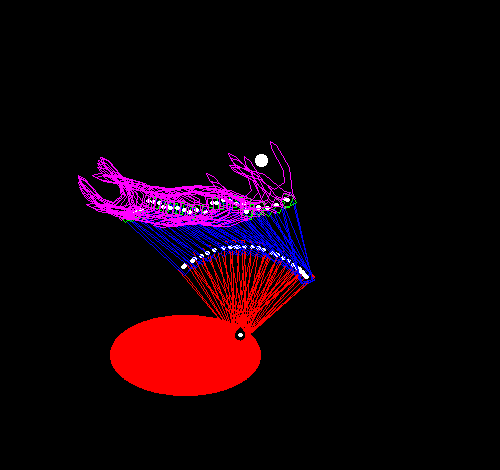}
        \caption{}
    \end{subfigure}

    \caption{
        \textit{Overview example}. (a) A simulated agent with a red elliptical body and an arm with 3 joints moving in the horizontal plane. Also shown is a white ball that the agent intends to reach. (b) The 3D manifold of the actuator parameters (joint angles), obtained by a random projection of the three joint angles into a 1000 dimensional space and then applying Isomap on the result. Colour coding is as follows: (increasing values) Red: $\theta_1$, Green: $\theta_2$, and Blue: $\theta_3$. (c) Same manifold after learning the contours of its own body; grey points are self-intersecting poses of the arm, which are removed from the model. (d) Trajectory followed by the agent to swat the white ball. 
    }

    \label{fig:overview-example}
  \end{center}
\end{figure}

\begin{figure}[ht!]
    \centering
    \includegraphics[page=3, clip, trim = 1.5cm 3cm 1.5cm 6cm, width=\textwidth]{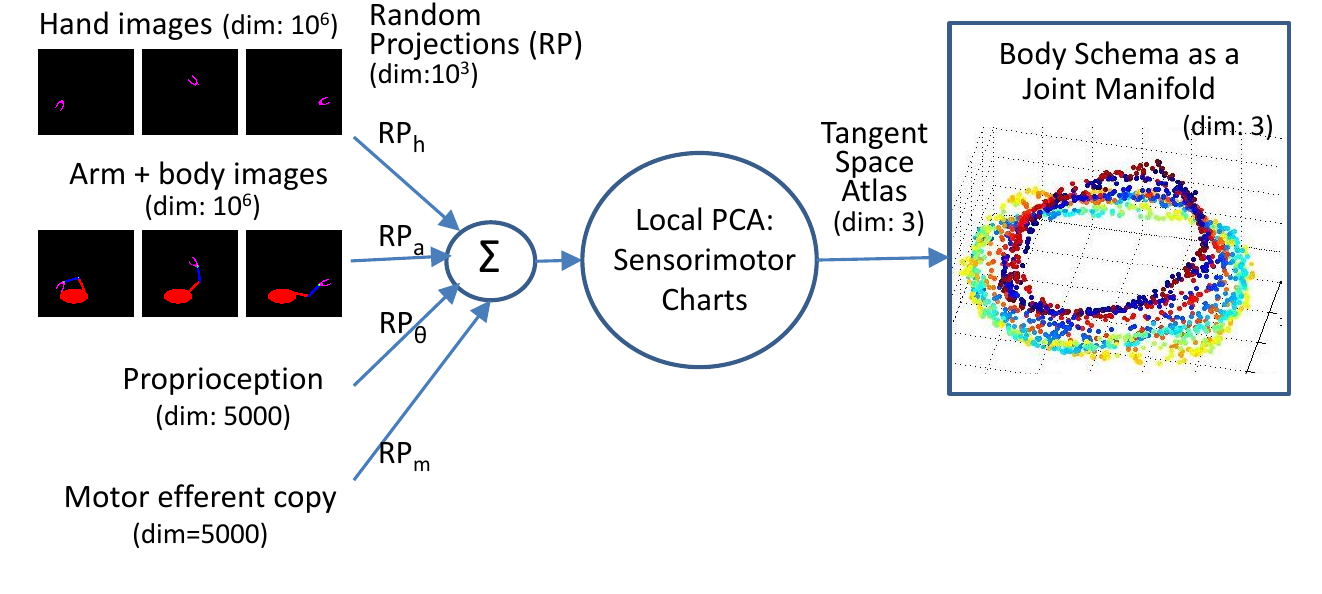}
  
    \caption{
        \textit{Body schema as a joint manifold}: 
        Multiple inputs  --- two visual feedbacks attending to only the hand and the whole arm, proprioception, and motor efferent copy, each signal with differing dimensions $D_i$ --- are all projected onto the same number of neurons, (say $D_{rp}$) by a set of random synaptic weights (each RP matrix is $D_i \times D_{rp}$,  with $D_{rp} << D_i$).  The joint manifold on which each of these data lies has the intrinsic dimension equalling the degrees of freedom $d$ of the system (for the right arm up to the wrist, $d=3$; other values of $D_i$ are suggestive only).  Though this \textit{joint manifold} is non-linear, local tangent spaces are computed all over the manifold, using neuronal computation for PCA.  This joint manifold serves to integrate the different sensory and motor modalities for moving the arm. To plan an action, the model is queried with the desired pose (image) of the hand (or arm). Interpolating on the relevant tangent space chart of the manifold (e.g. k-nearest neighbours), then returns the motor signals.  It also generates a set of expectations for proprioception, and visual feedback, without which rapid task feedback would not be possible. Several such  $d$-dimensional manifolds, with differing emphases (weights) for each of the modalities, may be maintained in motor and sensory areas.
    }
    \label{fig:joint-manifold}
\end{figure} 

We assume that during babbling, the system has explored most of the area in its reachable space. One aspect of the infant repertoire that has not received sufficient attention is how rapidly she learns to avoid motions that hit its own body. The motion that brings the hand to the mouth is often well developed before parturition, but knowledge of its body continues to mature.  Anecdotal evidence suggests that many parents are mortified when their newborn repeatedly pokes her own eyes. However, this behaviour disappears within a couple of days. 

Thus, in our little body+limb system, many of the possible motions self-intersect with its body, so like the newborn, our system quickly learns to rule out this part of  its motion space, while paying particular attention to poses on the boundary, where it is touching its own body, carrying food to the mouth or items in front of its eyes etc. This region of self-intersection is indicated in Figure~\ref{fig:overview-example} with gray points.

One question that arises in any visuomotor task is what part of the arm to attend to.  For most tasks, the hand is the part that will be interacting with the object; however in certain tasks, (e.g. in Figure~\ref{fig:3dof_path_w_obs} narrow gap), the whole body can be important.  We suggest that the system
maintains multiple representations so that it can adapt to different types of reasoning.

In this work, we focus on the following manifolds:
\begin{itemize}
    \item {
        \textit{Hand only}: Here the manifold is constructed from images that have only the hand, and ignore the rest of the arm. 
    }
    \item {
        \textit{Full arm} : Images show the full arm and the body.
    }
    \item {
        \textit{Actuator parameters} : Based on the joint angle data. 
    }
    \item {
        \textit{Combined maps - joint manifolds}: Combines these input data, by computing the local neighbourhood based on either (a) random projections of the different datasets to the same uniform lower dimension~\citep{ganguli-sompolinsky-12_compressed-sensing-sparsity-dimensionality}, or (b) by combining the metrics used in the three spaces into a single metric.  We can give different weightings to reflect different emphasis on the three types of input.
    }
\end {itemize}
Figures \ref{fig:hand-images} and \ref{fig:various_manifolds} illustrate these different manifolds. 

\begin{figure}[ht!]
    \centering
    \includegraphics[width=0.48\columnwidth]{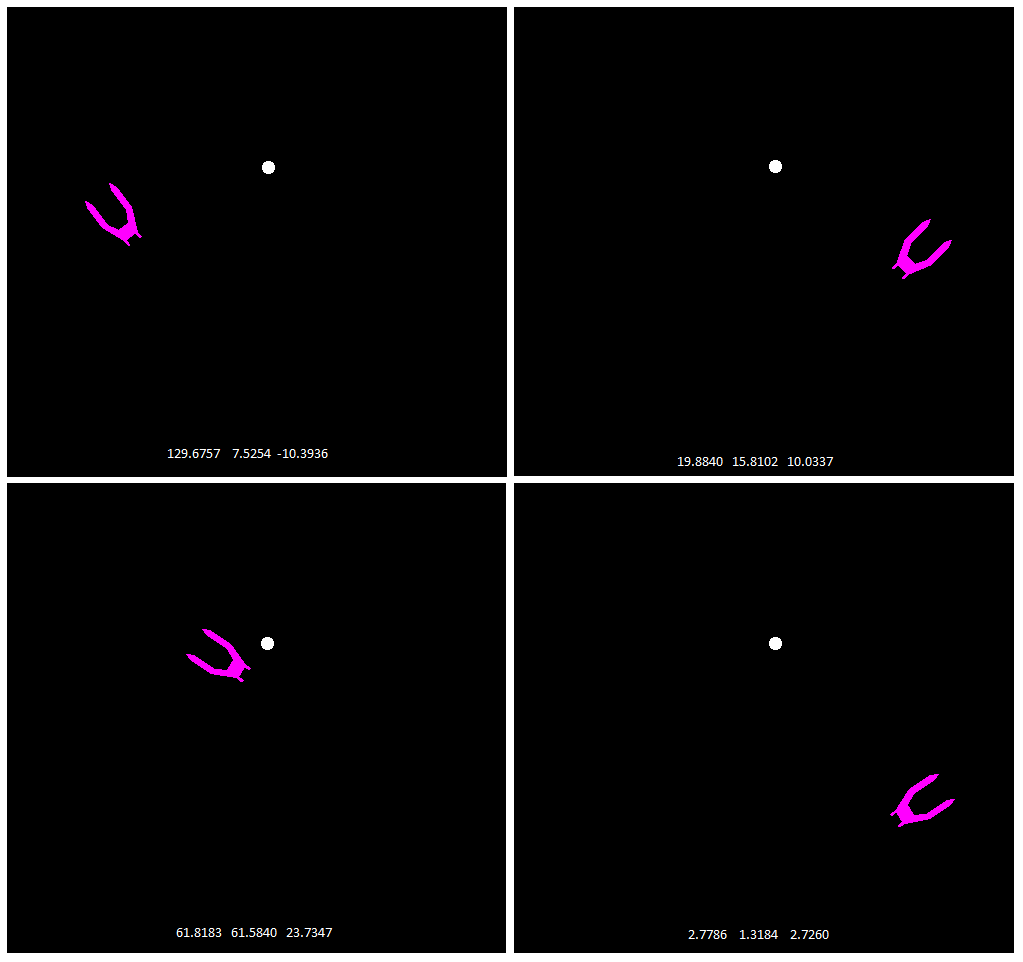}
    \includegraphics[width=0.48\columnwidth]{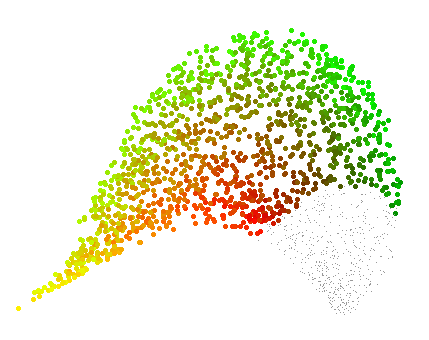} 
    \caption{
        \textit{Image space for just the hand, and its manifold}: a) Four hand images along with a ball. b) Resulting manifold with a target dimension $d$=2.  We observe that each image is very high-dimensional (e.g. $D=10^8$ in the retina, and $D=10^6$ in the optic nerve), but randomly assigning colours to the $D$ pixels will almost never create an image of a hand - so the hand-images subspace is a vanishingly small subspace.  If the hand is attached to an arm with $d$ joint angles, then it needs only $d$ parameters to fully specify its pose in space (in these images, $d=3$, corresponding to the angles shown in each image). Thus, the hand image subspace has an intrinsic dimensionality of $d$, which represents the number of ways in which the image can be altered while remaining an image of this hand.  In this particular manifold, almost the entire variability is explained by the first two dimensions (roughly, x and y), so the manifold shown here is for target dimension $d=2$. 
    }
    \label{fig:hand-images}
\end{figure} 
\clearpage
\begin{figure}[t]
  \begin{center}
    \begin{subfigure}[b]{0.42\columnwidth}
       \centering
        \includegraphics[width=\columnwidth]{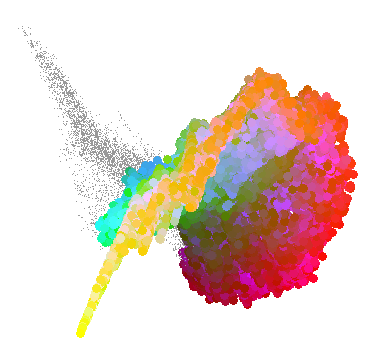}
        \caption{}
    \end{subfigure}
    \begin{subfigure}[b]{0.42\columnwidth}
       \centering
        \includegraphics[width=\columnwidth]{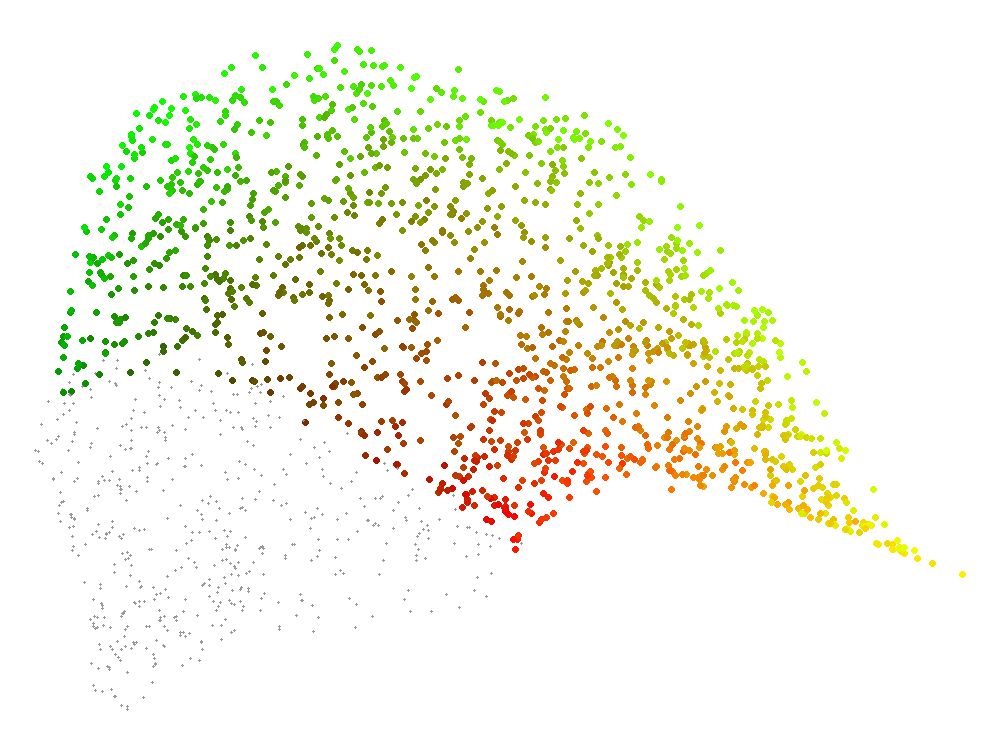}
        \caption{}
    \end{subfigure}
    \\
    \begin{subfigure}[b]{0.42\columnwidth}
       \centering
        \includegraphics[width=\columnwidth]{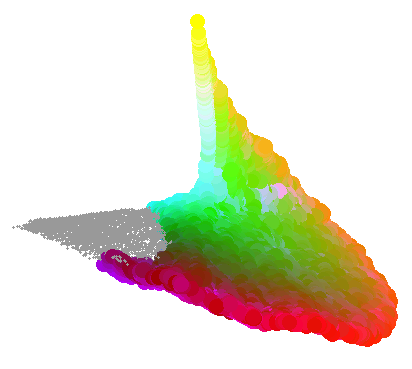}
        \caption{}
    \end{subfigure}
    \begin{subfigure}[b]{0.42\columnwidth}
       \centering
        \includegraphics[width=\columnwidth]{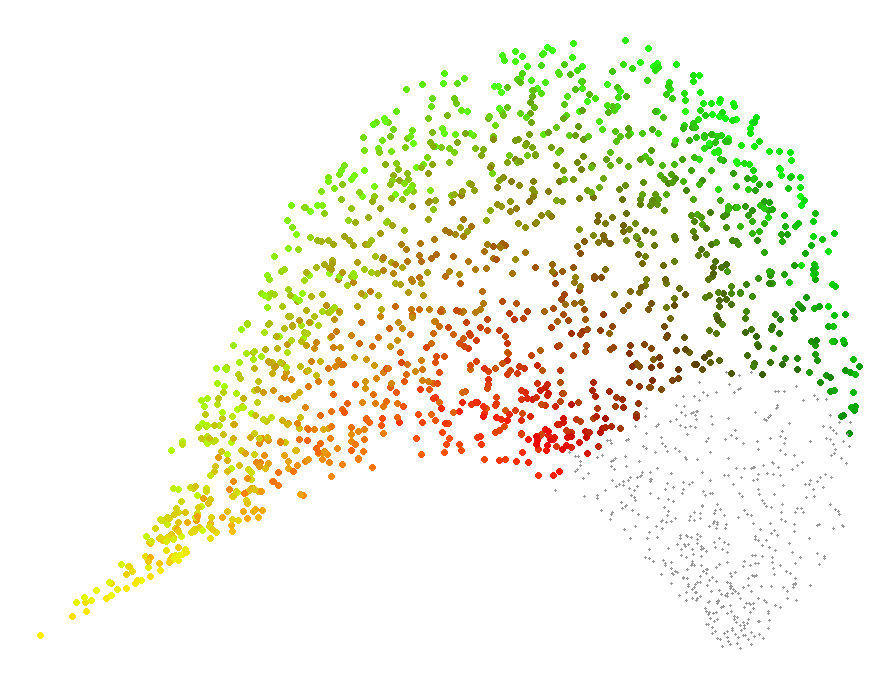}
        \caption{}
    \end{subfigure}
    \\
    \begin{subfigure}[b]{0.42\columnwidth}
       \centering
        \includegraphics[width=\columnwidth]{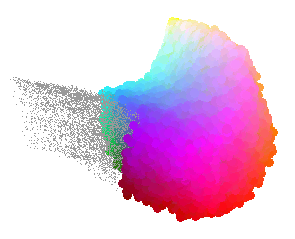}
        \caption{}
    \end{subfigure}
    \begin{subfigure}[b]{0.42\columnwidth}
       \centering
        \includegraphics[width=\columnwidth]{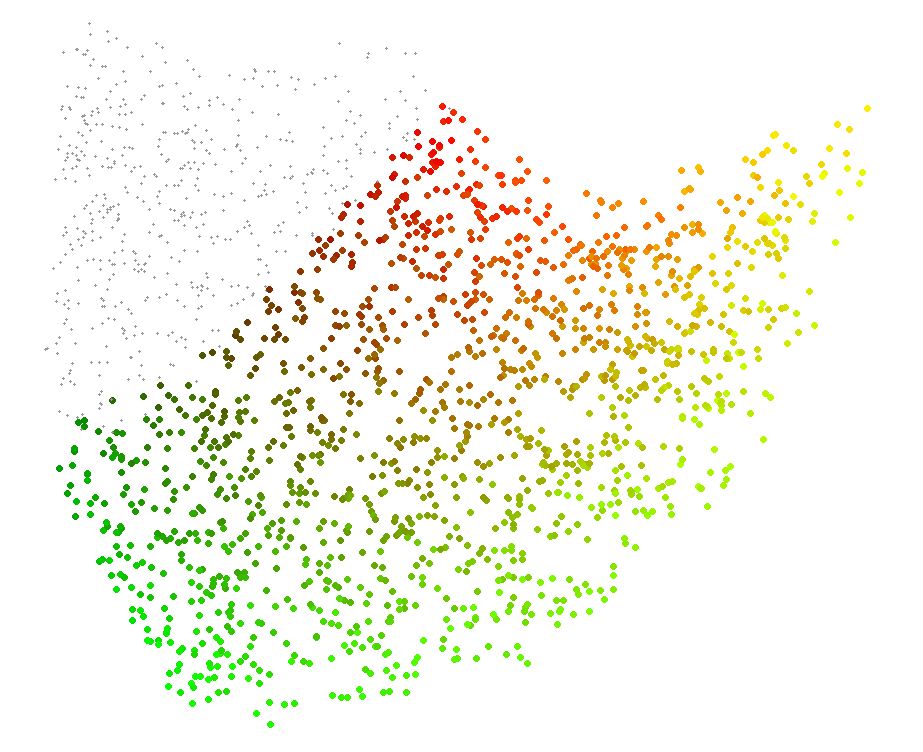}
        \caption{}
    \end{subfigure}
    \caption{
        \textit{Manifolds from full-arm and hand images, and joint manifold}. 
        The first column shows 3-manifolds (target dimension $d=3$) and the second column shows corresponding 2-manifolds (target dimension $d$=2) of images: (a, b) full-arm, (c, d) just the hand and (e, f) joint manifold. Joint manifold in (e, f) combines full arm + hand + actuator parameters. The 3-manifold for just the angle space was shown earlier in Figure~\ref{fig:overview-example}. The grey regions correspond to the poses of self-intersection.
    }
    \label{fig:various_manifolds}
  \end{center}
\end{figure} 
\clearpage

The manifolds shown here have been computed using the Isomap algorithm~\citep{tenenbaum2000global}. The system uses the Hausdorff distance~\citep{huttenlocher1993comparing} to compute the distance between any two images.  Given two sets of points $A$ and $B$, Hausdorff distance $d_H(A, B)$ is defined as 
\[
d_H(A, B) = \max \left\{ \adjustlimits \sup_{a \in A} \inf_{b \in B} d_E(a, b), \adjustlimits \sup_{b \in B} \inf_{a \in A} d_E(a, b) \right\},
\]
where $d_E(a, b)$ is the Euclidean distance between $a$ and $b$.

Since iterating over all the points in the images is expensive, here we approximate this distance over a set of high-contrast points~\citep{shi-tomasi_1994_cvpr_good-features-to-track}. For the angles, we use the $L_2$ (Euclidean) metric.  The Isomap algorithm then computes the nearest neighbour for every point using this metric and constructs a graph by joining these neighbourhoods.  It then finds an approximate geodesic distance between any two points and maps the entire set of distances to a lower dimension using Multi-Dimensional Scaling (MDS). Though other manifold learning approaches also work and generate similar maps, we find Isomap less deforming at the corners, possibly owing to its preserving the geodesic distances.  As mentioned earlier, neural structures may be simulating this construct via a series of PCA computations~\citep{oja-82_simplified-neuron-model-as-principal-component-analyzer}. The results in Figure~\ref{fig:overview-example} and ~\ref{fig:various_manifolds} are based on 20,000 samples, of which a little more than a fourth are rejected since they self-intersect with the body.  Even with fewer points (say 2,000), the manifold is well-defined, but it looks more patchy.

We note that as the infant is throwing its arms about, the efferent motor commands, the afferent proprioception, and the resulting image of hand or arm, are each of them extremely high dimensional.  Let us consider a system based on just the hand so that each input image shows just the hand, and nothing else.  A set of such images are shown in Figure~\ref{fig:hand-images}.  Each image is 500$\times$470 pixels = 235,000 dimensions.  However, if we randomly assign colours to each pixel, the resulting image will almost never be an image of this agent's hand.  At the same time, the image space is continuous, since arbitrarily small actuations result in arbitrarily small changes in the image.  Also, to change the image so that it remains within the subspace of robot images, at any point, we can do it only along $d$ basis directions -- the tangent space at this point. Thus, these high-dimensional inputs lie along a subspace whose local dimensionality everywhere is $d$.  However, these local spaces would typically not be parallel to each other, so that the overall manifold is curved (non-linear). 

Furthermore, we observe that if each of these inputs projects separately to different $d$-manifolds, then there would exist a one-to-one correspondence between the corresponding points on each of these manifolds and their neighbourhoods.  This is, in fact, the definition of \textit{homeomorphism}, the key requirement for data to constitute a manifold. Thus, these separate inputs are homeomorphic to each other, and they can also be mapped onto a single ``joint'' manifold~\citep{davenport-hegde-10_joint-manifolds-for-data-fusion}.   Also, the coordinates along each manifold specify a unique pose of the arm, so these coordinates constitute \textit{generalized coordinates}~\citep{mussa-95_geometrical-principles-in-motor-control}, which can be used not only in kinematics as demonstrated here but also in dynamics. 

\subsection{Joint Manifold via Metric Combination}
If our manifold discovery process can work with local distances, it is possible to obtain a projection for multiple datasets by combining their metrics.  Thus, a metric combining our three datasets, as shown in
fig:~\ref{fig:various_manifolds}c can be obtained as $d_{joint} = \alpha d_{full-arm} + \beta d_{hand} + \gamma d_{angle}$.  Results for $\alpha=\beta=\gamma=1$ are shown in Figure~\ref{fig:various_manifolds}(c).  Another approach to constructing a joint manifold would be to use random projections to map the data directly to the same intermediate dimension.

\subsection{Using the Joint Manifold as Body Schema}
Through the formulation of this joint manifold, the infant learns a combined map for visual poses and actuator/proprioception parameters.  Since the graph underlying the manifold already connects any two nodes via local neighbours, it can also plan paths on the graph that would \textit{reach} a desired object. Over time, it can also learn to overlay the location of objects in this space with the arm poses and try to allow it to reach around obstacles. Eventually, some salient parts of this sensorimotor space (e.g. the nose) can begin to acquire symbolic connotations as shown in Figure~\ref{fig:bs_map}. 

\begin{figure}[h!]
    \centering
    \begin{subfigure}[b]{0.45\columnwidth}
        \centering
        \includegraphics[width=\columnwidth]{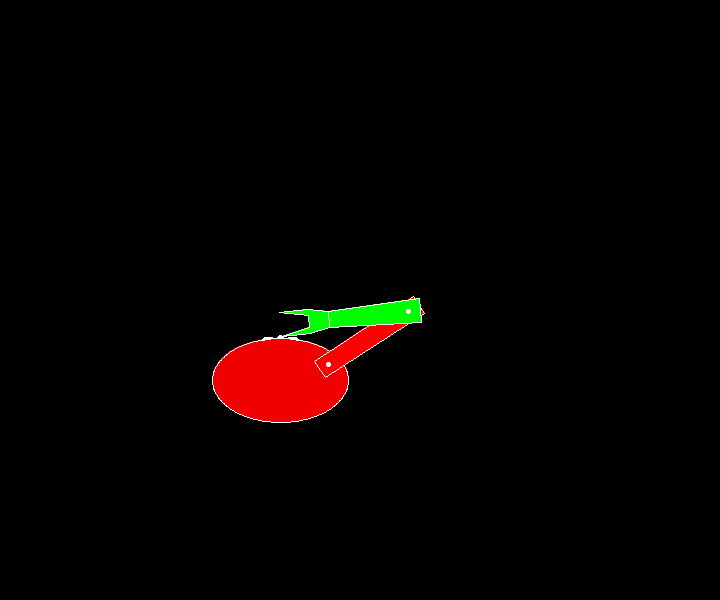}
        \caption{}
    \end{subfigure}
    \begin{subfigure}[b]{0.45\columnwidth}
        \centering
        \includegraphics[width=\columnwidth]{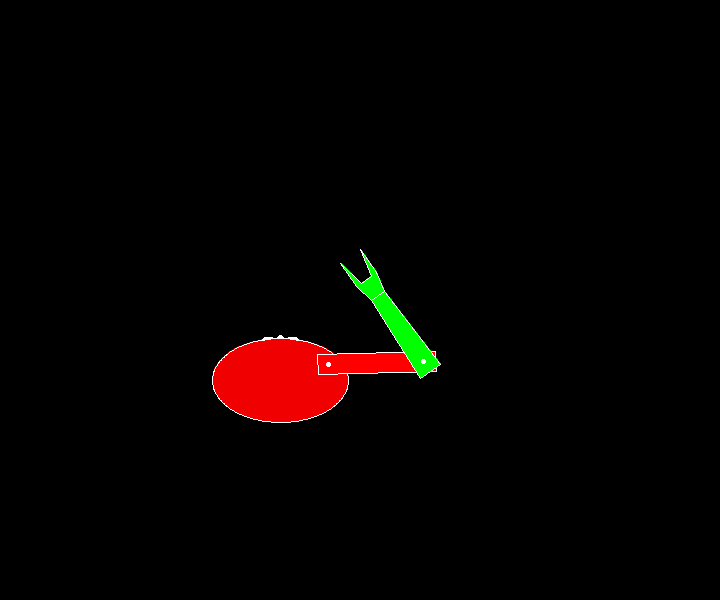}
        \caption{}
    \end{subfigure}
    
    \begin{subfigure}[b]{0.92\columnwidth}
        \centering
        \includegraphics[width=\columnwidth]{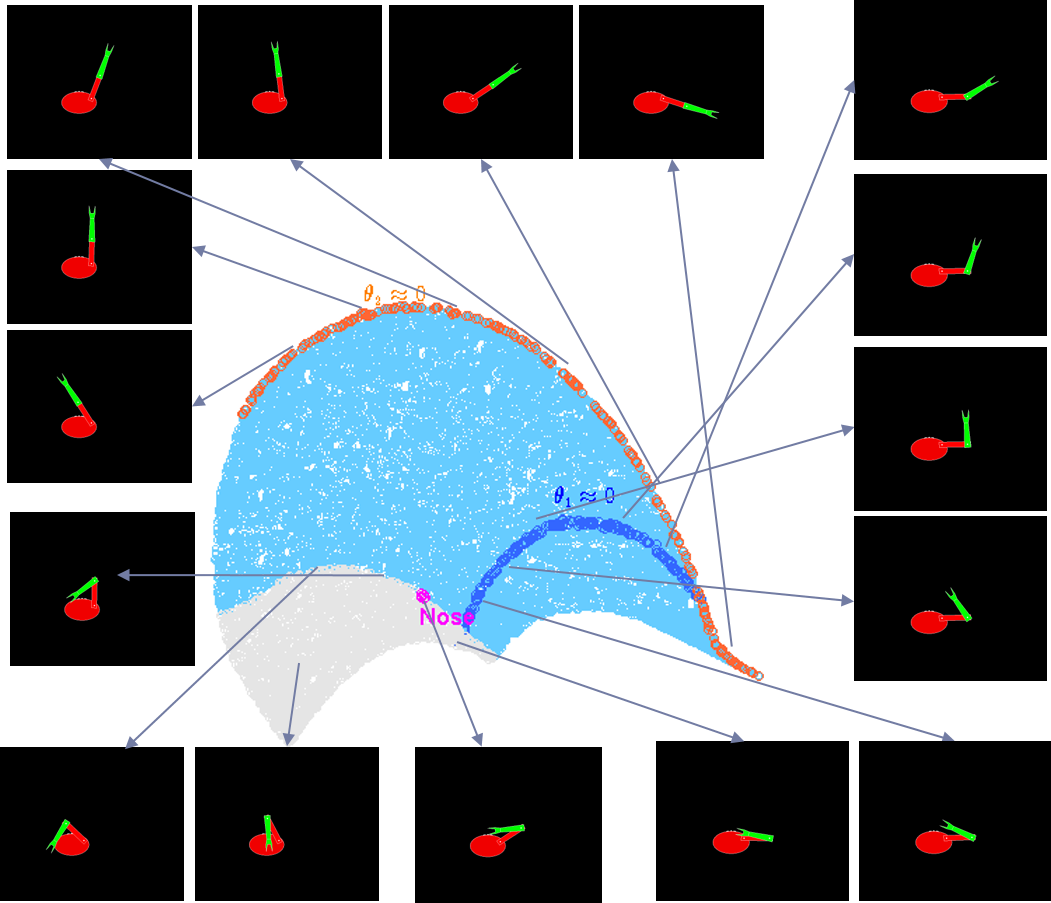}
        \caption{}
    \end{subfigure}
    
    \caption{
        \textit{Mapping of body parts on the manifold.}
        (a, b) Two images of an agent with a 2-DOF arm and an elliptical body with its nose and cheeks marked with small white regions on the body.
        (c) The 2-D manifold of this agent's images. Here, we show that the points on the manifold correspond to the different poses of the agent. The light blue region corresponds to poses where the arm does not intersect with the body and the grey region corresponds to poses where the arm intersects with the body. The boundary between these two regions makes up the boundary of the agent's body. Points corresponding to poses with joint angle $\theta_1 \approx 0$ are marked in blue, poses with $\theta_2 \approx 0$ are marked in orange and the points corresponding to the poses where the hand touches the nose are marked in magenta. 
    }
    \label{fig:bs_map}
\end{figure}
\clearpage

\section{Learning to Reach}
Having constructed the manifold as described above, we now illustrate how one may learn to reach using this map. Here we assume that the body dimensions are changing relatively slowly compared to the learning period, so the body is considered fixed.  We shall later elaborate on the question of transferring this knowledge to a growing
body. 

Reaching objects involves planning motions from the current pose to a target pose in which the hand can touch or hold the object. For purposes of motion planning, we consider a neighbourhood graph embedded on the body schema manifold. This graph is the \textit{visual roadmap} (VRM) discussed in the earlier chapters.

Let us say the agent's arm is in a given pose, and the agent wishes to move it to a new pose, which may be known as a desired hand image (it could also be a full arm or a set of angles).  We now find its nearest neighbours on the entire set of hand images and interpolate in the local tangent space to construct a neighbourhood for this desired reach position.  Then we can solve for a geodesic (shortest curve along the manifold) on any of the representations to obtain a path.  The geodesic is approximated as the shortest path on the graph for that particular manifold.  Interestingly, the paths are different for differing encodings (Figure~\ref{fig:reach_paths}), thus the hand-only manifold results in a motion that is more rectilinear in the task space, whereas the angle-based manifold results in a more convoluted path~\citep{danziger-mussa-12_influence-of-visual-motion-on-motor-learning}.

For an infant learning to reach, the paths will initially result in a swatting motion because her velocity control is poor, and also because the sampling of the motion space is not very good. Swatting gets better as the agent gains more and more visual experience with its own body. This is illustrated in Figure~\ref{fig:swatting-improves}.

When obstacles are introduced as in Figure~\ref{fig:3dof_path_w_obs}, we identify the full arm poses that overlap the obstacle in the image space, and remove them from the graph.  The remaining nodes in the free-space are used to compute the path. 

\clearpage
\begin{figure}[t]
 \begin{center}
    \begin{subfigure}[b]{0.36\columnwidth}
        \centering
        \includegraphics[width=\columnwidth]{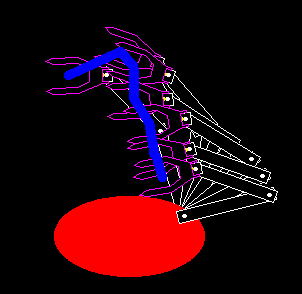}
        \caption{}
    \end{subfigure}
    \begin{subfigure}[b]{0.36\columnwidth}
        \centering
        \includegraphics[width=\columnwidth]{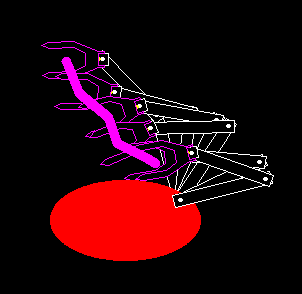}
        \caption{}
    \end{subfigure}
    
    \begin{subfigure}[b]{0.36\columnwidth}
        \centering
        \includegraphics[width=\columnwidth]{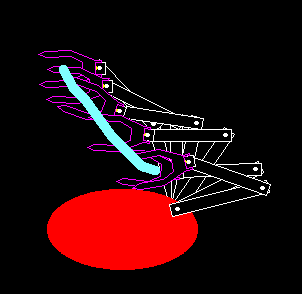}
        \caption{}
    \end{subfigure}
    \hspace{0.1mm}
    \begin{subfigure}[b]{0.36\columnwidth}
        \centering
        \includegraphics[width=\columnwidth]{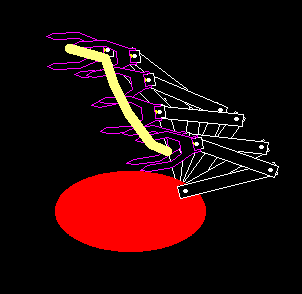}
        \caption{}
    \end{subfigure}
    
    \begin{subfigure}[b]{\columnwidth}
        \centering
        \includegraphics[width=0.72\columnwidth]{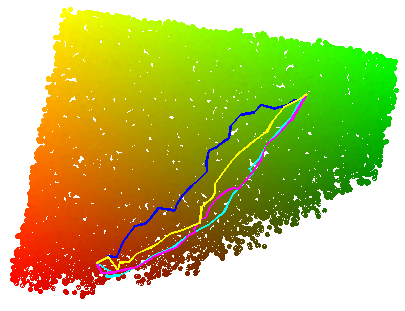}
        \caption{}
    \end{subfigure}
  
    \caption{\textit{Paths computed on different manifolds}:
        Paths between same source and destination, computed on different manifolds:
        (a) blue - actuator parameter manifold
        (b) magenta - manifold of corners of just the hand; 
        (c) cyan - manifold of corners of the full arm; 
        (d) yellow - joint manifold of angles + hand corners + full arm corners.
        (e) Paths superposed on a 2-manifold of actuator parameters. 
    }
    \label{fig:reach_paths}
  \end{center}
\end{figure} 

\clearpage

\begin{figure}[t]
  \begin{center}
    \begin{subfigure}[b]{0.42\columnwidth}
        \centering
        \includegraphics[width=\columnwidth]{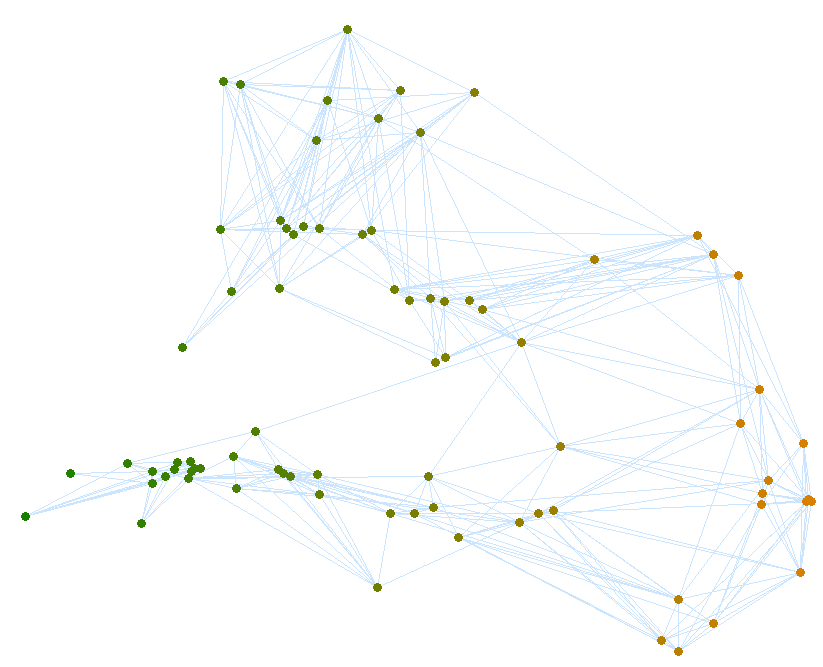}
        \caption{}
    \end{subfigure}
    \begin{subfigure}[b]{0.42\columnwidth}
        \centering
        \includegraphics[width=\columnwidth]{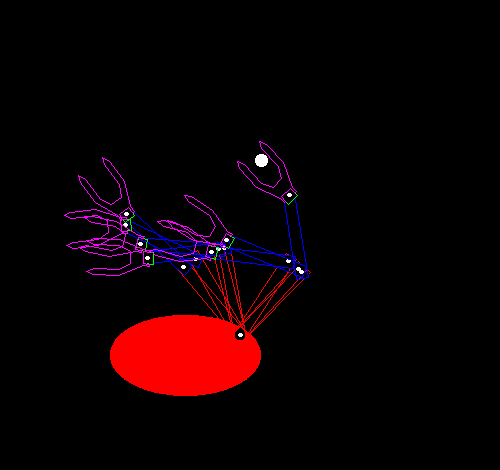}
        \caption{}
    \end{subfigure}
    
    \begin{subfigure}[b]{0.42\columnwidth}
        \centering
        \includegraphics[width=\columnwidth]{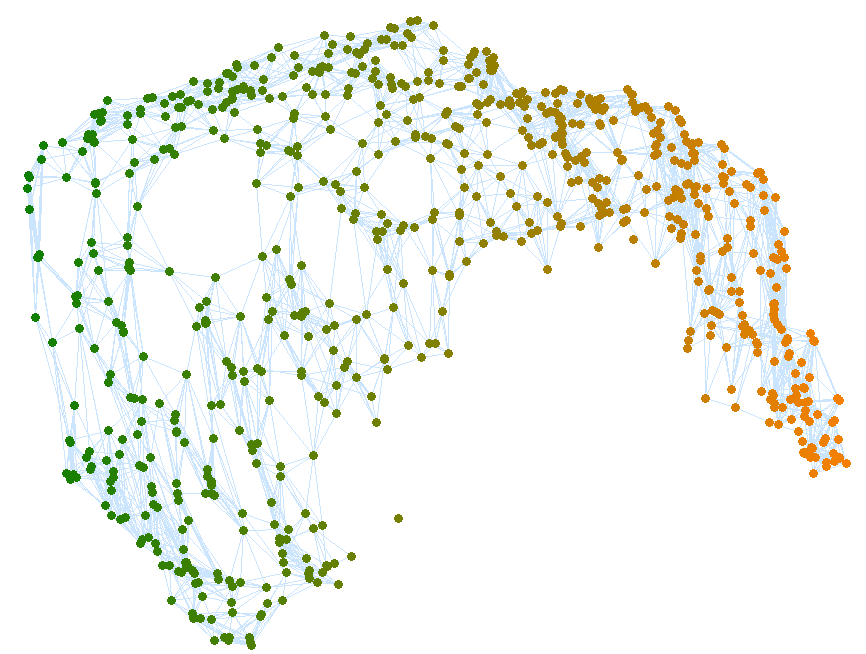}
        \caption{}
    \end{subfigure}
    \begin{subfigure}[b]{0.42\columnwidth}
        \centering
        \includegraphics[width=\columnwidth]{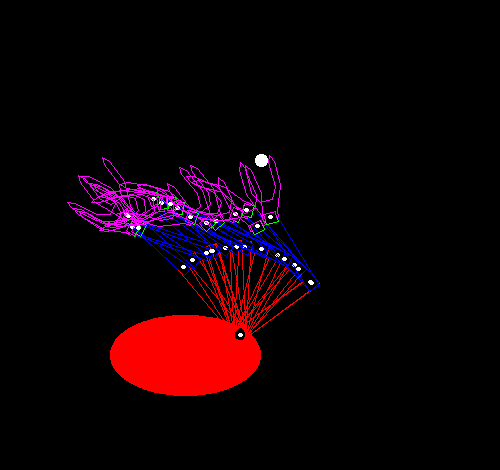}
        \caption{}
    \end{subfigure}
    
    \begin{subfigure}[b]{0.42\columnwidth}
        \centering
        \includegraphics[width=\columnwidth]{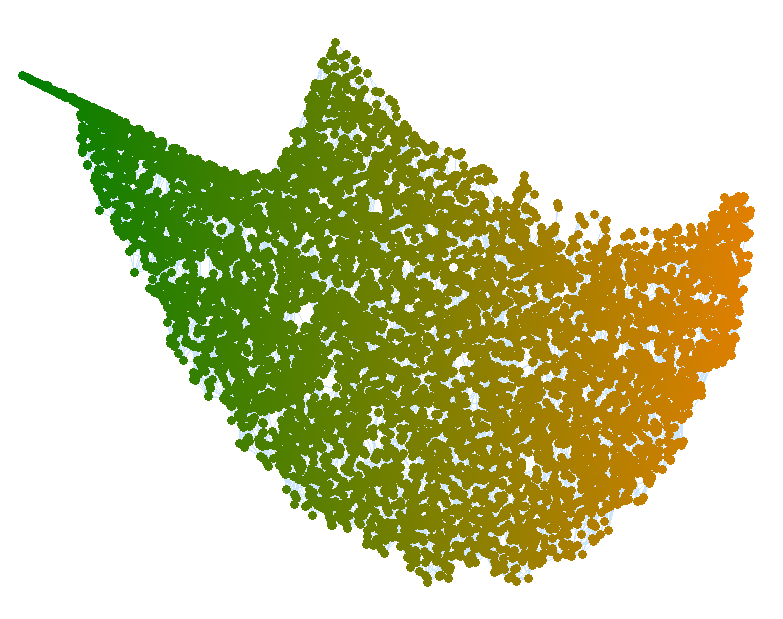}
        \caption{}
    \end{subfigure}
    \begin{subfigure}[b]{0.42\columnwidth}
        \centering
        \includegraphics[width=\columnwidth]{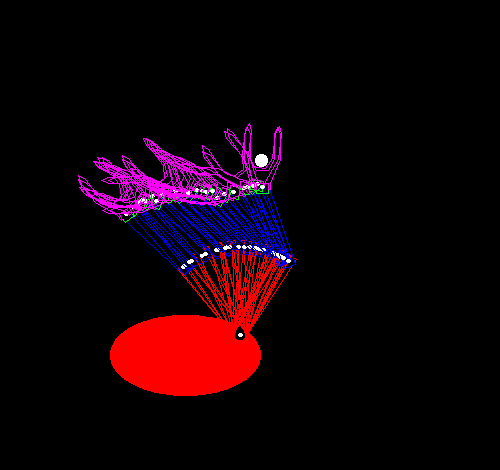}
        \caption{}
    \end{subfigure}
    
    \caption{
        {\em Swatting getting better with experience}.
        Column 1: Graphs used for motion planning. Number of nodes = 100, 1000, 10000 in the graphs of (a, c, e) respectively. Each node in the graph represents a random pose of the agent that it has visually observed; so, more nodes implies more visual experience. 
        Column 2: Trajectories followed by the agent in the workspace to swat the object. Figures (b, d, f) are the paths computed using the graphs of (a, c, e) respectively. It can be seen that the motion gets smoother with more experience (i.e., more nodes in the graph).
    }
    \label{fig:swatting-improves}
  \end{center}
\end{figure}

\clearpage

\begin{figure}[h!]
 \begin{center}
    \begin{subfigure}[b]{0.48\columnwidth}
        \includegraphics[width=\columnwidth]{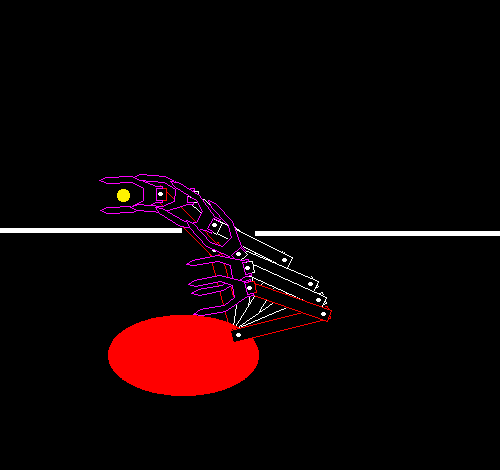}
        \caption{}
    \end{subfigure}
    \begin{subfigure}[b]{0.48\columnwidth}
        \includegraphics[width=\columnwidth]{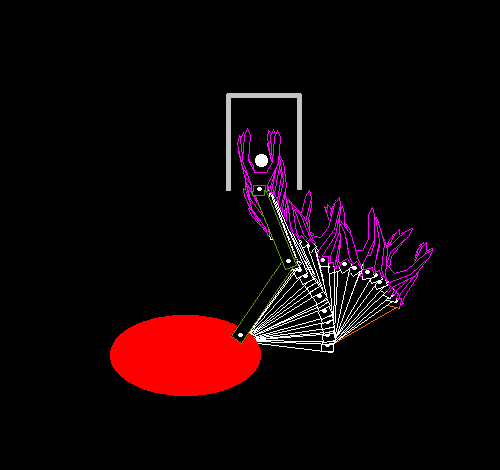}
        \caption{}
    \end{subfigure}
    \caption{
        \textit{Obstacle avoidance and object reaching}: (a) Path followed to reach an object thorough a small gap. (b) Path followed to reach an object inside a box. These paths were computed on the graphs embedded on the full-arm image manifold.
    } 
    \label{fig:3dof_path_w_obs}
  \end{center}
\end{figure}

\section{Growing Body}
\label{sec:growing}
So far we have considered the body dimensions to be fixed.  But how does the knowledge learned with a smaller body size translate to the situation where the body has grown larger?  We now consider such changes in the body schema that was learned earlier.  We note that since the body growth is fairly gradual, it should be possible to incrementally update the body schema at regular intervals, as the agent gains more and more experience with its own body and with its environment.

\begin{figure}[ht!]
\centering
    \begin{subfigure}[b]{0.32\columnwidth}
        \includegraphics[width=\columnwidth]{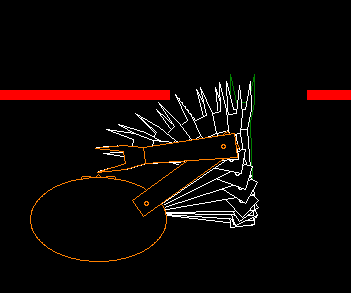}
        \caption{}
    \end{subfigure}
    \begin{subfigure}[b]{0.32\columnwidth}
        \includegraphics[width=\columnwidth]{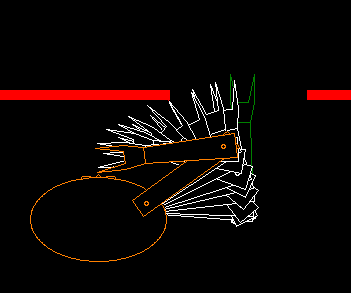}
        \caption{}
    \end{subfigure}
    \begin{subfigure}[b]{0.32\columnwidth}
        \includegraphics[width=\columnwidth]{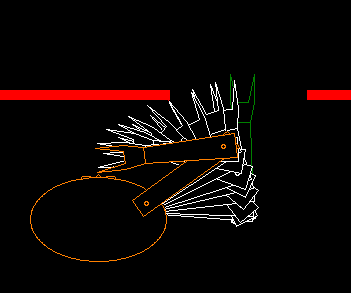}
        \caption{}
    \end{subfigure}
    
    \begin{subfigure}[b]{0.32\columnwidth}
        \includegraphics[width=\columnwidth]{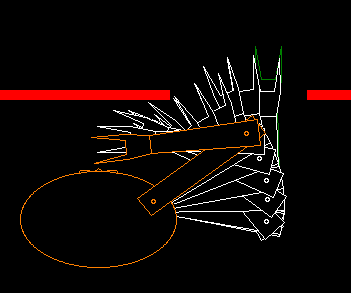}
        \caption{}
    \end{subfigure}
    \begin{subfigure}[b]{0.32\columnwidth}
        \includegraphics[width=\columnwidth]{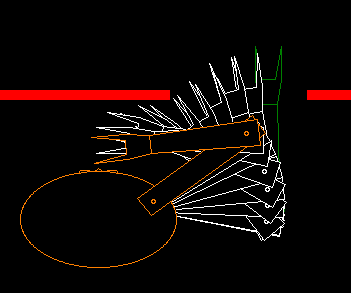}
        \caption{}
    \end{subfigure}
    \begin{subfigure}[b]{0.32\columnwidth}
        \includegraphics[width=\columnwidth]{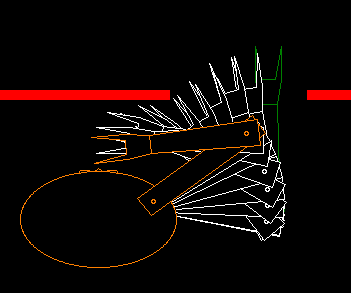}
        \caption{}
    \end{subfigure}
    
    \begin{subfigure}[b]{0.24\columnwidth}
        \includegraphics[width=\columnwidth]{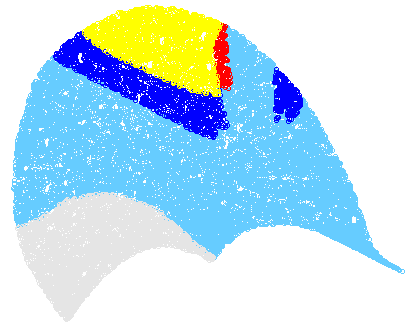}
        \caption{}
    \end{subfigure}
    \begin{subfigure}[b]{0.20\columnwidth}
        \includegraphics[width=\columnwidth]{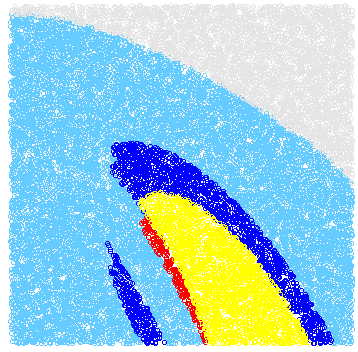}
        \caption{}
    \end{subfigure}
    \begin{subfigure}[b]{0.24\columnwidth}
        \includegraphics[width=\columnwidth]{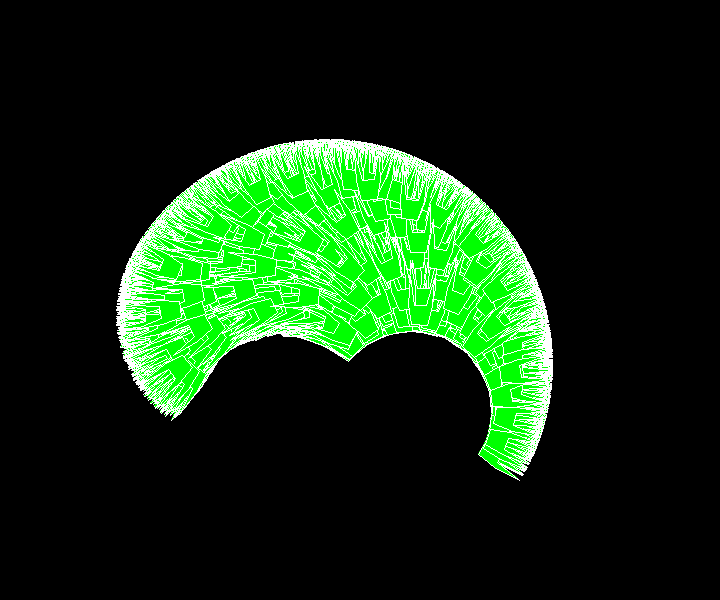}
        \caption{}
    \end{subfigure}
    \begin{subfigure}[b]{0.24\columnwidth}
        \includegraphics[width=\columnwidth]{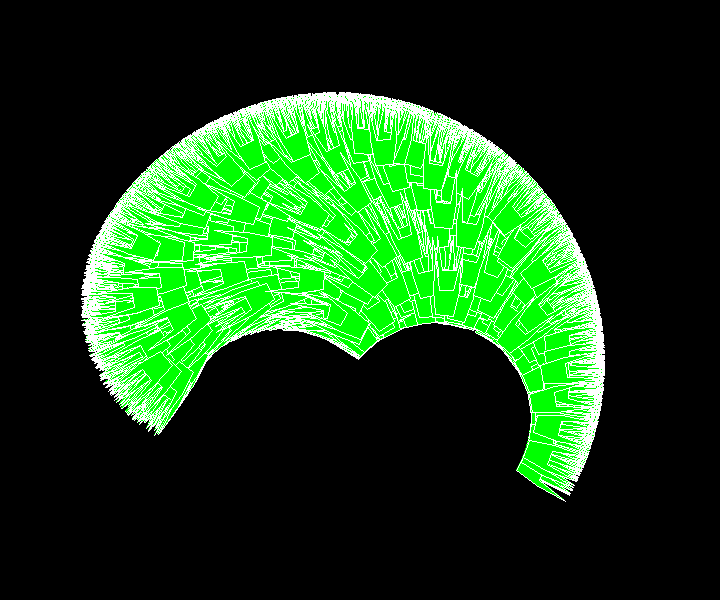}%
        \caption{}
    \end{subfigure}

    \caption{
        \textit{Growing robot}:
        Row 1: Path between a pair of poses of a robot in its infancy (time $t_1$) using  (a) joint angle distance (b) track distance on full arm images (c) track distance on just the gripper images. Track distance between two configurations is computed as the $L_2$ distance between the vectors obtained by concatenating all the corner point coordinates of each link. Here the arm is moving from a random pose to reach the mouth area through a window (gap in the red bar) in its workspace.
        Row 2: Same paths for the same robot after some growth (time $t_2$).
        Row 3: (g, h) Obstacle region marked on the manifold and in the angle space: yellow - common obstacle area at the two ages; red - obstacle for just the infant robot; blue - obstacle for just the bigger robot. (i, j) Gripper workspace at infancy ($t_1$) and after some growth ($t_2$).
    }
    \label{fig:growing-body}
\end{figure}

We consider the task of finding a path such that the agent’s hand would reach the mouth area from a random pose. We use the body schema learned at an earlier age, $t_1$ (Figure~\ref{fig:growing-body}, row 1), and consider using it to execute such a path with a grown body at a later age  $t_2$ (Figure~\ref{fig:growing-body}, row 2). Note that we do not reconstruct the body schema (roadmap graph) for the dimensions of  $t_2$; we use the earlier roadmap and measure how effective it is by looking at the collisions it may cause. From $t_1$ to $t_2$,  arm width, arm length, and body size were grown non-uniformly, by factors of 1.1, 1.2, and 1.15 respectively, to reflect the greater growth in limb lengths.

We investigated the adaptation to larger changes in body dimensions as follows. We considered the agent at three different ages: $t_1, t_2 \text{ and } t_3$. The same growth ratios were used from $t_2$ to $t_3$, as from $t_1$ to $t_2$. Based on the body schema (roadmap graph) at age $t_1$, we obtain the roadmaps $G', G'', \text{ and } G'''$ for ages $t_1, t2 \text{ and } t3$ by removing nodes where the images overlap with the obstacles shown (fig.~\ref{fig:growing-body} rows 1,2). In simulation we looked at the number of bad edges (edges between safe nodes, but lead to collision with obstacles when followed) in $G', G'' \text{ and } G'''$ and noticed that $G' \text{ and } G''$ did not have any bad edges while around 25\% of the edges in $G'''$ were bad. This suggests that the body schema computed at age $t_1$ could be used without much trouble at age $t_2$ , but needs to be updated before $t_3$. Such updates involve updating the local neighbourhood graph to account for collisions; this can be done based on the bad edges (unexpected, and therefore informative) that lead to collisions. 

\section{Conclusion}
\label{sec:conclusion}
In this chapter, we considered the problem of limb motion, and attempted to construct a visuomotor map that constitutes an implicit model of the organism's ``personal space''. The approach presented improves on earlier models by constructing a map directly from the image space onto actuator parameters, but even more significantly, the manifold provides a mechanism for feed-forward (open loop) modelling of actuator parameters from visual poses and also for imagining the arm pose for a putative actuator configuration. Such a sensorimotor map then constitutes a tight coupling of the image, proprioception, and motor actuation spaces.

Once such a manifold has been constructed, it is possible to imbue some parts of it with other relations (e.g. the part of the body that is the “mouth”, “eyes”, “nose”, ‘right”, “left” etc.). Another capability demonstrated is the ability to mark stable objects in the workspace as occupying a given part of the manifold. For example, if the limb is being used to type, then letters on a keyboard that are constantly being pressed in the same pose would map to those parts of each manifold, and would also be associated with the appropriate tactile feedback. This clearly suggests how such a model can make a start towards symbolic representations that are embedded on this personal space.

Such a map also generates a very compact representation in terms of a set of $d$ parameters (equal to the number of degrees of freedom). Thus, this representation constitutes a generalized coordinate --- any assignment of values to it specifies a unique pose of the arm. However, the process by which these coordinates are obtained is unstable, and hence these may not be as useful. But, they could still be there perhaps as an implicit or ``subconscious'' aspect of our computational infant. 

Also, the model is constantly being updated --- though this is computationally problematic with manifold learning approaches like Isomap --- the tangent-space charts can easily accommodate changes and new data. Thus, as the organism grows, or it picks up a stick, the relations between the poses can change subtly. 

One question that we have not addressed here, but one that is of interest, involves redundancy in most biological limbs. Surprisingly, in general, this is not a problem for this approach, since given a set of full-arm images and joint angles, it is still possible to create the image manifold and also the joint manifold. If given a target pose in terms of the full-arm image, everything above would still work. However, for a pose given only in terms of the hand pose, there are clearly multiple solutions.

An important ramification of this process would be that such a representation, after continued daily use, may lead to many kinds of generalization leading to an internal model for space itself, as hinted at by many others, e.g.~\citep{thelen-00_motor-development-foundation-for-dev-psych}.  A measure of distance can now be constructed in terms of the change required to reach other poses.  By generalizing over similar experiences, one may form notions of space such as direction, dimensionality, curvature, a hierarchical scale structure, and many other aspects relevant to space based on the action-perception pairings. By identifying the configurations that reach various parts of the workspace, the system is also constructing a model for space itself. This is a powerful argument and a possibility that such a computational model can be used to demonstrate. 

What we have presented here is just an initial step. The basic idea of discovering patterns from the lower-dimensional mapping of visual images is actually more general, and can also be used for locomotion; as one moves, the images change in a certain co-varying manner and can be learned in a similar manner. We hope that this initial work will open up these and many other questions which can be addressed using these tools. 
\chapter{Head Motion Animation on Motion Manifolds}
\label{chap:head_motion}
Animation plays a key role in adding realism to any graphical environment. In this chapter, we present a technique for animating avatars in a virtual environment designed to support remote collaboration between distributed work teams in which users are represented by avatars. There will be long periods of time when the user is not actively controlling the avatar and working on his official task. We need to animate his avatar with a `working-at-desk' animation that should be non-looping and sufficiently random for a single avatar as well as between multiple avatars to appear realistic. We present a technique for generating multiple head motions using the gaze space images to control the avatar motion. Our technique can automatically generate long sequences of motion without any user intervention. We present results from synthetic data.

\section{Introduction}
Motion image sequences with $m \times m$ images and $t$ frames may be thought of as points in an $m^{2t}$ dimensional space. Animation is difficult, time-consuming and laborious because it involves a search in this extremely high-dimensional space.

However, only an infinitesimally small set of points in the high-dimensional animation space are interesting motions that graphics systems need to care about. In recent years, this insight has led to a number of approaches based on dimensionality reduction for specifying animations of complex characters.

For example, \citep{Popovic:2007} decomposes the action space into a set of bases which can then be used to compose the very high dimensional motions of a stick figure changing gaits and avoiding obstacles. \citep{Kovar:Moextract} provides a way to generate new motions by blending clips from a space of motions parameterized by a continuous parameter. A survey touching on these approaches may be found in \citep{Moeslund:2006}. We address this body of work in more detail in section \ref{sec:head_motion:prevwork} below.

In this chapter, we present an approach motivated by mapping the motion in a low-dimensional manifold. The two key innovations of this work are that (a) the motion is specified in terms of the avatar's gaze, defined in terms of objects in the environment (e.g. ``look from the screen to the mouse"), and (b) although the computation inherently assumes the presence of a low-dimensional manifold for the motion, we are agnostic to the actual low-dimensional parameters, which are never computed.

Consider the following problem. A number of users log in every day to a virtual environment designed to support collaboration between distributed work teams. An avatar at a virtual workstation represents a user who could be physically at an office, home or other location. Avatars are animated depending on the activity the user is engaged in, for example, text chat, navigation or a virtual meeting. However, there could be long periods of time when a user is working alone. To depict this, the user's avatar needs to be animated with a non-repetitive and sufficiently random `working-at-desk' animation. Further, to appear realistic and natural, each avatar should be animated with a unique motion. Thus we need to generate multiple long sequences of animation in which an avatar could be reading a document, typing, thinking and so on. 

In order to create meaningful head motions for the required animation, we use the gaze space of the avatar in the virtual environment. Gaze space images are created as first-person camera views when the avatar moves only his head. Here we use the term \textit{gaze} to indicate the direction the head is looking at; we do not actively consider eyeball motion. Figure~\ref{fig:avatar_example_images} shows some example images of an avatar and the corresponding images in the gaze space. 

\begin{figure}[t]
    \centering
    \includegraphics[width=0.118\textwidth]{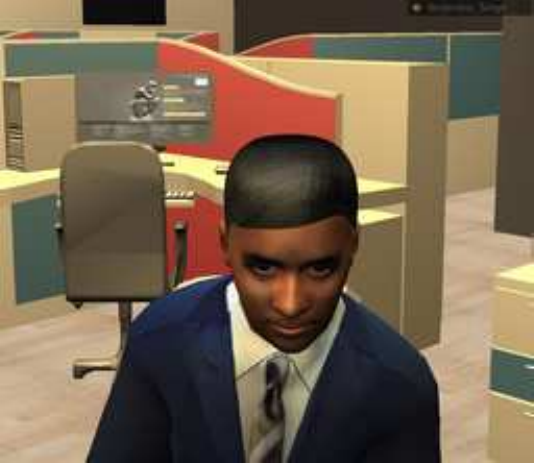}
    \includegraphics[width=0.118\textwidth]{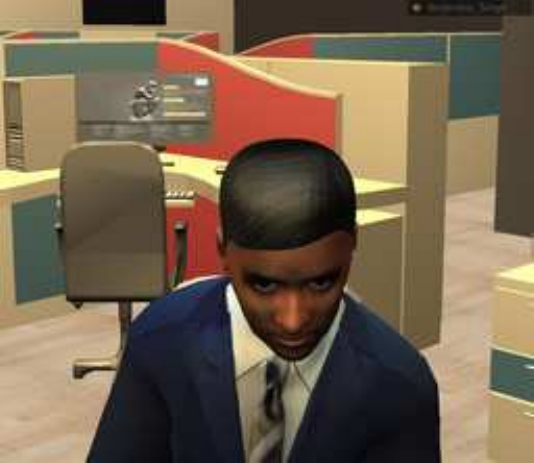}
    \includegraphics[width=0.118\textwidth]{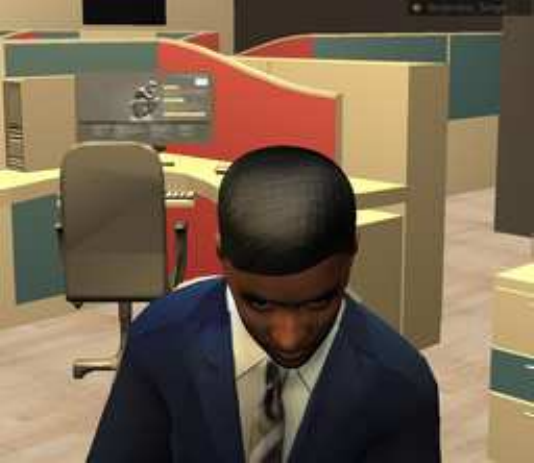}
    \includegraphics[width=0.118\textwidth]{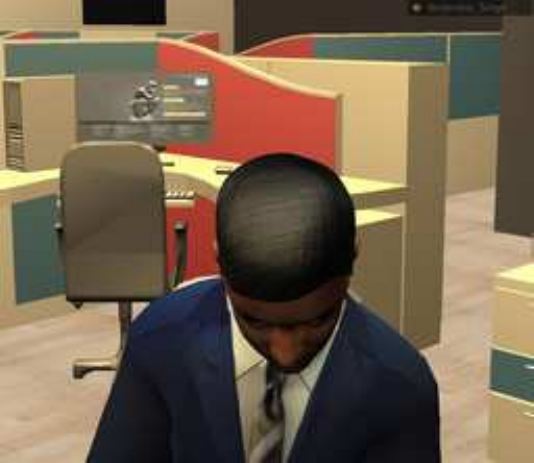}
    \includegraphics[width=0.118\textwidth]{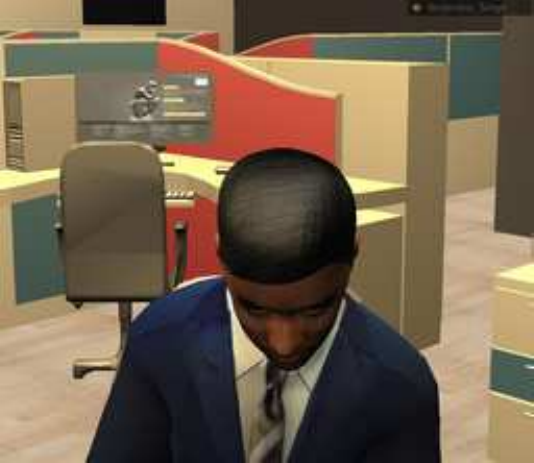}
    \includegraphics[width=0.118\textwidth]{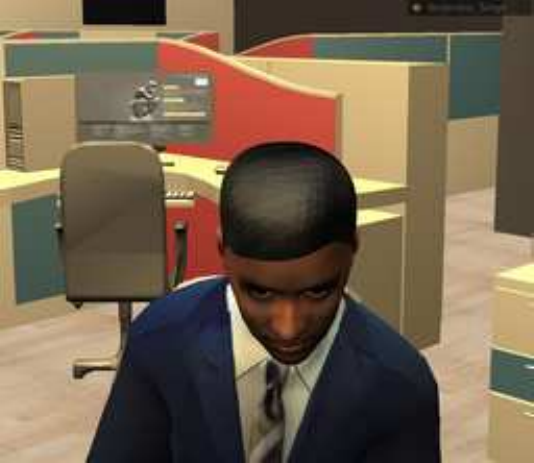}
    \includegraphics[width=0.118\textwidth]{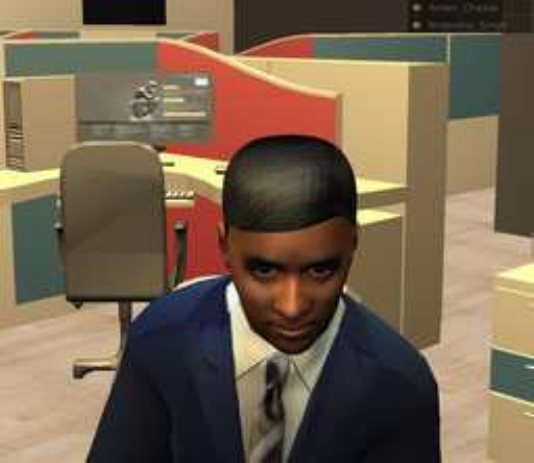}
    \includegraphics[width=0.118\textwidth]{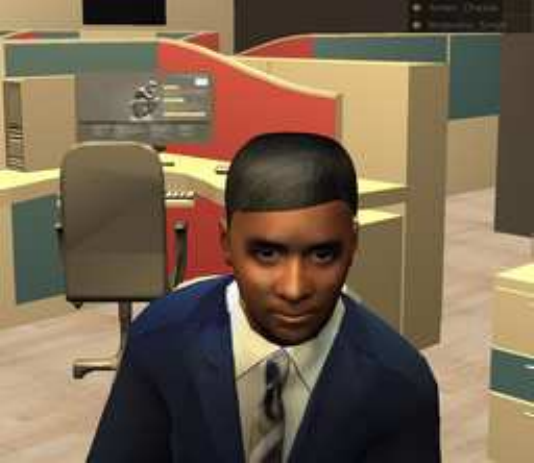}
    \\
    \includegraphics[width=0.118\textwidth]{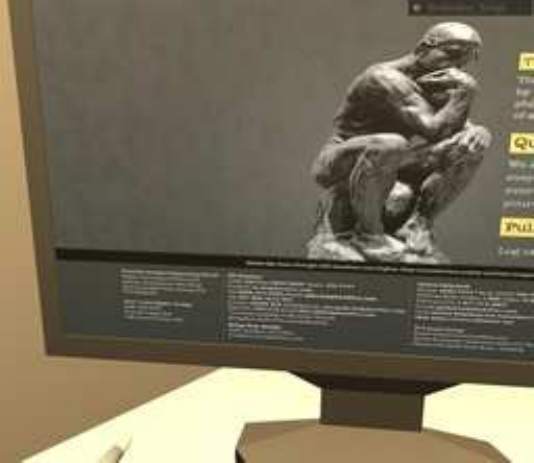}
    \includegraphics[width=0.118\textwidth]{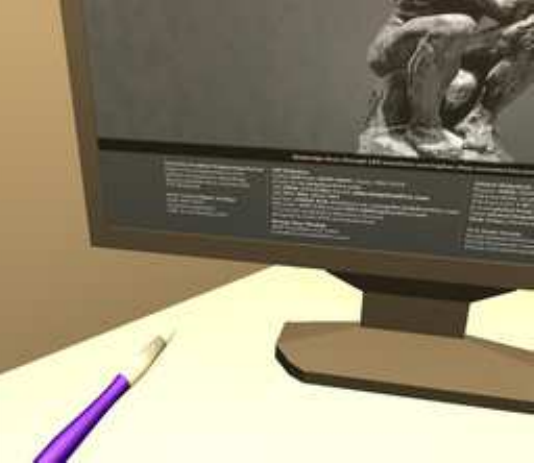}
    \includegraphics[width=0.118\textwidth]{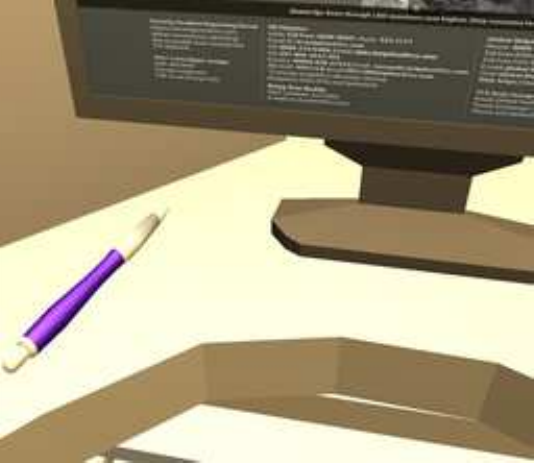}
    \includegraphics[width=0.118\textwidth]{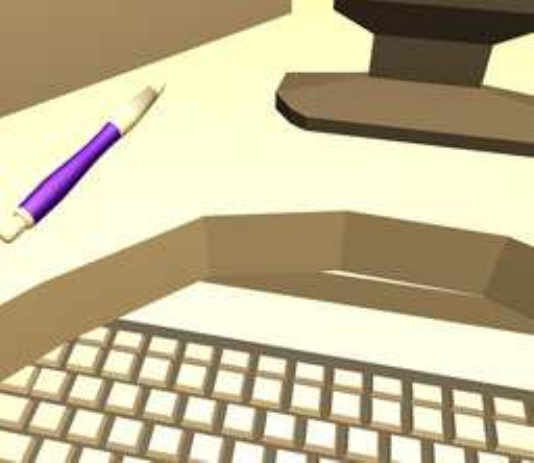}
    \includegraphics[width=0.118\textwidth]{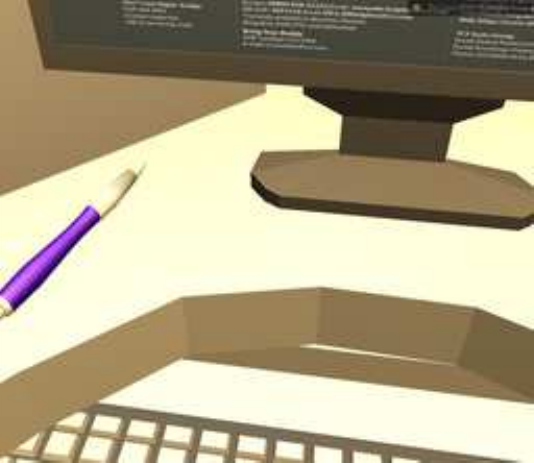}
    \includegraphics[width=0.118\textwidth]{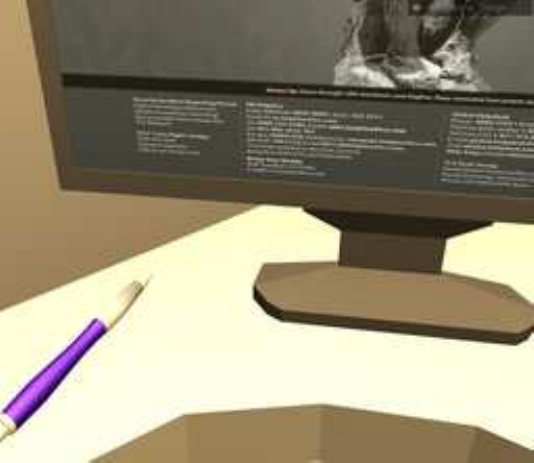}
    \includegraphics[width=0.118\textwidth]{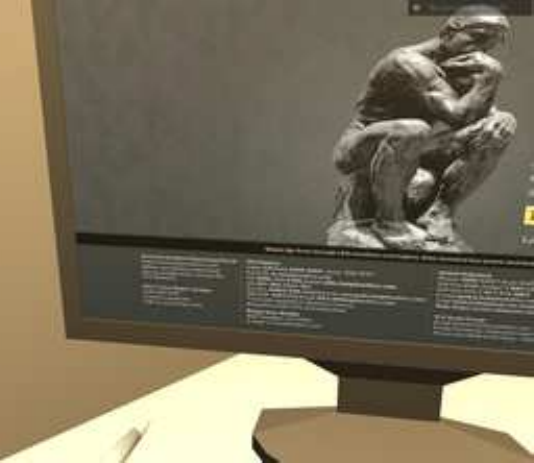}
    \includegraphics[width=0.118\textwidth]{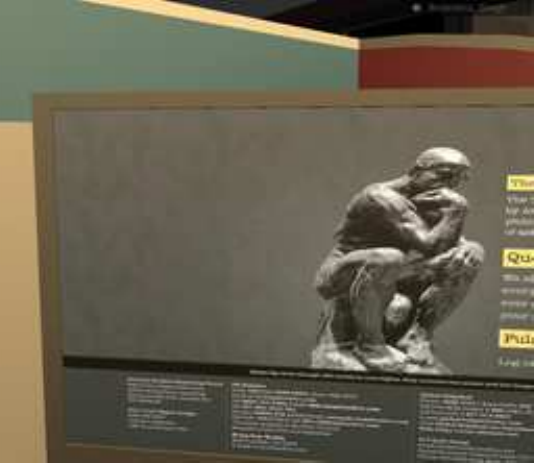}
    \caption{Some example images of an avatar (first row) and the corresponding images in the gaze space (second row).}
    \label{fig:avatar_example_images}
\end{figure}

We identify clusters of images in the gaze space based on the object that can be seen in them such as a screen, keyboard, coffee cup, etc. Each cluster is assigned a probability of occurrence based on its size. For example, screen images are expected to be largest in number as the user would be looking at his computer's screen most of the time. Long head animation sequences can then be generated automatically by choosing different source and destination images from various clusters based on their assigned probabilities. 

The rest of the chapter is organized as follows. We discuss previous work in section~\ref{sec:head_motion:prevwork}. Our technique is described in section~\ref{sec:head_motion:approach} followed by some results in section~\ref{sec:head_motion:results} and finally we conclude in section~\ref{sec:head_motion:conclusion}.

\section{Previous Work}
\label{sec:head_motion:prevwork}
Many techniques for blending motions have been proposed. Motion graphs \citep{kovar:Mographs} provide a mechanism to blend multiple motion capture sequences in different ways to create new animations. However, motion capture is not only expensive but also not suitable for the kind of motion sequences we need, as capturing motions with props such as chairs, tables and computers is hard. Further, the same captured motion will play for all avatars even if the order of blending is varied.   \citep{Kovar:Moextract} allows users to specify a motion parameterization function, such as the position of the hand, speed or curvature of walk, and then searches and blends from a database of motion capture clips. The manual effort required to create new motions makes it hard to use this technique for our requirement of multiple long sequences.

Graph techniques for interactively controlling characters in video games and virtual environments have been proposed in \citep{Shin:Fatgraphs}, \citep{McCann:2007:Characters} and \citep{Lee:2004}. While these blend short motion fragments instead of long sequences, the basic assumption is that the character is controlled by a user at runtime whereas we need the animation to run when the user is not controlling the avatar.

The `working-at-desk' animation is similar to an idle motion in the requirements of non-repetitiveness and uniqueness. \citep{Thalmann:2004} uses motion capture data of humans standing in different postures. Principal component representation of each key frame is extracted and variations in pose are created by Perlin noise functions. Finally, key frames are blended together to create the animation. This approach automatically generates animations which are non-repetitive but are personalized for the person whose motion was captured. Also, as stated earlier, motion capture is difficult with props.

\section{Our Approach}
\label{sec:head_motion:approach}
In this section, we describe our approach towards generating \emph{realistic} random head motions of avatars. Starting from a set of images of the avatar in various head poses, we first construct a $k$-nearest neighbours ($k$-NN) graph as described in section~\ref{sec:head_motion:knn_graph_construction}. Using this $k$-NN graph, we can animate the avatar from one head pose to another head pose as described in section~\ref{sec:head_motion:animating_head_poses}. In order to generate \emph{meaningful} head motions, we identify some semantic categories in the input images and tag some of the input images with these categories as described in section~\ref{sec:head_motion:adding_semantics}. To make the animations look \emph{realistic}, we then learn a probability distribution over the set of semantic categories as described in section~\ref{sec:head_motion:generating_realistic_motions}.

\subsection{Constructing k-Nearest Neighbours Graph}
\label{sec:head_motion:knn_graph_construction}
Here, we describe a method to build a representation that will be used later to generate animations. The input to this method is a set of images of the avatar in various head poses. Considering each image to be a grid of pixel intensities of say $r$ rows, $c$ columns and 3 channels (RGB), we will treat each $r \times c$ image as a point in $\mathbb{R}^{3rc}$ (i.e., a $3*r*c$ dimensional Euclidean space) by concatenating all $3rc$ pixel values into a vector. The distance between two images is the Euclidean distance between the corresponding points in $\mathbb{R}^{3rc}$. Algorithm~\ref{alg:head_motion:constructing_knn_graph} describes the steps to construct the $k$-nearest neighbours graph, which will be used later to generate animations.

\begin{algorithm}[ht!]
    \caption{Constructing $k$-Nearest Neighbours Graph}
    \label{alg:head_motion:constructing_knn_graph}
    \begin{algorithmic}[1]
        \REQUIRE{A set $X = \{x^{(1)}, x^{(2)}, \ldots x^{(n)}\}$ of $n$ points (flattened image vectors), and a neighbourhood size parameter $k$.}
        \ENSURE{A $k$-NN graph $G$}
        
        \STATE Compute the distance matrix $D_{n \times n}$ where $D_{i,j}$ is the Euclidean distance between $x^{(i)}$ and $x^{(j)}$. 
        \STATE Using $D$, compute the $k$ nearest neighbours of each point. 
        \STATE Construct a graph $G$ with $n$ vertices $\{v^{(1)}, v^{(2)}, \ldots, v^{(n)}\}$, so that each vertex $v^{(i)}$ corresponds to an input image $x^{(i)}$. Add an edge between two vertices $v^{(i)}$ and $v^{(j)}$ if $x^{(i)}$ is one of the k-nearest neighbours of $x^{(j)}$ or vice versa. Label the edge $(v^{(i)}, v^{(j)})$ with $D_{i, j}$. 
        \STATE Output $G$.
    \end{algorithmic}
\end{algorithm}

Constructing the $k$-NN graph is a one-time process for a given set of images.

\subsection{Animating between Head Poses}
\label{sec:head_motion:animating_head_poses}
In the nearest neighbours graph $G$, each vertex corresponds to an input image. Once such a graph is constructed, in order to generate an animation from one head pose to another head pose of the input image set, we compute the shortest path between the corresponding vertices in $G$ and animate the images corresponding to the vertices on the computed path. 

Figure~\ref{fig:isomap_2d_embedding} shows the 2D embedding of the images of the avatar shown in Figure~\ref{fig:avatar_example_images}, using the Isomap algorithm \citep{tenenbaum2000global}. The 2D embedding is shown in the form of a $k$-NN graph computed using the distances between image pairs in the avatar image set. Even though the actual 2D embedding is never computed explicitly in our method, we show it here for the purposes of visualization. Also shown is an example animation from a source pose to a destination pose in the form of a path in the $k$-NN graph between the corresponding source and destination vertices.

\begin{figure}[ht!]
    \centering
    \includegraphics[width=\textwidth]{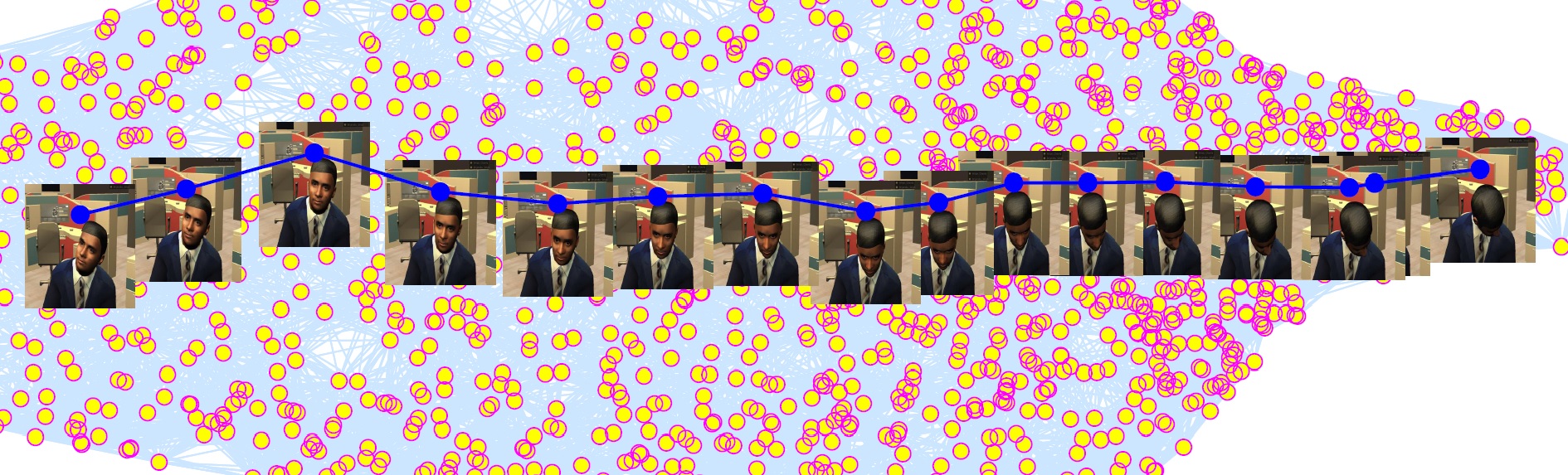}
    \caption{
        Head animation sequence shown as a path on a part of the 2-dimensional embedding of the avatar image manifold computed by Isomap.
    }
    \label{fig:isomap_2d_embedding}
\end{figure}

 Figure~\ref{fig:two_paths_face} shows two animations of the avatar between two different pairs of poses. It can be seen that near the intersection of paths images on the two paths look similar because the Isomap embedding preserves distances.

\begin{figure}[ht!]
    \centering
    \includegraphics[width=\textwidth]{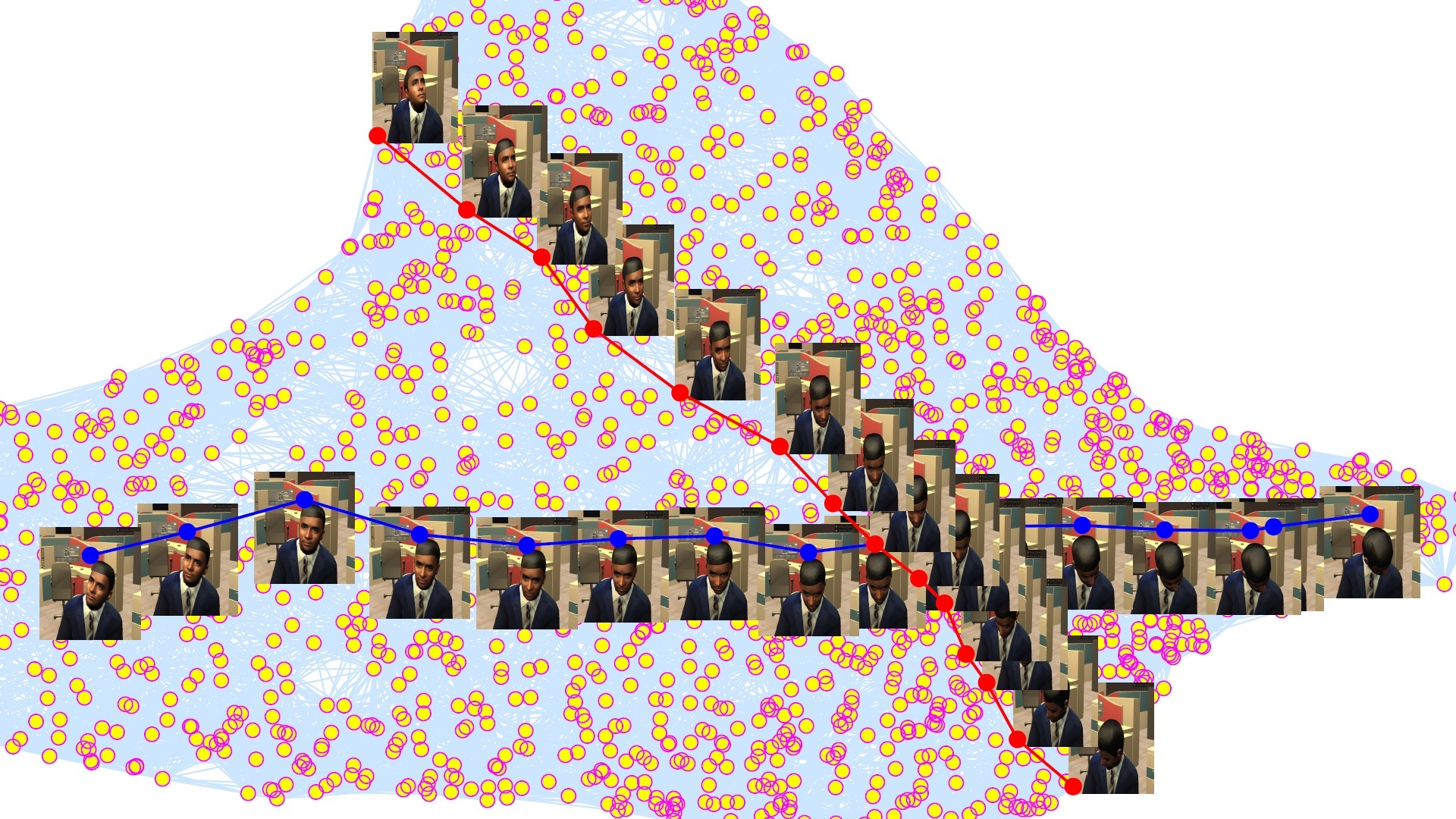}
    \caption{
        Head animation sequences between two different pairs of source and destination poses, shown as intersecting paths on a part of the 2-dimensional embedding of the avatar image manifold computed by Isomap. We can observe similar faces near the intersection of paths, indicating that near-by points on the manifold correspond to similar poses.
    }
    \label{fig:two_paths_face}
\end{figure}

\subsection{Adding Semantics to Images}
\label{sec:head_motion:adding_semantics}
 In order to be able to associate some meaning with motions, we wish to specify the source and destination poses of an animation in terms of what the avatar/person in the image is looking at. For example, in the case of an avatar sitting in front of a computer desk in a virtual office, we could identify some objects like \textit{keyboard}, \textit{screen}, \textit{mouse}, etc. These objects act as semantic categories to indicate what the avatar is looking at in a given image. We use these objects to specify the source and destination poses for an animation. Table~\ref{tab:exemplars} shows an example list of semantic categories along with some exemplars for each category for an avatar in a virtual office. Here exemplars are given in terms of the image file names from the input data set.

\begin{table*}[htp]
    \centering
    \caption{An example list of semantic categories along with some exemplars for each category}
    \begin{tabular}{|c|l|} \hline
        \label{tab:exemplars}
        \thead{Semantic \\ Category} & \thead{Exemplars (Image file names in the input data set)}\\ \hline
        Mouse& 0082.jpg, 0095.jpg, 0098.jpg, 0111.jpg, 0117.jpg, 0131.jpg, 0138.jpg\\ \hline
        Keyboard& 0006.jpg, 0011.jpg, 0016.jpg, 0025.jpg, 0059.jpg, 0075.jpg, 0076.jpg\\ \hline
        Screen& 0012.jpg, 0018.jpg, 0021.jpg, 0027.jpg, 0039.jpg, 0054.jpg, 0058.jpg\\ \hline
        Pen& 0206.jpg, 0215.jpg, 0192.jpg, 0220.jpg, 0221.jpg, 0244.jpg, 0251.jpg\\ \hline
        Wall& 0189.jpg, 0208.jpg, 0214.jpg, 0246.jpg, 0306.jpg, 0316.jpg, 0319.jpg\\ \hline
        Floor& 0197.jpg, 0485.jpg, 0490.jpg, 0331.jpg, 0332.jpg, 0333.jpg, 0574.jpg\\ \hline
        Space& 0009.jpg, 0557.jpg, 0562.jpg, 0565.jpg, 1255.jpg, 1303.jpg, 1320.jpg\\ \hline
    \end{tabular}
\end{table*}

Once such a list of semantic categories along with some exemplars is available, we can specify the source and destination of an animation in terms of these categories. For example, we can have the avatar animate from a pose in which it is looking at the \textit{screen} to a pose in which it is looking at the \textit{keyboard}. Given a pair of source and destination objects, we can randomly choose one of the images from the exemplars of the source object and one from the exemplars of the destination object and generate an animation between these two images. We can extend this ability to generate animations for an arbitrarily long sequence of objects. For example, we can have the avatar animated according to the following sequence: \textit{screen-screen-keyboard-screen-mouse-screen-screen-keyboard-screen}.

\subsection{Generating Realistic Head Motions}
\label{sec:head_motion:generating_realistic_motions}

In this section, we address the problem of generating head motions of avatars which are close to the movements of a real person. Towards this end, we learn a probability distribution, over the set of semantic categories, that tells how often a person/avatar looks at each of the objects in the environment. Algorithm~\ref{alg:learn_object_probabilities} describes the steps involved in learning such a probability distribution.

\begin{algorithm}[ht!]
    \caption{Learning Probabilities of Semantic Categories}
    \label{alg:learn_object_probabilities}
    \begin{algorithmic}[1]
        \REQUIRE{
            A set $X$ of $n$ images, a set $S$ of semantic categories along with some exemplars under each category. Let the set of all exemplars under all semantic categories be $E$.
        }
        \ENSURE{
            A probability distribution $P$ over $S$.
        }
        
        \FOR{each image $x \in X$}
            \IF{$x \notin E$}
                \STATE Assign to $x$ the category which is most frequent among its $k$ nearest neighbours in $E$.
            \ENDIF    
        \ENDFOR
        \FOR{each semantic category $s \in S$}
            \STATE $P(s) = $ (\textit{Number of images categorized as s}) / $n$
        \ENDFOR
        \STATE Output $P$.
    \end{algorithmic}
\end{algorithm}

\section{Results}
\label{sec:head_motion:results}

We applied the algorithms described in section~\ref{sec:head_motion:approach} on a set of images of an avatar sitting near a computer desk in a virtual environment. We have used a set of around 4500 images of the avatar and the same number of images in the gaze space. The semantic categories listed in Table~\ref{tab:exemplars} along with the avatar images were used to learn the probability distribution shown in Table~\ref{tab:probabilities} using Algorithm~\ref{alg:learn_object_probabilities}. This probability distribution was used to generate random sequences of objects using which animations can be generated as described in section~\ref{sec:head_motion:generating_realistic_motions}.

    \begin{table*}[ht!]
        \centering
        \caption{A probability distribution over the semantic categories.}
        \begin{tabular}{|c|l|} \hline
            \label{tab:probabilities}
            \textbf{Semantic Category} & \textbf{Probability of the Category}\\ \hline
                Mouse & 0.033011     \\ \hline
                Keyboard & 0.204710  \\ \hline
                Screen & 0.454308    \\ \hline
                Pen & 0.076288       \\ \hline
                Wall & 0.189815      \\ \hline
                Floor & 0.038245     \\ \hline
                Space & 0.003623     \\ \hline
        \end{tabular}
    \end{table*}

Figure~\ref{fig:screen_keyboard_screen} shows the result of animating the avatar from a pose in which it is looking at the screen to a pose in which it is looking at the keyboard and then back to looking at the screen. As the head of the avatar moves from one pose to another, what it would look at also changes. The corresponding changes in gaze space are shown in the figure immediately below each avatar head pose in the motion sequence.

\begin{figure}[ht!]
\centering

    \includegraphics[width=0.6in]{images/chap6/out3/000.pdf}
    \includegraphics[width=0.6in]{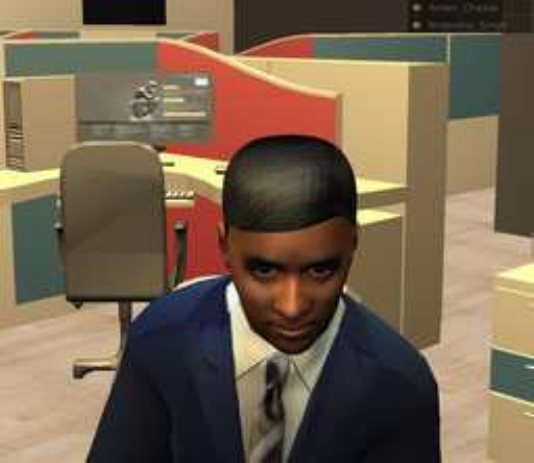}
    \includegraphics[width=0.6in]{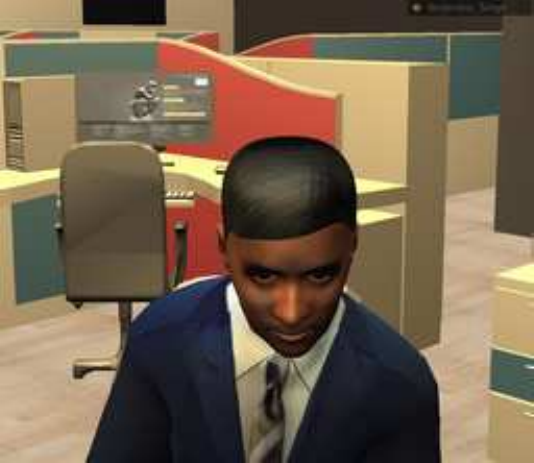}
    \includegraphics[width=0.6in]{images/chap6/out3/003.pdf}
    \includegraphics[width=0.6in]{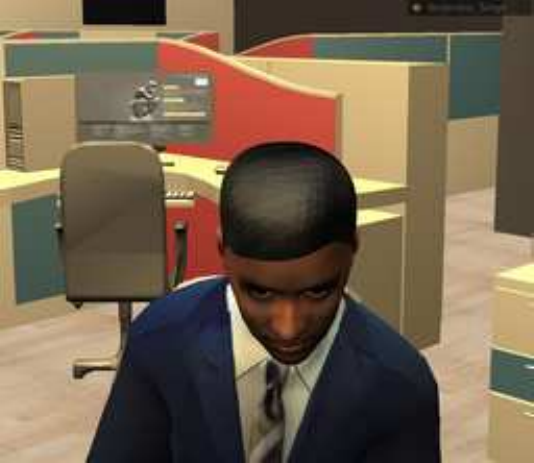}
    \includegraphics[width=0.6in]{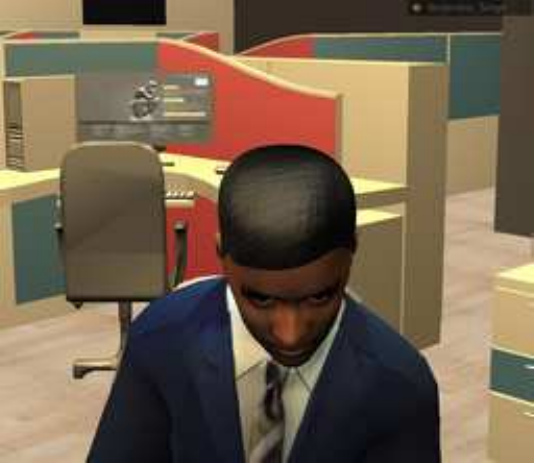}
    \includegraphics[width=0.6in]{images/chap6/out3/006.pdf}
    \includegraphics[width=0.6in]{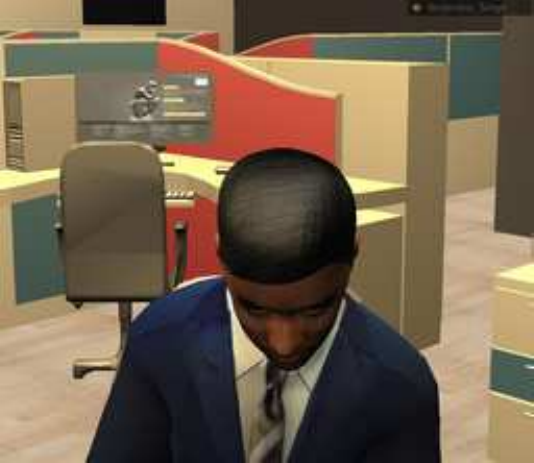}
    \\
    \includegraphics[width=0.6in]{images/chap6/out1/000.pdf}
    \includegraphics[width=0.6in]{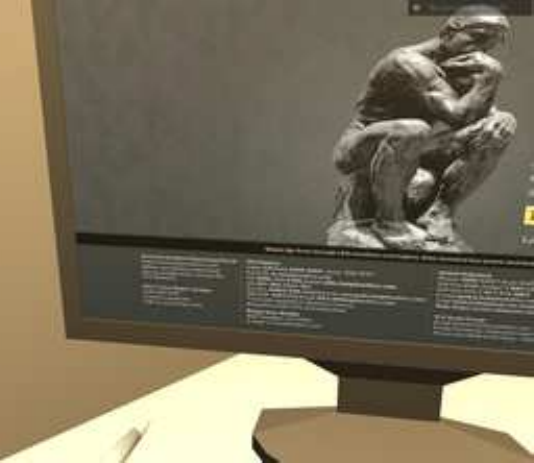}
    \includegraphics[width=0.6in]{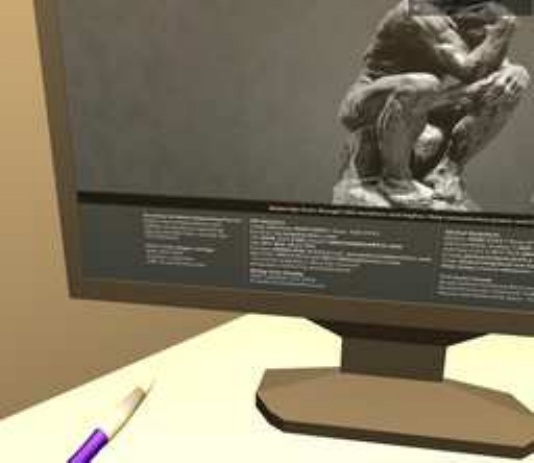}
    \includegraphics[width=0.6in]{images/chap6/out1/003.pdf}
    \includegraphics[width=0.6in]{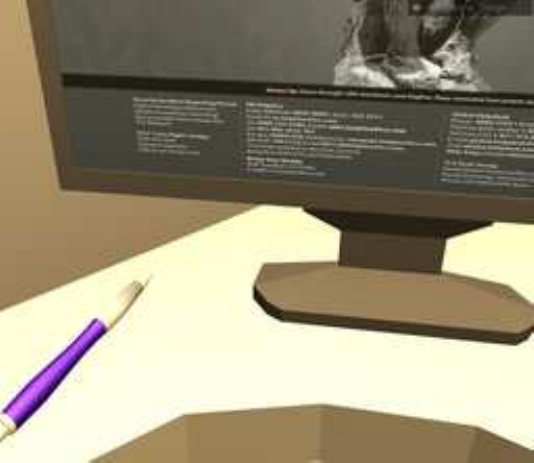}
    \includegraphics[width=0.6in]{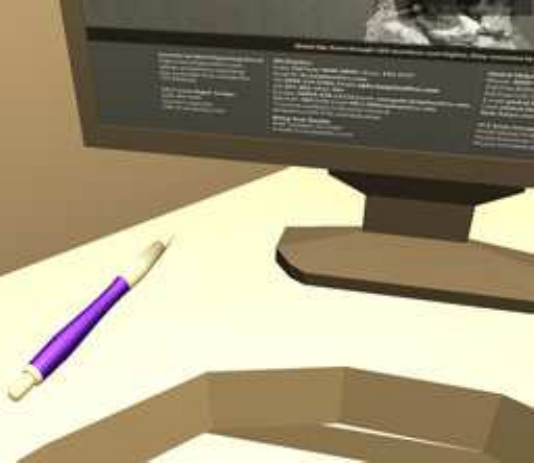}
    \includegraphics[width=0.6in]{images/chap6/out1/006.pdf}
    \includegraphics[width=0.6in]{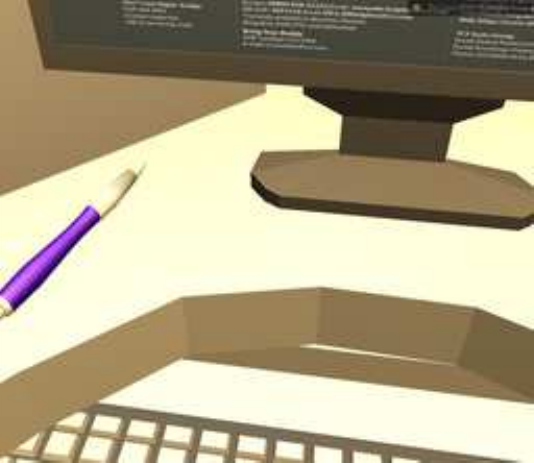}
    \\[2ex]    
    \includegraphics[width=0.6in]{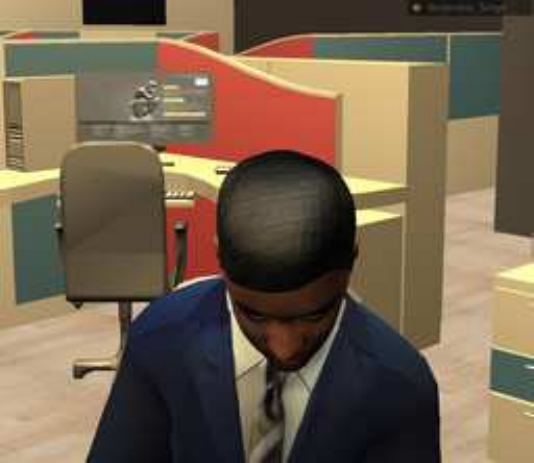}
    \includegraphics[width=0.6in]{images/chap6/out3/009.pdf}
    \includegraphics[width=0.6in]{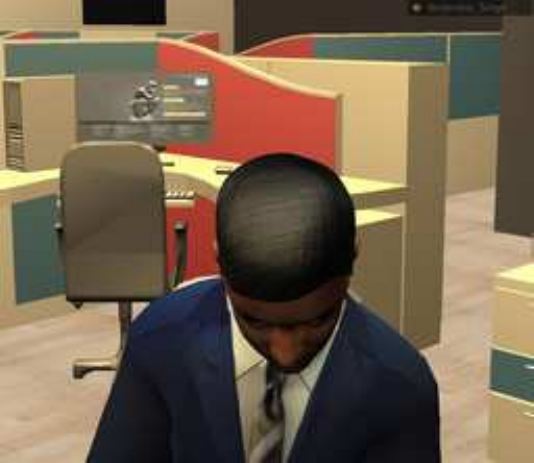}
    \includegraphics[width=0.6in]{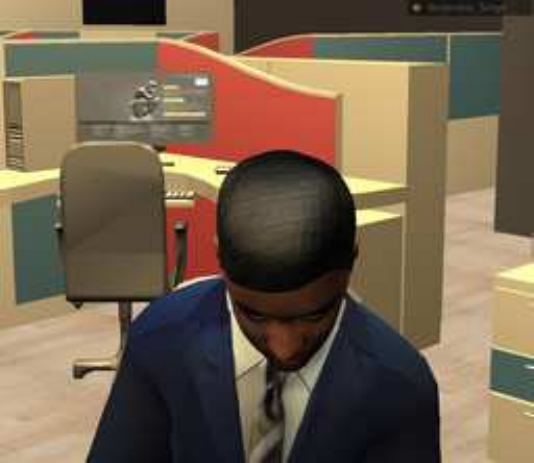}
    \includegraphics[width=0.6in]{images/chap6/out3/012.pdf}
    \includegraphics[width=0.6in]{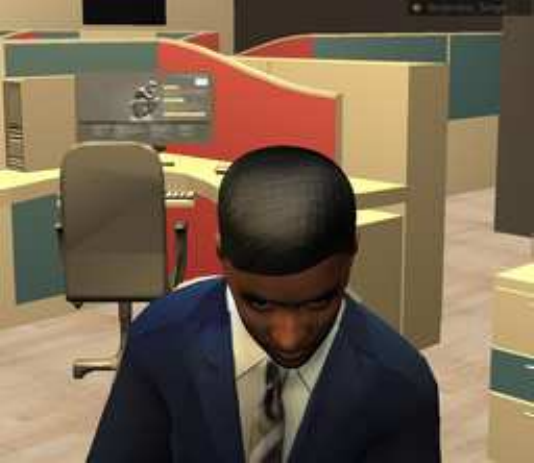}
    \includegraphics[width=0.6in]{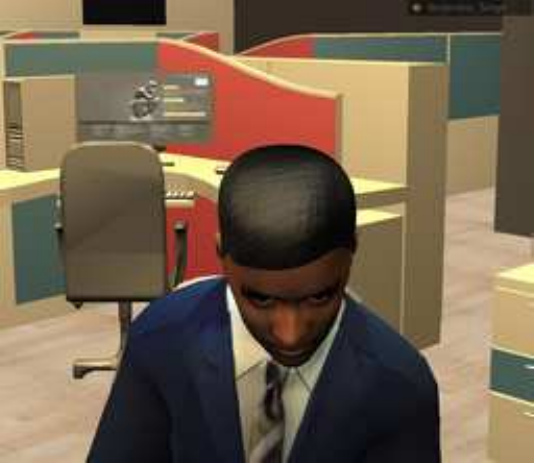}
    \includegraphics[width=0.6in]{images/chap6/out3/015.pdf}
    \\
    \includegraphics[width=0.6in]{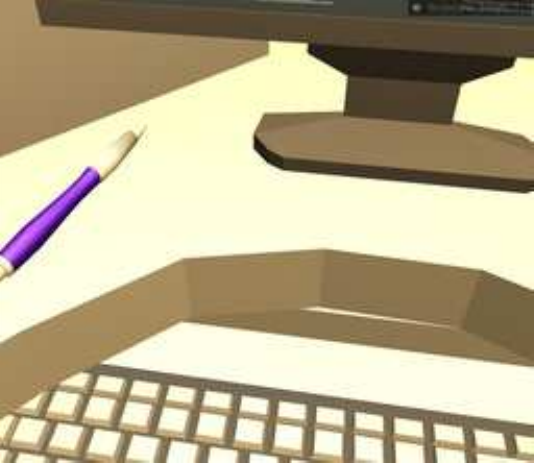}
    \includegraphics[width=0.6in]{images/chap6/out1/009.pdf}
    \includegraphics[width=0.6in]{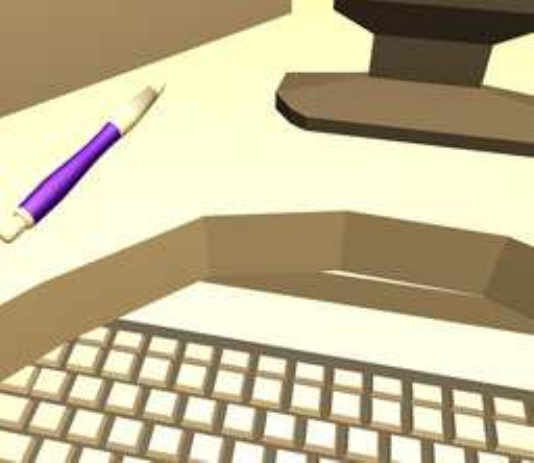}
    \includegraphics[width=0.6in]{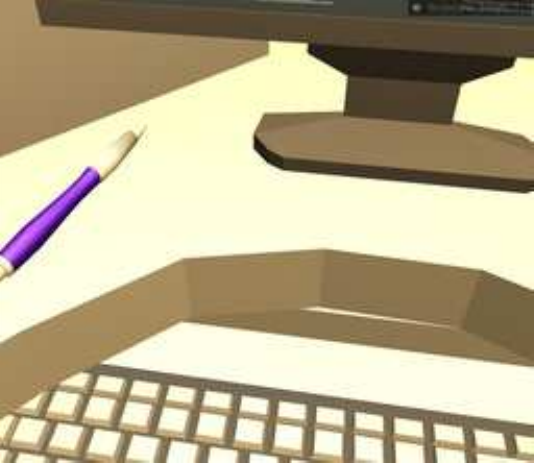}
    \includegraphics[width=0.6in]{images/chap6/out1/012.pdf}
    \includegraphics[width=0.6in]{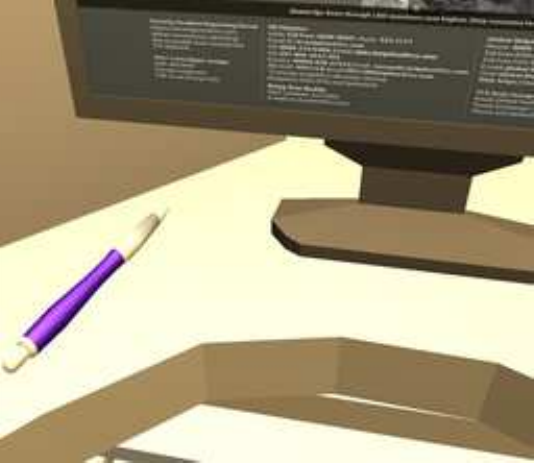}
    \includegraphics[width=0.6in]{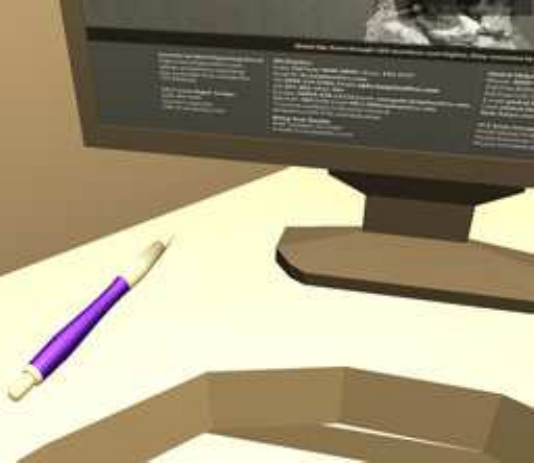}
    \includegraphics[width=0.6in]{images/chap6/out1/015.pdf}
    \\[2ex]    
    \includegraphics[width=0.6in]{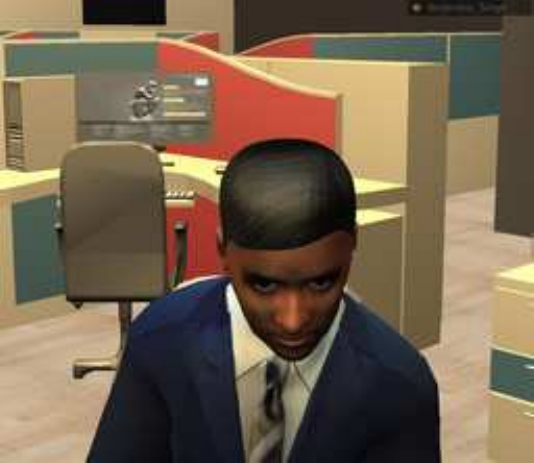}
    \includegraphics[width=0.6in]{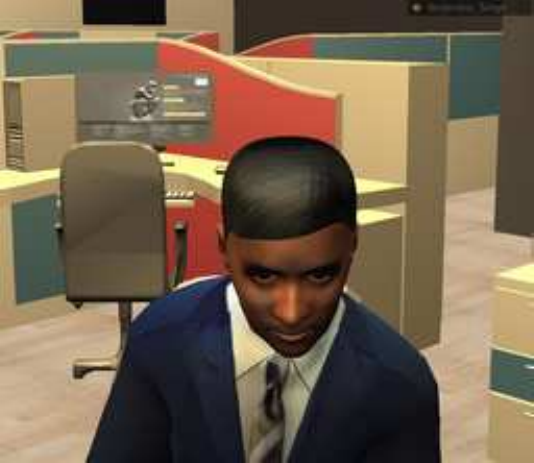}
    \includegraphics[width=0.6in]{images/chap6/out3/018.pdf}
    \includegraphics[width=0.6in]{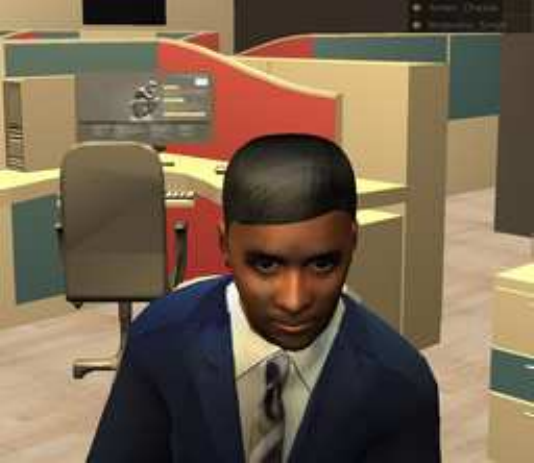}
    \includegraphics[width=0.6in]{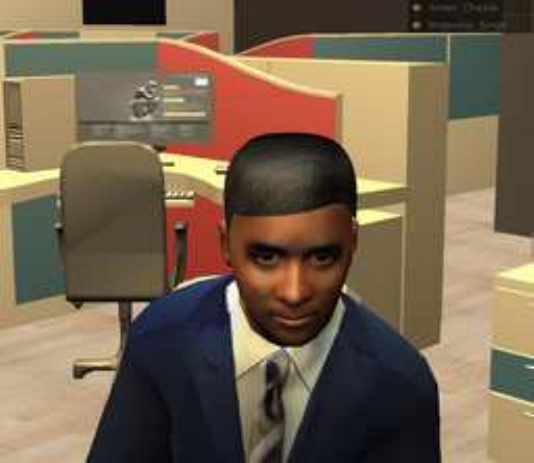}
    \includegraphics[width=0.6in]{images/chap6/out3/021.pdf}
    \includegraphics[width=0.6in]{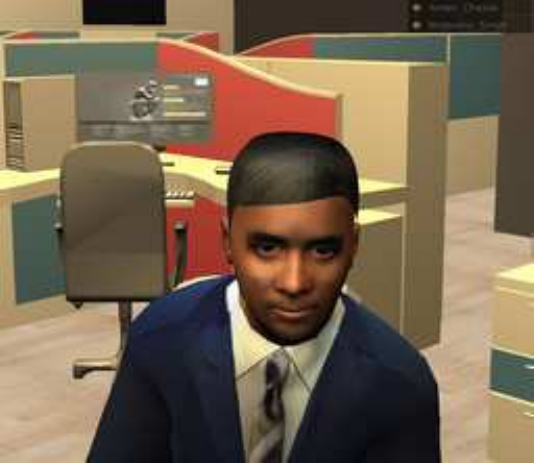}
    \includegraphics[width=0.6in]{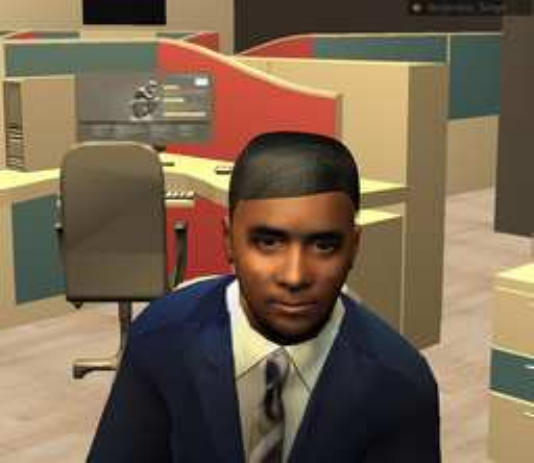}
    \\
    \includegraphics[width=0.6in]{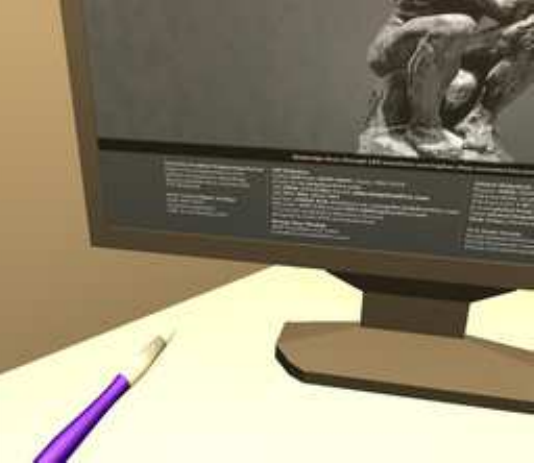}
    \includegraphics[width=0.6in]{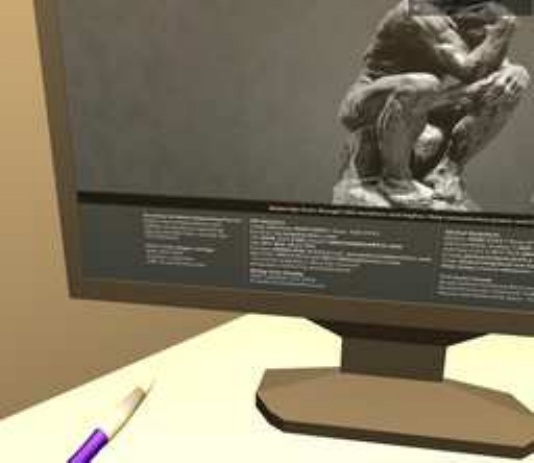}
    \includegraphics[width=0.6in]{images/chap6/out1/018.pdf}
    \includegraphics[width=0.6in]{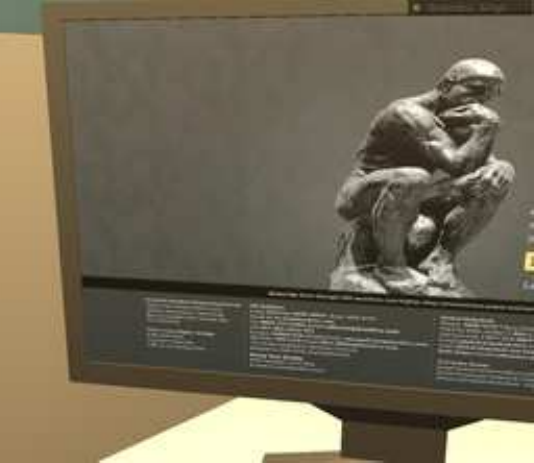}
    \includegraphics[width=0.6in]{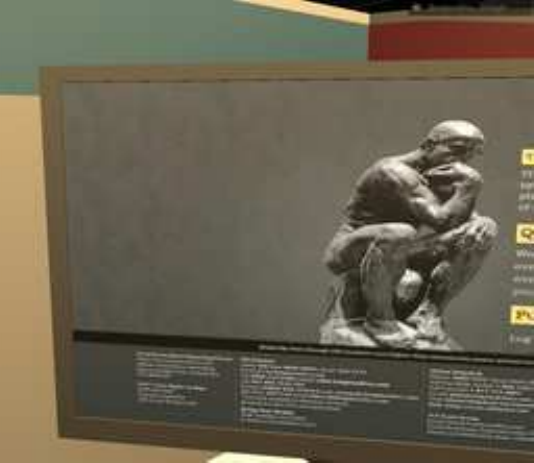}
    \includegraphics[width=0.6in]{images/chap6/out1/021.pdf}
    \includegraphics[width=0.6in]{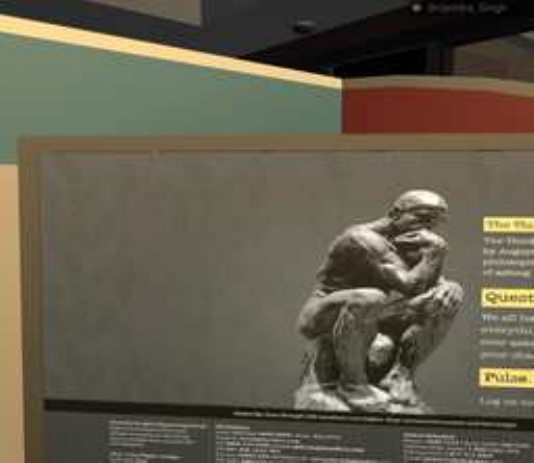}
    \includegraphics[width=0.6in]{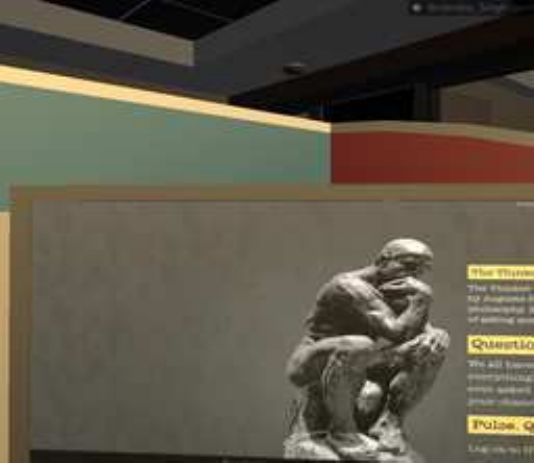}

    \caption{
        Sequence of images in a \textit{screen-keyboard-screen} animation of the avatar along with the changes in the gaze space. Animation frames proceed from left-to-right and top-to-bottom. There are 24 frames in this motion. Each pair of consecutive rows, starting from the first row, shows 8 frames from avatar's head animation and the corresponding 8 frames from the gaze space, respectively.
    }
    \label{fig:screen_keyboard_screen}
\end{figure}

 Similarly, we can generate arbitrarily long sequences of objects of gaze and generate animations through these object sequences using the methods described before.

\section{Conclusion}
\label{sec:head_motion:conclusion}
We have described a method towards generating \textit{realistic} random motion based on a set of images of the agent to be animated, and applied it on a set of images of an avatar sitting near a computer desk in a virtual environment. With a fairly simple algorithm, we were able to see some good results. These results can be further improved by using better models for motion probabilities. For example, in this work, we have only considered \textit{state probabilities} of gaze; we could extend it to consider \textit{transition probabilities}, in which case along with the probabilities of looking at an object, we will also have the probabilities of visually reaching a particular object (next state), given the current object of gaze (current state). Also in our work, we have taken constant velocity between all pairs of objects. Having different velocities between different pairs of objects, based on the \textit{importance} of objects in the environment, will be preferable. We believe that these extensions will significantly improve the results.

\chapter{Discussion and Conclusion} 

\label{chap:conclusion} 

In this work, we have introduced a new approach towards the longstanding perceptual robotics problem, which subsumes the problem of body schema learning~\citep{poincare-1895-space-and-geometry,philipona-oregan-2003_perception-of-structure-unknown-sensors,hoffmann2010body}. We consider the problem of a robot infant, and attempt to discover its sensorimotor map in a manner that draws upon some aspects of human infant cognition.

Although it has been long known that there may be many kinds of generalized coordinates (GC), so far there have been few attempts in robotics to build on this intuition.  The proposed paradigm attempts to develop such a non-traditional GC, and approximates the C-space that results from it in terms of a neighbourhood graph on a set of images.  We show that under some assumptions, an implicit learning system may be able to learn a mapping from the visuomotor space onto a joint manifold which encodes the degrees of freedom available to the system as a set of \emph{visual generalized coordinates} (VGC).   This sensorimotor map or \emph{i-representation} is an alternative, consistent, symbolic space and is as compact as the traditional \emph{e-representations}. We established some simple theoretical results for why such a system should work, and validated our approach with several demonstrations, in all of which we are able to learn an $i$-representation and execute motions in the space without invoking any kind of knowledge about the robot geometry, its kinematics, its dimensions, or the geometry of the obstacle. 

Unlike in methods used in robotics today, the \emph{visual configuration space} (VCS) approach eliminates several expensive aspects of robot modelling and planning.  First, it does not require a human expert to create models for robot geometry or kinematics.  It does not require precise obstacle shapes and poses, and does not require to calibrate the cameras so that this can be done.  There is no need for a precise simulator to test which poses collide with obstacles and which do not. The main benefit of this approach is that it discovers the motor self-structure and is able to construct obstacle maps on the VCS via visual overlap alone. For the purposes of obstacle avoidance, one may construct a \emph{Visual Roadmap} from the local neighbourhoods on the manifold.  Given an obstacle, putative collision poses (the images where the robot overlaps the obstacle) can be removed, and motion planning performed on the remaining free space. Even the local planner step, based on tracking image points to nearby images, results in a more principled approach than is available presently. 

This representation now allows the identification of objects in the workspace via visual overlap.  If the object is to be reached by a given part of the robot, poses for this can be identified by overlapping the obstacle image with a series of robot images.  

Another advantage is for environments that are changing rapidly, e.g. in interaction with humans or other robots. New obstacles are updated in $O(n)$ time, but small motions by another agent require $O(m)$, where there are $m$ nodes near the obstacle boundary. Additional obstacles or moving obstacles can be handled with incremental computation. 

The idea of generalized coordinates originated in Lagrangian dynamics, and here is another direction that needs to be pursued. Differentiating the GC would result in generalized velocities and accelerations and this may give rise to a \emph{visual dynamics}, when the same principles are applied on visual generalized coordinates. 
However, there are some significant trade-offs.  First, the approach is not \emph{complete} because the obstacle approximation is conservative, and there may exist paths which it cannot find. We observe that humans also face similar constraints where vision is less informative.  Secondly, it is applicable only to those situations where the entire C-Space is visible. Another constraint is the \emph{visual distinguishability assumption}, but this may not be very serious in practice.  As presented here, the algorithm requires that all robot pose images be stored, which can be done more efficiently via standard image compression techniques such as run length encoding. 

The approach presented is only a beginning for discovering generalized coordinates from sensorimotor data. One of the key future steps would be to fuse modalities other than vision into a joint manifold.  Thus, if we were to construct a fused visuomotor manifold, then even if poses that are separated in motion space look similar, they would remain distinguishable.  Similarly, touch stimuli could be modelled to predict the result of motions or in preparation for fine-motor tasks. Such a process would also make the model more robust against noise arising in any single modality.  On the whole, while the ideas presented seem promising, and open up many possibilities, much work remains to deploy VGC fully in theory and in practice. 

An important ramification of this process would be that such a representation, after repeated application in diverse situations, may lead to a generalization which may be considered to be an internal representation of space itself. Such a role for sensorimotor development has been suggested by many, for example, \citep{thelen-00_motor-development-foundation-for-dev-psych}.  As an example,  given a base pose for the robot, distant parts of the space are to be reached with a greater change in the manifold parameters (generalized coordinates).  Two locations may be close if they can be reached with similar configurations.  By generalizing over a large set of such experiences, one may acquire spatial concepts, such as its dimensionality, a hierarchical scale structure, and many other aspects based on the action-perception pairings.  By identifying the configurations that reach various parts of the workspace, the system is also constructing a model for space itself. 

What we have presented here is just an initial step. The basic idea of discovering patterns from the lower-dimensional mapping of visual images is actually more general, and can also be used for learning other regularities, as in learning the laws of Physics, or for handling self-motions of the eyes, based on the image space alone.  These and many other matters related to this approach remain to be explored. 

\appendix
\chapter{Metric Spaces and Topology}
\label{app:topology}

Recall that $\mathbb{R}$ is the real line and $\mathbb{R}^2$ is the Euclidean plane.

\begin{definition}
A nonempty set $A \subseteq \mathbb{R}$ is an \emph{open set} in $\mathbb{R}$ if $\forall x \in A, \exists r > 0 : (x-r, x+r) \subseteq A$.
\end{definition}

\begin{example}
\label{ex:metric_open_set}
Following are some examples of open sets in $\mathbb{R}$.
\begin{enumerate}
    \item {
        The empty set $\phi$ is open in $\mathbb{R}$, vacuously.
    }
    \item {
        $\mathbb{R}$ is open in $\mathbb{R}$. For every $x \in \mathbb{R}$, we can take any $r > 0$ to have $(x-r, x+r) \subseteq \mathbb{R}$.
    }
    \item {
        \label{ex:open_intervals_open_sets}
        All open intervals are open sets in $\mathbb{R}$.  \\
        Let $A = (a, b), x \in A$ and let $r_1 = x-a, r_2 = b-x$. If we choose $r = \min(r_1, r_2)/2$, then $(x-r, x+r) \subseteq (a, b)$. Such an $r$ can be chosen for all $x \in (a, b)$, hence $A = (a, b)$ is open in $\mathbb{R}$. See figure~\ref{fig:open_set_R}.
        \begin{figure}[ht]
            \centering
            \includegraphics[page=2,clip,trim=0cm 8cm 0cm 9cm,width=0.9\textwidth]{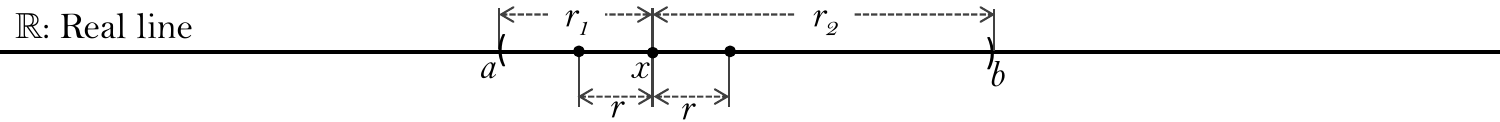}
            \caption{Open interval $(a, b)$ is an open set in $\mathbb{R}$}
            \label{fig:open_set_R}
        \end{figure}
    }
    \item {
        Unions of open intervals are open sets in $\mathbb{R}$. \\
        Let $\{A_i, A_2, \ldots, A_n\}$ be a collection of open intervals in $\mathbb{R}$ and $A = \bigcup\limits_{i=1}^n A_i$. Then for every $x \in A, \exists i: x \in A_i$ and an $r > 0$ can be found such that $(x-r, x+r) \subseteq A_i$, as was done in the previous example.
    }
    \item {
        Arbitrary union of open sets is open in $\mathbb{R}$. 
    }
    \item {
        Finite intersection of open sets is open in $\mathbb{R}$. \\
        Let $A_1, A_2$ be open sets in $\mathbb{R}$. If $x \in A_1 \cap A_2$, then $x \in A_1$ and $x \in A_2$ and since $A_1$, $A_2$ are open, $\exists r > 0 : (x-r, x+r) \subseteq A_1 \land (x-r, x+r) \subseteq  A_2$, which implies that $(x-r, x+r) \subseteq A_1 \cap A_2$ and hence $A_1 \cap A_2$ is open.
    }
\end{enumerate}
\end{example}

\begin{definition}
\label{def:topology}
Let $X$ be a set. A collection $\Tau \subseteq \mathcal{P}(X)$ is a \emph{topology} on $X$, if
\begin{enumerate}[label=(\roman*)]
    \item $\phi \in \Tau$ and $X \in \Tau$;
    \item $A_1, A_2, ... A_n \in \Tau \implies \bigcap\limits_{i=1}^n A_i \in \Tau$;
    \item $ \{A_\alpha\}_{\alpha \in I} \subseteq \Tau \implies \bigcup\limits_{\alpha \in I} A_\alpha \in \Tau$.
\end{enumerate}
Elements of $\Tau$ are, by definition, called \emph{open sets} of $\Tau$, and $(X, \Tau)$ is called a \emph{topological space}.
\end{definition}

\begin{example}
Following are some examples and non-examples of topologies.
\begin{enumerate}
    \item {
        For any set $X$, $\{\phi, X\}$ is a topology called the \emph{trivial topology} and $\mathcal{P}(X)$ is a topology called the \emph{discrete topology}.
    }
    \item {
        Let $X$ be an infinite set. Define a collection $\Tau \subseteq \mathcal{P}(X)$ as follows: 
        \[
            A \in \Tau \iff A = \phi \; \lor \; X \setminus A \text{ is finite }.
        \]
        Such a $\Tau$ is a topology on $X$ and is called the \emph{co-finite topology} of $X$.
    }
    \item {
        Unions of open intervals are the open sets defining a topology on $\mathbb{R}$, called the \emph{usual topology} or the \emph{Euclidean topology}.
    }
\end{enumerate}
\end{example}

Now, we will see the concept of a \emph{basis} of a topology $\Tau$, which is a small subset of $\Tau$  that can be used to generate all the open sets of $\Tau$.

\begin{example}
\label{ex:basis_R}
Consider the real line $\mathbb{R}$ along with the usual topology, $\Tau_\mathbb{R}$. Let $\mathcal{B} = \{(a, b) : a,b \in \mathbb{R} \land a < b\}$, be the set of all open intervals $(a, b)$ in $\mathbb{R}$, with $a < b$. We make the following observations about $\mathcal{B}$:
\begin{enumerate}[label=(\roman*)]
    \item {
        For every $x \in \mathbb{R}, \exists r > 0 : (x-r, x+r) \in \mathcal{B}$ (take r = 1, for example); i.e., $\forall x \in \mathbb{R}, \exists B \in \mathcal{B}: x \in B$.
    }
    \item {
        If $B_1, B_2 \in \mathcal{B}$ and $x \in B_1 \cap B_2$, then $\exists B_3 \in \mathcal{B}$ s.t. $x \in B_3 \subseteq B_1 \cap B_2$. See figure~\ref{fig:basis_Tau_R} for a proof.
        
        \begin{figure}[ht]
            \centering
            \captionsetup[subfigure]{labelformat=empty}
            \begin{subfigure}{.25\textwidth}
                \centering
                \includegraphics[page=3, clip, trim=0cm 6cm 16cm 5cm, width=\textwidth]{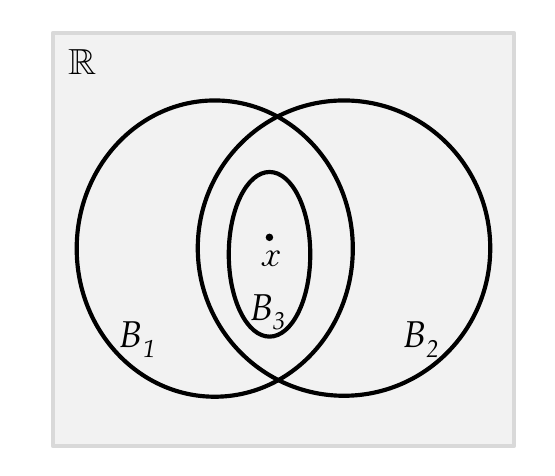}
            \end{subfigure}
            \qquad
            \begin{subfigure}{0.5\textwidth}
                \centering
                \includegraphics[page=3, clip, trim=9.5cm 9cm 2cm 8cm, width=\textwidth]{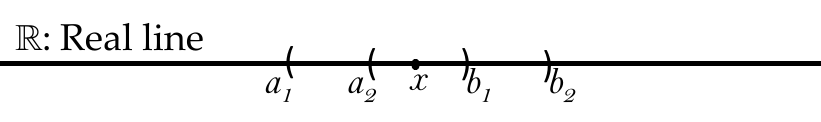} 
                \caption{If $B_1 = (a_1, b_1), B_2 = (a_2, b_2)$, then $\forall x \in (a_1, b_1) \cap (a_2, b_2)$, by taking $B_3 = (a_2, b_1)$, we have $x \in B_3 \subseteq B_1 \cap B_2$.}
            \end{subfigure}
            \caption{Open intervals as a basis for a topology on $\mathbb{R}$.}
            \label{fig:basis_Tau_R}
        \end{figure}
    }
\end{enumerate}
\end{example}

\begin{definition}
Let $X$ be a set. A function $d: X \times X \to \mathbb{R}$ is said to be a \emph{metric} or a \emph{distance function}, if $\forall x, y, z \in X$
\begin{enumerate}[label=(\roman*)]
    \item $d(x, y) \ge 0$
    \item $d(x, y) = d(y, x)$
    \item $d(x, y) + d(y, z) \ge d(x, z)$.
\end{enumerate}
Then $(X, d)$ is called a \emph{metric space}.
\end{definition}

\begin{example}
Following are some examples of metric spaces.
\begin{enumerate}
    \item {
        For $X = \mathbb{R}$, define $d: \mathbb{R} \times \mathbb{R} \to \mathbb{R}$ as 
        \[
            d(x, y) = |x - y|.
        \]
        Then $(\mathbb{R}, d)$ is a metric space.
    }
    \item {
        For $X = \mathbb{R}^n$, define $d_p: \mathbb{R}^n \times \mathbb{R}^n \to \mathbb{R}$ as 
        \[
            d_p(x, y) = \left(\sum\limits_{i=1}^n (x_i - y_i)^p\right)^{\frac{1}{p}};
        \]
        $d_p$ is called the $L_p$-norm metric. For $p=2$, it's called the Euclidean distance.
    }
\end{enumerate}
\end{example}

\begin{definition}
Let $(X, d)$ be a metric space. 
\begin{enumerate}
    \item {
        For any $x \in X, r > 0$, the \emph{open ball} or \emph{open neighbourhood} of $x$ of radius $r$ under $d$, is the set
        \[
            B_r(x) = \{y \in X: d(x, y) < r \}.
        \]
    }
    \item {
        A set $A \subseteq X$ is an \emph{open set} in $X$, if $\forall x \in A, \exists r > 0 : B_r(x) \subseteq A$.
    }
    \item {
        The \emph{topology $\Tau_d$ induced by $d$} is the set of all open sets in $X$.
    }
\end{enumerate}
\end{definition}

\begin{example}
\label{ex:basis_metric_space}
Let $(X, d)$ be a metric space. Let $\mathcal{B} = \{B_r(x) : x \in X, r > 0\}$ be the set of all open balls of $X$ under $d$. Then,
\begin{enumerate}[label=(\roman*)]
    \item {
        For every $x \in X, \exists r > 0 : B_r(x) \in \mathcal{B}$ (take r = 1, for example); i.e., $\forall x \in X, \exists B \in \mathcal{B}: x \in B$.
    }
    \item {
        If $B_1, B_2 \in \mathcal{B}$ and $x \in B_1 \cap B_2$, then $\exists B_3 \in \mathcal{B}$ s.t. $x \in B_3 \subseteq B_1 \cap B_2$. \\
        \begin{proof}
            Let $r$ be the distance to the closest point of $x$ in $B_1 \cap B_2$ and take $B_3 = B_r(x)$. Then $x \in B_3 \subseteq B_1 \cap B_2$.
        \end{proof}
    }
\end{enumerate}
\end{example}

\begin{definition}
Let $X$ be any set. A collection $\mathcal{B} \subseteq \mathcal{P}(X)$ is said to be a \emph{basis for a topology on $X$}, if 
\begin{enumerate}[label=(\roman*)]
    \item {
        $\forall x \in X, \exists B \in \mathcal{B} : x \in B$;
    }
    \item {
        If $B_1, B_2 \in \mathcal{B}$ and $x \in B_1 \cap B_2$, then $\exists B_3 \in \mathcal{B} : x \in B_3 \subseteq B_1 \cap B_2$. See figure~\ref{fig:basis_Tau_R} for an illustration.
    }
\end{enumerate}
Then the \emph{topology generated by $\mathcal{B}$} is:
\[
    \Tau_\mathcal{B} = \{ A \subseteq X \; | \; \forall x \in A, \exists B \in \mathcal{B} : x \in B \subseteq A \}.
\]
\end{definition}

\begin{claim*}
$\Tau_\mathcal{B}$ is indeed a topology.
\end{claim*} 
\begin{proof}
We will verify the three conditions required for $\Tau_\mathcal{B}$ to be a topology.
\begin{enumerate}[label=(\roman*)]
    \item {
        The empty set $\phi$ vacuously satisfies the condition ($\forall x \in \phi, \exists B \in \mathcal{B} : x \in B$) and hence $\phi \in \Tau_\mathcal{B}$. Also, $X \in \Tau_\mathcal{B}$ trivially because of the first condition in the definition of basis.
    }
    \item {
        Suppose $\{A_\alpha\}_{\alpha \in I} \subseteq \Tau_\mathcal{B}$ and let $A = \bigcup\limits_{\alpha \in I} A_\alpha$. For every $x \in A, \exists \alpha \in I, B \in \mathcal{B}: x \in B \subseteq A_\alpha \subseteq A$ and hence $A \in \Tau_\mathcal{B}$.
    }
    \item {
        Suppose $A_1, A_2 \in \Tau_\mathcal{B}$ and let $A = A_1 \cap A_2$. For every $x \in A, x \in A_1 \land x \in A_2$ and since $A_1, A_2 \in \Tau_\mathcal{B}, \exists B_1, B_2 \in \mathcal{B}: x \in B_1 \subseteq A_1 \land x \in B_2 \subseteq A_2$, which implies, by the second condition in the definition of basis, that $\exists B_3 \in \mathcal{B} : x \in B_3 \subseteq B_1 \cap B_2 \subseteq A$, since $B_1 \subseteq A_1$ and $B_2 \subseteq A_2$. Hence $A \in \Tau_\mathcal{B}$. See figure~\ref{fig:Tau_B_topology}.
        
        \begin{figure}[ht]
            \centering
            \includegraphics[page=4,clip,trim=5cm 4cm 5cm 6.5cm,width=0.5\textwidth]{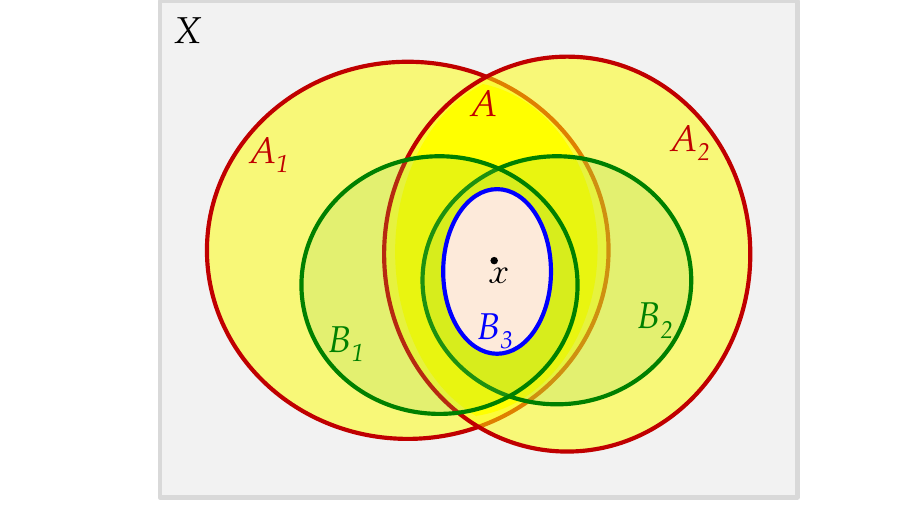}
            \caption{Illustration of existence of $B_3 \subseteq A_1 \cap A_2$.}
            \label{fig:Tau_B_topology}
        \end{figure}
        
    }
\end{enumerate}
Thus, $\Tau_\mathcal{B}$ is a topology.
\end{proof}

\begin{remark}
$\mathcal{B} \subseteq \Tau_\mathcal{B}$.
\end{remark}

\begin{example}
Following examples illustrate the notion of a basis of a topology.
\begin{enumerate}
    \item {
        For a metric space $(X, d)$, as discussed in examples~\ref{ex:basis_R} and \ref{ex:basis_metric_space}, the set of all open balls of $X$ forms a basis for the topology on $X$ whose open sets are arbitrary unions of all open balls. So, \textbf{every metric space is a topological space}.
    }
    \item {
        For any set $X, \mathcal{B} = \{ \{x\} : x \in X \}$ is a basis for the discrete topology on $X$.
    }
\end{enumerate}
\end{example}

\begin{proposition}
If $\Tau_1, \Tau_2$ are topologies generated by bases $\mathcal{B}_1, \mathcal{B}_2$ and $\mathcal{B}_1 \subseteq \mathcal{B}_2$, then $\Tau_1 \subseteq \Tau_2$.

\end{proposition}

\singlespacing
\cleardoublepage 
\phantomsection
\addcontentsline{toc}{chapter}{Bibliography}
\printbibliography

\end{document}